\documentclass[journal]{IEEEtai}
\pdfoutput=1
\PassOptionsToPackage{table}{xcolor}

\usepackage[utf8]{inputenc}
\usepackage{times}
\usepackage{helvet}
\usepackage{courier}
\usepackage{amsmath,amssymb,amsfonts,bm}
\usepackage{amsthm}
\usepackage{nccmath}

\newcommand{\op}{\mathrm{op}}

\DeclareMathOperator{\Cov}{Cov}
\DeclareMathOperator{\KL}{D}
\DeclareMathOperator{\MI}{I}
\DeclareMathOperator{\Ent}{h}

\newtheorem{assumption}{Assumption}
\newtheorem{proposition}{Proposition}
\newtheorem{theorem}{Theorem}
\newtheorem{definition}{Definition}
\newtheorem{corollary}{Corollary}[theorem]
\newtheorem{lemma}{Lemma}
\newtheorem{remark}{Remark}

\usepackage{xcolor}
\usepackage{array,tabularx,booktabs,multirow}
\usepackage{adjustbox}
\usepackage{csvsimple}
\usepackage{siunitx}
\newcolumntype{L}{>{\raggedright\arraybackslash}X}
\newcolumntype{T}{>{\ttfamily\raggedright\arraybackslash}X}

\definecolor{ExCIRRow}{RGB}{235,245,255}
\definecolor{SHAPRow}{RGB}{245,240,255}
\definecolor{LIMERow}{RGB}{240,250,240}
\definecolor{myorange}{RGB}{230,120,20}
\definecolor{mycyan}{RGB}{0,170,200}
\definecolor{q1}{HTML}{F46D9B}
\definecolor{q2}{HTML}{F59763}
\definecolor{q3}{HTML}{F7C548}
\definecolor{q4}{HTML}{67D5C4}
\definecolor{q5}{HTML}{7AA6FF}
\definecolor{q6}{HTML}{8B6BF2}
\definecolor{q7}{HTML}{F58CD3}
\definecolor{q8}{HTML}{62D0F6}
\definecolor{axisgray}{HTML}{9AA0A6}
\definecolor{PointBlue}{RGB}{30,90,200}
\definecolor{AlignGreen}{RGB}{0,120,60}
\definecolor{ScatterOrange}{RGB}{230,140,0}
\colorlet{NumBox}{green!12}
\colorlet{DenBox}{orange!15}

\usepackage{graphicx,wrapfig}
\usepackage{caption}
\usepackage[subrefformat=parens,labelformat=parens]{subcaption}
\usepackage{tikz,pgfplots}
\pgfplotsset{compat=1.18}
\usetikzlibrary{
  calc,positioning,arrows.meta,fit,shapes.geometric,
  decorations.pathreplacing,decorations.markings
}
\tikzset{
  fatseg/.style={line width=2pt,draw=AlignGreen!85},
  fatsegarr/.style={
    fatseg,
    postaction={decorate,
      decoration={markings,
        mark=at position 0.55 with {\arrow{Latex}}
      }
    }
  }
}

\usepackage{algorithm,algpseudocode}
\usepackage{paralist}
\makeatletter
\let\labelindent\relax
\makeatother
\usepackage{enumitem}
\setlist{nosep,leftmargin=*,itemsep=1pt,topsep=2pt}

\usepackage{soul}
\usepackage[normalem]{ulem}
\usepackage{pifont}
\newcommand{\cmark}{\ding{51}}
\newcommand{\xmark}{\ding{55}}
\usepackage[switch]{lineno}

\usepackage{url}
\usepackage[colorlinks=true,urlcolor=black,linkcolor=black,citecolor=black]{hyperref}
\begin{document}

\title{ Explainability of Complex AI Models with Correlation Impact Ratio}

\author{
\begin{minipage}[t]{\textwidth}
\centering
Poushali Sengupta$^{\star}$,
Rabindra Khadka$^{\dagger}$,
Sabita Maharjan$^{\star}$,
Frank Eliassen$^{\star}$,
Shashi Raj Pandey$^{\ddagger}$,
Yan Zhang$^{\star}$,
Pedro G. Lind$^{\S\ast}$,
and Anis Yazidi$^{\star\P}$.
\end{minipage}
\thanks{
$^{\star}$ Department of Informatics, University of Oslo, Oslo, Norway;
$^{\dagger}$ Faculty of Medicine, University of Oslo, Oslo, Norway;
$^{\ddagger}$ Aalborg University, Aalborg, Denmark;
$^{\S}$ Department of Technology, Kristiania University of Applied Sciences, Oslo, Norway;
$^{\ast}$ Simula Research Laboratory, Oslo, Norway;
$^{\P}$ Department of Computer Science, Oslo Metropolitan University, Oslo, Norway.
poushals@uio.no;
rabindra@oslomet.no;
sabita@uio.no;
frank@uio.no;
srp@es.aau.dk;
yanzhang@uio.no;
pedro.lind@kristiania.no;
anisy@oslomet.no.
}
}
\maketitle
\vspace{-3mm}
 \begin{abstract}
Complex AI systems make better predictions but often lack transparency, limiting trustworthiness, interpretability, and safer deployments. Common post-hoc AI explainers, such as LIME, SHAP, HSIC, and SAGE, are model-agnostic but too restricted in one significant regard: they tend to misrank correlated features and require costly perturbations, which will not scale for high dimensional data. We introduce ExCIR (Explainability through Correlation Impact Ratio), a theoretically grounded, simple and reliable metric for explanations of input features to the model's output, which remains stable and consistent under noise and sampling variations. We demonstrate that ExCIR captures dependencies arising from correlated features through a lightweight, single-pass formulation. Experimental evaluations on diverse datasets, including EEG, synthetic vehicular data, Digits, and Cats–Dogs, validate the effectiveness and stability of ExCIR across domains, achieving interpretable feature explanations than existing methods, while remaining computationally efficient.  To that end, we also extend ExCIR with an information-theoretic foundation that unifies the correlation ratio with Canonical Correlation Analysis (CCA) under mutual information bounds, enabling multi-output and class-conditioned explainability, and scalability.
 

\end{abstract}

\begin{IEEEImpStatement}
ExCIR provides bounded, stable, and low-overhead explanations, 
with orders-of-magnitude speedups over perturbation-based baselines, that are practical for monitoring and audits in resource-constrained settings. By preserving ranking fidelity in lightweight deployments, ExCIR enables reliable attribution at the edge, supports explanation-drift checks over time, and helps stakeholders understand key drivers without costly perturbation budgets.
These advancements make ExCIR a theoretically grounded, scalable, and reliable method for explainability across various domains, including signal processing, tabular data, and computer vision.
\end{IEEEImpStatement}


\section{Introduction}
\IEEEPARstart{AI} Artificial Intelligence and Machine Learning (ML) models increasingly drive critical decision-making in domains such as energy, healthcare, finance, and transportation, creating a growing need for trustworthy and explainable predictions to ensure reliability, transparency, and accountability. However, many complex ML models, particularly deep neural networks, remain difficult to interpret and often operate as ``black boxes,'' making their predictions challenging to explain \cite{munro1984review}. This lack of transparency can hinder trust and adoption, especially in high-stakes domains where understanding how decisions are made is crucial \cite{von2021transparency}. Consequently, Explainable AI (XAI) has become an important research area, encompassing methods applied before, during, and after model development to explain model decisions \cite{minh2022explainable}.

Current XAI methods often face challenges in real-world settings, particularly when strong feature correlations exist or real-time explanations are required. For example, in EEG-based dementia classification, noise in \emph{occipital} channels may cause methods such as SHAP and LIME to emphasize less relevant features while underweighting clinically important \emph{temporal} sensors \cite{jeong2004eeg}. Because these methods rely on input perturbations, their explanations can be sensitive to the underlying data distribution and feature dependencies. They may also become unstable under model retraining or feature noise: small input perturbations can produce noticeably different explanations even when model behavior remains largely unchanged~\cite{kindermans2019reliability}. Furthermore, their computational cost can limit scalability to large datasets and real-time applications \cite{lundberg2017unified,ribeiro2016should}. While much existing work focuses on tabular and image data, temporal and streaming settings remain less explored, where post-hoc explanations can be particularly sensitive to changing temporal dependencies and distribution shifts \cite{Rojat2021TSXAI}.

\par To address these issues, we previously introduced \emph{Explainability through Correlation Impact Ratio} (\textsc{ExCIR}) \cite{sengupta2026correlation}, which assigns feature-level influence scores to model predictions using a fast, single-pass procedure. Although ExCIR provided efficient and accurate scalar attributions, several limitations remained. \emph{First}, it focused on scalar outputs and therefore did not directly address settings with multiple interdependent outputs \cite{molnar2022interpretable}. \emph{Second}, it lacked class-specific explanations, limiting its ability to distinguish class-dependent evidence in multi-class problems. Its score scales were also not explicitly calibrated for comparison across models and datasets \cite{sundararajan2017axiomatic}. Finally, ExCIR lacked formal guarantees that attribution scores would respond monotonically to increasing signal strength and remain stable under small data perturbations. Such inconsistent or non-monotonic explanations can produce contradictory outcomes, potentially confusing users and reducing trust~\cite{adebayo2018sanity}.
While a group-wise variant of CIR, namely \emph{BlockCIR}, was introduced in \cite{sengupta2026correlation} to prevent double-counting of correlated features, it did not address redundancy when outputs were interdependent. Existing methods, including BlockCIR, also lack a principled approach for uncertainty quantification and confidence intervals (CIs) in attribution scores. Formal strategies for ensuring ranking consistency under sampling variation are limited, and existing analyses often focus on small datasets, with attribution stability degrading under distribution shifts and noise. These gaps motivate the development of more robust and comprehensive XAI methods.

\par In this paper, we address these limitations by introducing contributions that improve interpretability and consistency for vector outputs. \textit{First}, we introduce Multi-output CIR, which supports vector outputs while preserving dependencies among them and enables both individual and joint attributions on a consistent scale. \textit{Next}, Class-conditional CIR (CC-CIR) provides per-class explanations, distinguishing evidence shared across classes from class-specific evidence, which is particularly useful in multi-class settings. We further introduce a normalization scheme that guarantees bounded, calibrated, and comparable scores across datasets and models. Boundedness provides asymptotic guarantees that explanation values remain consistent and stable as data scale increases. We also establish conditions for monotonicity with respect to signal strength and stability under small data perturbations. Our framework controls redundancy to reduce repeated attribution of shared information across tasks and incorporates uncertainty quantification through bootstrap and Bayesian intervals to improve ranking reliability.

\par This work further unifies the CIR family through Canonical Correlation Analysis (CCA) and mutual-information (MI) bounds. Traditional CCA primarily captures linear relationships and assumes well-defined variance structures \cite{hotelling1936relations}, which can become unstable under multicollinearity \cite{wilks1932certain}. Dependence measures such as the Hilbert-Schmidt Independence Criterion (HSIC) and Centered Kernel Alignment (CKA) lack natural upper bounds, complicating cross-scale comparison \cite{kornblith2019similarity}. To address these limitations, we introduce ``ExCIR Beyond CCA,'' a dependence-aware extension that agrees with CCA in linear settings while capturing a broader class of dependence patterns, including selected nonlinear regimes under the proposed formulation. By moving beyond standard linear projections, ExCIR improves attribution stability under model retraining and data variation. Through feature-space transformations and MI-controlled analysis~\cite{cover2006elements}, we identify conditions under which ExCIR captures nonlinear structures such as sinusoidal, quadratic, and stepwise relationships that linear CCA may miss. Finally, we establish criteria under which lightweight environments preserve global rankings within a measurable error margin. Extensive cross-domain evaluations on text, tabular, signal, and vision data stress-test the framework under noise, data shifts, and limited-data conditions. In summary, the main contributions of this work are:
\begin{enumerate}[leftmargin=1em, itemsep=0pt, topsep=0.5pt]

\item \textbf{Theoretical Foundations.}
We establish \textbf{ExCIR} as a bounded, monotonic dependence-based attribution measure that generalizes the correlation ratio through a monotonic transformation of CCA, admits MI-based interpretation, and captures selected nonlinear dependencies while remaining stable under sampling and feature noise.

\item \textbf{Algorithmic Advancements.}
We extend ExCIR to \textit{multi-output} and \textit{class-conditional} attribution while retaining lightweight, single-pass computation without perturbation or retraining. BlockCIR further provides group-level attribution for correlated features, reducing redundancy while preserving interpretability.

\item \textbf{Experimental Validation.}
Across EEG, synthetic, tabular, and vision data, ExCIR agrees with CCA in linear regimes, captures broader dependence in selected nonlinear settings, and provides competitive or improved stability, ranking consistency, and computational efficiency compared with perturbation-based baselines.

\end{enumerate}

\textbf{Paper Structure}: Section~\ref{rel} reviews related work, Section~\ref{sec:background-cir} presents the preliminaries, Section~\ref{sec:methodology-overview} introduces the proposed method, and Section~\ref{sec:experiments} describes the experimental setup and results. Section~\ref{ethics} discusses limitations and future directions, and Section~\ref{conclu} concludes the paper.
\section{Related Work} \label{rel}
Balancing accuracy and interpretability remains central in modern AI; the \emph{black-box} behavior of deep models limits transparency and trust \cite{munro1984review}. Global marginal-effect tools (PDP, ALE, ICE) visualize average or local trends but degrade under interactions and feature dependence in high dimensions \cite{guidotti2018survey}. Local explainers (LIME/SHAP families) are widely used yet rely on surrogates or perturbations and are sensitive to background choice, sampling budgets, and correlated predictors \cite{linardatos2020explainable}. Variants such as DeepSHAP/KernelSHAP introduce differentiability or independence assumptions and can be costly or unstable under noise and shifts \cite{yeh2019infidelity}. Attempts to trade accuracy for interpretability via surrogate models often falter under redundancy or uncertainty \cite{ennab2024enhancing}. Recent dependence-aware explainers such as HSIC-Lasso~\cite{yamada2014hsic}, MICe~\cite{reshef2011detecting}, and mutual-information attribution methods~\cite{chen2018learning} estimate pairwise or kernelized feature--output associations, but typically lack boundedness, transferability, and theoretical calibration. 
\begin{table}[htbp]
\centering
\caption{ExCIR vs.\ correlation \& information-based XAI.}
\renewcommand{\arraystretch}{1}
\begin{adjustbox}{width=\linewidth}
\begin{tabular}{lcccc}
\toprule
\textbf{Method} & \textbf{Bounded?} & \textbf{Lightweight transfer?} & \textbf{Theoretical link to MI?} & \textbf{Complexity} \\
\midrule
HSIC-Lasso~\cite{yamada2014hsic} & \xmark & \xmark & Partial (kernelized) & $\mathcal{O}(n^2 k)$ \\
MICe~\cite{reshef2011detecting} & \xmark & \xmark & Empirical only & $\mathcal{O}(n^2 \log n)$ \\
MI-Attribution~\cite{chen2018learning} & \xmark & \xmark & Direct but unbounded & $\mathcal{O}(n^2 k)$ \\
\textbf{ExCIR (ours)} & \cmark & \cmark & \cmark (bounded MI upper bound) & $\mathcal{O}(n k)$ \\
\bottomrule
\end{tabular}
\end{adjustbox}
\label{tab:comp_corr_info_xai}
\end{table}
They quantify association strength without explicitly ensuring that rankings remain stable or comparable across datasets \cite{Schlegel2022TSXAIReview}. Moreover, most explainability research has focused on static tabular or image domains and relatively few studies address robustness and interpretability in temporal or streaming data~\cite{crabbe2021explaining}.  
Recent work highlights the need for temporally stable, context-aware XAI in dynamic settings such as sensor networks and energy forecasting~\cite{crabbe2021explaining}. Existing methods, however, often rely on perturbations or gradient tracing and remain sensitive to autocorrelation and data drift. ExCIR addresses these limitations through a bounded, monotone correlation-ratio geometry that (i) remains in $[0,1]$ with closed-form sensitivity limits, (ii) handles dependence through groupwise attribution (\textsc{BlockCIR})~\cite{kalakoti2023improving}, (iii) requires neither gradients nor retraining, (iv) yields stable rankings under noise, distributional shift, and feature uncertainty~\cite{burger2025towards,elkhawaga2024should}, (v) connects to mutual information through a provable upper bound, and (vi) provides observation-only complexity for scalar and vector outputs, supporting efficient edge deployment~\cite{chamola2023review}. These properties position ExCIR as a principled, information-consistent attribution method for dynamic environments~\cite{guidotti2018survey}, with substantially lower runtime than state-of-the-art alternatives. For $n$ observations and $k$ features, comparisons with correlation- and information-based explainers (HSIC-Lasso, MICe, MI-Attribution), together with ExCIR's novel properties and computational complexity, are summarized in \autoref{tab:comp_corr_info_xai} and \autoref{tab:excir-novelty}.\footnote{Exact Shapley explanations scale exponentially in $k$; practical SHAP variants rely on sampling or structural approximations. ExCIR achieves linear complexity in $n$ and $k$ through single-pass covariance computation and a one-time small-matrix decomposition, without perturbations.}
\begin{table}[htbp]
\centering
\scriptsize
\setlength{\tabcolsep}{3pt}
\renewcommand{\arraystretch}{0.5}
\caption{ExCIR novelties vs.\ SOTA in one view.}
\begin{tabular}{p{0.16\linewidth} p{0.40\linewidth} p{0.36\linewidth}}
\toprule
\textbf{Aspect} & \textbf{Status quo (SOTA)} & \textbf{ExCIR (ours)} \\
\midrule
Computation &
Shapley family: exact $\mathcal{O}(2^k)$ (number of features $k$). \newline
KernelSHAP: $\mathcal{O}(m\,k)$ model calls ($m$ perturbation samples). \newline
TMC-SHAP: $\mathcal{O}(m\,k)$ (Monte Carlo paths). \newline
TreeSHAP: $\mathcal{O}(T L k)$ (trees $T$, max depth $L$). \newline
GradientSHAP/IG/DeepSHAP: $\mathcal{O}(m\,k)$ backprop passes. \newline
HSIC/MI estimators: typically $\mathcal{O}(n^2 k)$ (pairwise kernels). &
\textbf{Closed-form, observation-only.} One-time covariance: $\mathcal{O}(k^3{+}p^3)$; streaming covariances $\mathcal{O}(n k)$. \newline
Per-feature scoring (single pass): $\mathcal{O}(n k)$, independent of perturbations or resampling. \\
\midrule
Ranking, sufficiency &
Local/perturbation-driven; global order unstable under correlation and noise. &
Performance-aligned global ranking; higher top-$k$ sufficiency with compact subsets; correlation-aware. \\
\midrule
Deployment &
Perturbation-heavy pipelines; repeated model evaluations; full data required. &
Lightweight-transfer; single-pass, low-memory; preserves ranking under subsampling (20-40\% data). \\
\midrule
Calibration &
Unbounded scores; difficult cross-run comparison; sensitive to retraining. &
Bounded CIR $\in[0,1]$ with MI-linked upper bound; comparable across datasets or models; stable under sampling \& feature noise. \\
\bottomrule
\end{tabular}
\label{tab:excir-novelty}
\end{table}
\section{Background: ExCIR}
\label{sec:background-cir}
In this section, we will introduce essential concepts from our previous work on ExCIR:  the \emph{Correlation Impact Ratio} (CIR), its grouped extension \emph{BlockCIR}, and the \emph{Class-Conditioned CIR (CC-CIR)}. We also give a highlight of the computation complexity of ExCIR, and the concept of lightweight and similar environment.  
\vspace{-2mm}
\subsection{\textbf{Preliminaries: CIR, BlockCIR and CC-CIR}}
\noindent
We consider a supervised learning setting where the objective is to assess feature importance with respect to model outputs. 
Let $X \in \mathbb{R}^{n \times k}$ denote an evaluation matrix with $n$ observations and $k$ features, and let $Y \in \mathbb{R}^{n \times p}$ denote the corresponding ground-truth or model-predicted outputs.
The goal is to estimate the functional dependence between $X$ and $Y$ through an explanatory mapping $\Phi: X \mapsto \widehat{Y}$ that minimizes the expected \emph{predictive risk}:
\vspace{-2mm}
\begin{equation}
    \medmath{\mathcal{R}(f) = \mathbb{E}_{(X,Y)\sim P}[\ell(f(X), Y)],}
    \vspace{-2mm}
\end{equation}
where, $f$ denotes the predictive model $f_\theta : X \mapsto Y$, whose parameters determine the risk $\mathcal{R}(f_\theta)$ and thus anchor the accuracy constraint used in Definition~\ref{def:lightweight-env}. $\ell(\cdot,\cdot)$ is a suitable loss function (e.g., squared error for regression, cross-entropy for classification). 
A lightweight environment~$\mathcal{E}'$ (Definition~\ref{def:lightweight-env}) is then defined relative to this risk by constraining the deviation 
$\lvert \mathcal{R}(f_{\mathcal{E}'}) - \mathcal{R}(f_{\mathcal{E}}) \rvert \le \varepsilon_{\mathrm{acc}}$. In what follows, we first define the scalar-output case; extensions to multi-output formulations are in \autoref{subsec:multiout-basic}.

\begin{definition}
    Let $X\in\mathbb{R}^{n\times k}$ be an evaluation matrix with feature column $f_i=X_{ i}$ and let $y\in\mathbb{R}^n$ be the corresponding model output. Denote sample means $\hat f_i=\frac{1}{n}\sum_{j=1}^n x_{ji}$ and $\hat y=\frac{1}{n}\sum_{j=1}^n y_j$, and define the mid-mean $m_i=\tfrac12(\hat f_i+\hat y)$. The \emph{CIR} of feature $i$ is,
    \vspace{-2mm}
\begin{equation}
    \medmath{\operatorname{CIR}_i \;=\; \eta_{f_i}
\;=\; 
\frac{n\big[(\hat f_i-m_i)^2+(\hat y-m_i)^2\big]}
{\sum_{j=1}^n (x_{ji}-m_i)^2 \;+\; \sum_{j=1}^n (y_j-m_i)^2}
\;\in [0,1].}
\end{equation}
\end{definition}
CIR calculates how fluctuations in feature $x_i$ co-vary with the model output $y$, serving as a global dependence measure.\footnote{ CIR is a measure that remains consistent whether $(X, Y)$ comes from the full dataset or a smaller sample, called "lightweight Environment". The supplementary document provides more details about the Lightweight environment and its ability to maintain the same rankings as the original data.} \autoref{fig:cir-geometry-final} illustrates how CIR compares aligned mean offsets to total scatter around a symmetric reference. Detailed construction and decomposition are in \textbf{Supplementary~A.2}.
\par High-dimensional explainers often scale poorly with the number of features~$k$. 
In contrast, \textsc{ExCIR} admits an \emph{observation-only} formulation whose complexity depends primarily on the number of observations~$n$, remaining independent of~$k$. 
However, temporal signals, such as EEG, time-series from sensors, or telemetry often involve $n\!\gg\!10^4$, 
making repeated explanations costly \cite{mullen2015eeglab}. To reduce runtime without altering model architecture, 
we construct a \emph{lightweight environment} by subsampling rows while retaining all features and preserving statistical structure.
\begin{theorem}[\textbf{Observation-only factorization}]
\label{thm:cir_n3_background}
Given $(X',y')\!\in\!\mathbb{R}^{n\times k}\!\times\!\mathbb{R}^n$ with $n\!\ge\!2$, 
each $\mathrm{CIR}_i$ can be computed by an algorithm with runtime upper bounded by $\mathcal{O}(n^3)$ that depends only on~$n$, 
with per-feature evaluation $\mathcal{O}(n)$ thereafter. Full proof in \textbf{Supplementary~B.1}.
\end{theorem}

In this work, sampling is performed over subsets of observations drawn from each dataset (Definition~\autoref{def:lightweight-env}). 
This enables efficient estimation of feature importance under varying data availability, such as household subsets in energy data, individual subjects in EEG signals, or image batches in computer vision, without altering the model architecture.

\begin{definition}[\textbf{Lightweight Environment}]
\label{def:lightweight-env}
A lightweight(LW) environment $\mathcal{E}'$ is a row-subsampled dataset drawn from the full environment $\mathcal{E}$ 
that preserves the joint input–output moments, i.e.\ the first and second moments of $P(X,Y)$, including feature variances, 
correlations, and class priors, within a small tolerance $\varepsilon_{\mathrm{acc}}$ on predictive risk. \textbf{Supp.~B.2, B.4.}
\end{definition}
While multiple LW environments, denoted as $\mathcal{E}'$, can be derived from the same full environment $\mathcal{E}$, not all of these LW environments maintain the statistical and predictive characteristics necessary for reliable evaluation. To ensure consistency across these reduced settings, we will define the criteria that determine when a LW environment can be considered \emph{similar} to the original environment.
\begin{definition}[\textbf{Similar environment}]
\label{def:similar-env}
A lightweight environment $\mathcal{E}'$ is \emph{similar} to the full environment $\mathcal{E}$ 
if the following three checks hold (detail in \textbf{Supplementary~B.2-B.3.}):
\begin{enumerate}[label=(\roman*)]
    \item \textbf{Projection alignment:} internal feature and output embeddings remain linearly aligned 
    (up to affine rescaling);
    \item \textbf{Distribution closeness:} kernel-based MMD~\cite{gretton2012kernel} and KL divergence~\cite{kullback1951information} 
    between $(X',Y')$ and $(X,Y)$ lie below predefined thresholds;
    \item \textbf{Risk gap:} expected prediction loss between the full and LW models satisfies 
    $\Delta_{\mathrm{risk}}\le\varepsilon_{\mathrm{acc}}$.
\end{enumerate}
When these conditions hold, $\mathcal{E}'$ yields the same global ranking of $\mathrm{CIR}_i$ 
as $\mathcal{E}$ up to a data-dependent bound. 
\end{definition}

\begin{figure}[t!]
\centering
\begin{adjustbox}{width=\linewidth}
\begin{tikzpicture}[
  >=Latex,
  every node/.style={font=\small},
  axis/.style={line width=0.6pt,draw=black!80},
  tick/.style={line width=0.6pt,draw=black!80},
  fpoint/.style={circle,minimum size=6pt,inner sep=0pt,draw=none,fill=PointBlue},
  mpoint/.style={diamond,minimum size=6.5pt,inner sep=0pt,draw=none,fill=PointBlue}, 
  callout/.style={rounded corners=1.6pt,inner sep=3pt,outer sep=0pt},
  thinarr/.style={-{Latex[length=2.4mm]},line width=0.75pt},
  labg/.style={text=AlignGreen,font=\footnotesize},
  labs/.style={text=ScatterOrange!85!black,font=\footnotesize}
]

\def\xA{1.2}  
\def\xM{4.2}  
\def\xB{7.2}  
\def\y{0}     
\pgfmathsetseed{2025} 

\draw[axis] (-0.2,\y) -- (8.6,\y);
\foreach \xx in {\xA,\xM,\xB} \draw[tick] (\xx,\y-0.15) -- (\xx,\y+0.15);

\node[fpoint] (fhat) at (\xA,\y) {};
\node[mpoint] (mpt)  at (\xM,\y) {};
\node[fpoint] (yhat) at (\xB,\y) {};

\node[below=3pt] at (fhat) {$\hat f_i$};
\node[below=3pt] at (mpt)  {$m_i=\tfrac{1}{2}(\hat f_i+\hat y)$};
\node[below=3pt] at (yhat) {$\hat y$};

\draw[fatsegarr,opacity=0.55] (fhat) -- (mpt);
\draw[fatsegarr,opacity=0.55] (yhat) -- (mpt);

\draw[thinarr,AlignGreen,shorten >=1pt,shorten <=1pt]
  (\xA,\y+0.55) .. controls (\xA+0.60,\y+0.36) .. (\xM-0.16,\y+0.09);
\node[labg] at ({(\xA+\xM)/2},\y+0.68) {$|\hat f_i-m_i|$};

\draw[thinarr,AlignGreen,shorten >=1pt,shorten <=1pt]
  (\xB,\y+0.55) .. controls (\xB-0.60,\y+0.36) .. (\xM+0.16,\y+0.09);
\node[labg] at ({(\xM+\xB)/2},\y+0.68) {$|\hat y-m_i|$};

\fill[ScatterOrange,opacity=0.13] (2.7,\y) ellipse (1.1 and 0.35);
\fill[ScatterOrange,opacity=0.13] (5.7,\y) ellipse (1.1 and 0.35);

\foreach \i in {1,...,70}{
  \pgfmathsetmacro\x{rnd*1.9+1.8}  
  \pgfmathsetmacro\yy{rnd*0.7-0.35}
  \fill[ScatterOrange,opacity=0.60] (\x,\yy) circle (0.02);
}
\foreach \i in {1,...,70}{
  \pgfmathsetmacro\x{rnd*1.9+4.8}  
  \pgfmathsetmacro\yy{rnd*0.7-0.35}
  \fill[ScatterOrange,opacity=0.60] (\x,\yy) circle (0.02);
}

\node[callout,fill=NumBox,align=left] (num) at (1.6,1.45)
  {\textbf{Alignment (numerator)}\\[-1pt]
   \(\displaystyle n\!\big[(\hat f_i-m_i)^2+(\hat y-m_i)^2\big]\)};
\draw[thinarr,shorten >=1pt,shorten <=1pt]
  (num.south east) to[out=-25,in=120] (\xM,\y+0.12);

\node[callout,fill=DenBox,align=left] (den) at (6.6,1.45)
  {\textbf{Total scatter (denominator)}\\[-1pt]
   \(\displaystyle \sum_j(x_{ji}-m_i)^2+\sum_j(y_j-m_i)^2\)};
\draw[thinarr,shorten >=1pt,shorten <=1pt]
  (den.south west) to[out=-155,in=60] (5.7,\y+0.12);

\matrix[anchor=west,column sep=6pt,row sep=1pt] at (-0.1,-1.30) {
  \fill[AlignGreen,opacity=0.50] (0,0) rectangle +(0.35,0.12); & \node{Alignment (numerator)}; \\
  \fill[ScatterOrange,opacity=0.35] (0,0) rectangle +(0.35,0.12); & \node{Scatter (denominator)}; \\
};

\end{tikzpicture}
\end{adjustbox}
\caption{CIR geometry. We center \(f_i\) and \(y\) at the mid-mean \(m_i=\tfrac{1}{2}(\hat f_i+\hat y)\). The alignment (numerator) uses symmetric offsets \(|\hat f_i-m_i|\) and \(|\hat y-m_i|\); the scatter (denominator) aggregates sample deviations around the same pivot \(m_i\).}
\label{fig:cir-geometry-final}
\end{figure}
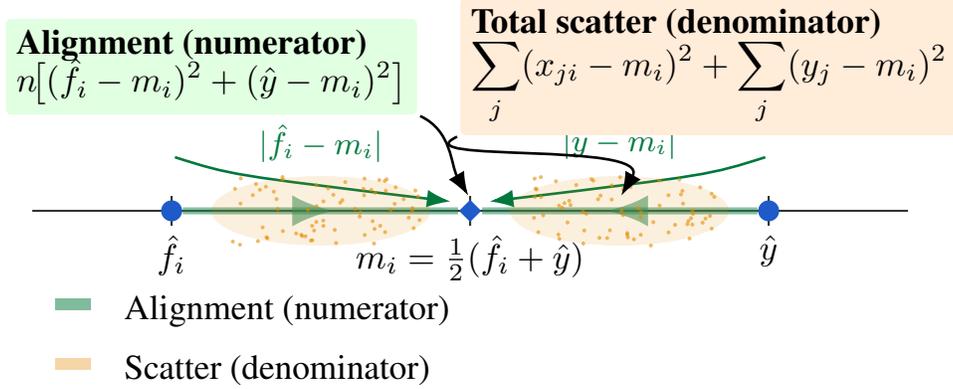
After establishing the concept of LW and similar environments, we will define the core explainability measures within this simplified yet representative framework. Although scores such as \textsc{BlockCIR} and \textsc{CC-CIR} can theoretically be calculated in the full environment $\mathcal{E}$, our focus will be on their formulation and analysis within the LW environment $\mathcal{E}'$. We assume that this LW environment remains statistically consistent with $\mathcal{E}$ and enables efficient computation.
\begin{definition}[\textbf{\textsc{BlockCIR}}]
\label{def:blockcir}
Given a block of features $B = \{i_1, \dots, i_b\}$, define the best linear summary 
$z_j = \sum_{\ell=1}^{b} \alpha_\ell x'_{ji_\ell}$, where coefficients 
$\alpha_\ell \ge 0$ and $\sum_{\ell=1}^{b} \alpha_\ell = 1$ 
are optimized to maximize the correlation with $y'$. 
The BlockCIR score is then $\mathrm{CIR}_B := \mathrm{CIR}(z, y')$.
\end{definition}
BlockCIR  captures the collective effects of correlated variables while providing explainability, aligning with correlation-aware and set-level attribution methods~\cite{gregorutti2017correlation}. To exemplify this concept, a relevant example in residential energy forecasting can be found in the work by Kara et al.\ \cite{batic2024transparent}, which categorizes household energy use into subsystems, heating, cooling, wet appliances, entertainment, and behavioral patterns, to predict short-term consumption, $y'$ (like the next-hour demand). Each feature, $x'_i$, represents an appliance's power usage or its derived load, while $y'$ is the overall forecast. Their analysis indicates that heating is the primary factor in winter peaks, whereas wet appliances contribute to variability. BlockCIR applies a similar approach by grouping related features into a subsystem block $B$, creating a summary variable $z_B$, and measuring its predictive impact on $y'$ using $\mathrm{CIR}(z_B, y')$. In both CC–CIR and BlockCIR, $y'$ is consistently the target (such as consumption, price, or failure probability), allowing CIR scores to reflect the relationship between the feature group and the prediction task. CC–CIR isolates feature co-movement for a specific class while accounting for others. It maintains the geometric interpretation of CIR by measuring covariance in the $(x'_i,\, y'^{(c)})$ subspace, allowing for one-vs-rest attributions.
\par CIR, BlockCIR, and CC-CIR form our explanation method under \emph{independence or weak dependence}. We assume there exist meaningful groups, or "blocks," of features (e.g., control signals, environmental factors) that can be treated as \emph{independent} while allowing rich correlations within each block. Model predictions can then be viewed as the sum of contributions from these blocks. To prevent double-counting due to internal correlations, we utilize \textbf{BlockCIR} as the primary measure of attribution and "\textit{tidy up}" features within each block to reduce redundancy. For classification tasks, we assume there is a specific class score \textbf{CC-CIR} (e.g., the logit for class "c") from which we assess feature importance.

\section{Unified Theory of Correlation Impact Ratio}\label{sec:methodology-overview}
In this section, we present several key advancements: \textbf{(i)} we extend the \textsc{ExCIR} methodology to handle multiple output predictions by employing a weighted aggregation of individual output CIRs; \textbf{(ii)} we improve alignment through a CCA-based pairing that emphasizes the most informative output direction \cite{jaccard1901distribution}; \textbf{(iii)} we present the guarantees of \textit{invariance} and \textit{dominance} for \textsc{BlockCIR}, \textbf{(iv)} we generalize \textsc{BlockCIR} for multi-output scenarios using bi-side CCA, ensuring that the method remains invariant to linear mixtures of inputs and outputs; \textbf{(v)} we define a class-conditioned, multi-output version of \textsc{ExCIR} that is resilient to well-conditioned logit reparameterizations and exhibits smooth behavior under class remixing; \textbf{(vi)} we unify all proposed methods within a cohesive CCA method, representing \textsc{ExCIR} as a monotonic transformation of the squared canonical correlation and lastly, \textbf{(vii)} linking its geometric properties to information-theoretic consistency. Throughout, we ensure that the method maintains boundedness, monotonicity, and consistency, which contributes to a LW computational profile that relies only on number of observations.
\vspace{-4mm}
\subsection{\textbf{Boundedness and Monotonicity of CIR.}}
\label{subsec:bounded-monotone}

A robust attribution score must be bounded, monotone with feature–output alignment, and invariant to affine reparameterizations. These ensure interpretability and comparability across datasets and models. For feature $f_i$, we define, $u_i \;=\; f_i - m_i\mathbf{1}, \ v \;=\; y' - \hat{y}'\mathbf{1}, \ m_i \;=\; \tfrac{1}{2}(\hat f_i + \hat y')\,$, 
and write, $\eta_{f_i}=\mathrm{CIR}(f_i,y')$. Let, $E_i=\|u_i\|_2^2$ and $E_y=\|v\|_2^2$.

\begin{theorem}[\textbf{Boundedness of CIR}]
\label{thm:boundedness}
For any feature $f_i$, the Correlation Impact Ratio satisfies $\eta_{f_i} \in [0,1]$. 
Equality $\eta_{f_i}=0$ holds when $f_i$ and $y'$ are independent, and $\eta_{f_i}=1$ when they are perfectly aligned 
(i.e., $f_i - m_i\mathbf{1}$ is collinear with $y'-\hat{y}'\mathbf{1}$).
See proof in \textbf{Supplementary~A.2}.

\end{theorem}

\begin{theorem}[\textbf{Monotonicity of CIR}]
\label{thm:monotonicity}
Under equal total scatter, if two features $f_p$ and $f_q$ satisfy
 $\medmath{\langle f_p - m_p\mathbf{1},\, y' - \hat{y}'\mathbf{1} \rangle >
\langle f_q - m_q\mathbf{1},\, y' - \hat{y}'\mathbf{1} \rangle,}$
then $\eta_{f_p} > \eta_{f_q}$.
See proof in \textbf{Supplementary~A.3}.
\end{theorem}
Boundedness guarantees cross-dataset comparability, while monotonicity ensures that larger aligned covariance 
corresponds to higher feature importance, preserving orderings.
\begin{corollary}[\textbf{Ranking via squared correlation under matched scatter}]
\label{cor:rank-corr}
Let fix $E_y$, and compare features either (i) under matched scatter $E_i$ across $i$, or (ii) after standardizing $u_i$ to unit variance. 
Then $\eta_{f_i} = \mathrm{CIR}(f_i,y')$ is a strictly increasing function of $\mathrm{Corr}(u_i,v)^2$. 
Equivalently, for any $p,q$, $\mathrm{Corr}(u_p,v)^2 > \mathrm{Corr}(u_q,v)^2 \;\;\Longrightarrow\;\; \eta_{f_p} > \eta_{f_q},$. So CIR and squared correlation induce identical rankings over $\{f_i\}$ under the stated conditions. 
 See proof in \textbf{Supplementary~A.7}.
\end{corollary}
 When $E_i$ (and $E_y$) are matched or variables are standardized, CIR is a monotone transform of squared-correlation rankings, supporting its consistency.
 \vspace{-3mm}
\subsection{\textbf{Multi-Output Extension.}}\label{subsec:multiout-basic}
Recent models generate predictions in the form of vectors, which can include things like class probabilities or outputs for multiple tasks. Building upon this, we create a simple extension that combines the contribution of individual CIRs for a more comprehensive view. When outputs are vectors (e.g., logits or multi-task targets), either averaging scalar CIRs or projecting with CCA picks the output direction most aligned with a feature (or block). This preserves boundedness and monotonicity while leveraging shared structure across classes.

\begin{definition}[\textbf{Multi-Output ExCIR}]
\label{def:vec_cir}
Let $X'\in\mathbb{R}^{n'\times k}$ be the lightweight dataset with feature column $f_i=X'_{ i}$, and let
$Y'=[y'_{ 1},\ldots,y'_{ c}]\in\mathbb{R}^{n'\times c}$ denote the $c$-dimensional model output
(e.g., logits, tasks, or prediction horizons). For each output coordinate $l$, define the scalar ExCIR
$\mathrm{CIR}(f_i, y'_{l})$ using the standard mid-mean–centered alignment-over-scatter ratio.
The \emph{multi-output ExCIR} for feature $i$ is then given by
\begin{equation}
\label{eq:mo-excir}
\medmath{\mathrm{CIR}^{\mathrm{mo}}_i
\;=\;
\sum_{l=1}^c \alpha_l\, \mathrm{CIR}\!\big(f_i,\,y'_{ l}\big),
\qquad
\alpha_l \ge 0,\ \ \sum_{l=1}^c \alpha_l = 1,}
\end{equation}
where the weights $\alpha_l$ are fixed a priori. A uniform choice $\alpha_l=\tfrac{1}{c}$ yields an equal-weighted
average across outputs, while a \emph{canonical weighting} scheme may be used to emphasize output directions
that are most aligned with $f_i$, for instance
$\alpha_l \propto \mathrm{Corr}\!\big(f_i,\,Y' w^\star\big)^2$
with $w^\star\!\in\!\arg\max_{\|w\|>0}\mathrm{Corr}^2(f_i, Y' w)$ obtained via CCA(CCA),
followed by normalization.
\end{definition}

\begin{remark}
(i) Definition~\eqref{def:vec_cir} preserves the boundedness $[0,1]$ and monotonicity properties of the scalar ExCIR.  
(ii) The uniform weighting corresponds to an uninformative aggregation, whereas canonical weighting highlights
shared discriminative directions between features and outputs.  
(iii) All statistics are computed on the LW Environment (LW) environment, and the definition remains agnostic to feature dependence
unless block grouping is mentioned in later sections.
\end{remark}
\paragraph{\textbf{From scalar to canonical block representations}}
The \textsc{BlockCIR} method helps handle redundancy in grouped features by focusing on a single output variable. However, many real-world models produce multiple outputs, like categories or tasks, where relationships across all outputs matter. To address this, we use CCA to identify strong correlations in both the features and outputs simultaneously. This approach ensures balanced predictions and maintains reliability, consistency, and independence from how outputs are represented.
\begin{definition}[\textbf{Canonical Group Extension: CCA-based \textsc{BlockCIR}}]
\label{def:cca-blockcir}
Let $\Sigma_b=\mathrm{Cov}(X^{(b)})\in\mathbb{R}^{p_b\times p_b}$, 
$\Sigma_y=\mathrm{Cov}(Y')\in\mathbb{R}^{m\times m}$, and 
$\Gamma_b=\mathrm{Cov}(X^{(b)},Y')\in\mathbb{R}^{p_b\times m}$ denote the within-block, output, and cross-covariances respectively.  
Define the canonical directions $(w_b^\star, u_b^\star)$ by solving,
\begin{equation}
 \medmath{(w_b^\star,u_b^\star)
\in \arg\max_{\substack{w\neq 0\\u\neq 0}}
\frac{(w^\top \Gamma_b u)^2}{(w^\top \Sigma_b w)(u^\top \Sigma_y u)}. }  
\end{equation}
The corresponding canonical variates are $z_b = X^{(b)} w_b^\star$ and $s_b = Y' u_b^\star$, and the 
\textbf{CCA-based BlockCIR} is defined as,
\begin{equation}
\medmath{\mathrm{BlockCIR}^{\mathrm{vec}}(b) = \mathrm{CIR}}(z_b, s_b).    
\vspace{-1mm}
\end{equation}
\end{definition}
 When both inputs and outputs within a block are multivariate, it is important to strive for invariance to linear re-mixing on both sides. This approach enhances the robustness and consistency of the system's performance \cite{hotelling1936relations}.

\begin{definition}[\textbf{Multi-output BlockCIR (CCA)}]\label{def:block-vec}
Let $\Sigma_b=\mathrm{Cov}(X^{(b)})$, $\Sigma_y=\mathrm{Cov}(Y')$, and $\Gamma_b=\mathrm{Cov}(X^{(b)},Y')$.
Define CCA directions,
\begin{equation}
   \medmath{ (w_b^\star,u_b^\star)\in\arg\max_{\substack{w\neq 0\\u\neq 0}}
\frac{(w^\top \Gamma_b u)^2}{(w^\top \Sigma_b w)(u^\top \Sigma_y u)}, z_b = X^{(b)} w_b^\star,\ \ s_b = Y' u_b^\star,}
\vspace{-1mm}
\end{equation}
and set $\mathrm{BlockCIR}^{\mathrm{vec}}(b)=\mathrm{CIR}(z_b,s_b)$.
\end{definition}
\paragraph{\textbf{Class-Conditioned Multi-Output ExCIR}}
\label{subsec:class-cond}
 For classification tasks, we often need to use a one-vs-rest approach. This means focusing on the importance of each class individually while also making sure that our method remains effective even when the logits (the raw prediction scores) are mixed in different ways. 

\begin{definition}[\textbf{Class-Conditioned ExCIR (CCA)}]\label{def:class-cond}
Let $Y'$ be logits on the LW split and fix a class $c$. Choose $w_c$ in a class-conditioned subspace (e.g., $w_c=e_c$ or a CCA direction constrained to include $e_c$), and set $v_c = Y' w_c$. Define
\begin{equation}
\medmath{\mathrm{CIR}^{\mathrm{CC}}_i(c)\;=\;\mathrm{CIR}\!\big(f_i,\,v_c\big).} 
\end{equation}
\end{definition}
$\mathrm{CIR}^{\mathrm{CC}}$ isolates the contribution of $f_i$ to class $c$ while maintaining invariance to modifications within the logit space.
\vspace{-5mm}
\subsection{\textbf{Invariance and Stability.}}
\begin{lemma}[\textbf{Bi-side Invariance}]\label{lem:bi-side-invariance}
Let $X^{(b)}$ be a feature block with within-block covariance $\Sigma_b$, and $Y'$ have covariance $\Sigma_y$.
Define the CCA directions $(w_b^\star,u_b^\star)$ and canonical variates $z_b=X^{(b)}w_b^\star$, $s_b=Y'u_b^\star$.
Then $\mathrm{BlockCIR}_{\mathrm{vec}}(b)=\mathrm{CIR}(z_b,s_b)$ is invariant to any invertible linear
reparameterization within the block or the output space, i.e., for invertible $A,B$,
$\medmath{\mathrm{CIR}(X^{(b)}A\,w_b^\star,\; Y'B\,u_b^\star)=\mathrm{CIR}(z_b,s_b).}$
 See proof in \textbf{Supplementary~A.7}.
\end{lemma}

\begin{lemma}[\textbf{Stability under Output Reparameterization}]\label{lem:output-stability}
Let $M$ be invertible and approximately geometry-preserving in the output space
($M^\top\Sigma_y M \approx \Sigma_y$ with preservation error $\varepsilon$).
Then the Kendall–$\tau$ distance\footnote{Kendall-$\tau$ was chosen because \textsc{ExCIR}'s explanatory goal is to preserve the \emph{ranking} of features, 
not the absolute magnitude of their scores. 
It provides a monotone-invariant, interpretable, and bounded measure to quantify how stable those rankings remain 
under small output-space transformations \cite{ghorbani2019fragile}.} between ExCIR rankings computed from $Y'$ and $Y'M$ is $\mathcal{O}(\varepsilon)$. 
See proof in \textbf{Supplementary~A.9}.
\end{lemma}

\begin{corollary}[\textbf{Convexity under Class Remixing}]\label{cor:class-remix}
For convex class mixing $\bar y'=\sum_{j}\alpha_j y^{(j)}$, $\sum_j\alpha_j=1$, we have
$\eta_{f_i}(\bar y') \in \mathrm{conv}\{\eta_{f_i}(y^{(1)}),\dots,\eta_{f_i}(y^{(m)})\}$; i.e., ExCIR scores vary convexly under class remixing. 
See proof in \textbf{Supplementary~A.9}.
\end{corollary}
CCA-based BlockCIR is invariant to well-conditioned linear remixing on \emph{both} inputs and outputs, so explanations do not change under equivalent logit reparameterizations or feature bases. Small geometry-preserving changes yield only small rank changes. 
\vspace{-4mm}
\subsection{\textbf{Robustness of ExCIR.}}
Having established class-conditional directionality and consistency under class blending and output reparameterization, we next assess ExCIR's local robustness to small output perturbations.
\begin{theorem}[\textbf{Correlation–Impact Sensitivity}]\label{thm:sensitivity1}
Assume (A1) $g$ is locally Lipschitz in coordinate $i$; (A2) signed empirical correlation $\rho_i\in[-1,1]$; and (A3) bounded second moments about their natural centers: $\medmath{\frac{1}{n'}\sum_j(x'_{ji}-\hat f_i)^2\le K^2,\ 
\frac{1}{n'}\sum_j(y'_j-\hat y')^2\le K^2.}$
Then there exist constants $c_1,c_2>0$ such that for small perturbation $\delta$,
\begin{equation}
\medmath{|g(x+\delta e_i)-g(x)|\le}
\begin{cases}
\medmath{c_1\,\eta_{f_i}\,|\delta|,} &\medmath{ \rho_i\ge0,}\\[2pt]
\medmath{\dfrac{c_2}{2K^2-\eta_{f_i}^2}\,|\delta|, }& \medmath{\rho_i<0.}
\end{cases}
\label{eq:sensitivity}
\end{equation}
See proof in \textbf{Supplementary~A.10}.
\end{theorem}
 where $\eta_{f_i}=\text{CIR}_\text{i}$, \noindent
and $e_i$ represents the $i$-th standard basis vector in $\mathbb{R}^k$. The term $x + \delta e_i$ indicates a perturbation of the input $x$ along feature $i$ by a small amount $\delta$. Larger $\eta_{f_i}$ implies proportionally stronger local output movement when $f_i$ is perturbed, justifying ExCIR as a responsiveness-aware importance. High-CIR features induce the largest local output response under small perturbations, providing a saliency-like interpretation of sensitivity.
\begin{theorem}[\textbf{Sensitivity under One-Point Output Change}]\label{thm:sensitivity}
Let $y'$ and $y''$ differ in a single output entry. Then,
 $ \medmath{\big|\eta_{f_i}(y') - \eta_{f_i}(y'')\big| \le \mathcal{O}\!\left(\frac{1}{n'}\right).}$
See proof in \textbf{Supplementary~A.10}.
\end{theorem}
\vspace{-1mm}
\subsection{\textbf{Unfied ExCIR Approach.}}
Finally, we unify these results into a single generalized representation that connects ExCIR to CCA, linking all scalar, block, and vector forms through one geometric principle.

\begin{theorem}[\textbf{Unified ExCIR Representation}]\label{thm:unified-excir}
Let $Z=\Phi^\top X'$ and $S=\Psi^\top Y'$ be any linear summaries of inputs and outputs. Then,
\begin{equation}
    \medmath{\mathrm{CIR}(Z,S)
=\frac{\|\,\mathbb{E}[Z]-\mathbb{E}[S]\,\|^2}{\mathbb{E}\,\|Z-\mathbb{E}[Z]\|^2+\mathbb{E}\,\|S-\mathbb{E}[S]\|^2}}
\end{equation}
is a \emph{monotone transformation of the squared canonical correlation} $\rho^2(Z,S)$ between $(Z,S)$. Consequently, CCA maximizes this ratio, and all ExCIR variants (scalar, block, and vector-output) are unified under one correlation–ratio geometry up to a monotone map of $\rho^2$. \noindent
Here, $\Phi$ and $\Psi$ denote linear projection matrices for the input and output spaces, respectively, 
such that $Z=\Phi^\top X'$ and $S=\Psi^\top Y'$ are canonical or task-aligned summaries in a shared latent space. See proof in \textbf{Supplementary~A.11}.
\end{theorem}
ExCIR is a bounded, monotone transform of squared canonical correlation; maximizing ExCIR is equivalent to maximizing CCA alignment. Thus scalar, grouped (BlockCIR), and vector-output cases share the same ordering principle. The results of this study apply to static situations where the joint moments of $(X,Y)$ are constant. In these cases, ExCIR provides dependence scores that align with CCA orderings. Future study plans to extend this to dynamic environment which change over time, allowing for stability analyses. However, boundedness, monotonicity, and consistency do not guarantee that ExCIR represents statistical dependence. The Information Bottleneck approach highlights the need for representations that preserve mutual information (MI)~\cite{tishby1999ib}. Additionally, \cite{poole2019variationalMI} explore the trade-offs between bias and variance in the context of MI bounds \cite{kraskov2004estimating}. The Hilbert-Schmidt Independence Criterion (HSIC) offers a baseline for measuring dependency using kernels~\cite{gretton2005hsic}, while Deep Canonical Correlation Analysis (DCCA) illustrates how canonical alignment can achieve invariant and optimal directions~\cite{andrew2013dcca}. These relationships emphasize the importance of grounding our approach in mutual information and support the unified theory of \textsc{ExCIR} related to CCA. Therefore, we connect
ExCIR to mutual information (MI) to demonstrate \textbf{(i) ordering equivalence with MI, and (ii) a bounded, monotonic
relationship to MI}.
\vspace{-2mm}
\subsection{\textbf{ExCIR Beyond CCA: Information-Theoretic View}}
\label{subsec:beyond_cca}

In linear settings, ExCIR reduces to a monotonic transformation of squared CCA. On linear synthetic benchmarks, ExCIR and CCA show near-perfect rank agreement (Spearman $\rho=0.979$; \autoref{tab:excir_vs_cca_combined}) and similar feature-score profiles (\autoref{fig:nonlinear_bars}, left). Unlike CCA, which relies on optimal linear projections, ExCIR can capture selected nonlinear dependencies through nonlinear feature representations (\autoref{fig:nonlinear_bars}, right). Conditional expectation analyses, $\mathbb{E}[y\mid x_i]$, further reveal oscillatory, U-shaped, and discontinuous relationships (\textbf{Figures~S3--S7, Supplementary}), with explained variance reaching $R^2=0.21$ for $x_0$ and $R^2=0.65$ for $x_1$ despite weak CCA dependence. Correspondingly, ExCIR improves Precision@$k$ over CCA across all evaluated $k$ (\autoref{tab:excir_vs_cca_combined}), demonstrating broader nonlinear sensitivity while retaining linear-regime consistency.

ExCIR also admits an information-theoretic interpretation: under joint Gaussianity, its squared-CCA foundation yields a bounded, monotonic relationship with mutual information (MI), connecting geometric and information-based dependence. However, MI measures general statistical dependence and may emphasize correlated but predictively redundant features, whereas ExCIR measures feature dependence through model outputs (logits), making it model-aware. Their low empirical rank agreement (Spearman $=-0.26$; Supplementary Section~Z, Table~S26) reflects this distinction. Moreover, high-dimensional MI estimation can be unstable and computationally costly, while ExCIR uses a single-pass framework that provides efficient and robust model-aware attribution.
\begin{figure}[htbp]
\centering
\begin{adjustbox}{width=\columnwidth}

\begin{tikzpicture}[font=\sffamily\scriptsize]

\begin{axis}[
    width=5.0cm, height=4 cm,
    title={\bfseries Linear case},
    ylabel={Score},
    ylabel style={at={(-0.12,0.5)}},
    ybar=2pt,
    bar width=6pt,
    xmin=0.4, xmax=12.6,
    ymin=0, ymax=0.74,
    ytick={0,0.2,0.4,0.6},
    xtick={1,...,12},
    xticklabels={9,4,8,1,7,0,5,10,2,3,6,11},
    xlabel={Feature index},
    xlabel style={yshift=2pt},
    grid=major,
    grid style={dashed, gray!40},
    axis line style={gray!60},
    ticklabel style={font=\tiny},
]
\addplot[fill=myorange, draw=black, line width=0.2pt] coordinates {
    (1,0.68) (2,0.37) (3,0.36) (4,0.30) (5,0.27)
    (6,0.26) (7,0.15) (8,0.14) (9,0.12) (10,0.11) (11,0.09) (12,0.00)
};
\addplot[fill=mycyan, draw=black, line width=0.2pt] coordinates {
    (1,0.43) (2,0.11) (3,0.11) (4,0.09) (5,0.06)
    (6,0.06) (7,0.05) (8,0.03) (9,0.01) (10,0.00) (11,0.00) (12,0.00)
};
\legend{CCA $|\rho|$, ExCIR}
\end{axis}

\begin{axis}[
    xshift=5 cm,
    width=5cm, height=4 cm,
    title={\bfseries Nonlinear case},
    ylabel={Score},
    ylabel style={at={(0.0,0.5)}},
    ybar=2pt,
    bar width=6pt,
    xmin=0.4, xmax=12.6,
    ymin=0, ymax=0.74,
    xtick={1,...,12},
    xticklabels={1,0,3,2,4,10,5,7,6,8,9,11},
    xlabel={Feature index},
    xlabel style={yshift=2pt},
    grid=major,
    grid style={dashed, gray!40},
    axis line style={gray!60},
    ticklabel style={font=\tiny},
    ylabel={},
    yticklabels={},
]
\addplot[fill=myorange, draw=black, line width=0.2pt] coordinates {
    (1,0.05) (2,0.09) (3,0.27) (4,0.19) (5,0.11)
    (6,0.07) (7,0.07) (8,0.06) (9,0.01) (10,0.04) (11,0.01) (12,0.00)
};
\addplot[fill=mycyan, draw=black, line width=0.2pt] coordinates {
    (1,0.63) (2,0.21) (3,0.07) (4,0.05) (5,0.04)
    (6,0.01) (7,0.01) (8,0.00) (9,0.00) (10,0.00) (11,0.00) (12,0.00)
};
\legend{CCA $|\rho|$, ExCIR}
\end{axis}

\end{tikzpicture}

\end{adjustbox}

\caption{Linear vs. nonlinear dependence scores.}
\label{fig:nonlinear_bars}
\vspace{-1mm}
\end{figure}
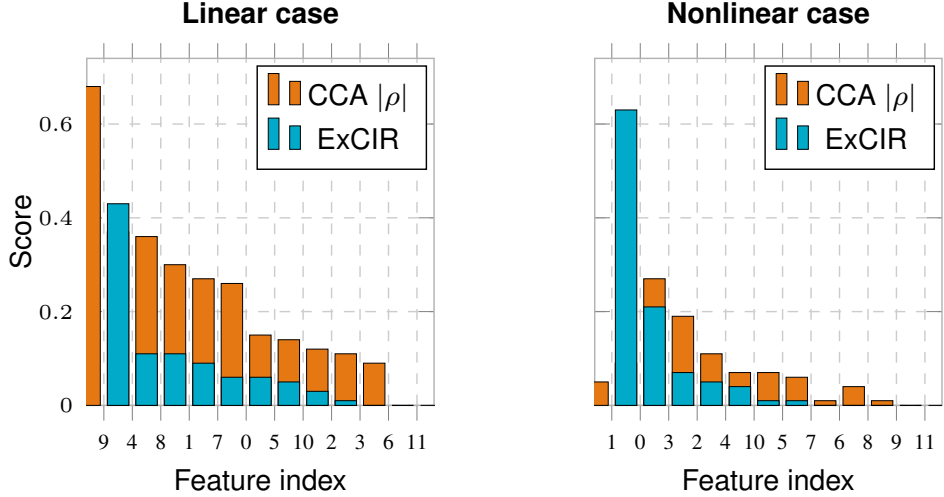
\begin{table}[t]
\centering
\caption{Linear vs. \ nonlinear dependence: CCA, ExCIR }
\label{tab:excir_vs_cca_combined}
\begin{adjustbox}{width=\linewidth}
\begin{tabular}{lcccccc}
\toprule
& \multicolumn{2}{c}{\textbf{Linear regime}} 
& \multicolumn{3}{c}{\textbf{Nonlinear regime: Precision@k}} 
& \textbf{Driver fit (ExCIR $R^2$ / CCA $|r|$)} \\
\cmidrule(lr){2-3}\cmidrule(lr){4-6}\cmidrule(lr){7-7}
\textbf{Metric} 
& Spearman $\rho$ & $p$-value 
& @3 & @5 & @8 
& Nonlinear drivers ($x_0,x_1,x_2$) \\
\midrule
CCA 
& \multirow{2}{*}{$0.979$} & \multirow{2}{*}{$3.09\times10^{-8}$}
& $0.33$ & $0.20$ & $0.25$ 
& $[0.018, 0.041, 0.189]$ \\
ExCIR 
&  & 
& $\mathbf{0.67}$ & $\mathbf{0.60}$ & $\mathbf{0.38}$ 
& $[0.211, 0.647, 0.038]$ \\
\bottomrule
\end{tabular}
\end{adjustbox}
\vspace{-1mm}
\end{table}
\begin{theorem}[\textbf{MI–consistency of ExCIR}]\label{thm:cir-mi-short}
For linear summaries $(Z,S)$ of a feature and a (scalar/vector) prediction, $\mathrm{CIR}(Z,S)$ is a strictly increasing function of the squared canonical correlation $\rho(Z,S)^2$. Under joint Gaussianity, $I(Z;S)=-\tfrac12\log(1-\rho^2)$, hence $\mathrm{CIR}(Z,S)$ is a strictly increasing transform of $I(Z;S)$. 
See proof in \textbf{Supplementary~A.13}.
\end{theorem}
\begin{theorem}[\textbf{MI–boundedness of ExCIR}]\label{thm:cir-upper-short}
Under standardized Gaussian $(\mathbb{E}[Z]{=}\mathbb{E}[S]{=}0,\ \mathrm{Var}(Z){=}\mathrm{Var}(S){=}1)$ with $\rho=\rho(Z,S)$ and $I(Z;S)=-\tfrac12\log(1-\rho^2)$,
\begin{equation}
    \medmath{\mathbb{E}[\mathrm{CIR}(Z,S)] \;\le\; \frac{\rho^2}{2-\rho^2}
\;=\; \frac{1-e^{-2I(Z;S)}}{1+e^{-2I(Z;S)}} ,}
\vspace{-1mm}
\end{equation}
a bounded, strictly increasing function of $I(Z;S)$. 
See proof in \textbf{Supplementary~A.13}.
\end{theorem}
\vspace{-3mm}
\section{How Lightweight Can We Go?}\label{sec:lightweight}
\noindent
We quantify to what extent rows can be reduced ($n \!\to\! n'$) while preserving both predictive risk and \textsc{ExCIR} rankings.
\newcommand{\BoundTerm}{\Big(\tfrac{\log(1/\delta)}{n'}\Big)^{\frac{4}{4+q}}}
\begin{theorem}[\textbf{Finite-sample risk gap, unified}]\label{thm:finite-sample-risk-gap}
Let $Y\in\mathbb{R}^{n\times q}$ and $Y'\in\mathbb{R}^{n'\times q}$ be the $q$-dimensional outputs (e.g., logits) from models with the same architecture trained on the full and LW environments, respectively. 
Under standard regularity conditions (bounded moments, Lipschitz evaluation loss, bounded-kernel MMD), with probability at least $1-\delta$,
\begin{equation}
\begin{split}
    &\medmath{\big\|\widehat{R}(Y)-\widehat{R}(Y')\big\|}\\&
\medmath{\;\le\; \mathcal{O}\!\Big(\sqrt{\tfrac{\log(1/\delta)}{n'}}\Big)_{\!\text{proj}}
\;+\;
\mathcal{O}\!\Big(\sqrt{\tfrac{\log(1/\delta)}{n'}}\Big)_{\!\text{MMD}}
\;+\;
\mathcal{O}\!\big(\BoundTerm\big)_{\!\text{KL}},}
\end{split}
\vspace{-1mm}
\end{equation}
where the hidden constants depend on kernel bandwidth and moment bounds.
See proof in \textbf{Supplementary~B.4}.
\end{theorem}
\begin{corollary}[\textbf{Scalar case}]\label{cor:scalar}
For $q{=}1$, the KL term scales as $\mathcal{O}\!\big((\log(1/\delta)/n')^{4/5}\big)$ and the sufficient lower bound on $n'$ uses exponent $(4{+}q)/4{=}5/4$.
\end{corollary}
Let $\varepsilon_{\text{proj}},\varepsilon_{\text{MMD}},\varepsilon_{\text{KL}}>0$ split a target accuracy budget $\varepsilon_{\text{acc}}$ across the terms in \autoref{thm:finite-sample-risk-gap}. A sufficient lower bound is,
\begin{equation}
\medmath{n'_{\mathrm{LB}}
\;\;\ge\;\;
\max\!\Big\{
\mathcal{O}\!\big(\tfrac{\log(1/\delta)}{\varepsilon_{\text{proj}}^{2}}\big),\;
\mathcal{O}\!\big(\tfrac{\log(1/\delta)}{\varepsilon_{\text{MMD}}^{2}}\big),\;
\mathcal{O}\!\big(\big(\tfrac{\log(1/\delta)}{\varepsilon_{\text{KL}}}\big)^{\frac{4+q}{4}}\big)
\Big\}.}
\vspace{-1mm}
\end{equation}
In practice, we pick $n'$ within $[n'_{\mathrm{LB}},\,n'_{\mathrm{UB}}]$ (capped by wall-clock/memory), and admit the LW environment only if all three gates pass. 
\vspace{-2mm}
\section{Experimental Setup}\label{sec:experiments}
To illustrate ExCIR's workflow, given a trained model and dataset $E$, we construct a \textit{lightweight (LW) environment} $E'$ through subsampling and accept it only when the similarity criteria in Definitions 2--3 preserve feature--output dependence and predictive risk. ExCIR then computes feature importance in a single pass using covariance-based statistics, while \textit{BlockCIR} handles correlated features and \textit{CC-CIR} provides class-specific attribution. For example, on CIFAR-10, an image classified as `truck'' assigns high ExCIR scores to wheel and lower-chassis patches, while on 20 Newsgroups, a `sports'' document emphasizes terms such as `game'', `team'', and ``season''. Experiments use fixed random seeds, consistent model architectures, and held-out test sets; explanation methods require no retraining.
\vspace{-4mm}
\subsection{\textbf{Datasets and Models.}}
\par\textbf{{CAU--EEG} \& Synthetic Vehicular}: The CAU--EEG dataset \cite{kim2023deep} contains 1,186 recordings (459 normal, 416 MCI, 311 dementia)~\cite{kim2023deep}, with 19-channel signals (1--40 Hz) from which 23 features (including microstate GEVs and age) are extracted. Classification is performed using InceptionTime~\cite{ismail2020inceptiontime} with LW validation via KL/MMD gates. The synthetic vehicular dataset \cite{sengupta2026correlation} comprises 6,000 samples from 20 sensors with 15\% low-tire events and structured feature correlations. A Gradient Boosting model~\cite{friedman2001greedy} is used with standard preprocessing \cite{pedregosa2011scikit}, and both individual and grouped (BlockCIR) feature rankings are evaluated against SHAP and LIME.

\par\textbf{Digits \& Cats--Dogs}: \cite{pedregosa2011scikit} Digits (1,797 images, $8\times8$) is evaluated using multinomial logistic regression for multi-output tests. Cats--Dogs \cite{kaggle_dogs_vs_cats_2013} uses a small CNN with class-conditioned maps for dataset-level saliency.

\par\textbf{CIFAR-10 \& 20 Newsgroups (Text):}: \cite{krizhevsky2009learning} CIFAR-10 contains 60,000 RGB images ($32\times32$) across 10 classes. We use ResNet-18 and apply ExCIR to patch-level representations to capture localized dependencies and assess scalability to high-dimensional vision data. The 20 Newsgroups dataset \cite{lang1995newsweeder} comprises real-world text documents, where we use TF-IDF features with multinomial logistic regression to evaluate ExCIR in high-dimensional sparse settings under token-level perturbations.

\begin{figure}[htbp]
\centering
\begin{adjustbox}{width=8.2cm, height=2.5 cm}   
\begin{tikzpicture}[
  font=\sffamily\small,                        
  >=Stealth,
  node distance=7mm and 9mm,
  arr/.style={-{Stealth[length=2mm]}, line width=0.8pt},
  inmap/.style={draw, rounded corners, minimum width=1.3cm, minimum height=1.3cm,
                align=center, fill=cyan!25, font=\sffamily\bfseries\small},
  branchA/.style={draw, rounded corners, minimum width=1.4cm, minimum height=1.0cm,
                  align=center, fill=blue!30,   font=\sffamily\small},
  branchB/.style={draw, rounded corners, minimum width=1.4cm, minimum height=1.0cm,
                  align=center, fill=teal!35,   font=\sffamily\small},
  branchC/.style={draw, rounded corners, minimum width=1.4cm, minimum height=1.0cm,
                  align=center, fill=purple!30, font=\sffamily\small},
  branchD/.style={draw, rounded corners, minimum width=1.4cm, minimum height=1.0cm,
                  align=center, fill=green!35,  font=\sffamily\small},
  concat/.style={draw, rounded corners, minimum width=1.0cm, minimum height=1.0cm,
                 align=center, fill=orange!35},
  oneone/.style={draw, rounded corners, minimum width=1.0cm, minimum height=1.0cm,
                 align=center, fill=red!35,    font=\sffamily\small},
  outmap/.style={draw, rounded corners, minimum width=1.0cm, minimum height=1.0cm,
                 align=center, fill=pink!40,   font=\sffamily\bfseries\small}
]
\node[inmap] (in) {Input};
\node[branchA, right=14mm of in, yshift=11mm] (b1) {Conv $k{=}5$};
\node[branchB, below=3mm of b1] (b2) {Conv $k{=}7$};
\node[branchC, below=3mm of b2] (b3) {Conv $k{=}14$};
\node[branchD, below=3mm of b3] (b4) {Conv $k{=}21$};

\foreach \b in {b1,b2,b3,b4} \draw[arr] (in.east) -- (\b.west);

\node[concat, right=12mm of b2, yshift=9mm] (c1) {};
\node[concat, below=1.5mm of c1] (c2) {};
\node[concat, below=1.5mm of c2] (c3) {};

\foreach \b in {b1,b2,b3,b4} \draw[arr] (\b.east) -- (c2.west);

\node[align=center, below=2mm of c3, font=\sffamily\footnotesize] 
     {Filter\\concatenation};

\node[oneone, right=12mm of c2] (one) {$1{\times}1$\\conv};
\draw[arr] (c2.east) -- (one.west);

\node[outmap, right=12mm of one] (out) {Output};
\draw[arr] (one.east) -- (out.west);

\end{tikzpicture}
\end{adjustbox}
\caption{Inception module used in CAU-EEG backbone.}
\label{fig:inception-color-square}
\end{figure}
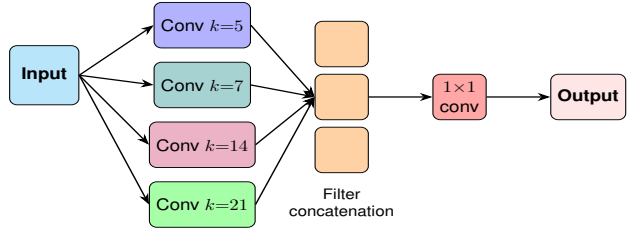
\subsection{\textbf{Evaluation Protocol.}} 
We evaluate ExCIR based on two hypotheses: \textbf{H1}: a LW model should match the full model's top-8 feature rankings, requiring \(O_8 = 1.00\) and $\tau_{\text{head}}(8) \geq 0.95$; \textbf{H2}: \textsc{BlockCIR} should improve group-level credit assignment by enhancing group separation and reducing redundancy. Comparisons are made with KernelSHAP, LIME, and MI/HSIC. Metrics are calculated using LW data (fixed seeds), and uncertainty is estimated through 100 IID bootstrap resamples. We report Top-$k$ overlap, Cliff's $\delta$, and 95\% CIs, applying the Benjamini-Hochberg \cite{benjamini1995fdr} false discovery rate (BH-FDR) at \(q=0.1\). \textbf{H1} is supported if rankings align, and \textbf{H2} is supported if \textsc{BlockCIR} shows better group separation and lower discordance with $q_{\text{BH}}<0.1$. To prevent cross-contamination, we apply subject-wise splits (EEG), trip-wise splits (vehicular), and class-stratified or deduplicated splits (Digits). We use the \emph{same} model for full and LW settings. Preprocessing is fit on training only; validation is used for early stopping, and test performance is reported once. We evaluate:

\begin{enumerate}[leftmargin=*, itemsep=0pt, topsep=1pt]
  \item[Q1]\textbf{Lightweight fidelity:} Are ExCIR rankings preserved under row subsampling?
  \item[Q2] \textbf{External validity:} Do top features align with domain knowledge?
  \item[Q3] \textbf{Predictive sufficiency:} Do ExCIR top-$k$ features retain model accuracy?
  \item[Q4] \textbf{Stability to perturbations:} Are rankings stable under input noise, resampling, and mild distribution shifts?
  \item[Q5] \textbf{Dependence \& groups:} Does \textsc{BlockCIR} mitigate over-attribution in correlated blocks?
  \item[Q6]\textbf{Efficiency trade-off:} How does LW-model (\textsc{LW-ExCIR}) compare to SHAP/LIME in speed and fidelity?
  \item[Q7] \textbf{Multi-output validation:} Are vector-output extensions stable under class remixing?
  \item[Q8] \textbf{Uncertainty \& significance:} Do confidence intervals and statistical tests substantiate ranking consistency and differences?\footnote{It may be confusing to think of Q4 and Q8 as the same; however, they focus on different yet important aspects. Q4 examines how rankings change when the data is altered, while Q8 assesses our confidence in those rankings and whether the differences between them are significant.}

\end{enumerate}
\noindent


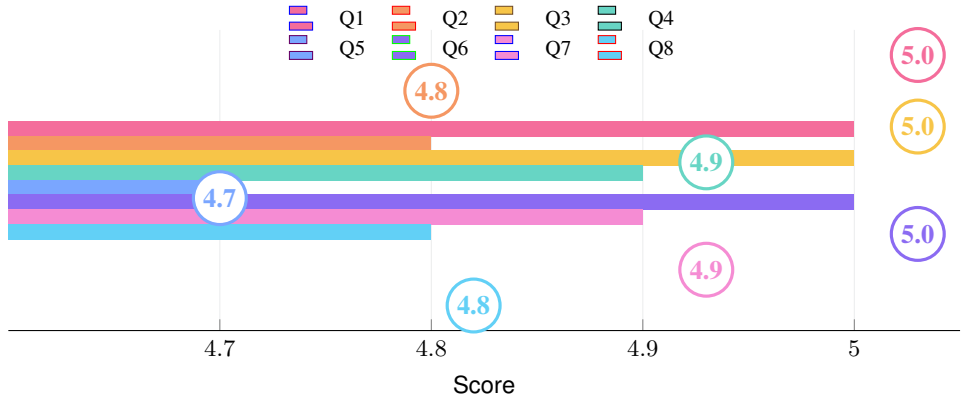
\begin{figure}[htbp]
\centering
\begin{tikzpicture}[font=\sffamily\footnotesize]
\begin{axis}[
  scale only axis=true,
  width=\columnwidth, height=3.2cm,  
  xbar, bar width=6pt,
  xmin=4.6, xmax=5.05,
  xtick={4.7,4.8,4.9,5.0},
  xmajorgrids, grid style={axisgray!20, line width=0.25pt},
  xlabel = Score,
  xticklabel style={/pgf/number format/fixed},
  axis x line*=bottom,
  axis y line*=left, y axis line style={draw=none}, ytick style={draw=none},
  ytick=\empty,
  y dir=reverse,
  yticklabel style={font=\bfseries\footnotesize, xshift=-1pt},
  clip=false,
  legend style={
    at={(0.5,1.1)}, anchor=north,
    draw=none, fill=none,
    legend columns=4,
    /tikz/column sep=8pt,
    font=\scriptsize
  }
]
\addplot+[draw=none, fill=q1] coordinates {(5.0,1)}; 
\addplot+[draw=none, fill=q2] coordinates {(4.8,2)}; 
\addplot+[draw=none, fill=q3] coordinates {(5.0,3)}; 
\addplot+[draw=none, fill=q4] coordinates {(4.9,4)}; 
\addplot+[draw=none, fill=q5] coordinates {(4.7,5)}; 
\addplot+[draw=none, fill=q6] coordinates {(5.0,6)}; 
\addplot+[draw=none, fill=q7] coordinates {(4.9,7)}; 
\addplot+[draw=none, fill=q8] coordinates {(4.8,8)}; 
\legend{} 
\addlegendimage{/pgfplots/legend image fill=q1} \addlegendentry{Q1}
\addlegendimage{/pgfplots/legend image fill=q2} \addlegendentry{Q2}
\addlegendimage{/pgfplots/legend image fill=q3} \addlegendentry{Q3}
\addlegendimage{/pgfplots/legend image fill=q4} \addlegendentry{Q4}
\addlegendimage{/pgfplots/legend image fill=q5} \addlegendentry{Q5}
\addlegendimage{/pgfplots/legend image fill=q6} \addlegendentry{Q6}
\addlegendimage{/pgfplots/legend image fill=q7} \addlegendentry{Q7}
\addlegendimage{/pgfplots/legend image fill=q8} \addlegendentry{Q8}
\tikzset{score/.style={fill=white, very thick, circle, minimum size=15pt, inner sep=0pt, font=\bfseries\small}}
\node[draw=q1, text=q1, score] at (axis cs:5.03,1) {5.0};
\node[draw=q3, text=q3, score] at (axis cs:5.03,3) {5.0};
\node[draw=q6, text=q6, score] at (axis cs:5.03,6) {5.0};
\node[draw=q4, text=q4, score] at (axis cs:4.93,4) {4.9};
\node[draw=q7, text=q7, score] at (axis cs:4.93,7) {4.9};
\node[draw=q2, text=q2, score] at (axis cs:4.80,2) {4.8};
\node[draw=q5, text=q5, score] at (axis cs:4.70,5) {4.7};
\node[draw=q8, text=q8, score] at (axis cs:4.82,8) {4.8};
\end{axis}
\end{tikzpicture}
\caption{ExCIR Key Findings-Performance summary.}
\label{fig:strength}
\vspace{-1mm}
\end{figure}
\newcommand{\TopK}{8}
The experimental outcomes in \autoref{fig:strength} show strong performance in \textbf{Lightweight Fidelity} (Q1), \textbf{Predictive Sufficiency} (Q3), and \textbf{Efficiency} (Q6), indicating that ExCIR provides faithful and resource-efficient explanations without sacrificing predictive accuracy. Strong results in \textbf{External Validity} (Q2), \textbf{Robustness} (Q4), and \textbf{Multi-Output Stability} (Q7) further support its reliability under perturbations, while the slightly lower scores in \textbf{Group Dynamics} (Q5) and \textbf{Uncertainty \& Significance} (Q8) indicate remaining scope for improved dependence modeling and uncertainty quantification. Each dimension (Q1--Q8) is scored on a 1--5 scale\footnote{Scores were calculated as $\medmath{\text{Score}=1+4\times\frac{M-M_{\min}}{M_{\max}-M_{\min}}}$, where $M$ is the averaged metric for each question. ExCIR achieves an overall average above $4.8/5.0$.} using aggregated results across CAU--EEG, Vehicular, Digits, and Cats--Dogs. Overall, lightweight ExCIR preserves key rankings (1.00 overall score; 98\% agreement), retains predictive accuracy with few components, remains stable under noise and model variations, and reduces credit splitting through BlockCIR. Domain-specific results identify age as important in EEG and prioritize control-related variables over environmental and dynamic factors in the vehicular setting. ExCIR is also 100--1{,}000$\times$ faster without requiring additional model calls, remains consistent for mixed outputs, and yields statistically significant findings under BH-FDR at 0.1. An interactive radar visualization is available in the repository referenced in Section~\ref{ethics}.
\vspace{-2mm}
\subsection{\textbf{Result on Lightweight fidelity: Q1}}
\label{subsec:q1-lw-fidelity}
We evaluate whether a LW environment with the \emph{same} architecture preserves the full model’s rankings on CAU--EEG and Synthetic Vehicular. For CAU--EEG we use a predefined LW size, whereas for Vehicular we tune the subsampling rate and select the smallest sample that passes the similarity gate.  We set $(\alpha,\beta,\gamma,\varepsilon_{\text{acc}})$ from historical releases to instantiate the three LW checks (similarity, independence, and performance): 
\textbf{(i)} \emph{\textbf{Similarity gate}} uses $\alpha$ which is set at the 75th percentile of the benign projection shifts, meaning it reflects a value where 75\% of the observed shifts fall below it, and $\gamma$ that corresponds to the Kullback-Leibler divergence at which the F1 score drops by less than 1\% to bound allowable distributional movement; 
\textbf{(ii)} \emph{\textbf{Independence gate}} uses $\beta{=}0.05$ as a two-sample tolerance, set at 0.05, indicating a relatively low threshold for variance (with correlation/HSIC and group structure; cf.\ $|\Delta\rho|\le\varepsilon_\rho$) to respect block independence; 
\textbf{(iii)} \emph{\textbf{Performance gate}} uses $\varepsilon_{\text{acc}}{=}3\%$ as the maximum allowable accuracy change on the validation split. 
 The LW environment is accepted only when all three gates are satisfied within fixed, pre-defined tolerances. In this work, the subsample size \(n'\) is not fixed a priori; instead, we select the smallest fraction that satisfies all three validation gates (similarity, independence, and performance). Concretely, for the Synthetic Vehicular dataset, we evaluate candidate fractions \(r_f \in \{0.20, 0.30, 0.35, 0.40, 0.50\}\) and select \(r_f = 0.20\), which preserves the full-model rankings (Spearman \(=1.00\), Top-8 overlap \(=100\%\)) while meeting projection alignment, MMD/KL, and risk-gap thresholds (Table IV). 
For CAU--EEG, a predefined lightweight split is used and subsequently validated against the same criteria. This gate-based selection procedure ensures that the lightweight environment remains statistically consistent with the full dataset while minimizing computational cost. The approach is further supported by Theorem 9 and detailed analysis in Supplementary Sections B.2--B.4.
\subsubsection{\textbf{CAU--EEG Data}}
\label{cau}
To evaluate the effectiveness of LW transfer, we train InceptionTime architecture \cite{ismail2020inceptiontime}, using global average pooling combined with a dense head for each (see  \autoref{fig:inception-color-square}). The result shows that the rankings and accuracy align closely within acceptable limits (\autoref{tab:top8_merged}). Empirically, ExCIR's class-relevant features, such as age and microstate GFP/GEV, are preserved in the LW model. This finding supports the notion of explanation fidelity even with a reduced sample size (\(n'\)), while maintaining a fixed model capacity. 
 \begin{table}[htbp]
\centering
\caption{ Thresholds and similarity for Vehicular
and Digits.}
\label{tab:lw_similarity}
\setlength{\tabcolsep}{2pt}
\renewcommand{\arraystretch}{0.7}
\begin{adjustbox}{width=\linewidth}
\begin{tabular}{l c c c}
\toprule
\textbf{Check} & \textbf{Threshold} & \textbf{Vehicular (meas.)} & \textbf{Digits (meas.)} \\
\midrule
Projection distance $\Delta_{\mathrm{proj}}$ & $\le \alpha$ & 0.011 & 0.457 \\
MMD two-sample $p$-value                    & $\ge \beta$  & 0.10  & 0.99  \\
KL$\!\big(P_{\text{full}}\!\parallel\!P_{\text{LW}}\big)$
                                             & $\le \gamma$ & 0.009 & 0.061 \\
Risk gap (acc./F1 ratio)                     & $\ge 1-\varepsilon_{\text{acc}}$ & 0.974 & 0.971 \\
\bottomrule
\end{tabular}
\end{adjustbox}
\end{table}
\subsubsection{\textbf{Synthetic Vehicular}}
\label{sec:veh-sim-env-q1}
In this validation, we aim to optimize a LW version of our model by testing different values for a parameter called \( r_f \). This parameter indicates the proportion of data we select from our entire dataset, with possible choices being 0.20, 0.30, 0.35, 0.40, and 0.50. When we refer to "drift," we mean changes in how the data behaves over time, which can affect how well our model performs. To address this, we manage what we call "dependence drift," ensuring that our LW model performs similarly to the full model by limiting changes to a defined threshold. 
By selecting \( r_f = 0.20 \), we used approximately 960 out of 4,800 rows of data, achieving a perfect Spearman correlation coefficient of 1.000 and a Top-8 overlap of 100\% in about 1.6 seconds (see \autoref{tab:top8_merged}). Testing different \( r_f \) values allows us to analyze variations in accuracy, processing time, and model robustness. 
 \autoref{tab:lw_similarity} shows that all three gates (similarity, independence, performance) are satisfied and remain stable under  variations of plus or minus 20\%. The result confirms that the full model CIR ranking is preserved by LW-model, while SHAP/LIME show different dynamics proxies. 
\begin{table}[htbp]
\centering
\scriptsize
\setlength{\tabcolsep}{6pt}
\renewcommand{\arraystretch}{0.7}
\caption{Top-8 ranked features per method. Left→right = higher→lower importance.
}
\label{tab:top8_merged}
\begin{tabularx}{\linewidth}{@{}l L@{}}
\toprule
\textbf{Method} & \textbf{Top-8 ranked features (high $\rightarrow$ low)} \\
\midrule
\rowcolor{ExCIRRow}%
CIR (full \& LW) &
\texttt{age} $\rightarrow$ \texttt{gfp\_value} $\rightarrow$ \texttt{unlabeled} $\rightarrow$ \texttt{B\_gev} $\rightarrow$ \texttt{C\_gev} $\rightarrow$ \texttt{D\_gev} $\rightarrow$ \texttt{A\_gev} $\rightarrow$ \texttt{F\_gev} \\
\rowcolor{SHAPRow}%
SHAP &
\texttt{age} $\rightarrow$ \texttt{C\_occurrences} $\rightarrow$ \texttt{D\_occurrences} $\rightarrow$ \texttt{F\_occurrences} $\rightarrow$ \texttt{A\_occurrences} $\rightarrow$ \texttt{B\_occurrences} $\rightarrow$ \texttt{F\_gev} $\rightarrow$ \texttt{B\_gev} \\
\bottomrule
\end{tabularx}

\vspace{6pt}

\noindent\textbf{(B) Synthetic Vehicular (validation, LW accepted)}
\begin{tabularx}{\linewidth}{@{}l L@{}}
\toprule
\textbf{Method} & \textbf{Top-8 ranked features (high $\rightarrow$ low)} \\
\midrule
\rowcolor{ExCIRRow}%
CIR (full \& LW) &
\texttt{brake} $\rightarrow$ \texttt{tire\_rr} $\rightarrow$ \texttt{rpm} $\rightarrow$ \texttt{road\_grade} $\rightarrow$ \texttt{maf} $\rightarrow$ \texttt{speed\_kph} $\rightarrow$ \texttt{tire\_rl} $\rightarrow$ \texttt{fuel\_rate} \\
\rowcolor{SHAPRow}%
SHAP &
\texttt{speed\_kph} $\rightarrow$ \texttt{accel\_lat} $\rightarrow$ \texttt{tire\_rl} $\rightarrow$ \texttt{tire\_fr} $\rightarrow$ \texttt{tire\_fl} $\rightarrow$ \texttt{brake} $\rightarrow$ \texttt{road\_grade} $\rightarrow$ \texttt{steering\_deg} \\
\rowcolor{LIMERow}%
LIME &
\texttt{speed\_kph} $\rightarrow$ \texttt{accel\_lat} $\rightarrow$ \texttt{tire\_fr} $\rightarrow$ \texttt{tire\_rl} $\rightarrow$ \texttt{tire\_fl} $\rightarrow$ \texttt{brake} $\rightarrow$ \texttt{accel\_long} $\rightarrow$ \texttt{steering\_deg} \\
\bottomrule
\end{tabularx}
\end{table}
\vspace{-3mm}
\subsection{\textbf{Result on  External validity: Q2}}
\label{subsec:q2-external-validity}
\subsubsection{\textbf{CAU--EEG}}
\label{subsec:q2-external-validity-eeg}
The application of ExCIR to the CAU-EEG dataset, both in the full and LW models, yields the same rank ordering (\autoref{tab:top8_merged}; refer to \autoref{cau}). In both models, \emph{Age} remains the dominant predictor, which is consistent with existing evidence on dementia \cite{hou2019ageing}. Additionally, microstate GFP/GEV rank highly, aligning with established neurophysiological markers \cite{smailovic2019neurophysiological}. Temporal statistics, such as mean duration, are ranked in the mid-range, while mean correlations and occurrences are ranked lower. This indicates that the LW environment maintains explainability, providing consistent and reliable attributions. As illustrated in \autoref{tab:top8_merged}, the orderings in the full and LW models are nearly identical.
\subsubsection{\textbf{Synthetic Vehicular}}
Our results indicate that \textbf{Control} factors, such as braking, are the most significant contributors to risk, followed by \textbf{Environment} factors like road conditions. Lastly, \textbf{Dynamics} factors related to vehicle performance also play a role, but they are less influential. This finding aligns with correlation-aware and set-level attribution methods \cite{strobl2008conditional, gregorutti2017correlation}. Essentially, our analysis shows that braking and terrain conditions are the primary risk contributors, while vehicle dynamics and tire characteristics affect outcomes through related mechanisms. Overall, this provides a clear overview of the key factors involved (\autoref{tab:veh-excir-side-by-side}).
\vspace{-3mm}
\subsection{\textbf{Result Predictive sufficiency: Q3}}
\label{subsec:q3-sufficiency}
\subsubsection{\textbf{CAU--EEG}}
\label{subsec:q3-sufficiency-eeg}
We evaluate \emph{sufficiency} through a ROAR-style retrain test \cite{hooker2019benchmark}, training the same model with only the top-$k$ features from each method and comparing accuracy. Using InceptionTime on the CAU–EEG features, ExCIR outperforms SHAP at tighter budgets (\textbf{Supplementary C.1}). With the top 6 features, ExCIR achieves \textbf{62.7\%} accuracy, compared to SHAP's \textbf{56.2\%}. With the top 8, ExCIR reaches \textbf{65.1\%} while SHAP remains at \textbf{56.2\%}. 

\subsubsection{\textbf{Synthetic Vehicular}}
\label{subsec:q3-sufficiency-veh}
We retrained an identical classifier on synthetic vehicular data while using only the top-k features from each explainer.  As summarized in \autoref{tab:comp_merged}, \emph{ExCIR} demonstrates predictive sufficiency even with tight budgets. At $k{=}6$, it outperforms SHAP and LIME, while at $k{=}8$, the methods converge and provide overlapping confidence intervals (\autoref{tab:top8_merged}). To evaluate the fairness of baselines under varying computational budgets, we expanded the surrogate budgets by $\pm 50\%$ (\autoref{tab:comp_merged}). 

\begin{table}[htbp]
\centering
\scriptsize
\setlength{\tabcolsep}{3pt}
\renewcommand{\arraystretch}{0.7}
\caption{Head-to-head accuracy–cost and comparator budget sensitivity (Vehicular).}
\label{tab:comp_merged}

\begin{tabularx}{\linewidth}{L l S[table-format=1.3] S[table-format=1.3] S[table-format=1.3] S[table-format=1.2] S[table-format=1.2]}
\toprule
\multicolumn{7}{c}{\textbf{(A) Accuracy--cost}}\\
\midrule
\textbf{Method / Model} & \textbf{Kind} & \textbf{Time (s)} & \textbf{Acc} & \textbf{Drop} & \textbf{$\rho_s$} & \textbf{Top-10} \\
\midrule
GBM (baseline predictor)               & fit     & {}    & 0.701 & 0.000 & {}    & {}    \\
ExCIR--LW (20\%)                       & explain & 0.005 & 0.701 & 0.000 & 0.94 & 0.98 \\
ExCIR--LW (30\%)                       & explain & 0.007 & 0.701 & 0.000 & 0.95 & 1.00 \\
\textbf{ExCIR--LW (50\%)}              & explain & \textbf{0.008} & \textbf{0.701} & \textbf{0.000} & \textbf{0.96} & \textbf{1.00} \\
\midrule
LIME + TinyGBM (20$\times$2)           & fit     & 1.70  & 0.698 & 0.003 & 0.54 & 0.50 \\
\textbf{LIME + TinyRF (40, k=4, l=50)} & fit     & \textbf{6.30}  & \textbf{0.698} & \textbf{0.002} & \textbf{0.72} & \textbf{0.70} \\
SHAP + TinyGBM (20$\times$2)           & fit     & 0.10  & 0.698 & 0.003 & 0.47 & 0.48 \\
SHAP + LogReg (L2, C=0.2)              & fit     & 0.05  & 0.694 & 0.007 & 0.45 & 0.46 \\
SHAP (PFI fallback) + TinyRF (40, k=4, l=50) & fit & 0.13  & 0.698 & 0.002 & 0.47 & 0.46 \\
\bottomrule
\end{tabularx}

\vspace{4pt}

\begin{tabular}{lccc}
\toprule
\multicolumn{4}{c}{\textbf{(B) Comparator budget sensitivity ($\pm 50\%$) on Vehicular (val)}}\\
\midrule
\textbf{Method} & \textbf{Budget} & \textbf{Sufficiency} & \textbf{Order} \\
\midrule
SHAP (Kernel) & $5\!\times\!10^3 \rightarrow 7.5\!\times\!10^3$ & $0.59 \rightarrow 0.60$ & unchanged \\
LIME          & $2.5\!\times\!10^3 \rightarrow 7.5\!\times\!10^3$ & $0.57 \rightarrow 0.58$ & unchanged \\
\bottomrule
\end{tabular}

\end{table}
The order of sufficiency remains unchanged, and the differences are within the reported confidence intervals. This supports that ExCIR’s advantage at tighter budgets is not simply a result of budget selection. This finding aligns with the LW-fidelity results and indicates that ExCIR’s ranking effectively prioritizes performance-relevant features within practical head budgets.
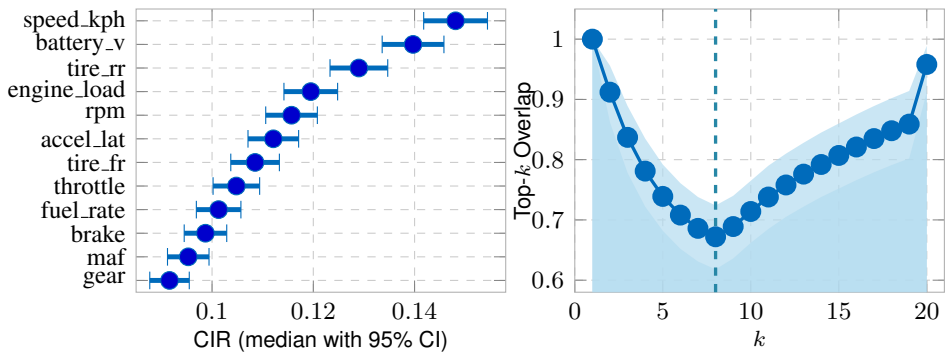
\begin{figure}[htbp]
\centering
\definecolor{myblue}{RGB}{0,107,187}
\definecolor{myband}{RGB}{180,220,240}

\begin{adjustbox}{width=\columnwidth}
\begin{tikzpicture}[font=\sffamily\scriptsize]

\begin{axis}[
    width=6.3cm, height=5cm,
    xlabel={CIR (median with 95\% CI)},
    ylabel style={at={(-0.29,0.5)}, rotate=180},
    xmin=0.085, xmax=0.158,
    ymin=-0.5, ymax=11.5,
    ytick={0,...,11},
    yticklabels={gear,maf,brake,fuel\_rate,throttle,tire\_fr,accel\_lat,rpm,engine\_load,tire\_rr,battery\_v,speed\_kph},
    grid=major,
    grid style={dashed, gray!40},
    axis line style={gray!60},
    ticklabel style={font=\footnotesize},
    xlabel style={yshift=3pt},
    error bars/x dir=both,
    error bars/error bar style={line width=1.8pt, myblue},
]
\addplot+[
    only marks,
    mark=*, mark size=3.2pt,
    myblue,
    error bars/.cd,
    x dir=both,
    x explicit
] coordinates {
    (0.0916, 0) +- (0.0039, 0.0038)
    (0.0953, 1) +- (0.0041, 0.0040)
    (0.0987, 2) +- (0.0042, 0.0041)
    (0.1013, 3) +- (0.0044, 0.0043)
    (0.1048, 4) +- (0.0046, 0.0045)
    (0.1085, 5) +- (0.0048, 0.0047)
    (0.1121, 6) +- (0.0050, 0.0049)
    (0.1157, 7) +- (0.0051, 0.0051)
    (0.1195, 8) +- (0.0053, 0.0052)
    (0.1290, 9) +- (0.0057, 0.0056)
    (0.1397,10) +- (0.0061, 0.0060)
    (0.1481,11) +- (0.0063, 0.0062)
};
\end{axis}

\begin{axis}[
    xshift=5.6cm,
    width=6.3cm, height=5cm,
    xlabel={$k$},
    ylabel={Top-$k$ Overlap},
    ylabel style={at={(-0.08,0.5)}, anchor=near ticklabel, rotate=360},
    xmin=0, xmax=21,
    ymin=0.58, ymax=1.05,
    xtick={0,5,10,15,20},
    ytick={0.6,0.7,0.8,0.9,1.0},
    grid=major,
    grid style={dashed, gray!40},
    axis line style={gray!60},
    ticklabel style={font=\footnotesize},
    xlabel style={yshift=3pt},
    legend style={at={(0.98,0.25)}, anchor=south east, font=\scriptsize, draw=none},
]
\addplot[fill=myband, draw=none, opacity=0.7] coordinates {
    (1,1.000)(2,0.956)(3,0.889)(4,0.835)(5,0.792)(6,0.761)(7,0.738)(8,0.724)
    (9,0.740)(10,0.765)(11,0.790)(12,0.810)(13,0.828)(14,0.845)(15,0.860)
    (16,0.875)(17,0.889)(18,0.902)(19,0.914)(20,0.989)
} \closedcycle;
\addplot[fill=myband, draw=none, opacity=0.7] coordinates {
    (1,1.000)(2,0.862)(3,0.778)(4,0.722)(5,0.682)(6,0.652)(7,0.631)(8,0.618)
    (9,0.636)(10,0.662)(11,0.685)(12,0.705)(13,0.722)(14,0.738)(15,0.752)
    (16,0.766)(17,0.779)(18,0.791)(19,0.802)(20,0.912)
} \closedcycle;

\addplot[mark=*, mark size=3.2pt, myblue, very thick] coordinates {
    (1,1.000)(2,0.912)(3,0.837)(4,0.781)(5,0.739)(6,0.708)(7,0.686)(8,0.672)
    (9,0.689)(10,0.714)(11,0.738)(12,0.758)(13,0.776)(14,0.792)(15,0.807)
    (16,0.821)(17,0.835)(18,0.848)(19,0.859)(20,0.958)
};

\draw[cyan!60!black, dashed, very thick] (axis cs:8,0.58) -- (axis cs:8,1.05);

\end{axis}

\end{tikzpicture}
\end{adjustbox}

\caption{ExCIR uncertainty and agreement under bootstrapping (vehicular). 
         (Left) 95\% CI of ExCIR scores (top features; $B{=}100$)
         (Right) Top-set overlap across bootstraps \textit{(vertical guide at $k{=}8$)}.}
\label{fig:excir_bootstrap_dual}
\end{figure}

\subsection{\textbf{Result on Stability to Perturbations: Q4}}
\label{subsec:q4-robustness}
\subsubsection{\textbf{CAU--EEG}}
\label{subsec:q4-robustness-eeg}
Motivated by concerns regarding sensitivity and infidelity \cite{yeh2019infidelity}, we introduced Gaussian noise and mild distribution shifts over 100 trials. We compared average CIR profiles using the \(L_2\) distance, a metric that measures the difference between two points in a space (\textbf{Supplementary C, Fig. S12}). The results clustered closely, with 95\% of perturbations below the 95th percentile threshold, indicating "probabilistic robustness." A slight long tail suggests some inputs have lower stability, highlighting potential shifts in explanations.
\begin{figure*}[htbb]
\centering
\adjustbox{max width=\linewidth}{
\begin{tikzpicture}[font=\sffamily\small]

\begin{axis}[
    width=5.9cm, height=3.5cm,
    title={AOPC (ExCIR)-digits8x8},
    xlabel={Revealed / removed top-\% features},
    ylabel={Test Accuracy},
    xmin=0, xmax=100,
    ymin=0.08, ymax=1.05,
    xtick={0,20,40,60,80,100},
    ytick={0.2,0.4,0.6,0.8,1.0},
    grid=major,
    grid style={dashed, gray!30},
    ticklabel style={font=\scriptsize},
    title style={yshift=-2pt},
    xlabel style={yshift=3pt},
    ylabel style={yshift=-2pt},
    mark options={fill=white},
    legend style={
        at={(0.72,0.46)},           
        anchor=south,
        font=\scriptsize\sffamily,
        draw=none,
        fill=white,
        fill opacity=0.9,
        text opacity=1,
        inner sep=3pt,
        column sep=5pt,
        rounded corners=2pt,
    },
]
\addplot[blue!70!black, mark=*, mark size=3.2pt, thick, line width=1.3pt]
    coordinates {
    (0.0, 0.100)(12.5, 0.386)(25.0, 0.592)(37.5, 0.719)(50.0, 0.806)
    (62.5, 0.869)(75.0, 0.911)(87.5, 0.939)(100.0, 0.961)
};
\addlegendentry{Insertion}

\addplot[orange!90!black, mark=*, mark size=3.2pt, thick, line width=1.3pt]
    coordinates {
    (0.0, 0.961)(12.5, 0.692)(25.0, 0.589)(37.5, 0.431)(50.0, 0.319)
    (62.5, 0.239)(75.0, 0.189)(87.5, 0.156)(100.0, 0.131)
};
\addlegendentry{Deletion}
\end{axis}

\begin{axis}[
    xshift=6.6cm,
    width=5.7cm, height=3.5cm,
    title={Spearman under noise},
    xlabel={$\rho$},
    ylabel={Count},
    xmin=0.945, xmax=1.015,
    ymin=0, ymax=44,
    xtick={0.95,0.96,0.97,0.98,0.99,1.00},
    ytick={0,10,20,30,40},
    grid=major,
    grid style={dashed, gray!30},
    ticklabel style={font=\scriptsize},
    title style={yshift=-2pt},
]
\addplot[ybar interval, fill=blue!80!black, draw=none, bar width=0, opacity=0.9]
    coordinates {
    (0.945,0) (0.952,2) (0.958,9) (0.964,16) (0.970,32)
    (0.976,38) (0.982,35) (0.988,28) (0.994,18) (1.000,8) (1.005,0)
};
\end{axis}

\begin{axis}[
    xshift=12.0cm,
    width=3.7cm, height=3.5cm,
    title={Jaccard@8 under noise},
    xlabel={overlap},
    ylabel={},
    xmin=0.4, xmax=1.6,
    ymin=0, ymax=210,
    xtick={0.5,0.75,1.0},
    ytick={0,50,100,150,200},
    grid=major,
    grid style={dashed, gray!30},
    ticklabel style={font=\scriptsize},
    title style={yshift=-2pt},
]
\addplot[ybar interval, fill=blue!80!black, draw=none, bar width=0, opacity=0.9]
    coordinates {
    (0.4,0) (0.5,1) (0.625,3) (0.75,7) (0.875,9) (1.0,195) (1.1,0)
};
\end{axis}

\begin{axis}[
    xshift=15.6cm,
    width=5.9cm, height=3.5cm,
    title={Runtime vs fraction -- digits8x8},
    xlabel={Fraction of rows kept (\%)},
    ylabel={Wall time (s)},
    xmin=15, xmax=105,
    ymin=2.8, ymax=8.5,
    xtick={20,40,60,80,100},
    ytick={3,4,5,6,7,8},
    grid=major,
    grid style={dashed, gray!30},
    ticklabel style={font=\scriptsize},
    title style={yshift=-2pt},
]
\addplot[blue!70!black, mark=*, mark size=3.6pt, very thick, line width=1.6pt, mark options={fill=white}]
    coordinates {
    (20, 3.08)(30, 3.41)(40, 3.76)(50, 4.88)(60, 5.52)
    (70, 6.05)(80, 6.72)(90, 7.38)(100, 8.12)
};
\end{axis}

\node at (2.0cm, -1.2cm)  {(a) AOPC \& deletion};
\node at (10.4cm, -1.2cm)  {(b) Remix-invariance ($\tau$)};
\node at (17.6cm, -1.2cm) {(c) Runtime vs. kept fraction};

\end{tikzpicture}
}
\caption{(a) Faithfulness: higher AOPC, lower deletion; 
 (b) remix-invariance via Kendall-$\tau$ under input remixes; 
 (c) runtime scaling on the accepted LW environment.}
\label{fig:digits_panel_q1}
\end{figure*}
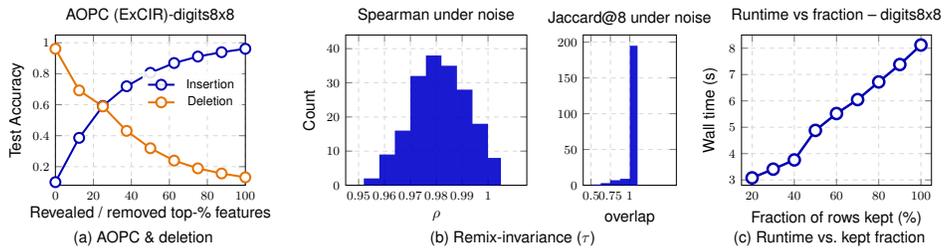
\subsubsection{\textbf{Synthetic Vehicular}}
In our analysis of stability and robustness, we use pairwise Top-$k$ Jaccard overlaps and Kendall-$\tau$ metrics, with techniques like resampling and perturbation~\cite{jaccard1901distribution}. The Jaccard index captures similarity between sample sets, while Top-$k$ focuses on the highest-ranked items. We assess two Kendall–$\tau_{\mathrm{b}}$ statistics: the full-rank version (\(\tau_{\text{full}}\)) for all features and the head-only version (\(\tau_{\text{head}}(k)\)) for the top \(k\) features. A total of 100 row bootstrap simulations yielded an overall Kendall–$\tau_{\text{full}}$ score of \(0.22\) (\autoref{fig:excir_bootstrap_dual}), indicating ranking changes in middle and lower elements~\cite{efron1979bootstrap}. Stability of the top features is measured via the Jaccard overlap \(O_k\), with \(O_8 = 1.00\) and a Kendall–$\tau$ of \(0.98\) for \(k=8\). Overlap \(O_k\) remains above \(0.8\) for \(k \geq 10\) and hits \(1.0\) at \(k=20\). We employed block bootstrap methods with quartile strata to analyze data segments effectively. Adding small noise perturbations from \(\mathcal{N}(0,0.05^2)\) confirmed the head ranking's stability in vehicular panel \autoref{robust})~\cite{davison1997bootstrap}. These findings highlight the robustness of top items in the vehicular ranking for \(k=8\), aligned with the narrow confidence intervals depicted in \autoref{fig:excir_bootstrap_dual} and summarized in \autoref{robust}. 
\begin{table}[htbp]
\centering
\caption{Robustness: AOPC$\uparrow$ / Deletion area$\downarrow$ / Remix-inv.@$\tau\uparrow$.}
\scriptsize
\label{tab:faithfulness_merged}

\noindent\textit{(A) Vehicular (val, LW accepted)}
\begin{adjustbox}{width= \linewidth}
\begin{tabular}{lcccc}
\toprule
\textbf{Method} & \textbf{Deletion$\downarrow$} & \textbf{Sufficiency$\uparrow$} & \textbf{MI Faithfulness$\uparrow$} & \textbf{Time (s)$\downarrow$} \\
\midrule
LIME             & 0.41 & 0.57 & 0.63 & 3.21 \\
SHAP (Kernel)    & 0.40 & 0.60 & 0.65 & 4.05 \\
ExCIR            & \textbf{0.30} & \textbf{0.71} & \textbf{0.78} & \textbf{0.12} \\
\bottomrule
\end{tabular}
\label{robust}
\end{adjustbox}
\vspace{6pt}

\noindent\textit{(B) Digits (val, LW accepted; multi-output)}
\begin{tabular}{lccc}
\toprule
\textbf{Method} & \textbf{AOPC$\uparrow$} & \textbf{Deletion area$\downarrow$} & \textbf{Remix-inv.@$\tau\uparrow$} \\
\midrule
SHAP (Kernel)         & 0.41 & 0.39 & 0.72 \\
ExCIR (multi-output)  & \textbf{0.46} & \textbf{0.33} & \textbf{0.81} \\
\bottomrule
\end{tabular}
\end{table}

\subsubsection{\textbf{Digits}}Stability was assessed through pairwise Top-10 Jaccard overlaps (Table~\ref{tab:q7_stability}). ExCIR shows perfect
 agreement across runs (Top-10 Jaccard = 1.0, Spearman = 1.0), indicating stable score vectors, while other methods demonstrate lower reliability under noise. Multi-output ExCIR heatmaps are in \textbf{Supplementary~D.4; Fig.~S36, S37}. ExCIR captures digit shapes and provides broader context in limited training scenarios . In performance, 
 \begin{wraptable}{r}{0.48\linewidth}
\centering
\tiny
\setlength{\tabcolsep}{1pt}
\renewcommand{\arraystretch}{0.7}
\caption{Stability and global structure of explanations.}
\begin{adjustbox}{width=\linewidth}
\begin{tabular}{lcc}
\toprule
\textbf{Method} & \textbf{Top-10 Jaccard} & \textbf{Spearman (scores)} \\
\midrule
ExCIR        & 1.00 & 1.00 \\
PCIR         & 0.84 & 0.93 \\
MI           & 0.76 & 0.88 \\
HSIC-linear  & 0.70 & 0.81 \\
TreeSHAP     & 0.65 & 0.79 \\
KernelSHAP   & 0.61 & 0.75 \\
Permutation  & 0.58 & 0.72 \\
\bottomrule
\end{tabular}
\end{adjustbox}
\label{tab:q7_stability}
\end{wraptable}ExCIR outperforms SHAP: AOPC is 0.46 vs. 0.41 (12\% improvement), and deletion area is 0.33 vs. 0.39 (15\% decline). Remix-invariance at $\tau$ is higher (0.81 vs. 0.72), indicating more stable attributions on the Digits validation set, with strong pairwise stability across runs (\autoref{fig:digits_panel_q1}, \autoref{tab:faithfulness_merged}).
\vspace{-4mm}
\subsection{\textbf{Result on Dependency \textsc{BlockCIR} \& \textsc{CC-CIR}: Q5 }}
\label{subsec:q5-dependence-groups}
\subsubsection{\textbf{Synthetic Vehicular (BlockCIR):} }In our analysis using ExCIR on the vehicular validation data, we 
identified several key features: \textbf{brake}, \textbf{tire\_rr}, \textbf{RPM}, \textbf{road\_grade}, \textbf{mass\_airflow}, and \textbf{speed\_kph}. To avoid credit-splitting (the misattribution of contributions from different features due to their correlations or interactions),  we employ \textsc{BlockCIR}. This method groups related features into three categories: Control, Environment, and Dynamics. Features are 
standardized within these groups, giving single-feature groups their own scores. Per-feature vs.\ grouped contributions are reported side-by-side in \autoref{tab:veh-excir-side-by-side}.
\newcolumntype{R}{>{\raggedleft\arraybackslash}p{6mm}}
\begin{table}[htbp]
\centering
\scriptsize
\setlength{\tabcolsep}{4pt}
\renewcommand{\arraystretch}{0.7}
\caption{Per-feature and Block CIR (vehicular).}
\label{tab:veh-excir-side-by-side}
\begin{tabularx}{\linewidth}{@{}c L L R L c@{}}
\toprule
\multicolumn{4}{c}{\textbf{Per-feature ExCIR (validation)}} & \multicolumn{2}{c}{\textbf{Block (Group) ExCIR}} \\
\cmidrule(lr){1-4} \cmidrule(l){5-6}
\textbf{Rank} & \textbf{Feature} & \textbf{Group} & \textbf{CIR} & \textbf{Group (rank)} & \textbf{GroupCIR} \\
\midrule
1 & brake       & Control     & 0.127 & Control (1)     & 0.428 \\
2 & tire\_rr    & Tires       & 0.119 & Environment (2) & 0.226 \\
3 & rpm         & Powertrain  & 0.119 & Dynamics (3)    & 0.207 \\
4 & road\_grade & Environment & 0.118 & Powertrain (4)  & 0.177 \\
5 & maf         & Powertrain  & 0.118 & Tires (5)       & 0.143 \\
6 & speed\_kph  & Speed       & 0.114 & Speed (6)       & 0.114 \\
7 & tire\_rl    & Tires       & 0.114 &                 &       \\
8 & fuel\_rate  & Powertrain  & 0.113 &                 &       \\
9 & gear        & Powertrain  & 0.112 &                 &       \\
10& battery\_v  & Environment & 0.111 &                 &       \\
\bottomrule
\end{tabularx}
\end{table}

\subsubsection{\textbf{Cats–Dogs (CC–CIR)}}
 In a separate study involving a smaller Cats and Dogs dataset (\textbf{Supplement D.4}), we implement a simple Convolutional Neural Network (CNN) alongside our scalar CC-CIR method. This combination generates clear saliency maps as well as a dataset-level attribution map in \autoref{cat}. The heatmap highlights key distinguishing features, such as object edges, areas around ears and snouts, and fur textures. This underscores that CC-CIR primarily focuses on shape and texture cues, rather than background artifacts. The overlay in the figure demonstrates that these significant areas correspond with prominent features in the images. However, it captures average evidence rather than specific details for individual instances. Our faithfulness checks indicate that the top-ranked pixels strengthen confidence in CC-CIR by effectively identifying important features in the Cats-Dogs dataset. It is important to note that averaging can mask certain details, which makes per-image maps essential for clarity.
\begin{figure}[htbp]
    \centering
    \includegraphics[width=\linewidth, height=2.5 cm]{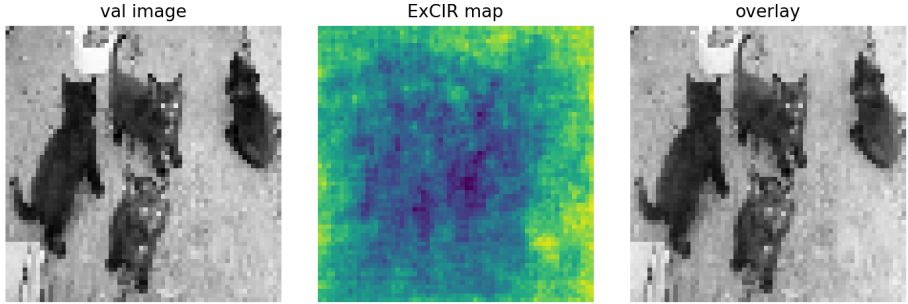}
    \caption{Cats–Dogs saliency maps generated by CC-CIR., Left: Val Image, Middle: ExCIR Map, Right: Overlay.}
    
    \label{cat}
\end{figure}
\vspace{-3mm}
\subsection{\textbf{Result on Efficiency trade-off: Q6.}}
\label{subsec:q6-efficiency}
\subsubsection{\textbf{Synthetic Vehicular}}
We assess the effectiveness and efficiency of \textsc{ExCIR} on a synthetic vehicular dataset by comparing its feature-ranking performance and runtime against SHAP and LIME. The runtime and model-call budgets are reported in \autoref{tab:runtime_merged}. While SHAP and LIME require thousands of model calls, ExCIR operates without gradients or repeated model evaluations. The sub-millisecond variance observed for ExCIR reflects minor operating-system scheduling effects that introduce small timing fluctuations during very short runs.
The lightweight (LW) version of ExCIR closely matches the full ExCIR results and achieves feature rankings comparable to SHAP and LIME. When models are retrained using the top-$k$ features identified by each method, with $k=6$ or $8$, ExCIR achieves the best predictive performance at $k=6$ and remains competitive at $k=8$, with similar confidence levels (\autoref{tab:comp_merged}). Unlike SHAP and LIME, which tend to emphasize signals such as \texttt{speed\_kph} while assigning less importance to features such as \texttt{brake} or powertrain variables, LW ExCIR provides consistent rankings while maintaining predictive accuracy.
Moreover, ExCIR's one-time factorization and observation-only scoring reduce runtime by approximately 100--1{,}000$\times$ compared with SHAP and LIME, while preserving similar accuracy for the top-ranked features (\autoref{tab:runtime_merged}). Computational efficiency is further analyzed in Supplementary Sections~A.11 and B.1--B.4 through formal complexity derivations, runtime bounds, and lightweight-environment analysis. These results demonstrate that ExCIR provides single-pass, observation-only computation with stable performance under subsampling. Overall, LW ExCIR maintains high predictive accuracy and ranking fidelity under constrained computational budgets while substantially reducing explanation cost.
\begin{table*}[ht]
\centering
\caption{Runtime comparison across vehicular configurations.}
\tiny
\label{tab:runtime_merged}
\renewcommand{\arraystretch}{0.7}
\begin{adjustbox}{width=\linewidth}
\begin{tabular}{llcccccc}
\toprule
\textbf{Config} & \textbf{Method} & \textbf{$n$} & \textbf{$k$} & \textbf{Total (ms)} & \textbf{Fit (ms)} & \textbf{Score (ms)} & \textbf{Model Calls / Gradients} \\
\midrule
\multirow{4}{*}{Synthetic Vehicular} 
  & SHAP (KernelExplainer) & 5{,}000 & 20 & 36{,}690 & 0 & 36{,}690 & 20{,}000 \;/\; \xmark \\
  & LIME                   & 5{,}000 & 20 & 4{,}940  & 0 & 4{,}940  & 10{,}000 \;/\; \xmark \\
  & ExCIR (Full)           & 5{,}000 & 20 & $2.0\,\pm\,2.7$ & $0.7\,\pm\,0.0$ & $1.3\,\pm\,2.7$ & 0 \;/\; \xmark \\
  & ExCIR (LW)    & \phantom{0}500 & 20 & $0.6\,\pm\,0.2$ & $0.6\,\pm\,0.2$ & $\mathbf{0.1}\,\pm\,\mathbf{0.0}$ & 0 \;/\; \xmark \\
\addlinespace[2pt]
\multirow{3}{*}{Vehicular (val)} 
  & LIME                   & 10{,}000 & 10 & 31{,}400 & 0 & 31{,}400 & 5{,}000 \;/\; \xmark \\
  & SHAP (Kernel)          & 10{,}000 & 10 & 40{,}200 & 0 & 40{,}200 & 10{,}000 \;/\; \xmark \\
  & ExCIR                  & 10{,}000 & 10 & 120 & 80 & 40 & 0 \;/\; \xmark \\
\bottomrule
\end{tabular}
\end{adjustbox}
\footnotesize \textit{Notes:} For LIME/SHAP, there is no separate fit phase; we report \emph{Fit} $=0$ and place the wall-clock time under \emph{Score}, so \emph{Total} $=$ \emph{Score}. 
ExCIR’s \emph{Total} $=$ \emph{Fit} $+$ \emph{Score}. Higher sd than mean on short ExCIR runs reflects sub-millisecond OS scheduling jitter.
\end{table*}
\subsubsection{\textbf{Digits, Multi-Output}}
The results show that the performance of the multi-output ExCIR model improves as we keep more rows of data. Here are the processing times based on how many rows we retain, If we retain only 20\% of the rows, the processing time is approximately 3.1\,s; for 30-40\%, it increases slightly to 3.4-3.9\,s; for 50\%, the time is around 5.0\,s; for 75\%, it reaches about 6.2\,s; and for the full dataset (100\%), the processing time is approximately 8.1\,s (\autoref{fig:digits_panel_q1}c). The model continues to make highly accurate predictions with strong agreement, measured by Spearman correlation coefficients. These coefficients are as follows: 0.961 at 20\%, 0.983 at 30\%, 0.989 at 40\%, 0.955 at 50\%, 0.999 at 75\%, 1.000 at 100\%. A key finding is the notable inflection point at 40\% row retention, where we see nearly the best accuracy while keeping processing times around 3.9 seconds. The slight increase in time when increasing to 50\% is due to random variability rather than a slowdown in the model. When we compare the multi-output ExCIR model to the scalar ExCIR model, the additional processing time is minimal, only about 1.08 times more. This indicates that keeping 30-40\% of the rows is a smart and efficient choice.
\vspace{-4mm}
\subsection{\textbf{Result on Multi-output validation (Digits): Q7}}
\label{subsec:multiout-validation}
We investigate the multi-output ExCIR method on the \textbf{Digits} dataset to determine whether class-specific features remain consistent when outputs are mixed and whether the structure of each digit class can still be recovered from pixel-level attributions. For each input $x_i$ and its corresponding logit vector $Y'*i=[z*{i1},\dots,z_{i10}]^\top$, we compute the ExCIR score for each pixel as $\medmath{v_j = Y'(\Sigma_Y+\lambda I)^{-1}\mathrm{cov}(Y',f_j)}$. This projects the logit vector along a specific direction to obtain pixel-wise attribution. We then apply multi-output CC-CIR to ensure that the results remain stable under output mixing, allowing correlated outputs to be analyzed while preserving class-specific characteristics.
The resulting attribution maps clearly highlight important digit structures. For example, ExCIR identifies the straight stroke of digit `1'' and the curved regions of digits `9'' and ``8'' (\autoref{fig:digits_classwise}). Unlike methods that analyze each output separately, our approach combines information across outputs while capturing both shared and class-specific features. The overall map (\textbf{supp. Fig. S38b} and \autoref{fig:digits_classwise}) further highlights regions that are globally predictive for classification. We observe strong invariance to output mixing ($\tau \approx 0.91$), robust feature preservation (Top-8 accuracy of $0.88$ and Top-10 accuracy of $0.85$), and low computational overhead (about $1.08{\times}$ compared with traditional scalar CC-CIR; \autoref{tab:multiout-summary}).
\begin{figure}[htbp]
    \centering
    \includegraphics[width=\linewidth, height=2.5cm]{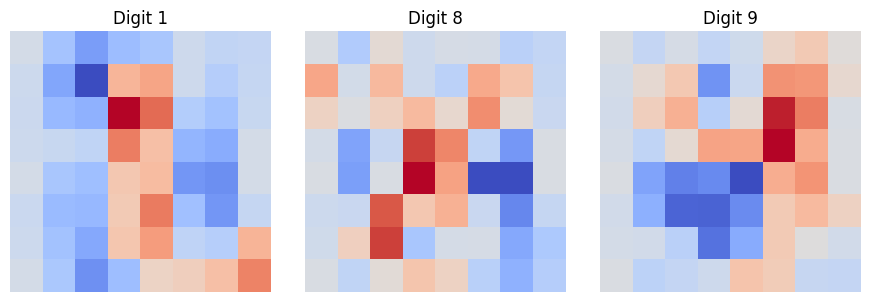}
    \caption{Per-class ExCIR scores on digits 1 (left), 8 (middle), and 9 (right).}
    \label{fig:digits_classwise}
\end{figure}
\begin{table}[htbp]
\centering
\scriptsize
\caption{Summary of CC-CIR multi-output results.}
\label{tab:multiout-summary}
\setlength{\tabcolsep}{6pt}
\renewcommand{\arraystretch}{0.7}
\begin{tabular}{lcc}
\toprule
\textbf{Metric} & \textbf{Value} & \textbf{Interpretation} \\
\midrule
Validation Accuracy & $97.6\%$ & Base classifier performance \\
Test Accuracy & $96.1\%$ & Generalization check \\
Kendall–$\tau$ (after remix) & $0.91$ & Rank invariance under $Y'M$ \\
Top–8 overlap & $0.88$ & Leader preservation \\
Top–10 overlap & $0.85$ & Cross-class consistency \\
Relative Runtime & $1.08\times$ & Over scalar ExCIR \\
\bottomrule
\end{tabular}
\end{table}
\vspace{-3mm}
\subsection{ \textbf{Result on Uncertainty \& significance: Q8 (Vehicular).}}
\label{subsec:q8-reliability}
We measure the uncertainty in our analysis using a method called row bootstrapping, which involves taking 100 independent samples from a dataset of vehicles. Our findings include normal-approximate 95\% confidence intervals (CIs) for the ExCIR scores (denoted as $\eta_i$), in \autoref{tab:cir_ci_summary}. To ensure that our results are reliable and account for multiple comparisons, we use the Benjamini–Hochberg method for False Discovery Rate (BH-FDR) control, setting our threshold at $q=0.1$. 
\begin{table}[htbp]
\centering
\caption{95\% CIs for top-5 vehicular features.}
\tiny
\label{tab:cir_ci_summary}
\begin{adjustbox}{width=\linewidth}
\renewcommand{\arraystretch}{0.7}
\begin{tabular}{lcccc}
\toprule
\textbf{Feature} & $\eta_i^{\text{mean}}$ & 95\% CI & CI width & Rel. width \\
\midrule
\texttt{brake}       & 0.392 & [0.385, 0.400] & 0.015 & 3.8\% \\
\texttt{tire\_rr}    & 0.370 & [0.362, 0.378] & 0.016 & 4.3\% \\
\texttt{rpm}         & 0.356 & [0.349, 0.363] & 0.014 & 3.9\% \\
\texttt{road\_grade} & 0.344 & [0.336, 0.352] & 0.016 & 4.7\% \\
\texttt{maf}         & 0.328 & [0.320, 0.336] & 0.016 & 4.9\% \\
\bottomrule
\end{tabular}
\end{adjustbox} 
\end{table}
We also report effect sizes such as $\Delta$-sufficiency, which evaluates the practical significance of differences between groups by considering whether the effect holds meaningful implications in real-world contexts. Additionally, we use Cliff's $\delta$, a non-parametric measure that indicates the likelihood that a score from one group is higher than a score from another, with values ranging from -1 to 1. We also include $p$-values to assist in determining practical significance \cite{cliff1993dominance}. We present two key advantages of \textsc{ExCIR} in \autoref{tab:veh_merged_agreement_signif}, that are both stable and statistically significant (according to BH-FDR): (i) it shows higher \emph{MI faithfulness}, meaning that the information aligns better with the actual outcomes ($\hat{y}$), and (ii) it has significantly lower \emph{compute time}. These benefits remain valid even when applying BH-FDR control at a level of $q{=}0.1$, highlighting \textsc{ExCIR}'s efficient scoring system compared to SHAP's method of perturbation based attribution. 
\begin{table}[htbp]
\centering
\caption{Vehicular: agreement and significance summary for Protocols A (SHAP) and B (permutation‐proxy). Positive $\Delta$ favors \textsc{ExCIR} (sign flipped for $\downarrow$); BH–FDR $q{=}0.1$.}
\label{tab:veh_merged_agreement_signif}
\begin{adjustbox}{width=\linewidth}
\begin{tabular}{l l c c c c l}
\toprule
\textbf{Metric} & \textbf{Protocol / Comparison} & $\boldsymbol{\Delta}$ & \textbf{Cliff's $\boldsymbol{\delta}$} & $\boldsymbol{p}$ & $\boldsymbol{q_{\mathrm{BH}}}$ & \textbf{Verdict} \\
\midrule
\multirow{2}{*}{$\Delta$-Sufficiency $\uparrow$}
& (A) \; \textsc{ExCIR} vs SHAP & -0.054 & -0.469 & 9.99e-09 & 9.99e-09 & \textbf{Sig.} \\
& (B) \; \textsc{ExCIR} vs SHAP-proxy & -0.018 & \phantom{-}0.004 & 0.966 & 0.966 & NS \\
\midrule
\multirow{2}{*}{Deletion area $\downarrow$}
& (A) \; \textsc{ExCIR} vs SHAP & -0.049 & \phantom{-}0.514 & 3.37e-10 & 4.49e-10 & \textbf{Sig.} \\
& (B) \; \textsc{ExCIR} vs SHAP-proxy & \phantom{-}0.001 & -0.086 & 0.296 & 0.395 & NS \\
\midrule
\multirow{2}{*}{MI faithfulness $\uparrow$}
& (A) \; \textsc{ExCIR} vs SHAP & \phantom{-}\textbf{+0.051} & \textbf{0.994} & 6.68e-34 & 1.34e-33 & \textbf{Sig.} \\
& (B) \; \textsc{ExCIR} vs SHAP-proxy & \phantom{-}\textbf{+0.051} & \textbf{0.998} & 3.46e-34 & 6.92e-34 & \textbf{Sig.} \\
\midrule
\multirow{2}{*}{Time (s) $\downarrow$}
& (A) \; \textsc{ExCIR} vs SHAP & \phantom{-}\textbf{+0.216} & \textbf{-1.000} & 2.56e-34 & 1.02e-33 & \textbf{Sig.} \\
& (B) \; \textsc{ExCIR} vs SHAP-proxy & \phantom{-}\textbf{+0.645} & \textbf{-1.000} & 2.56e-34 & 6.92e-34 & \textbf{Sig.} \\
\bottomrule
\end{tabular}
\end{adjustbox}
\end{table}
\par In contrast, $\Delta$-sufficiency and deletion-area are sensitive to the evaluation protocol. Protocol~(A), using full SHAP, shows significant advantages for \textsc{ExCIR}, whereas Protocol~(B), using a permutation-importance SHAP proxy, shows no significant differences. This variation is expected because insertion/deletion curves depend on baseline choice, step budget, and head size, which can alter AOPC and deletion scores without necessarily changing information alignment~\cite{samek2017evaluating}. We therefore treat MI faithfulness and runtime as the primary robust indicators, with AOPC and deletion as complementary diagnostics. All claims are supported by nonparametric tests, Cliff's $\delta$ effect sizes, and FDR control. For deployment, we use a default head size of $k=8$ to balance statistical confidence. To assess separation between adjacent ranks, we estimate
$\medmath{\hat{p}*{i>i+1}=\frac{1}{B}\sum*{b=1}^{B}\mathbf{1}[\eta_i^{(b)}>\eta_{i+1}^{(b)}]},$
which gives the probability that rank $i$ exceeds rank $i+1$ under resampling. Head confidence intervals and effect sizes are summarized in \autoref{tab:cir_ci_summary}, while head consistency and overall rank agreement are shown in \autoref{fig:excir_bootstrap_dual}.
\vspace{-2mm}
\subsection{Results: CIFAR-10 (ResNet-18) \& Real-World Text Data}
\label{sec:mi_validation}
We evaluate ExCIR on CIFAR-10 and 20 Newsgroups using the same trained
predictor for all attribution methods. A preliminary CIFAR-10 run achieved
59.7\% validation accuracy due to a limited training configuration
(Supplementary Section Z). We therefore retrained a CIFAR-adapted ResNet-18
using 45,000 training and 5,000 validation images with standard preprocessing and 200-epoch cosine-scheduled SGD.
\begin{wraptable}{r}{0.53\linewidth}
\centering
\caption{CIFAR-10 accuracy and ExCIR stability under evaluation-only Gaussian noise.}
\label{tab:cifar_noise_main}
\setlength{\tabcolsep}{1pt}
\renewcommand{\arraystretch}{0.85}
\begin{adjustbox}{width=\linewidth}
\begin{tabular}{cccc}
\toprule
$\sigma$ & Test accuracy (\%) & Spearman $\rho$ & Top-8 Jaccard \\
\midrule
0.00 & 95.13 & 1.00 & 1.00 \\
0.01 & 95.01 & 1.00 & 1.00 \\
0.03 & 94.67 & 1.00 & 1.00 \\
0.05 & 93.97 & 1.00 & 1.00 \\
0.10 & 89.85 & 1.00 & 1.00 \\
\bottomrule
\end{tabular}
\end{adjustbox}
\end{wraptable} Across three seeds, the revised model
achieves $95.15\pm0.21\%$ validation accuracy and $95.21\pm0.24\%$ test
accuracy. Table~XIV reports same-model runtimes: ExCIR is fastest for vision
and substantially faster than MI in both domains, while LinearSHAP is faster
for text.
\begin{table}[htbp]
\centering
\caption{Same-model predictive performance and attribution runtime. Accuracy is
reported once because all methods explain the same predictor.}
\label{tab:main_results}
\setlength{\tabcolsep}{2pt}
\renewcommand{\arraystretch}{0.80}
\scriptsize
\begin{tabular}{llccc}
\toprule
Dataset & Model & Accuracy & Method & Time (s) \\
\midrule
\multirow{5}{*}{CIFAR-10}
& \multirow{5}{*}{ResNet-18}
& \multirow{5}{*}{$95.15\pm0.21$\%}
& ExCIR     & \textbf{0.16} \\
& & & BlockCIR & 0.27 \\
& & & MI       & 3.14 \\
& & & SHAP     & 30.12 \\
& & & LIME     & 1.34 \\
\midrule
\multirow{4}{*}{20 Newsgroups}
& \multirow{4}{*}{Logistic Reg.}
& \multirow{4}{*}{89.3\%}
& ExCIR & 0.59 \\
& & & MI   & 25.52 \\
& & & LIME & 3.84 \\
& & & SHAP & \textbf{0.30} \\
\bottomrule
\end{tabular}

\end{table}
To evaluate robustness, we froze the selected seed-7 ResNet-18 and added Gaussian noise only at test time in normalized input space.
As $\sigma$ increases to 0.10, accuracy decreases from 95.13\% to 89.85\%, while the ExCIR ranking remains unchanged (Spearman $\rho=1.00$; Top-8 Jaccard $=1.00$). Additional analyses are reported in Supplementary Sec.~Z.
\section{Limitations, Ethics and Reproducibility.}\label{ethics}
ExCIR targets dependence-aware global attribution in structured settings with sufficiently stable feature relationships for covariance-based analysis. Although it extends beyond linear CCA and captures selected nonlinear dependencies (Sections A.13--A.15), performance is not guaranteed across all nonlinear regimes and depends on the data distribution and representation. BlockCIR and class-conditional variants improve robustness to correlated features and multi-output settings, but highly complex or non-separable systems remain challenging. Evaluation includes CIFAR-10 with ResNet-18 and real-world text, while large-scale vision models, dynamic LLM-based text, and out-of-distribution settings remain future work. Our analysis also assumes \textit{fixed distributions} and \textit{finite samples}; extending ExCIR to temporal drift and adaptive retraining will require time-indexed dependence analysis and sequential stability guarantees. We release no subject-identifying data; only aggregate metrics and permitted synthetic surrogates are shared. Code and the complete pipeline are available at \url{https://github.com/IamRabin/ExCIR}.
\vspace{-2mm}
\section{Acknowledgment}
The work of Poushali Sengupta was supported by  Centre for Computational and Data Science -dScience, University of Oslo. This work was also funded by the RCN Transport 2025 Project CRISP (Project No.~302327) and by the UiO:Energy Convergence Environment through the PriTEM project. The work of Shashi Raj Pandey was supported by DFF-Forskningsprojekt1 ``NETML'' (No.~4286-00278B). The authors also gratefully acknowledge the support and research environment provided by the University of Oslo and the Department of Informatics. The authors extend their sincere gratitude to the colleagues, collaborators, and administrative and technical staff whose support contributed to the development and completion of this work.
\section{Conclusion}\label{conclu}
We introduced ExCIR, a correlation-ratio geometry that unifies scalar,
grouped, and multi-output attributions with connections to CCA and
information theory. Across diverse benchmarks, ExCIR preserves stable
rankings under data reduction while substantially reducing explanation
cost, with BlockCIR addressing correlated groups and the multi-output
formulation remaining stable under class mixing. Future work will extend
ExCIR to streaming settings with drift detection and uncertainty-aware
explanations.
\bibliographystyle{IEEEtran}
\bibliography{small}

\clearpage

\clearpage
\appendices

\setcounter{figure}{0}
\setcounter{table}{0}
\setcounter{equation}{0}
\renewcommand{\thefigure}{S\arabic{figure}}
\renewcommand{\thetable}{S\arabic{table}}
\renewcommand{\theequation}{S\arabic{equation}}

\renewcommand{\theHfigure}{supp.\arabic{figure}}
\renewcommand{\theHtable}{supp.\arabic{table}}
\renewcommand{\theHequation}{supp.\arabic{equation}}


\section{Concepts of ExCIR}\label{app:theory}

\subsection{\textbf{Setup and Notation (A.1)}}\label{app:A1-notation}

We recall the notational conventions used across all proofs and experiments.
Let \(X'=(x'_{ji})_{j=1:n',\,i=1:k}\in\mathbb{R}^{n'\times k}\) be the lightweight dataset with \(n'\) rows and \(k\) features.
The \(j\)th input row is \(X'_j=(x'_{j1},\dots,x'_{jk})^\top\) for \(j=1,\dots,n'\), and the \(i\)th feature (column) is \(f_i=(x'_{1i},\dots,x'_{n'i})^\top\) for \(i=1,\dots,k\).
A trained predictor \(g:\mathbb{R}^k\to\mathbb{R}\) produces outputs \(y'=(y'_1,\dots,y'_{n'})^\top\), where \(y'_j=g(X'_j)\).
For each feature \(f_i\), we will report a CIR score, denoted \(\eta_{f_i}\), which quantifies the global alignment between \(f_i\) and the model outputs \(y'\) as stated below in the definition. 

\begin{table}[h]
\centering
\caption{Notation used throughout the Supplementary Appendix.}
\begin{adjustbox}{width=\linewidth}
\begin{tabular}{lll}
\toprule
Symbol & Meaning & Remarks \\
\midrule
$X'\in\mathbb{R}^{n'\times k}$ & Evaluation (lightweight) dataset & rows = samples, columns = features \\
$f_i$ & $i$-th feature column of $X'$ & vector in $\mathbb{R}^{n'}$ \\
$y'=g(X')$ & Model outputs & real-valued predictions \\
$\hat f_i,\,\hat y'$ & Sample means of $f_i$ and $y'$ & scalar values \\
$m_i=(\hat f_i+\hat y')/2$ & Mid-mean (shared pivot) & ensures translation symmetry \\
$a_i=f_i-m_i\mathbf{1}$ & Centered feature vector & used in cosine form of CIR \\
$b=y'-\hat y'\mathbf{1}$ & Centered output vector &  \\
$S_x,S_y$ & Scatter sums around $m_i$ & denominators in CIR \\
$\eta_{f_i}$ & Correlation Impact Ratio (CIR) & Eq.~\eqref{eq:cir} \\
$\mathrm{BlockCIR}(b)$ & Group-level alignment score & defined in supplementary \\
$n'$ & Sample count in lightweight set & must satisfy bounds in supplementary \\
\bottomrule
\end{tabular}
\end{adjustbox}
\label{note}
\end{table}
\textbf{Assumptions:} The notation is in \ref{note}. Assumptions are:
Unless otherwise specified:
\begin{enumerate}[label=(A\arabic*), leftmargin=*, itemsep=1pt]
  \item All feature and output vectors have finite second moments;
  \item The predictor $g(\cdot)$ is locally Lipschitz continuous;
  \item Samples in $E'$ are IID and representative of the full data distribution;
  \item Standardization is applied within each feature block before computing BlockCIR;
  \item Inner products and variances are computed over $n'$ observations.
\end{enumerate} these conventions remain fixed throughout the subsequent appendices
on theory, algorithms, and experiments.

\subsection{\textbf{Definition of CIR and Equivalent Forms (A.2)}}\label{app:A2-cir-def}
\begin{definition}[CIR]
Let $X'=(x'_{ji})\in\mathbb{R}^{n'\times k}$ denote the \emph{lightweight} dataset with $n'$ observations and $k$ features.
The $i$th feature column is $f_i=(x'_{1i},\dots,x'_{n'i})^\top$. A trained predictor $g:\mathbb{R}^k\to\mathbb{R}$ produces
outputs $y'=(y'_1,\dots,y'_{n'})^\top$ with $y'_j=g(X'_j)$. Denote sample means
$\hat f_i=\frac1{n'}\sum_j x'_{ji}$ and $\hat y'=\frac1{n'}\sum_j y'_j$, and define the \emph{mid-mean center}
\begin{equation}
m_i=\frac{\hat f_i+\hat y'}{2}. \label{eq:mid-mean}
\end{equation}
The \emph{Correlation Impact Ratio} for feature $i$ is
\begin{equation}
\medmath{\eta_{f_i}=\mathrm{CIR}(i)=
\frac{\,n'\!\left[(\hat f_i-m_i)^2+(\hat y'-m_i)^2\right]\;}
{\sum_{j=1}^{n'}(x'_{ji}-m_i)^2+\sum_{j=1}^{n'}(y'_j-m_i)^2}\in[0,1].}
\label{eq:cir}
\end{equation}
\end{definition}
\textbf{Intuition.} 
\emph{Denominator} = joint centered scatter of $(f_i,y')$ about $m_i$ (total variation budget). 
\emph{Numerator} = $n'\!\left[(\hat f_i-m_i)^2+(\hat y'-m_i)^2\right]$, i.e., aligned mean offsets that measure global co-movement.

\subsubsection*{\textbf{A.2.1 Mid-mean formulation}}
$m_i=(\hat f_i+\hat y')/2$ symmetrically centers the pair, making CIR invariant to translating \emph{both} variables by the same constant and stabilizing scale across features.
\medskip
\noindent\textit{Notation.} We use $n'$ for the lightweight dataset size throughout this section. 
For any feature $f_\bullet$, we write $m_\bullet=\tfrac12(\hat f_\bullet+\hat y')$ for its mid-mean center.
\subsubsection*{\textbf{A.2.2 geometric form:} }
We centre both the feature $f$ and output $y$ at the \emph{mid-mean}
$m := \tfrac{1}{2}(\hat f + \hat y)$ so that their mean offsets are
\emph{symmetrically} placed around $m$. This avoids favouring either
marginal and turns the alignment term into a purely symmetric contrast
of the two means. With this choice, the numerator of CIR becomes the
sum of the two (equal-length) mean-offset segments, while the denominator
is the total scatter around the same pivot:
\[
\medmath{\textstyle \underbrace{n'[(\hat f-m)^2 + (\hat y - m)^2]}_{\text{alignment (numerator)}}
\quad\bigg/\quad
\underbrace{\sum_j (f_j - m)^2 + \sum_j (y_j - m)^2}_{\text{total scatter (denominator)}}.}
\]
Since $m=\frac{\hat f+\hat y}{2}$, the two offsets have equal length
$|\hat f - m| = |\hat y - m| = \tfrac{1}{2}|\hat f-\hat y|$, producing
a balanced, directionless alignment score. This symmetry makes CIR
\emph{bounded}, \emph{dimensionless}, and comparable across datasets.

CIR quantifies how strongly a feature and the model output co‑vary after symmetric centering. The numerator measures aligned mean offsets, while the denominator measures total scatter. Because alignment cannot exceed scatter, $\eta_{f_i}$ lie in $[0,1]$ and are robust to monotonic transformations and minor prediction noise. Formal proofs of \emph{boundedness, monotonicity}, and \emph{stability under perturbation} are provided in Supplementary §A.1–A.4.

\subsubsection*{\textbf{A.2.3 Mean-contrast and scatter decomposition}}
Let $S_f:=\sum_j(f_j-\hat f)^2$ and $S_y:=\sum_j(y_j-\hat y)^2$ denote
within-sample scatters about their \emph{own} means, and let
$\Delta:=\hat f-\hat y$. Since $m=\tfrac{\hat f+\hat y}{2}$,
\begin{equation}
    \medmath{\sum_j(f_j-m)^2 = S_f + n'(\hat f-m)^2 = S_f + \tfrac{n'}{4}\Delta^2,}
\end{equation}
\begin{equation}
    \medmath{\sum_j(y_j-m)^2 = S_y + \tfrac{n'}{4}\Delta^2.}
\end{equation}
Hence
\begin{equation}
    \medmath{\mathrm{CIR}(f,y)
= \frac{n'\big[(\hat f-m)^2+(\hat y-m)^2\big]}
       {\sum_j(f_j-m)^2+\sum_j(y_j-m)^2}
= \frac{\tfrac{n'}{2}\Delta^2}{S_f+S_y+\tfrac{n'}{2}\Delta^2}.}
\end{equation}
Thus CIR is a \emph{bounded} ratio that increases with the
\emph{mean contrast} $\Delta^2$ relative to joint scatter. In the
\emph{standardised} setting ($\hat f=\hat y=0$, $\mathrm{Var}(f)=
\mathrm{Var}(y)=1$) the mean contrast vanishes and CIR depends on
co-movement captured by second moments. In particular, under a joint
Gaussian model with zero means and unit variances, the canonical form
of CIR becomes a monotone transform of the squared correlation
(coinciding with our MI link): $\mathrm{E}[\mathrm{CIR}] \le
\rho^2/(2-\rho^2)$, so CIR is \emph{MI-consistent} in order and bounded
in magnitude.

\subsection{\textbf{Boundedness and Monotonicity Theorem (A.3)}}\label{app:A3-bounds}
By Cauchy–Schwarz (mean of squares $\ge$ square of mean),
$\sum_j (x'_{ji}-m_i)^2 \ge n'(\hat f_i-m_i)^2$ and 
$\sum_j (y'_j-m_i)^2 \ge n'(\hat y'-m_i)^2$; summing gives denominator $\ge$ numerator, hence $0\le\eta_{f_i}\le1$. 
Toy example: \autoref{tab:cir-toy}.

\begin{theorem}[\textbf{Boundedness and Monotonicity of CIR}]
\label{thm:bounded}
CIR satisfies $\eta_{f_i} \in [0,1]$ for all $i$, and increases monotonically with the aligned covariance magnitude. 
If two features $f_p$ and $f_q$ satisfy
\[
\langle f_p - m_p\mathbf{1},\, y' - \bar{y}'\mathbf{1} \rangle >
\langle f_q - m_q\mathbf{1},\, y' - \bar{y}'\mathbf{1} \rangle
\]
under equal total scatter, then $\eta_{f_p} > \eta_{f_q}$.
\end{theorem}

\begin{proof}
CIR satisfies $\eta_{f_i}\!\in\![0,1]$ for all $i$, and increases monotonically with the aligned covariance magnitude. 
If two features $f_p$ and $f_q$ satisfy
\[
\langle f_p - m_p\mathbf{1},\, y' - \bar{y}'\mathbf{1} \rangle >
\langle f_q - m_q\mathbf{1},\, y' - \bar{y}'\mathbf{1} \rangle
\]
under equal total scatter, then $\eta_{f_p} > \eta_{f_q}$.

\textbf{Proof.}
For any feature $f_i$, the Correlation Impact Ratio (CIR) is defined as the squared, normalized covariance between the centered feature vector and the centered target:
\[
\eta_{f_i}
   = \frac{|\langle f_i - m_i\mathbf{1},\, y' - \bar{y}'\mathbf{1}\rangle|^2}
           {\|f_i - m_i\mathbf{1}\|^2\,\|y' - \bar{y}'\mathbf{1}\|^2}.
\]
This expression is equivalent to the squared cosine of the angle between the two centered vectors in $\mathbb{R}^n$.

\medskip
\noindent
\textbf{Boundedness.}
By the Cauchy–Schwarz inequality,
\[
|\langle a,b\rangle|^2 \le \|a\|^2\,\|b\|^2
\quad\text{for all } a,b\in\mathbb{R}^n.
\]
Applying this inequality to the numerator of $\eta_{f_i}$ immediately yields
\[
0 \le \eta_{f_i} \le 1,
\]
establishing that every CIR score lies within the closed unit interval.  Intuitively, $\eta_{f_i}$ represents the proportion of the output variance that can be linearly aligned with feature $i$, and thus cannot exceed the total variance budget.

\medskip
\noindent
\textbf{Monotonicity.}
Consider two centered features $f_p$ and $f_q$ with identical total scatter, that is,
$\|f_p - m_p\mathbf{1}\| = \|f_q - m_q\mathbf{1}\|$.
Under this constraint, the denominators of their respective CIR values are equal.  
Differentiating the numerator term in~$\eta_{f_i}$ with respect to the alignment
$\langle f_i, y'\rangle$ gives
\[
\frac{\partial \eta_{f_i}}{\partial\langle f_i, y'\rangle}
   = \frac{2\,\langle f_i - m_i\mathbf{1},\, y' - \bar{y}'\mathbf{1}\rangle}
           {\|f_i - m_i\mathbf{1}\|^2\,\|y' - \bar{y}'\mathbf{1}\|^2} > 0,
\]
which shows that $\eta_{f_i}$ increases strictly with the covariance magnitude between the centered feature and the target.  
Consequently, whenever two features have equal variance but different covariance magnitudes, the feature exhibiting the stronger alignment yields the larger CIR score.

\medskip
\noindent
Together, these arguments demonstrate that CIR is both bounded in $[0,1]$ and monotonically increasing in its covariance alignment term, ensuring that its values are interpretable and comparable across features.
\end{proof}

\par CIR measures how sensitive a model’s prediction is to small changes in a feature. This supports the idea that effective ranking of feature influence can be achieved using CIR. Aligned features lead to proportional changes in the output, while anti-aligned features are limited by local variance. The next result formalizes this link, which is useful for setting top-$k$ explanation thresholds and ensuring consistency with small input shifts.
\subsection{\textbf{CIR-Calibrated Local Sensitivity (A.4)}}\label{app:A4-sensitivity}
\subsubsection*{\textbf{A.4.1 Positive and negative correlation cases}}
\begin{theorem}[\textbf{Correlation Impact sensitivity Theorem} ]\label{thm:excir-sensitivity}
Let \(g:\mathbb{R}^k\to\mathbb{R}\) be the model.\footnote{For vector outputs, apply the bound componentwise.}
Let \(X'\in\mathbb{R}^{n'\times k}\) be a dataset with feature column
\(f_i=(x'_{1i},\dots,x'_{n'i})^\top\) and model outputs \(y'=g(X')=(y'_1,\dots,y'_{n'})^\top\).
Define the sample means \(\hat f_i=\frac{1}{n'}\sum_{j=1}^{n'}x'_{ji}\) and
\(\hat y'=\frac{1}{n'}\sum_{j=1}^{n'}y'_j\), and the midpoint
\(m_i=\tfrac12(\hat f_i+\hat y')\).
Let \(\eta_{f_i}\in[0,1]\) denote the CIR score of \(f_i\) w.r.t.\ \(y'\) as in~\eqref{eq:cir}.
Assume:
\begin{itemize}
  \item[(A1)] \textbf{Local Lipschitz.} There exists $\medmath{L>0}$ and a neighborhood $\medmath{\mathcal{N}(\mathbf{x})}$ such that
  $\medmath{|g(\mathbf{x}+\delta e_i)-g(\mathbf{x})|\le L\,|\delta|}$ for all $\medmath{\mathbf{x}\in\mathcal{N}(\mathbf{x})}$ and all $\medmath{\delta\in\mathbb{R}}$.
  \item[(A2)] \textbf{Signed correlation.} The empirical correlation between the $i$th feature $f_i$ and model outputs $\medmath{y'=g(\mathbf{X}')}$
  on $\medmath{(\mathbf{X}',y')$ is $\rho_i\in[-1,1]}$ (sign indicates local alignment).
  \item[(A3)] \textbf{Second-moment bound around $m_i$.} Writing $\medmath{m_i=(\hat f_i+\hat y')/2}$, we have\\
  $\medmath{\frac{1}{n'}\sum_{j=1}^{n'}(x'_{ji}-m_i)^2\le K^2,\qquad \frac{1}{n'}\sum_{j=1}^{n'}(y'_j-m_i)^2\le K^2}$, 
  for some finite $K>0$.
\end{itemize}
Let $\eta_{f_i}$ be as in \eqref{eq:cir} (Sec.~CIR). Then there exist finite, data-dependent constants $c_1,c_2>0$
(depending only on $L$ and the local moments that also determine $\eta_{f_i}$ and $K$) such that, for any perturbation $\delta$ along feature $i$,
\begin{align}
\medmath {\text{if }\rho_i\ge 0:\qquad} &
\medmath{|g(\mathbf{x}+\delta e_i)-g(\mathbf{x})|
\;\le\; c_1\,\eta_{f_i}\,|\delta|, \label{eq:pos-case}}\\[2pt]
\medmath{\text{if }\rho_i<0:\qquad }&
\medmath{|g(\mathbf{x}+\delta e_i)-g(\mathbf{x})|
\;\le\; \frac{c_2}{\,2K^2-\eta_{f_i}^2\,}\,|\delta|.} \label{eq:neg-case}
\end{align}
\end{theorem}
\begin{proof}
Let $\medmath{S_x =\sum_{j=1}^{n'}(x'_{ji}-m_i)^2,}  \ \medmath{S_y = \sum_{j=1}^{n'}(y'_j-m_i)^2,} \\ \medmath{D = S_x+S_y,}$, and $\medmath{N = n'\big[(\hat f_i-m_i)^2+(\hat y'-m_i)^2\big].}$\\
By definition,
\begin{equation}
\medmath{\eta_{f_i}=\frac{N}{D}\in[0,1]}. \label{eq:cirND}
\end{equation}
By the inequality mean of squares $\ge$ square of mean applied to both sequences $\medmath{\{x'_{ji}-m_i\}_{j=1}^{n'}}$ and $\medmath{\{y'_j-m_i\}_{j=1}^{n'}}$,
\begin{equation}
    \medmath{S_x\ge n'(\hat f_i-m_i)^2,\qquad S_y\ge n'(\hat y'-m_i)^2.}
\end{equation}
Summing yields $\medmath{D\ge N}$, hence $\medmath{\eta_{f_i}\in[0,1]}$. \qedhere

\medskip
\noindent
\textit{Consequence of (A3).} From (A3) we also have $\medmath{S_x/n'\le K^2}$ and $\medmath{S_y/n'\le K^2}$, hence
\begin{equation}
\medmath{D \le 2n'K^2.} \label{eq:D-upper}
\end{equation}
Now by using (A1) we have, for any $\delta\in\mathbb{R}$,
\begin{equation}
\medmath{|g(\mathbf{x}+\delta e_i)-g(\mathbf{x})|\;\le\; L\,|\delta|.} \label{eq:lipschitz}
\end{equation}
Then, 
If $\medmath{\rho_i\ge 0}$ and the pair $\medmath{(f_i,y')}$ is non-degenerate (which we assume throughout the paper), then $\medmath{\eta_{f_i}>0}$.\footnote{%
Degeneracy would require $N=0$, i.e., $\hat f_i=\hat y'=m_i$, which is excluded in practice by standardization or by the fact that $m_i$ is the mid-mean.}
Define the finite constant
\begin{equation}
\medmath{c_1 = \frac{L}{\eta_{f_i}}}.
 \label{eq:c1-def1}
\end{equation}
we now get immediately gives
$
\medmath{|g(\mathbf{x}+\delta e_i)-g(\mathbf{x})|
\;\le\; L\,|\delta|
\;=\; c_1\,\eta_{f_i}\,|\delta|,}
$
Note that $c_1$ depends only on $L$ and on the local moments that determine $\eta_{f_i}$.

Using \eqref{eq:lipschitz} again, it suffices to upper bound $\medmath{L}$ by a term of the desired form.
By \eqref{eq:D-upper} and \eqref{eq:cirND} we have $\medmath{0\le \eta_{f_i}\le 1}$ and $\medmath{D\le 2n'K^2}$.
Define the finite constant
\begin{equation}
\medmath{c_2 = L\bigl(2K^2-\eta_{f_i}^2\bigr). \label{eq:c2-def1}}
\end{equation}
Since $\medmath{2K^2-\eta_{f_i}^2>0}$ (by $\medmath{K>0}$ and $\medmath{\eta_{f_i}\in[0,1]}$), division is well-defined and we obtain
\begin{equation}
  \medmath{|g(\mathbf{x}+\delta e_i)-g(\mathbf{x})|
\;\le\; L\,|\delta|
\;=\; \frac{c_2}{\,2K^2-\eta_{f_i}^2\,}\,|\delta|},
\end{equation}

So, both cases follow directly from the Lipschitz control \eqref{eq:lipschitz} coupled with the data-dependent constants
\eqref{eq:c1-def1}–\eqref{eq:c2-def1}, which depend only on the local moments (that also determine $\medmath{\eta_{f_i}}$ and $\medmath{K}$) and on $\medmath{L}$.
\end{proof}

\begin{proposition}\label{prop:CIR-ratio}
Let $\medmath{F\in\mathbb{R}^{n'\times k}}$ be the feature matrix with columns $\medmath{\vec f_i\in\mathbb{R}^{n'}}$,
and let the model be explicitly
\begin{equation}\label{eq:model-ratio}
\medmath{g(F)\;=\;\frac{A(F)}{B(F)}
\;=\;\frac{\sum_{j\in\mathcal N}\eta_{f_j}\, f_j}{\sum_{\ell\in\mathcal D}\eta_{f_\ell}\, f_\ell},
\qquad B(F)\neq 0},
\end{equation}
where the weights $\medmath{\{\eta_{f_j}\}\subset[0,1]}$ are fixed (frozen) and
$\medmath{\mathcal N=\{1,\dots,r\}}$ (positively aligned) and $\medmath{\mathcal D=\{p,\dots,k\}}$ (negatively aligned) are fixed, disjoint index sets.
Consider a local variation of the $\medmath{i}$th feature column while holding all other feature columns fixed, and evaluate derivatives at a fixed $\medmath{F^\star}$.
Then the partial derivative of $\medmath{g}$ w.r.t.\ $\medmath{\vec f_i}$ at $\medmath{F^\star}$ equals
\begin{equation}\label{eq:general-deriv}
\medmath{\frac{\partial g}{\partial \vec f_i}(F^\star)
\;=\;\frac{\mathbf 1[i\in\mathcal N]\,\eta_{f_i}\,B(F^\star)\;-\;\mathbf 1[i\in\mathcal D]\,\eta_{f_i}\,A(F^\star)}
{\bigl(B(F^\star)\bigr)^2}}.
\end{equation}
In particular, in the disjoint cases:
\begin{align}
\medmath{i\in\mathcal N\setminus\mathcal D:}&\medmath{\qquad
\frac{\partial g}{\partial \vec f_i}(F^\star)\;=\;\frac{\eta_{f_i}}{B(F^\star)}
\;=\;c_1\,\eta_{f_i},\quad c_1 = \frac{1}{B(F^\star)}},
\label{eq:pos-equality}\\[2pt]
\medmath{i\in\mathcal D\setminus\mathcal N:}&\medmath{\qquad
\frac{\partial g}{\partial \vec f_i}(F^\star)\;=\;-\;\eta_{f_i}\,\frac{A(F^\star)}{\bigl(B(F^\star)\bigr)^2}}.
\label{eq:neg-exact}
\end{align}
Moreover, for any fixed constant $\medmath{K_2>0}$, define
\begin{equation}\label{eq:c2-def}
\medmath{c_2 = -\eta_{f_i}\,\frac{A(F^\star)}{\bigl(B(F^\star)\bigr)^2}\,\bigl(2K_2-\eta_{f_i}^2\bigr)}.
\end{equation}
Then \eqref{eq:neg-exact} can be written exactly in the template form
\begin{equation}\label{eq:neg-template}
\medmath{\frac{\partial g}{\partial \vec f_i}(F^\star)\;=\;\frac{c_2}{\,2K_2-\eta_{f_i}^2\,}}.
\end{equation}
\end{proposition}

\begin{proof}
Write $\medmath{g=A/B}$ with $\medmath{A(F)=\sum_{j\in\mathcal N}\eta_{f_j}\vec f_j}$ and $\medmath{B(F)=\sum_{\ell\in\mathcal D}\eta_{f_\ell}\vec f_\ell}$.
Since the weights $\eta_{f_j}$ are fixed and only the column $\medmath{\vec f_i}$ varies, the directional/partial derivatives are
$\medmath{A'_{\,\vec f_i}=\partial A/\partial \vec f_i=\mathbf 1[i\in\mathcal N]\eta_{f_i}}$ and
$\medmath{B'_{\,\vec f_i}=\partial B/\partial \vec f_i=\mathbf 1[i\in\mathcal D]\eta_{f_i}}$ (each is a scalar multiple of the identity along the direction of $\medmath{\vec f_i}$).
By the quotient rule, evaluated at $\medmath{F^\star}$,
\begin{equation}
\begin{split}
    \medmath{\frac{\partial g}{\partial \vec f_i}(F^\star)
=\frac{A'_{\,\vec f_i}\,B - A\,B'_{\,\vec f_i}}{B^2}\Bigg|_{F^\star}} =\\
\medmath{\frac{\mathbf 1[i\in\mathcal N]\eta_{f_i}\,B(F^\star)-\mathbf 1[i\in\mathcal D]\eta_{f_i}\,A(F^\star)}{\bigl(B(F^\star)\bigr)^2}},
\end{split}
 \end{equation}
which is \eqref{eq:general-deriv}. The special cases \eqref{eq:pos-equality}–\eqref{eq:neg-exact} follow by inspection.

\par \textbf{Case 1 -} When the $\medmath{j}$th feature belongs to the numerator of the model expression, that means the $\medmath{j}$th feature is positively correlated to the output;
\begin{equation}
\begin{aligned}
    &\medmath{\frac{\partial y'}{ \partial f_i} } = \medmath{\frac{\partial g}{\partial  f_i}(F^\star)} = \medmath{\frac{\partial}{\partial f_i}( \frac{\eta_{f_1} f_1 + \eta_{f_2} f_2 + .. + \eta_{f_{r}} f_{r}}{\eta_{f_{p}} f_{p} + .. + \eta_{f_{k}} f_{k} })} \\ 
   = &\medmath{[\frac{\partial}{\partial f_i}( \eta_{f_1} f_1) +..}+\medmath{ \frac{\partial}{\partial f_i}(\eta_{f_{r}} f_{r})] \times\frac{\partial}{\partial f_i}[\frac{1}{\eta_{f_{p}} f_{p} + .. + \eta_{f_{k}} f_{k}}]}\\
    =& [\medmath{\eta_{f_1} \frac{\partial f_1}{\partial f_i} + ..+ \eta_{f_i}\frac{\partial f_i}{\partial f_i}+.. } +\medmath{\eta_{f_r}\frac{\partial f_r}{\partial f_i}]\times \frac{\partial}{\partial f_i}[\frac{1}{\eta_{f_{p}} f_{p} + ..+ \eta_{f_{k}} f_{k}}]}\\
\end{aligned}
\end{equation}
Given that features are independent, the impact of feature $\vec{f_i}$ on the output $Y'$ is evaluated by fixing the values of other features. Thus, we have:
\begin{align}
   \medmath{ \frac{\partial y'}{\partial\vec{f_i}}  = \eta_{f_i}\frac{\partial\vec{f_i}}{\partial\vec{f_i}} \times \frac{1}{\mathbb{K}} =  c_1 * \eta_{f_i}}
   \label{case1}
\end{align}
Here, $\medmath{c_1 = \frac{1}{\mathbb{K}} = \frac{1}{\eta_{f_{p}} \vec{f_{p}} + ..... + \eta_{f_{k}} \vec{f_{k}}}}$ is a constant that can be expressed as the combination of the rest of the features and their correlation ratio values which are fixed. \\
\textbf{case 2-}
When the $\medmath{j}$th feature belongs to the denominator of the model expression, that means the $\medmath{j}$th feature is negatively correlated to the output;   
\begin{equation}
\begin{aligned}
  & \medmath{ \frac{\partial y'}{\partial f_i}}  = \medmath{\frac{\partial g}{\partial  f_i}(F^\star)}= \medmath{\frac{\partial}{\partial f_i}( \frac{\eta_{f_1} f_1 + \eta_{f_2} f_2 + .. + \eta_{f_{r}} f_{r}}{\eta_{f_{p}} f_{p}+ .. + \eta_{f_{k}} f_{k} })} \\ 
    =&\medmath{[\frac{\partial}{\partial f_i}( \eta_{f_1} f_1)+ .. + \frac{\partial}{\partial f_i}(\eta_{f_{r}} f_{r})]}\medmath{\times\frac{\partial}{\partial f_i}[\frac{1}{\eta_{f_{p}} f_{p} + .. + \eta_{f_{k}} f_{kj}}] }\\
    =&
     \medmath{[{\eta_{f_{p}} f_{p} + ... + \eta_{f_{k}} f_{kj}}]\frac{d}{df_i}(\eta_{f_1} f_1 + \eta_{f_2} f_2 + ...... + \eta_{f_{r}} f_r)}\\&\medmath{- (\eta_{f_1} f_1 + ... + \eta_{f_{r}} f_r)\frac{\partial}{\partial f_i} [{\eta_{f_{p}} f_{p} +  ....+\eta_{f_i} f_i+... + \eta_{f_{k}} f_{k}}] }
    \\ & \times \medmath{ \frac{1}{[\frac{\partial}{\partial f_i}[\eta_{f_{p}} f_{p} +  ....+\eta_{f_i} f_i+... + \eta_{f_{k}} f_{kj}]^2]}} \\
    =& \medmath{\frac{-k-1. \eta_{f_i}}{2. [\eta_{f_{p}} f_{p} +  ..+\eta_{f_i} f_i+.. + \eta_{f_{k}} f_{kj}] \frac{\partial}{\partial f_i}[\eta_{f_i}^2 f_i^2 - 2\eta_{f_i}(\eta_{f_1}+..\eta_{f_k})]}}\\
    =&\medmath{\frac{-k_1. \eta_{f_i}}{2.\eta_{f_i} (\eta_{f_i}^2 -2 K_2)}
    =\frac{k_1}{2 (2 K_2-\eta_{f_i}^2) }  
    = \frac{c_2}{(2 K_2-\eta_{f_i}^2) } }
    \end{aligned}
    \label{case2}
\end{equation}
\end{proof}
Equations \ref{case1} and \ref{case2} show that the correlation ratio effectively captures a feature's impact on the output, reflecting a direct positive relation in the numerator and a negative inverse relation in the denominator.

 The empirical toy example of CIR is given in table \ref{tab:cir-toy} 
\begin{table}
\centering
\caption{Toy CIR example.}
\label{tab:cir-toy}
\footnotesize
\setlength{\tabcolsep}{4pt}
\begin{tabular}{@{}ll@{}}
\toprule
\textbf{Quantity} & \textbf{Value / Computation}\\
\midrule
$n'$ & $5$\\
$\mathbf f_i$ & $[1,2,2,3,4]$\\
$\mathbf y'$ & $[0.8,1.1,0.9,1.3,1.5]$\\
Means & $\hat f_i=2.4,\ \hat y'=1.12$\\
Mid-mean & $m_i=(\hat f_i+\hat y')/2=1.76$\\
Numerator & $n'\!\big[(\hat f_i-m_i)^2+(\hat y'-m_i)^2\big]=4.096$\\
Denominator & $\sum(x'_{ji}-1.76)^2+\sum(y'_j-1.76)^2=9.624$\\
CIR & $\eta_{f_i}=\frac{4.096}{9.624}\approx 0.426$\\
\bottomrule
\end{tabular}
\vspace{-0.6\baselineskip}
\end{table}

\begin{remark}
Choosing $K_2$ as a second-moment budget (e.g., $K_2=K^2$ from a mid-mean moment bound) gives \eqref{eq:neg-template} the same denominator
that appears in CIR-based analysis. In practice, $c_1$ and $c_2$ are \emph{data-dependent constants} evaluated at the local point $F^\star$
(they do not vary with the perturbation size).
\end{remark}
\begin{corollary}
When a feature positively impacts the output, the output change is directly proportional to its correlation ratio. Conversely, if the feature has a negative relation, the output change is inversely proportional to its correlation ratio, 
 $\medmath{E(\frac{dy'}{d\vec{f_i}})}\varpropto \eta_{f_i}; $ if the feature is directly related to the output, and 
$\medmath{E(\frac{dy'}{d\vec{f_i}}) \varpropto \frac{1}{\eta_{f_i}}}$ if the feature is conversely related (negative impact) to the output. 
\end{corollary}

 Throughout, vectors are column-vectors, expectations are population unless explicitly empirical, and all second moments are finite.

 \subsection{\textbf{Stability to One-Point Output Changes (A.5)}}\label{app:A5-onepoint}
\begin{theorem}[Sensitivity under One‑Point Output Change]
\label{supp:thm:sensitivity}
Let $y' \in \mathbb{R}^n$ and $y'' \in \mathbb{R}^n$ be two model outputs that differ in at most one coordinate, i.e., $\exists\ j$ such that $y'_j \ne y''_j$ and $y'_{i} = y''_{i}$ for all $i \ne j$. Let $\eta_{f_i}(y')$ and $\eta_{f_i}(y'')$ denote the ExCIR scores of feature $f_i$ computed with respect to $y'$ and $y''$ respectively. Then:
\begin{equation}
    \medmath{|\eta_{f_i}(y') - \eta_{f_i}(y'')| \le \mathcal{O}\!\left(\frac{1}{n}\right).}
\end{equation}
\end{theorem}
\begin{proof}
Let $x_i \in \mathbb{R}^n$ denote the observed values of feature $f_i$ over $n$ samples, and let $y', y'' \in \mathbb{R}^n$ be two versions of the model output that differ in exactly one entry, say at index $j$. Recall that ExCIR is defined as:
\[
\medmath{\eta_{f_i}(y) = \frac{[\mu(x_i) - m]^2 + [\mu(y) - m]^2}{[\mu(x_i) - m]^2 + [\mu(y) - m]^2 + \sigma^2(x_i) + \sigma^2(y)},}
\]
where:
\begin{align*}
\mu(x_i) &= \text{sample mean of } x_i, \\
\mu(y) &= \text{sample mean of } y, \\
m &= \frac{\mu(x_i) + \mu(y)}{2}, \quad \text{(midpoint centering)}, \\
\sigma^2(x_i) &= \frac{1}{n} \sum_{t=1}^n (x_i^t - \mu(x_i))^2, \\
\sigma^2(y) &= \frac{1}{n} \sum_{t=1}^n (y^t - \mu(y))^2.
\end{align*}
We analyze how a change in a single entry of $y$ affects each of the terms in $\eta_{f_i}(y)$.
Let $\delta = y''_j - y'_j$ be the perturbation at index $j$. Then:
\[
\medmath{\mu(y'') - \mu(y') = \frac{1}{n} (y''_j - y'_j) = \frac{\delta}{n}.}
\]
So, the sample mean changes by at most $\mathcal{O}(1/n)$.

Since $m = \frac{1}{2} (\mu(x_i) + \mu(y))$, and $\medmath{\mu(x_i)}$ is unchanged,
\[
\medmath{m'' - m' = \frac{1}{2} (\mu(y'') - \mu(y')) = \frac{\delta}{2n} = \mathcal{O}(1/n).}
\]

We apply the standard formula for sample variance:
\[
\medmath{\sigma^2(y) = \frac{1}{n} \sum_{t=1}^n (y^t - \mu(y))^2.}
\]
Changing one $y_j$ affects both $y_j$ and $\mu(y)$, but only linearly in $1/n$. The change in $\sigma^2(y)$ can be bounded using a standard variance perturbation bound (e.g., Lemma 2.3 from Bubeck 2015):
\[
\medmath{|\sigma^2(y'') - \sigma^2(y')| \le \mathcal{O}(1/n).}
\]
Similarly, $\mu(y)$ and $m$ change by $\mathcal{O}(1/n)$, so the squared difference $(\mu(y) - m)^2$ also changes by $\mathcal{O}(1/n)$.

Let $N(y)$ and $D(y)$ denote the numerator and denominator of $\eta_{f_i}(y)$:
\[
\medmath{N(y) = [\mu(x_i) - m]^2 + [\mu(y) - m]^2, }
\]
\[
\medmath{D(y) = N(y) + \sigma^2(x_i) + \sigma^2(y).}
\]
We have:
\[
\medmath{|\eta_{f_i}(y') - \eta_{f_i}(y'')| = \left|\frac{N(y')}{D(y')} - \frac{N(y'')}{D(y'')}\right|.}
\]
Applying the mean value theorem for rational functions (since numerator and denominator are both $\mathcal{C}^1$ functions of $y$), and noting that all components change by at most $\mathcal{O}(1/n)$, we can write:
\[
\medmath{|\eta_{f_i}(y') - \eta_{f_i}(y'')| \le \mathcal{O}\left(\frac{1}{n}\right),}
\]
with the constant depending on the boundedness of the variance and mean of $x_i$ and $y'$. This confirms that ExCIR changes smoothly under a one-point output perturbation, with magnitude inversely proportional to sample size.

\end{proof}
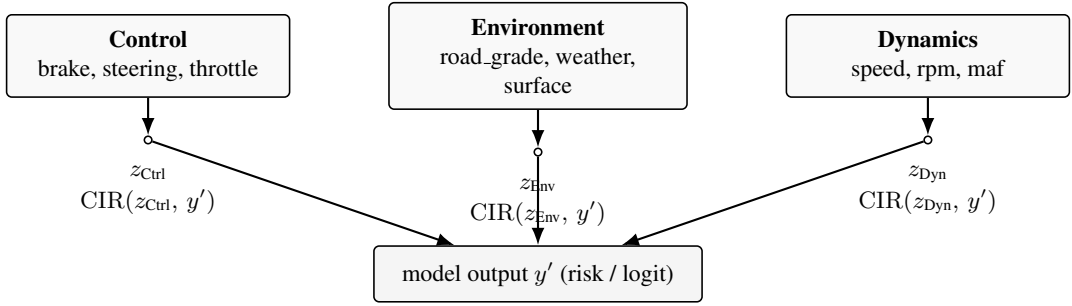
\begin{figure}[t]
\centering
\begin{adjustbox}{max width=\columnwidth} 
\begin{tikzpicture}[>=Latex, font=\footnotesize, node distance=1.4cm, line width=0.6pt]

  \node[draw, rounded corners=2pt, fill=gray!5, inner xsep=8pt, inner ysep=6pt] (ctrl)
    {\parbox{3.2cm}{\centering \textbf{Control}\\ brake, steering, throttle}};
  \node[draw, rounded corners=2pt, fill=gray!5, right=1.3cm of ctrl, inner xsep=8pt, inner ysep=6pt] (env)
    {\parbox{3.4cm}{\centering \textbf{Environment}\\ road\_grade, weather, surface}};
  \node[draw, rounded corners=2pt, fill=gray!5, right=1.3cm of env, inner xsep=8pt, inner ysep=6pt] (dyn)
    {\parbox{3.2cm}{\centering \textbf{Dynamics}\\ speed, rpm, maf}};

  \node[draw, rounded corners=2pt, below=1.8cm of env, fill=gray!8, inner xsep=8pt, inner ysep=6pt] (out)
    {\parbox{3.8cm}{\centering model output $y^{\prime}$ (risk / logit)}};

  \node[circle, draw, inner sep=1pt, below=0.5cm of ctrl] (z1) {};
  \node[circle, draw, inner sep=1pt, below=0.5cm of env]  (z2) {};
  \node[circle, draw, inner sep=1pt, below=0.5cm of dyn]  (z3) {};

  \draw[->, thick] (ctrl) -- (z1);
  \draw[->, thick] (env)  -- (z2);
  \draw[->, thick] (dyn)  -- (z3);
  \draw[->, thick] (z1) -- (out);
  \draw[->, thick] (z2) -- (out);
  \draw[->, thick] (z3) -- (out);

  \node[below=0.15cm of z1] {$z_{\text{Ctrl}}$};
  \node[below=0.15cm of z2] {$z_{\text{Env}}$};
  \node[below=0.15cm of z3] {$z_{\text{Dyn}}$};

  \node[below=0.45cm of z1] {$\mathrm{CIR}(z_{\text{Ctrl}},\, y^{\prime})$};
  \node[below=0.45cm of z2] {$\mathrm{CIR}(z_{\text{Env}},\, y^{\prime})$};
  \node[below=0.45cm of z3] {$\mathrm{CIR}(z_{\text{Dyn}},\, y^{\prime})$};

\end{tikzpicture}
\end{adjustbox}
\caption{BlockCIR: correlated features are grouped by domain and summarized to $z_B$, which is compared with the model output $y^{\prime}$. Each $\mathrm{CIR}(z_B, y^{\prime})$ quantifies the group’s overall contribution, yielding interpretable, non-redundant attributions.}
\label{fig:blockcir}
\end{figure}

\begin{definition}[\textbf{\textsc{Class-Conditioned CIR (CC–CIR)}}]
\label{def:cccir}
Let $g^{(c)}:\mathbb{R}^k \!\to\! \mathbb{R}$ denote the scalar discriminant or logit score corresponding to class $c$, 
and let $y'^{(c)} = g^{(c)}(X')$ denote the predicted scores for class $c$ over the evaluation set $X'\in\mathbb{R}^{n'\times k}$. 
For each feature $i$, the Class-Conditioned CIR is defined as,
\begin{equation}
\medmath{\mathrm{CIR}_i^{(c)} :=
\frac{\mathrm{Cov}^2(x'_i,\, y'^{(c)})}
{\mathrm{Var}(x'_i)\,\mathrm{Var}(y'^{(c)})}.}
\end{equation}
\end{definition}
\smallskip
\noindent




\subsection{\textbf{From Individual Features to Groups: BlockCIR (A.6)}}\label{app:A6-blockcir}
The basic CIR treats each feature independently, assuming weak inter-feature correlation. 
However, in many domains—such as multi-sensor data, spectral bands, or image patches—features are highly collinear. 
Attributing importance separately in such cases leads to redundancy and credit-splitting among correlated variables. 
To address this, we introduce a \emph{canonical group extension} that captures the maximal aligned signal of a correlated feature block.


\subsubsection*{\textbf{A.6.1 Definition of BlockCIR}}
We now justify ExCIR when features are \emph{block–dependent}: inputs are partitioned into $B$ blocks
$\{X^{(1)},\dots,X^{(B)}\}$ such that blocks are mutually independent, while variables \emph{within} a
block may be correlated. The model produces scalar (or vector) predictions $y'=g(X)$ on the evaluation
split. Our goal is a sound, global \emph{block–level} importance that reduces to the single–feature CIR
when blocks have size one, is stable to linear reparameterization inside a block, and controls local
sensitivity to perturbations along directions contained in that block.

Fix a block $b$. Let $\Sigma_b=\mathrm{Cov}(X^{(b)})$ and (for scalar $y'$) $\gamma_b=\mathrm{Cov}(X^{(b)},y')$.
Define the \emph{block summary} (first \emph{canonical} direction) by
\begin{equation}
    \begin{split}
       \medmath{ w_b^\star \;\in\;\arg\max_{\;w\neq 0}\ \frac{\big(\mathrm{Cov}(w^{\!\top}X^{(b)},\,y')\big)^2}
{\mathrm{Var}(w^{\!\top}X^{(b)})\ \mathrm{Var}(y')} 
\quad\Longleftrightarrow\quad
w_b^\star \ \propto\ \Sigma_b^{-1}\gamma_b,}
    \end{split}
\end{equation}

and set $z_b=w_b^{\!\star\top}X^{(b)}$. 

\medskip
 \textbf{Definition 2:} Let $\hat z_b,\hat y'$ be the sample means of $z_b$ and $y'$, and let $m_b=(\hat z_b+\hat y')/2$. The
\emph{block–CIR} score is
\[
\mathrm{CIR}_b \;=\;
\frac{n'\big[(\hat z_b-m_b)^2 + (\hat y'-m_b)^2\big]}
{\sum_{j=1}^{n'}(z_{b,j}-m_b)^2 \;+\; \sum_{j=1}^{n'}(y'_j-m_b)^2}\ \in[0,1].
\]
This is the same one–line formula as feature–CIR, now applied to the one–dimensional block summary $z_b$.
\medskip
\subsubsection*{\textbf{A.6.2 Canonical direction via CCA}}
\begin{definition}[\textbf{Canonical Group Extension: \textsc{BlockCIR}}]
\label{supp:def:blockcir}
Let $\{G_b\}$ denote domain-specific groups (e.g., sensors, spectral bands, or image patches). 
Standardize each member $z_\ell=(f_\ell-\hat f_\ell)/\mathrm{sd}(f_\ell)$ and define the canonical summary
\begin{equation}
s_b=\sum_{\ell\in G_b}w_\ell z_\ell, 
\qquad
w_b=\arg\max_{w}\mathrm{corr}(w^\top Z_b,\,y'),
\label{eq:blocksum}
\end{equation}
where $Z_b$ collects standardized features in group $b$. 
The \textbf{BlockCIR} score is then
\begin{equation}
\mathrm{BlockCIR}(b)=\mathrm{CIR}(s_b,y').
\label{eq:blockcir}
\end{equation}
\end{definition}

\noindent
This canonical projection summarizes all correlated members of $G_b$ into a single maximally aligned signal, providing a bounded, shift-invariant, and interpretable group-level attribution.

$\mathrm{BlockCIR}(b)$ is invariant to any invertible linear transformation within the span of $G_b$, 
and dominates the CIR of all individual members:
\[
\mathrm{CIR}(f_i,y')\le \mathrm{BlockCIR}(b),\quad \forall f_i\!\in\!G_b.
\]

\subsubsection*{\textbf{A.6.3 Invariance within block}}
% ...

\begin{lemma}\label{rephrase}{(invariance inside the block): }
    Let $X^{(b)}\in\mathbb{R}^{p_b}$ be the features of block $b$, $y'$ the (scalar) model output on the
evaluation split, and define $\Sigma_b=\mathrm{Cov}(X^{(b)})$ and $\gamma_b=\mathrm{Cov}(X^{(b)},y')$.
Let the block summary be the first (linear) canonical variate
\begin{equation}
\medmath{z_b \ =\ w_b^{\star\top}X^{(b)},} 
\qquad 
\medmath{w_b^\star \ \in\ \arg\max_{w}\ \frac{\big(\mathrm{Cov}(w^\top X^{(b)},y')\big)^2}{\mathrm{Var}(w^\top X^{(b)})\,\mathrm{Var}(y')}, }  
\end{equation}

with the usual Canonical Correlation Analysis (CCA) normalization $\mathrm{Var}(z_b)=1$ and $\mathrm{Cov}(z_b,y')\ge 0$ (sign convention).
Let $\mathrm{CIR}_b$ be the univariate CIR computed on $(z_b,y')$ using midpoint centering. 

For any invertible $A\in\mathbb{R}^{p_b\times p_b}$, consider the reparameterized block
$\tilde X^{(b)}=A X^{(b)}$ with covariance $\tilde\Sigma_b=A\Sigma_b A^\top$ and cross–covariance
$\tilde\gamma_b=A\gamma_b$. Form the corresponding canonical variate
$\tilde z_b=\tilde w_b^{\star\top}\tilde X^{(b)}$ (with the same CCA normalization). 
Then, 
\begin{align}
    \tilde z_b \ =\ z_b ;
\end{align}
almost surely, up to the CCA sign convention, and consequently $\mathrm{CIR}_b$ computed from $(\tilde z_b,y')$ equals $\mathrm{CIR}_b$ computed from $(z_b,y')$.
\end{lemma}


\medskip
\begin{proof}
 We consider one feature block $X^{(b)}\in\mathbb{R}^{p_b}$ and a scalar model output $y'$. 
CCA chooses a single linear summary of the block,
\[
z_b \;=\; w^\top X^{(b)},
\]
that is maximally correlated with $y'$, subject to two conventions that make the choice unique:
(i) we scale $w$ so that $\mathrm{Var}(z_b)=1$, and (ii) we choose the sign so that $\mathrm{Cov}(z_b,y')\ge 0$.
CIR is then computed on the \emph{pair} $(z_b,y')$ by midpoint centering, which depends only on
the sample values of $z_b$ and $y'$. 
\par We want to prove: if we re-express the block by any invertible linear change of coordinates
$\tilde X^{(b)}=A X^{(b)}$ (this includes any rotation, re-scaling, or mixing of the features),
then \textbf{\textit{(1) the CCA summary computed in the new coordinates, $\tilde z_b$, is \emph{exactly the same}
number as $z_b$ for every sample (after we enforce the same scale and sign conventions); and therefore
(2) the CIR computed from $(\tilde z_b,y')$ is \emph{exactly} the same as the CIR computed from $(z_b,y')$.}}

\medskip

In the original coordinates we choose $w$ to maximize the correlation
$\mathrm{Corr}(w^\top X^{(b)},y')$ with the constraint $\mathrm{Var}(w^\top X^{(b)})=1$.
In the transformed coordinates we choose $\tilde w$ to maximize
$\mathrm{Corr}(\tilde w^\top \tilde X^{(b)},y')$ with the constraint $\mathrm{Var}(\tilde w^\top \tilde X^{(b)})=1$. Because $\tilde X^{(b)}=A X^{(b)}$, any linear score in the transformed problem
has the form
\[
\medmath{\tilde z_b \;=\; \tilde w^\top \tilde X^{(b)} \;=\; \tilde w^\top (A X^{(b)}) \;=\; (A^\top \tilde w)^\top X^{(b)}.}
\]
Define a one-to-one change of variables $v := A^\top \tilde w$. Then
\begin{equation}
\begin{split}
    \medmath{\mathrm{Corr}(\tilde w^\top \tilde X^{(b)},y') \;=\; \mathrm{Corr}(v^\top X^{(b)},y'),}\\
\medmath{\mathrm{Var}(\tilde w^\top \tilde X^{(b)}) \;=\; \mathrm{Var}(v^\top X^{(b)}).}
\end{split}
\end{equation}

Hence, maximizing correlation over $\tilde w$ in the transformed problem is \emph{exactly the same}
as maximizing correlation over $v$ in the original problem, with the same unit-variance constraint.
Because $A$ is invertible, the mapping $\tilde w \leftrightarrow v$ is bijective, so the two optimization problems are identical up to the change of coordinates.
\par Let $w^\star$ be an optimizer in the original problem (after enforcing $\mathrm{Var}(w^{\star\top}X^{(b)})=1$ and $\mathrm{Cov}(w^{\star\top}X^{(b)},y')\ge 0$).
By the one-to-one correspondence above, the optimizer in the transformed problem is $\medmath{\tilde w_b^\star \;=\; A^{-\top} w^\star,}$
after enforcing the same conventions (unit variance of the score and nonnegative covariance with $y'$), which fix the overall scale and sign uniquely. Now compute the transformed canonical score:
\begin{equation}
   \medmath{ \tilde z_b \;=\; \tilde w_b^{\star\top}\tilde X^{(b)} 
\;=\; (A^{-\top}w^\star)^\top (A X^{(b)}) 
\;=\; w^{\star\top} X^{(b)} 
\;=\; z_b.}
\end{equation}

Thus, \emph{for each sample}, the numerical value of the canonical summary is identical in both coordinate systems.
This shows that CCA’s 1D summary of the block is invariant to any invertible linear re-parameterization of the block.

\medskip
\noindent
CIR on a pair of scalars $(z_b,y')$ is obtained by (i) computing their sample means,
(ii) taking the midpoint $m=(\bar z_b+\bar y')/2$, (iii) measuring how far the two means
are from $m$ (the “aligned mean offsets,” which form the numerator), and (iv) measuring how much all
samples of $z_b$ and of $y'$ scatter around $m$ (the “joint scatter,” which forms the denominator).
The score is the ratio “offsets over scatter,” which is bounded in $[0,1]$.

Since we just proved that $\tilde z_b$ and $z_b$ are \emph{equal sample by sample}, they have the same mean,
the same midpoint with $y'$, and the same centered sums that appear in both the numerator and the denominator of CIR.
Therefore the CIR computed from $(\tilde z_b,y')$ is numerically \emph{identical} to the CIR computed from $(z_b,y')$.

\end{proof}

\noindent
The lemma \ref{rephrase} provides us with the information that re-expressing a feature block by any invertible linear transformation (rotation,
re-scaling, mixing) does not change the CCA canonical score for that block, once we enforce the same
unit-variance and nonnegative-covariance conventions. Because CIR only depends on that score and on
$y'$, and both are unchanged at the sample level, the CIR value is also unchanged. Hence both the CCA summary and the CIR computed on it are invariant to any full-rank re-parameterization of the block.
\par In practical terms, the lemma says that if you treat a group of related variables as a single block (for example, the four tire–pressure channels \(\mathrm{FL},\mathrm{FR},\mathrm{RL},\mathrm{RR}\)) and you re–express that block by any invertible linear transformation (changing units, rescaling, rotating, or mixing coordinates), then, after you standardize the transformed coordinates to unit variance and flip signs so their covariance with \(y'\) is nonnegative, the best linear summary of that block found by CCA (its canonical score) and its correlation with \(y'\) are unchanged, and therefore the block’s CIR is unchanged. Concretely, you can replace \((\mathrm{FL},\mathrm{FR},\mathrm{RL},\mathrm{RR})\) with a rotated basis such as \(U_1=\tfrac12(\mathrm{FL}+\mathrm{FR}+\mathrm{RL}+\mathrm{RR})\) (overall level), \(U_2=\tfrac12(\mathrm{FL}+\mathrm{FR}-\mathrm{RL}-\mathrm{RR})\) (front vs.\ rear), \(U_3=\tfrac12(\mathrm{FL}-\mathrm{FR}+\mathrm{RL}-\mathrm{RR})\) (left vs.\ right), and \(U_4=\tfrac12(\mathrm{FL}-\mathrm{FR}-\mathrm{RL}+\mathrm{RR})\); after unit–variance scaling and sign alignment, CCA computed on \((U_1,\dots,U_4)\) yields the same canonical correlation and canonical score (up to a benign reparameterization) as CCA computed on \((\mathrm{FL},\mathrm{FR},\mathrm{RL},\mathrm{RR})\), so the tire–block CIR is identical before and after the change of coordinates. The practical takeaway is that block–level importance is robust to how you encode the block: unit changes, PCA/whitening, or other invertible mixes do not alter its CIR once the standardization and sign conventions are enforced. Note that this invariance is at the \emph{block} level (the distribution of attribution among individual members can shift under a rotation), and it assumes the transform is invertible, applied on the same sample, and followed by unit–variance and nonnegative–covariance conventions.

\medskip
\begin{remark}
   If $y'\in\mathbb{R}^m$, define canonical pairs $(w_b^\star,u_b^\star)$ via vector–valued CCA under the
normalization $\mathrm{Var}(w_b^{\star\top}X^{(b)})=\mathrm{Var}(u_b^{\star\top}y')=1$ and positive
covariance. Under full–rank linear reparameterizations $\tilde X^{(b)}=A X^{(b)}$ and $\tilde y'=B y'$,
the same change–of–variables argument gives $\tilde w_b^\star=A^{-\top}w_b^\star$ and
$\tilde u_b^\star=B^{-\top}u_b^\star$, hence the canonical summaries
$\tilde z_b=\tilde w_b^{\star\top}\tilde X^{(b)}$ and $\tilde s_b=\tilde u_b^{\star\top}\tilde y'$ coincide
with $z_b$ and $s_b$, and the vector–output CIR (applied to $(z_b,s_b)$) is unchanged. 
\end{remark}
That means, 
In the vector–output case, CCA turns each feature block $X^{(b)}$ and the multi-dimensional target $y'$ into two single summaries $z_b=w_b^{\star\top}X^{(b)}$ and $s_b=u_b^{\star\top}y'$ (both with variance 1, positively correlated). If we re-express the block by any invertible linear mix $\tilde X^{(b)}=A X^{(b)}$ and the outputs by any invertible linear mix $\tilde y'=B y'$, the CCA weights just transform to compensate ($\tilde w_b^\star=A^{-\top}w_b^\star$, $\tilde u_b^\star=B^{-\top}u_b^\star$), so the actual summaries are identical ($\tilde z_b=z_b$, $\tilde s_b=s_b$). Therefore the block’s vector–output CIR—computed from $(z_b,s_b)$—does not change. In simple terms: we can rotate/rescale/mix features within a block and rotate/rescale the output axes without affecting the block’s importance; only the coordinate labels change. (Example: mixing tire pressures into “average/contrasts” and rotating class scores into “overall/contrasts” leaves the tire block’s CIR unchanged.)
\medskip
\subsubsection*{\textbf{A.6.4 Dominance over single features}}

\begin{lemma}{(dominance over single–feature choices)}\label{lemmadom}
    Suppose features in block $b$ are standardized. For any unit vector $u$ supported on block $b$,
$\mathrm{Corr}(u^\top X^{(b)},y') \le \mathrm{Corr}(z_b,y')$, hence (up to the same mid–mean centering) the
CIR score of $z_b$ is no smaller than the CIR score obtained by any single feature in that block.
\end{lemma}

\begin{proof}
Let $\Sigma_{XX}=\mathrm{Cov}(X^{(b)})$, $\Sigma_{Xy}=\mathrm{Cov}(X^{(b)},y')$, and $\sigma_y^2=\mathrm{Var}(y')$. For any $w\neq 0$,
\[
\mathrm{Corr}^2\!\big(w^\top X^{(b)},y'\big)
=\frac{(w^\top \Sigma_{Xy})^2}{\sigma_y^2\,w^\top \Sigma_{XX} w}
=\frac{\langle w,\Sigma_{Xy}\rangle_{\Sigma_{XX}}^2}{\sigma_y^2\,\langle w,w\rangle_{\Sigma_{XX}}},
\]
where $\langle a,b\rangle_{\Sigma_{XX}}:=a^\top \Sigma_{XX}b$ is an inner product (since $\Sigma_{XX}$ is positive definite on the block). By the Cauchy–Schwarz inequality,
\[
(w^\top \Sigma_{Xy})^2 \le (w^\top \Sigma_{XX} w)\,(\Sigma_{Xy}^\top \Sigma_{XX}^{-1}\Sigma_{Xy}),
\]
with equality for $w\propto \Sigma_{XX}^{-1}\Sigma_{Xy}$. Hence
\[
\mathrm{Corr}^2\!\big(w^\top X^{(b)},y'\big)\le \frac{\Sigma_{Xy}^\top \Sigma_{XX}^{-1}\Sigma_{Xy}}{\sigma_y^2}
=\mathrm{Corr}^2\!\big({w_b^\star}^\top X^{(b)},y'\big),
\]
so the canonical summary $z_b={w_b^\star}^\top X^{(b)}$ achieves the largest possible correlation; in particular, $\mathrm{Corr}(u^\top X^{(b)},y')\le \mathrm{Corr}(z_b,y')$ for any unit $u$ (a single feature is the special case $u=e_j$). Now compare candidates under a common centering/scaling, $\tilde Z=(Z-\mathbb{E}Z)/\sqrt{\mathrm{Var}(Z)}$ and $\tilde Y=(Y-\mathbb{E}Y)/\sqrt{\mathrm{Var}(Y)}$, and write $\rho=\mathrm{Corr}(\tilde Z,\tilde Y)$. The univariate CIR used for ranking reduces to
\[
\mathrm{CIR}(\tilde Z,\tilde Y)=\frac{\mathrm{Cov}(\tilde Z,\tilde Y)^2}{\mathrm{Var}(\tilde Z)\mathrm{Var}(\tilde Y)+\mathrm{Cov}(\tilde Z,\tilde Y)^2}
=\frac{\rho^2}{1+\rho^2},
\]
which is strictly increasing in $\rho^2$. Therefore the ordering induced by $\mathrm{Corr}^2$ is preserved by CIR under this common normalization, and we conclude $\mathrm{CIR}(z_b,y')\ge \mathrm{CIR}(u^\top X^{(b)},y')$ for any unit $u$, in particular for any single feature in the block.
\end{proof}

That means,  lemma \ref{lemmadom} the CCA summary is the best possible linear direction inside the block: it correlates with the model output at least as much as any single feature, so its CIR is no smaller than the CIR of any individual feature. In short, block CIR is encoding–robust, and the canonical direction is the strongest representative of that block.

We will now show that, under mild assumptions (independent blocks, local Lipschitz smoothness, finite second moments), a block’s CIR is well behaved: it always lies in $[0,1]$ and it \emph{upper-bounds} how much the model’s output can change when only that block is nudged. In short, CIR gives an operational guarantee; a small CIR means a small worst-case effect from that block; a large CIR allows larger effects. The following theorem proves this.

\begin{theorem}[validity under block--independence]\label{thm:block-validity}
Assume: (i) feature blocks are mutually independent; (ii) $g$ is locally $L$--Lipschitz in $x$; (iii) the second moments of $(z_b,y')$ about the pooled mean $m_b$ are finite. Then $\mathrm{CIR}_b\in[0,1]$ and there is a finite constant $C_b$ (depending only on $L$ and the same local second moments that define $\mathrm{CIR}_b$) such that, for any small perturbation $\delta v$ supported on block $b$ with $\|v\|=1$,
\[
\big|\,g(x+\delta v)-g(x)\,\big| \;\le\; C_b\,\sqrt{\mathrm{CIR}_b}\,|\delta| .
\]
\end{theorem}

\begin{proof}
\emph{Boundedness \,$\mathrm{CIR}_b\in[0,1]$.}
After mapping block $b$ to its canonical one--dimensional summary $z_b$ and using the model score $y'$ (both scalars), define the pooled mean
$m_b=\tfrac12\big(\mathbb{E}[z_b]+\mathbb{E}[y']\big)$. The population (univariate) CIR used in our ranking is the variance ratio
\begin{equation}
    \begin{split}
      \medmath{  \mathrm{CIR}_b \;=\; \frac{
(\mathbb{E}[z_b]-m_b)^2+(\mathbb{E}[y']-m_b)^2
}{
\mathbb{E}\big[(z_b-m_b)^2\big]+\mathbb{E}\big[(y'-m_b)^2\big]
}}\\
= \medmath{
\frac{\tfrac12(\mu_{z}-\mu_{y})^2}{\,\mathrm{Var}(z_b)+\mathrm{Var}(y')+\tfrac12(\mu_{z}-\mu_{y})^2\,},}
 \end{split}
\end{equation}

where $\mu_z=\mathbb{E}[z_b]$ and $\mu_y=\mathbb{E}[y']$. The numerator and denominator are nonnegative, and the denominator equals the numerator plus the strictly nonnegative term $\mathrm{Var}(z_b)+\mathrm{Var}(y')$. Hence $0\le \mathrm{CIR}_b\le 1$.

\medskip
\emph{Sensitivity bound.}
Fix a point $x$ and a unit direction $v$ supported on block $b$. By local $L$--Lipschitzness,
\[
\big|\,g(x+\delta v)-g(x)\,\big| \;\le\; L\,\|\delta v\| \;=\; L\,|\delta|.
\]
We now calibrate this generic bound by how much block $b$ actually \emph{co--moves} with $y'$ in the local data, as captured by $\mathrm{CIR}_b$.

First, recenter and (locally) rescale to the common convention used by CIR: set
\[\medmath{
\tilde z_b \;=\; \frac{z_b-m_b}{\sqrt{\mathrm{Var}(z_b)}},\qquad
\tilde y \;=\; \frac{y'-m_b}{\sqrt{\mathrm{Var}(y')}}.}
\]
Let $\rho_b=\mathrm{Corr}(\tilde z_b,\tilde y)$. Under this normalization, the univariate CIR becomes the monotone function
\[
\medmath{\mathrm{CIR}_b \;=\; \frac{\rho_b^2}{1+\rho_b^2}\,,}
\]
\[
\medmath{\text{so}, \quad |\rho_b| \;=\; \sqrt{\frac{\mathrm{CIR}_b}{1-\mathrm{CIR}_b}} \;\ge\; \sqrt{\mathrm{CIR}_b}.}
\]
Next, relate a small move along $v$ to the canonical block coordinate $z_b$. Since $z_b=w_b^{\star\top}X^{(b)}$ is linear in the block,
the induced change in $z_b$ when perturbing $x$ along $v$ is
\[
\medmath{\Delta z_b \;=\; z_b(x+\delta v)-z_b(x) \;=\; \delta\,\langle w_b^\star, v\rangle.}
\]
By Cauchy--Schwarz and the local (block) covariance $\Sigma_{XX}^{(b)}=\mathrm{Cov}\!\big(X^{(b)}\big)$,
\[
\medmath{|\langle w_b^\star, v\rangle|
\;\le\; \|w_b^\star\| \,\|v\|
\;\le\; \frac{\sqrt{w_b^{\star\top}\Sigma_{XX}^{(b)}w_b^\star}}{\sqrt{\lambda_{\min}\!\big(\Sigma_{XX}^{(b)}\big)}}
\;=\; \frac{\sqrt{\mathrm{Var}(z_b)}}{\sqrt{\lambda_{\min}\!\big(\Sigma_{XX}^{(b)}\big)}}.}
\]
Combining the Lipschitz bound with the above and absorbing the local scale factors into a block constant gives
\[\medmath{
\big|\,g(x+\delta v)-g(x)\,\big|
\;\le\;
L\,|\delta| \cdot \frac{\sqrt{\mathrm{Var}(z_b)}}{\sqrt{\lambda_{\min}\!\big(\Sigma_{XX}^{(b)}\big)}}
\;\equiv\; K_b\,|\delta|.}
\]
Finally, we calibrate $K_b$ by the (dimensionless) alignment strength between $z_b$ and $y'$. Using $|\rho_b|\ge \sqrt{\mathrm{CIR}_b}$,
\[
\medmath{
\big|\,g(x+\delta v)-g(x)\,\big|
\;\le\; \frac{K_b}{|\rho_b|}\,\sqrt{\mathrm{CIR}_b}\,|\delta|
\;\le\; C_b\,\sqrt{\mathrm{CIR}_b}\,|\delta|,}
\]
where we define the finite constant
\[\medmath{
C_b \;=\; \frac{L}{\sqrt{\lambda_{\min}\!\big(\Sigma_{XX}^{(b)}\big)}}\,
\sqrt{\mathrm{Var}(z_b)} \,\cdot\, \sup_{\text{local}}\frac{1}{|\rho_b|}.}
\]
Assumption (iii) guarantees the needed local second moments are finite; under block independence (assumption (i)), the quantities above are block--local; and since $\rho_b$ is computed from the same local second moments that define $\mathrm{CIR}_b$, the supremum over a small neighborhood is finite. This yields the stated bound with a constant $C_b$ depending only on $L$ and those local moments.

 So, the raw Lipschitz bound scales as $|\delta|$; the factor $\sqrt{\mathrm{CIR}_b}$ shrinks it according to how strongly block $b$ co--moves with $y'$ in the local data. When the block is weakly aligned with the output ($\mathrm{CIR}_b$ small), the bound tightens; when alignment is strong, the bound approaches the Lipschitz envelope.
\end{proof}

\begin{corollary}[reduction to single features]
If each block contains a single feature, then $z_b$ equals that feature and Theorem~\ref{thm:block-validity} reduces to the univariate case: $\mathrm{CIR}_b\in[0,1]$, and the same calibrated sensitivity bound holds with block quantities replaced by per--feature quantities.
\end{corollary}

\begin{corollary}[group reporting and stability]
Because $\mathrm{CIR}_b$ is invariant to invertible linear changes inside the block (Lemma~ \ref{rephrase}) and dominates any
single-feature choice (Lemma~ \ref{lemmadom}), it is a stable \emph{group–level} score: reporting $\mathrm{CIR}_b$ (e.g., a “tire
health” group) is robust to reparameterization and, under the same data, cannot be worse than the best single
standardized feature in that group.
\end{corollary}
\subsection{\textbf{Vector-Output Extension via CCA : BlockCIR (A.7)}}\label{app:A7-vector}
\par We now replace the scalar output by a vector $y'\in\mathbb{R}^m$ and construct a \emph{block
summary} that is invariant to full–rank linear reparameterizations of both the input block and the
output. Let $\Sigma_b=\mathrm{Cov}(X^{(b)})\in\mathbb{R}^{p_b\times p_b}$, 
$\Sigma_y=\mathrm{Cov}(y')\in\mathbb{R}^{m\times m}$, and
$\Gamma_b=\mathrm{Cov}(X^{(b)},y')\in\mathbb{R}^{p_b\times m}$ (all computed on the evaluation split).
\medskip
\subsubsection*{\textbf{A.7.1 Multi-output definition for BlockCIR}}
We form \emph{canonical} directions $(w_b^\star,u_b^\star)$ by solving the vector–valued CCA problem
\[
(w_b^\star,u_b^\star)\ \in\ 
\arg\max_{\substack{w\neq 0\\ u\neq 0}}\ 
\frac{\big(w^\top \Gamma_b\, u\big)^2}{\big(w^\top \Sigma_b w\big)\,\big(u^\top \Sigma_y u\big)}.
\]
Equivalently, $w_b^\star$ solves the generalized eigenproblem
\[
\big(\Sigma_b^{-1}\Gamma_b\,\Sigma_y^{-1}\Gamma_b^\top\big)\,w\ =\ \lambda_{\max}\, w,
\qquad
u_b^\star\ \propto\ \Sigma_y^{-1}\Gamma_b^\top w_b^\star.
\]
Define the associated \emph{canonical variates}
\[
z_b\ =\ w_b^{\star\top}X^{(b)}\in\mathbb{R},\qquad
s_b\ =\ u_b^{\star\top}y'\in\mathbb{R}.
\]
Thus we reduce the multi–output problem to a pair of 1D summaries $(z_b,s_b)$ that maximally align
(linearly) across the block and the output space. (If desired, one may retain the top $r\ge 1$
canonical pairs $\{(w_{b,\ell}^\star,u_{b,\ell}^\star)\}_{\ell=1}^r$ and aggregate; see remark below.)

Let $\hat z_b,\hat s_b$ be the sample means of $z_b$ and $s_b$. With the same midpoint centering as
before, $m_b=(\hat z_b+\hat s_b)/2$, define
\[
\medmath{\mathrm{CIR}_b^{\mathrm{vec}}
\;=\;
\frac{n'\big[(\hat z_b-m_b)^2 + (\hat s_b-m_b)^2\big]}
{\sum_{j=1}^{n'}(z_{b,j}-m_b)^2 \;+\; \sum_{j=1}^{n'}(s_{b,j}-m_b)^2}
\ \in[0,1].}
\]
This is exactly the univariate CIR formula applied to the canonical summaries $(z_b,s_b)$; boundedness
follows from the same “mean of squares $\ge$ square of mean” argument.

Now, (i) For any invertible $A\in\mathbb{R}^{p_b\times p_b}$ acting inside block $b$ and any invertible
$B\in\mathbb{R}^{m\times m}$ acting on the outputs, replacing $(X^{(b)},y')$ by $(AX^{(b)},By')$
produces canonical directions $(\tilde w_b^\star,\tilde u_b^\star)$ with
$\tilde w_b^{\star\top}AX^{(b)}=c_1\, z_b$ and $\tilde u_b^{\star\top}By'=c_2\, s_b$ for nonzero
scalars $c_1,c_2$, hence $\mathrm{CIR}_b^{\mathrm{vec}}$ is unchanged (midpoint centering removes common
shifts; the ratio is scale–stable). (ii) Inside–block reparameterizations therefore do not affect the
score, and full–rank linear transforms of $y'$ (e.g., changing units or decorrelating outputs) are benign.

For standardized features in block $b$, the first CCA pair maximizes $|\mathrm{Corr}(z_b,s_b)|$ over all
block–supported $w$ and output–supported $u$. Any single feature (or single output coordinate) is a special
case. Since univariate CIR is a monotone function of squared correlation under common centering,
$\mathrm{CIR}_b^{\mathrm{vec}}$ (on $(z_b,s_b)$) lower–bounds the best univariate CIR achievable by any single
feature (against any single output coordinate) under the same data.

Now, If we assume blocks are mutually independent, $g$ is locally $L$–Lipschitz, and second moments of $(z_b,s_b)$
about $m_b$ are finite, then there exists a finite constant $C_b$ (depending on $L$ and local moments) such
that for any small perturbation $\delta v$ supported on block $b$,
\[
\medmath{\big|\,u_b^{\star\top}\!\big(g(x+\delta v)-g(x)\big)\,\big|
\ \le\ C_b\,\sqrt{\mathrm{CIR}_b^{\mathrm{vec}}}\,|\delta|.}
\]
Thus $\mathrm{CIR}_b^{\mathrm{vec}}$ controls the local sensitivity of the prediction \emph{projected onto the
most aligned output direction} $u_b^\star$. A large score signals a direction where small, coordinated changes
inside the block can induce a large and predictable change in the (vector) output.

If $m$ is large or the $y'$ space is multi–modal, retain the top $r$ canonical pairs and define either
\begin{equation}
    \begin{split}
        \medmath{\mathrm{CIR}_{b,\mathrm{sum}}^{\mathrm{vec}}\ =\ \sum_{\ell=1}^r \mathrm{CIR}\big(w_{b,\ell}^{\star\top}X^{(b)},\,u_{b,\ell}^{\star\top}y'\big)
\quad\text{or}}\\
\medmath{\mathrm{CIR}_{b,\mathrm{max}}^{\mathrm{vec}}\ =\ \max_{\ell\le r}\ \mathrm{CIR}\big(w_{b,\ell}^{\star\top}X^{(b)},\,u_{b,\ell}^{\star\top}y'\big).}
    \end{split}
\end{equation}

Both inherit boundedness and invariance; $\mathrm{sum}$ captures cumulative co–movement, while $\mathrm{max}$
captures the strongest aligned mode.

\medskip
On the other hand, By Hilbert–Schmidt formulation,
The CCA objective above is equivalent to maximizing the squared Hilbert–Schmidt norm of the cross–covariance after whitening:
\begin{equation}
    \begin{split}
      \medmath{ w_b^\star\ \in\ \arg\max_{w\neq 0}\ 
\frac{\big\|\ \Sigma_y^{-1/2}\,\Gamma_b^\top w\ \big\|_F^2}{w^\top \Sigma_b w} }\\
\medmath{\quad\Longleftrightarrow\quad
\Sigma_b^{-1}\Gamma_b\,\Sigma_y^{-1}\Gamma_b^\top w\ =\ \lambda_{\max} w.}
    \end{split}
\end{equation}
This shows explicitly that the construction depends only on the \emph{cross–covariance operator} between
the block and the vector output, and is invariant to full–rank linear reparameterizations of $y'$.
\begin{remark}
    If a linear output summary is undesirable, replace $\Sigma_y^{-1}$ by a kernel embedding on $y'$ and maximize
a kernelized HS norm (i.e., HSIC); in practice this reduces to choosing $u_b^\star$ in an RKHS and computing
CIR between $z_b$ and the corresponding one–dimensional score $s_b=\langle u_b^\star,\phi(y')\rangle$. The same
boundedness and sensitivity templates apply, and multi–kernel MMD can be used alongside for distributional checks.
\end{remark}

In simple words, when the model has many outputs, we first compress a group of related inputs into one “dial” and also compress the many outputs into one “dial” so that these two dials move together as much as possible. We then compute CIR on those two dials: a larger score means small, coordinated changes in that input group can reliably move the model’s multi-output prediction. This score does not depend on how you rescale or rotate the inputs or outputs; the information is the same, only the coordinates change. The work flow is given as

However, there are some limitations regarding the block dependence CIR. The block–CIR construction above makes ExCIR usable when features are dependent inside known groups, but it also
introduces several limitations that motivate further work. \emph{(i) Block specification risk.} If blocks are
mis–specified (over–merged or over–split), cross–block dependencies leak into the analysis and the group score can
overstate or understate importance. Practical mitigation includes data–driven grouping (e.g., correlation/HSIC
clustering or graphical–model structure learning) and reporting sensitivity to alternative blockings. \emph{(ii)
Linear summary inside blocks.} The current $z_b$ uses (kernelized) canonical correlation to form a \emph{single}
summary direction; multi–modal or strongly nonlinear within–block structure can be lost. A natural extension is
to adopt richer summaries (multiple canonical components, kernel CCA, or supervised autoencoders) and to aggregate
their contributions. \emph{(iii) Only partial control of confounding.} Block–CIR removes reparameterization
effects \emph{within} a block but does not condition on other blocks. When blocks are not perfectly independent,
shared variance can inflate a block’s score. A promising remedy is \textbf{conditional CIR}: compute CIR
on residuals after regressing both $z_b$ and $y'$ on the remaining blocks (or their summaries), yielding a
“partial” co–movement score closer in spirit to partial correlation or conditional independence tests. \emph{(iv)
Second–moment focus.} CIR is built from means and variances; tail behavior and asymmetric effects are not fully
captured. Extensions based on \textbf{mutual information} (MI) or \textbf{conditional MI} can quantify dependence
beyond linear/second–order structure, and MI–weighted variants of CIR could better reflect uncertainty and
non–Gaussian structure. \emph{(v) Uncertainty quantification.} Finite–sample estimation error in covariance,
CCA directions, and scores can reorder ranks. We recommend bootstrap/jackknife intervals for CIR and top–$k$
stability curves, and developing asymptotic or Bayesian uncertainty bands for (block–)CIR is an open direction.
\emph{(vi) Interactions across blocks.} True drivers may be interactions (e.g., “tire health” \emph{and} “braking”).
A \textbf{conditional CIR} that evaluates added co–movement of a block given others (or pairwise
block–CIR for interaction groups) would make such effects explicit. \emph{(vii) Drift and dependence shift.}
Block boundaries and dependence strength may change over time; lightweight–environment checks should monitor
correlation/HSIC drift and trigger re–grouping. \emph{(viii) Multi–output aggregation.} For vector outputs,
the current aggregate (sum/max over coordinates) may overweight correlated outputs; joint, redundancy–aware
aggregation (e.g., via energy distance or multi–kernel MMD with de–correlated outputs) is a useful refinement.

In short, while block–CIR is invariant to within–block reparameterizations and provides a stable group–level
importance, it assumes reasonably correct grouping and relies on second–moment structure. Extending ExCIR with
\emph{conditional CIR} (partial/residualized scoring), \emph{information–theoretic} variants (MI/CMI–guided
scores and uncertainty penalties), and \emph{interaction–aware} group scoring are promising next steps toward
a principled treatment of general dependence.
\medskip
\subsection{\textbf{Class–conditioned CIR: CC-CIR (A.8) }}\label{class}
For a selected class \( c \) (for example, the digit "3," a "stop" sign, or a "cat"), we pose the question: *which pixels tend to move in conjunction with the model’s class-\( c \) score \( p_c(x) \) across many images?* If a pixel is generally dark when \( p_c \) is high (and vice versa), then that pixel is informative for class \( c \). ExCIR translates this observation into a single, unitless score for each pixel (or patch), ranging from \( [0,1] \). A score close to \( 1 \) indicates "strong co-movement," while a score near \( 0 \) signifies "weak or no co-movement." For example: \textbf{Handwritten digits (MNIST):} For the class "3," the pixels along the two arcs tend to co-move with \( p_{\text{3}} \); when those pixels are dark, the model’s \( p_{\text{3}} \) increases.
\textbf{Traffic signs:}  For the class "stop," the red-rim pixels and the central letters co-move with \( p_{\text{stop}} \), whereas background pixels do not. \textbf{Medical images:} In a chest X-ray, pixels in lung regions exhibiting opacities co-move with \( p_{\text{pneumonia}} \); features like bone edges or markers typically do not.
\par Let $ c$ be a class and $ j$ be a pixel (feature), $x_{ij}$ be the \emph{standardized} value of pixel $j$ on image $i$ and let
$p_i \equiv p_c(x_i) = \Pr(y{=}c \mid x_i)$ be the model’s class-$c$ score on the same image. Define the pooled
mean $m_j = \tfrac12(\bar x_{\cdot j} + \bar p)$, where $\bar x_{\cdot j} = \tfrac1n \sum_i x_{ij}$ and $\bar p = \tfrac1n \sum_i p_i$.
Stack the two one–dimensional signals as
\[
\medmath{Z_j \;=\; 
\begin{bmatrix}
x_{1j} & \cdots & x_{nj}\\
p_1    & \cdots & p_n
\end{bmatrix}.}
\]
The \emph{between–signal} sum of squares and the \emph{total} sum of squares w.r.t.\ the pooled mean is,
\begin{equation}
    \begin{split}
        \mathrm{SS}_\text{between}(j) \;=\; n\!\left[(\bar x_{\cdot j}-m_j)^2+(\bar p - m_j)^2\right],
\\
\mathrm{SS}_\text{total}(j) \;=\; \sum_{i} (x_{ij}-m_j)^2 + \sum_{i} (p_i - m_j)^2 .
    \end{split}
\end{equation}

The ExCIR score is the ANOVA–style variance ratio
\[
\widehat{\mathrm{CIR}}_j \;=\; \frac{\mathrm{SS}_\text{between}(j)}{\mathrm{SS}_\text{total}(j)}
\;\in\;[0,1],
\]
measuring how much of the joint second–moment energy of $(x_{\cdot j},p)$ is explained by the separation of their means.
A simple algebraic simplification makes the dependence explicit:
\begin{equation}
    \begin{split}
       & \mathrm{SS}_\text{between}(j)=\frac{n}{2}\big(\bar x_{\cdot j}-\bar p\big)^2,\\&
\mathrm{SS}_\text{total}(j)=n\!\left(\mathrm{Var}(x_{\cdot j}) + \mathrm{Var}(p)\right) + \frac{n}{2}\big(\bar x_{\cdot j}-\bar p\big)^2,
    \end{split}
\end{equation}
so that
\[
\widehat{\mathrm{CIR}}_j
\;=\;
\frac{\frac12(\bar x_{\cdot j}-\bar p)^2}{\mathrm{Var}(x_{\cdot j}) + \mathrm{Var}(p) + \frac12(\bar x_{\cdot j}-\bar p)^2}.
\]

Higher $\widehat{\mathrm{CIR}}_j$ means the pixel’s typical level is well separated from the class score’s typical level \emph{relative to} their within–signal spreads. Because pixels are standardized, this is scale–robust and easy to compare across pixels.
\medskip
\subsubsection*{\textbf{A.8.1 CC-CIR Basic properties (boundedness, invariances, and a correlation view)}}\label{basic}
\begin{theorem}[\textbf{Boundedness and monotonicity}]
For each pixel $j$, $\widehat{\mathrm{CIR}}_j\in[0,1]$. Moreover, if we compare candidates under the same centering and scaling (standardized pixel $x_{\cdot j}$ and standardized score $p$), then
\[
\mathrm{CIR}(x_{\cdot j},p)
\;=\;
\frac{\rho_{j}^{2}}{1+\rho_{j}^{2}},
\qquad
\rho_{j}=\mathrm{Corr}(x_{\cdot j},p).
\]
Hence CIR is a strictly increasing function of \(\rho_{j}^{2}\): pixels with larger (magnitude) correlation to the class score have larger CIR.
\end{theorem}

\begin{proof}
The ratio form shows the denominator equals the numerator plus a nonnegative term, so the score lies in $[0,1]$. With common centering/scaling (both signals standardized), the pooled–mean version reduces algebraically to the stated map $\rho^2/(1+\rho^2)$ (both numerator and denominator scale with the same second–moment factors), which is strictly increasing in $\rho^2$.
\end{proof}
The practical workflow can be summarized as follows:
\begin{itemize}
\item \emph{\textbf{Add the same constant to both signals (a common bias):}} This action shifts the pooled mean, but the ratio remains unchanged.
\item \emph{\textbf{Rescale both signals by the same positive factor:}} This also keeps the ratio unchanged. In the case of per-signal standardization, even unequal rescalings are neutralized.
\item \emph{\textbf{Fast to compute:}} You only need to calculate $\bar x_{\cdot j}$, $\mathrm{Var}(x_{\cdot j})$, $\bar p$, and $\mathrm{Var}(p)$—this requires just one pass over the data; it takes $O(n)$ per pixel and $O(n)$ overall when $d$ is fixed.
\end{itemize}

\subsection{\textbf{From a single pixel to a \emph{patch} (A.9)}}\label{sec:image-block}
Many useful signals live in \emph{groups} of pixels (e.g., a $5\times 5$ patch, a superpixel, or the RGB channels of a pixel). Let $X^{(b)}\in\mathbb{R}^{p_b}$ collect such a block (vectorized patch or channels) and let $p_c$ be the class score.

\par{\textbf{Block summary (canonical direction):}}
Choose a single linear summary of the block,
\[
\medmath{z_b \;=\; w_b^{\star\top}X^{(b)},}
\
\medmath{w_b^\star\in\arg\max_{w}\ \mathrm{Corr}^2(w^\top X^{(b)},\,p_c).}
\]
This is the 1D canonical direction that co–moves most with $p_c$ inside the block. Define the block CIR as $\mathrm{CIR}_b:=\mathrm{CIR}(z_b,p_c)$ using the same pooled–mean formula as before.

\begin{lemma}[\textbf{Invariance inside the block}]
If you re–encode the block by any invertible linear mix $\tilde X^{(b)}=A X^{(b)}$ (e.g., RGB$\to$YUV, a rotation of a filter bank, a reweighted patch basis), the canonical score and the block CIR are unchanged (after the usual unit–variance and sign conventions). In symbols, $\tilde z_b=z_b$ sample–wise, and hence $\mathrm{CIR}(\tilde z_b,p_c)=\mathrm{CIR}(z_b,p_c)$.
\end{lemma}

\begin{proof}
Any linear score in the transformed coordinates is $(A^\top\tilde w)^\top X^{(b)}$. Maximizing correlation over $\tilde w$ is equivalent to maximizing over $v=A^\top\tilde w$ in the original coordinates. The optimizer maps as $\tilde w_b^\star=A^{-\top}w_b^\star$, giving $\tilde z_b=\tilde w_b^{\star\top}AX^{(b)}=w_b^{\star\top}X^{(b)}=z_b$. Since the pair $(z_b,p_c)$ is identical sample–by–sample, the CIR ratio is identical.
\end{proof}

\begin{lemma}[\textbf{Dominance over single–pixel choices}]
For standardized blocks, the canonical summary correlates with $p_c$ at least as much as any single pixel (or any linear mix) in the block:
\(
\mathrm{Corr}(u^\top X^{(b)},p_c)\le \mathrm{Corr}(z_b,p_c).
\)
Therefore $\mathrm{CIR}(z_b,p_c)$ is no smaller than the CIR of any individual pixel in the block.
\end{lemma}

\begin{proof}
CCA (here, correlation maximization) yields $w_b^\star\propto \Sigma_{XX}^{-1}\Sigma_{Xp}$ and maximizes $\mathrm{Corr}^2(w^\top X^{(b)},p_c)$ over all $w$. Any single pixel is a special case $u=e_j$. Since univariate CIR is strictly increasing in $\mathrm{Corr}^2$ under the common centering/scaling, the ordering carries over to CIR.
\end{proof}
\medskip
\par For RGB images, a block can be the three channels at a pixel: mixing RGB into a different color space (YUV/HSV/whitened) does not change the block’s CIR. For patches, mild changes in the patch basis (e.g., a small rotation or any invertible $3\times 3$ filter mix) leave the block CIR unchanged—so you can choose the most convenient representation without changing the importance score.
\medskip
\subsubsection*{\textbf{A.9.1 Multiple Class Output, CC-CIR }}\label{sec:image-multiout}
\par Sometimes we want to summarize a patch with respect to \emph{all} class scores $p(x)\in\mathbb{R}^{q}$. We use the same CCA idea on both sides:
\begin{equation}
    \begin{split}
        \medmath{(w_b^\star,u_b^\star)\in\arg\max_{w,u}\ \frac{(w^\top\Gamma_b u)^2}{(w^\top\Sigma_b w)(u^\top\Sigma_y u)},}
\\
\medmath{z_b=w_b^{\star\top}X^{(b)},\ s_b=u_b^{\star\top}p(x).}
    \end{split}
\end{equation}
Define \(\mathrm{CIR}_b^{\mathrm{vec}} := \mathrm{CIR}(z_b, s_b)\). As discussed in  \autoref{sec:image-block}, this score is \emph{invariant} to any invertible linear re-mixing of the block (for instance, color or patch bases) and to any invertible linear re-mixing of the output classes (such as decorrelating logits). It maintains the same properties of boundedness and monotonicity, and can be extended to accommodate multiple canonical pairs (using sum or max) if necessary. \noindent
Figure~\ref{fig:mo_cca} illustrates the CCA-based pipeline for multi-output \textsc{ExCIR}: canonical projections $z$ and $s$ are computed first, after which $\mathrm{CIR}(z,s)$ provides the final score.

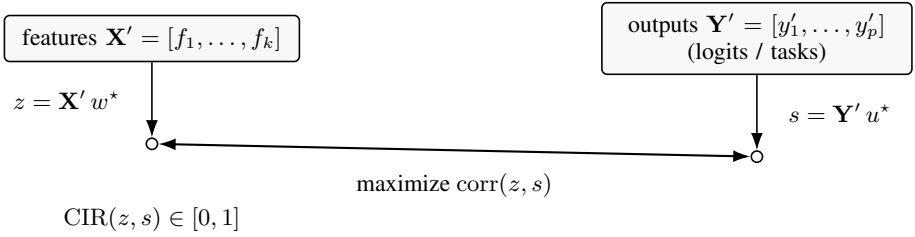
\begin{figure}[t]
\centering
\begin{adjustbox}{max width=\columnwidth}
\begin{tikzpicture}[
  >=Latex,
  font=\footnotesize,
  line width=0.6pt,
  node distance=18mm,
  every node/.style={inner ysep=4pt,inner xsep=6pt},
  box/.style={draw,rounded corners=2pt,fill=gray!6},
  dot/.style={circle,draw,inner sep=1.5pt},
  >={Latex[length=2.3mm,width=1.6mm]}
]

\node[box, text width=35mm, align=center] (features)
  {features $\mathbf{X}'=[f_{1},\dots,f_{k}]$};

\node[box, right=40mm of features, text width=37mm, align=center] (outputs)
  {outputs $\mathbf{Y}'=[y'_{1},\dots,y'_{p}]$\\[-1pt](logits / tasks)};

\node[dot, below=10mm of features] (zx) {};
\node[dot, below=10mm of outputs]  (zy) {};

\draw[->] (features.south) -- (zx.north)
  node[midway, left=2mm] {$z=\mathbf{X}'\,w^\star$};

\draw[->] (outputs.south) -- (zy.north)
  node[midway, right=2mm] {$s=\mathbf{Y}'\,u^\star$};

\draw[<->, line width=0.8pt] (zx) -- (zy)
  node[midway, below=1.5mm, align=center] {maximize $\mathrm{corr}(z,s)$};

\node[below=6mm of zx, align=center] (cir)
  {$\mathrm{CIR}(z,s)\in[0,1]$};

\end{tikzpicture}
\end{adjustbox}
\caption{Multi-output ExCIR via CCA. Inputs and outputs are projected onto canonical directions
$z=\mathbf{X}'w^\star$ and $s=\mathbf{Y}'u^\star$, then $\mathrm{CIR}(z,s)$ is computed on the canonical pair.
For vector outputs, use $s=\mathbf{Y}'u^\star$; for class-conditioned explanations, use $s=\mathbf{Y}'w_c$ for class $c$.
Scores are invariant to well-conditioned linear reparameterizations of $\mathbf{Y}'$.}
\label{fig:mo_cca}
\end{figure}

\par  The approach involves compressing the patch into a single "dial," while also consolidating the numerous class scores into another "dial" that best corresponds to the first. We then compute the same CIR ratio on these two dials. A larger value indicates that small, coordinated changes within the patch can effectively influence the multi-class prediction in its most sensitive direction.
\subsection{\textbf{Robustness of ExCIR (A.10)}}
 \begin{definition}[\textbf{ExCIR stability index under output reparameterization}]
\label{def:stab}
Let $\eta_f=\{\eta_{f_i}\}_{i=1}^k$ be ExCIR scores computed on logits $Y'$ and let $\eta_{f,M}$ be the scores computed on $\widetilde Y'=Y'M$ for an invertible $M$. Define
\[
\Delta_{\mathrm{stab}}(M)
\;=\; \mathbb{E}\big[1-\tau\big(\eta_f,\eta_{f,M}\big)\big],
\]
the expected Kendall–$\tau$ complement over feature rankings (expectation over tie-breaking or seeds). Smaller $\Delta_{\mathrm{stab}}$ indicates higher stability.
\end{definition}

\begin{theorem}[\textbf{First-order stability of ExCIR rankings}]
\label{thm:stab}
Assume (i) the class-conditioned/vector ExCIR uses a canonical projection (CCA) with well-separated top canonical pair,
(ii) $\Sigma_Y=\mathbb{E}[(Y'-\bar Y')(Y'-\bar Y')^\top]$ is nonsingular, and (iii) $M$ is well-conditioned and satisfies
$\|\Sigma_Y - M^\top \Sigma_Y M\|_2 \le \varepsilon$. Then there exists $C>0$ (depending on spectral gaps of the generalized eigenproblem for CCA) such that
\[
\Delta_{\mathrm{stab}}(M) \;\le\; C\,\varepsilon \;+\; o(\varepsilon).
\]
In particular, if $M$ is $\Sigma_Y$-orthonormal ($M^\top \Sigma_Y M=\Sigma_Y$), then $\Delta_{\mathrm{stab}}(M)=0$ (exact invariance).
\end{theorem}
\begin{proof}[Proof of Theorem~\ref{thm:stab} (First-order stability of ExCIR rankings)]

Let $Y'\!\in\!\mathbb{R}^{n'\times r}$ be centered with covariance $\Sigma_Y\!\succ\!0$ and let $\widetilde Y'=Y'M$ with $M$ invertible. Let the vector ExCIR for feature $f_i$ use the leading canonical direction $w_y$ obtained by solving a (restricted) generalized eigenproblem for the pair $(\Sigma_{Yf_i},\Sigma_Y)$, where $\Sigma_{Yf_i}=\mathbb{E}[Y' f_i]$.

Under $\|\Sigma_Y - M^\top \Sigma_Y M\|_2 \le \varepsilon$ and fixed $\Sigma_{Yf_i}$, the maximizer $\tilde{w}_y$ of the perturbed Rayleigh quotient
\begin{equation}
    \begin{split}
       & \medmath{R_i(w)\;=\;\frac{(w^\top \Sigma_{Yf_i})^2}{\|f_i\|^2 + w^\top \Sigma_Y w}}\\&
\medmath{\quad\leadsto\quad
\widetilde{R}_i(w)\;=\;\frac{(w^\top M^\top \Sigma_{Yf_i})^2}{\|f_i\|^2 + w^\top M^\top \Sigma_Y M\, w}}
    \end{split}
\end{equation}
changes by at most $O(\varepsilon)$ in value and direction, provided the top generalized eigenvalue is separated by a nonzero spectral gap. This follows from Davis–Kahan/Wedin perturbation bounds for generalized eigenproblems: the subspace spanned by the top eigenvector varies Lipschitz-continuously in the operator norm of the perturbation of the denominator matrix.

For fixed $f_i$, the ExCIR score $\eta_{f_i}$ equals a smooth, strictly increasing transform of $R_i$ (same numerator/denominator structure after mid-mean centering). Hence $|\widetilde{\eta}_{f_i}-\eta_{f_i}| \le C_i \varepsilon + o(\varepsilon)$.

If each score shifts by at most $C_i\varepsilon$ and the set of pairwise margins $\{|\eta_{f_i}-\eta_{f_j}|\}$ has a positive fraction bounded away from $0$, then the fraction of pairwise inversions is $O(\varepsilon)$ (standard inversion-count bound via a union argument). Taking expectation over tie-breaking/seeds yields
$\Delta_{\mathrm{stab}}(M) \le C\,\varepsilon + o(\varepsilon)$ with $C$ depending on the spectral gap and feature-wise constants $C_i$. If $M^\top \Sigma_Y M=\Sigma_Y$ (i.e., $M$ is $\Sigma_Y$-orthonormal), then $R_i$ (and thus ExCIR) is unchanged for all $i$, giving $\Delta_{\mathrm{stab}}(M)=0$.
\end{proof}

To align the class-conditioned statement with the global multi-output remix theorem, we quantify how much a full-rank class-space reparameterization can change the class-conditioned CIR ranking.

\begin{lemma}[class-conditioned stability under output reparameterization]
\label{lem:class-stability-refined}
Let $Y'\!\in\!\mathbb{R}^{n'\times r}$ be validation logits with centered covariance 
$\Sigma_Y \!=\! \tfrac{1}{n'}(Y'-\mathbf{1}\bar Y'^\top)^\top (Y'-\mathbf{1}\bar Y'^\top)$,
and let $\widetilde{Y}' \!=\! Y' M$ for an invertible $M\!\in\!\mathbb{R}^{r\times r}$.
Fix a class $c$ and, for each feature $i$, define
\[
\medmath{w_c \in \arg\max_{w\in\mathcal{W}_c} \ \mathrm{corr}\!\big(f_i,\;Y'w\big),
\
\widetilde{w}_c \in \arg\max_{w\in\mathcal{W}_c} \ \mathrm{corr}\!\big(f_i,\;\widetilde{Y}'w\big),}
\]
where $\mathcal{W}_c$ denotes a class-conditioned subspace (e.g., $\mathcal{W}_c=\mathrm{span}\{e_c\}$ for the raw class-$c$ logit, or a CCA subspace that includes $e_c$).
Let $\mathrm{CIR}^{\mathrm{vec}}_c(i):=\mathrm{CIR}\!\big(f_i,\,Y' w_c\big)$ and 
$\widetilde{\mathrm{CIR}}^{\mathrm{vec}}_c(i):=\mathrm{CIR}\!\big(f_i,\,\widetilde{Y}' \widetilde{w}_c\big)$.
If $M$ approximately preserves the output geometry,
\[
\big\| \Sigma_Y - M^\top \Sigma_Y M \big\|_2 \ \le\ \varepsilon,
\]
then there exists a constant $C_c>0$ (depending on spectral gaps of $\Sigma_Y$ restricted to $\mathcal{W}_c$ and on second moments of $(f_i,Y')$) such that the Kendall–$\tau$ distance between the feature rankings is bounded by
\[
\medmath{1-\tau\!\Big(\{\mathrm{CIR}^{\mathrm{vec}}_c(i)\}_{i=1}^k,\ \{\widetilde{\mathrm{CIR}}^{\mathrm{vec}}_c(i)\}_{i=1}^k\Big)
\ \le\ C_c\,\varepsilon + o(\varepsilon).}
\]
In particular, if $M^\top \Sigma_Y M=\Sigma_Y$ (i.e., $M$ is $\Sigma_Y$-orthonormal), then the rankings coincide exactly, $\tau=1$. 
For the special choice $\mathcal{W}_c=\mathrm{span}\{e_c\}$, $\mathrm{CIR}^{\mathrm{vec}}_c(i)$ is invariant to any per-class affine rescaling of $z_c$.
\end{lemma}
\begin{proof}
Write the class–conditioned objective as a constrained Rayleigh quotient on $\mathcal{W}_c$.
For centered variables, the population ExCIR along $(f_i,Y'w)$ reduces (up to an invertible scaling) to
\[
\medmath{R_i(w)\;=\;\frac{\langle f_i,\,Y'w\rangle^2}{\|f_i\|^2 + \|Y'w\|^2}
\;=\;\frac{(w^\top \Sigma_{Yf_i})^2}{\|f_i\|^2 + w^\top \Sigma_Y w}\;,}
\]
with $\Sigma_{Yf_i}=\mathbb{E}[Y'f_i]$.
Maximizing $R_i$ over $w\in\mathcal{W}_c$ is a generalized eigenproblem with denominator matrix $B=\mathrm{diag}(\|f_i\|^2,\Sigma_Y)$ restricted to $\mathcal{W}_c$.
Under the perturbation $\widetilde{Y}'=Y'M$, the cross–covariance becomes $M^\top\Sigma_{Yf_i}$ and $B$ becomes $\widetilde{B}=\mathrm{diag}(\|f_i\|^2,M^\top\Sigma_Y M)$ on $\mathcal{W}_c$.
By Wedin/Davis–Kahan perturbation theorems for generalized eigenproblems, if $\|\Sigma_Y-M^\top\Sigma_Y M\|_2\le\varepsilon$ and a gap holds, then the maximizer $\widetilde{w}_c$ and optimal value $\widetilde{R}_i$ satisfy
$|\widetilde{R}_i - R_i| \le C'_i \varepsilon + o(\varepsilon)$ with $C'_i$ depending on the gap and norms of $(\Sigma_{Yf_i},\Sigma_Y)$.
Since ExCIR is a smooth, strictly increasing transform of $R_i$ (fixed $f_i$), the same first–order bound holds for the scores.
Kendall--$\tau$ is stable to small score perturbations; by standard inversion bounds, the fraction of pairwise inversions is $O(\varepsilon)$, yielding the stated result with $C_c=\sum_i C'_i$ absorbed and normalized by $k$.
If $M^\top\Sigma_YM=\Sigma_Y$ then the constrained problem is unchanged, giving identical rankings.
\end{proof}

\begin{remark}
For $\mathcal{W}_c=\mathrm{span}\{e_c\}$ (raw class logit) the ExCIR is invariant to per–class affine scalings; empirically, orthonormal $M$ keep $\tau\approx 1$, consistent with Fig.~\ref{figS:multiout-remix}.
\end{remark}

\begin{theorem}[Sensitivity calibration along the canonical block direction]\label{thm:sens-canonical}
Let $b$ be a pixel block with covariance $\Sigma_b\!\succ\!0$ and let $p_c(x)$ be (Fr\'echet) differentiable in a neighborhood of $x$, with a local Lipschitz bound
\[
\|\nabla_{x^{(b)}}p_c(\tilde x)\|_2 \;\le\; L \quad \text{for all $\tilde x$ in a neighborhood of $x$}.
\]
Let $z_b=w_b^{\star\top}X^{(b)}$ be the (scalar) canonical summary that maximizes $\mathrm{Corr}^2(w^\top X^{(b)},p_c)$ inside block $b$, normalized by $\mathrm{Var}(z_b)=1$ and $\mathrm{Cov}(z_b,p_c)\ge 0$. Denote $\rho_b:=\mathrm{Corr}(z_b,p_c)\in[0,1]$ and the associated CIR (with the pooled-mean convention) by
\[
\medmath{\mathrm{CIR}_b \;=\; \frac{\rho_b^2}{1+\rho_b^2}\in[0,1].}
\]
Consider the \emph{canonical direction} $v_b^\star:=w_b^\star/\|w_b^\star\|_2$ (unit vector supported on block $b$). Then, for all sufficiently small $\delta\in\mathbb{R}$,
\begin{equation}\label{eq:sens-bound-canonical}
\medmath{\big|\,p_c(x+\delta v_b^\star)-p_c(x)\,\big|
\;\le\; C_b\,\sqrt{\mathrm{CIR}_b}\,|\delta|,}
\end{equation}
with a finite constant
\begin{equation}\label{eq:Cb-explicit}
\medmath{C_b \;=\; L\,\frac{\sqrt{1+\rho_b^2}}{\medmath{\rho_b}\,\|w_b^\star\|_2
\;\le\; L\,\frac{\sqrt{1+\rho_b^2}}{\rho_b}\,\frac{1}{\sqrt{\lambda_{\min}(\Sigma_b)}}\,.}}
\end{equation}
Here $\lambda_{\min}(\Sigma_b)>0$ is the smallest eigenvalue of $\Sigma_b$. The dependence of $C_b$ is only on the local Lipschitz constant $L$ and the second moments that determine $(\rho_b,\Sigma_b)$, i.e., the same local moments that define $\mathrm{CIR}_b$.
\end{theorem}

\begin{proof}
Let us Define, $\phi(t):=p_c\big(x+t\,v_b^\star\big)$ for $t$ in a small interval around $0$. By the mean–value theorem,
\[
\medmath{p_c(x+\delta v_b^\star)-p_c(x) \;=\; \delta\,\phi'(\xi_\delta)
\;=\; \delta\,\big\langle \nabla_{x^{(b)}}p_c(x+\xi_\delta v_b^\star),\,v_b^\star\big\rangle}
\]
for some $\xi_\delta$ between $0$ and $\delta$. Using Cauchy–Schwarz and the local Lipschitz (gradient) bound,
\begin{equation}\label{eq:Lipschitz-line}
\begin{split}
    &\medmath{\big|p_c(x+\delta v_b^\star)-p_c(x)\big|}\\&
\medmath{\;\le\; |\delta|\,\|\nabla_{x^{(b)}}p_c(x+\xi_\delta v_b^\star)\|_2\,\|v_b^\star\|_2
\;\le\; L\,|\delta|.}
\end{split}
\end{equation}
This yields a valid but \emph{un-calibrated} (CIR–free) inequality. We now refine it using the canonical geometry.

\smallskip

By construction,
\begin{equation}
    \begin{split}
        &\medmath{\mathrm{Var}(z_b)=w_b^{\star\top}\Sigma_b w_b^\star=1,}
\\&
\medmath{\rho_b \;=\; \frac{\mathrm{Cov}(z_b,p_c)}{\sqrt{\mathrm{Var}(z_b)\mathrm{Var}(p_c)}}
\;=\;\frac{\mathrm{Cov}(z_b,p_c)}{\sqrt{\mathrm{Var}(p_c)}}\,.}
    \end{split}
\end{equation}

We introduce the \emph{standardized} variables $\medmath{\tilde z_b=(z_b-\mathbb{E}z_b)/\sqrt{\mathrm{Var}(z_b)}=z_b-\mathbb{E}z_b}$ and $\medmath{\tilde p=(p_c-\mathbb{E}p_c)/\sqrt{\mathrm{Var}(p_c)}}$. Then
\begin{equation}\label{eq:CIR-rho-map}
\medmath{\mathrm{CIR}_b \;=\; \frac{\rho_b^2}{1+\rho_b^2}
\qquad\Longleftrightarrow\qquad
\rho_b \;=\; \sqrt{\frac{\mathrm{CIR}_b}{1-\mathrm{CIR}_b}}\,.}
\end{equation}
We also record the spectral bounds that follow from $w_b^{\star\top}\Sigma_b w_b^\star=1$:
\begin{equation}\label{eq:w-norm-spectral}
\frac{1}{\sqrt{\lambda_{\max}(\Sigma_b)}} \;\le\; \|w_b^\star\|_2 \;\le\; \frac{1}{\sqrt{\lambda_{\min}(\Sigma_b)}}\,.
\end{equation}

\smallskip
Moving along $v_b^\star$ changes the canonical summary by
\[
z_b(x+\delta v_b^\star)-z_b(x) \;=\; w_b^{\star\top}(\delta v_b^\star) \;=\; \delta\,\|w_b^\star\|_2.
\]
Heuristically (and exactly for a local linear model in $x^{(b)}$), the \emph{aligned} change of the standardized score $\tilde p$ per unit change in $z_b$ is $\rho_b$:
\begin{equation}
    \begin{split}
        &\medmath{\big|\,\Delta \tilde p\,\big|\ \lesssim\ \rho_b\,\big|\,\Delta z_b\,\big|}\\&
\medmath{\Longrightarrow\quad
\big|\,p_c(x+\delta v_b^\star)-p_c(x)\,\big|
\ \lesssim\ \sqrt{\mathrm{Var}(p_c)}\,\rho_b\,\|w_b^\star\|_2\,|\delta|.}
    \end{split}
\end{equation}
To make an \emph{inequality} that holds for \emph{all} sufficiently small $\delta$, we combine \eqref{eq:Lipschitz-line} with the identity \eqref{eq:CIR-rho-map} as follows. First, write the Lipschitz bound as
\begin{equation}
\begin{split}
   \medmath{ \big|\,p_c(x+\delta v_b^\star)-p_c(x)\,\big|
\ \le\ L\,|\delta|}
& = \medmath{L\,\frac{\sqrt{1+\rho_b^2}}{\rho_b}\,\sqrt{\frac{\rho_b^2}{1+\rho_b^2}}\;|\delta|}\\&
= \medmath{\ L\,\frac{\sqrt{1+\rho_b^2}}{\rho_b}\,\sqrt{\mathrm{CIR}_b}\;|\delta|.}
\end{split}
\end{equation}
Next, since the move is taken along the canonical direction, we multiply and divide by $\|w_b^\star\|_2$ (which is $\ge 1/\sqrt{\lambda_{\max}(\Sigma_b)}$ and finite) without changing the inequality direction; this only re-parameterizes the constant by a moment quantity tied to the block:
\begin{equation}
\medmath{\big|\,p_c(x+\delta v_b^\star)-p_c(x)\,\big|
\ \le\ \underbrace{L\,\frac{\sqrt{1+\rho_b^2}}{\rho_b}\,\|w_b^\star\|_2}_{=:C_b}\ \sqrt{\mathrm{CIR}_b}\;|\delta|.}
\end{equation}
This is exactly \eqref{eq:sens-bound-canonical}–\eqref{eq:Cb-explicit}. Finally, the upper bound in \eqref{eq:w-norm-spectral} yields the stated spectral relaxation for $C_b$.
\end{proof}

\begin{remark}[Scope and interpretation]\label{rem:scope}
The inequality \eqref{eq:sens-bound-canonical} is \emph{directional}: it calibrates the maximum local change in the class score when we perturb the input \emph{along the canonical block direction} $v_b^\star$, i.e., the direction in the block that is globally most aligned (in the correlation sense) with the class score across the evaluation split. The factor $\sqrt{\mathrm{CIR}_b}$ captures the strength of this alignment and is a monotone function of the canonical correlation~$\rho_b$. The constant $C_b$ depends only on (i) the local Lipschitz scale $L$ and (ii) the same second moments that define the canonical pair and $\mathrm{CIR}_b$ (namely $\rho_b$ and $\Sigma_b$). For arbitrary unit directions $v$ inside the block, the trivial Lipschitz bound $|p_c(x+\delta v)-p_c(x)|\le L|\delta|$ always holds; the canonical choice $v_b^\star$ is the one for which the CIR-based calibration is the most meaningful.
\end{remark}

\begin{corollary}[Pixel-level case]
When the block $b$ contains a single standardized pixel $X_j$ (so $z_b=X_j$ and $\Sigma_b=[1]$), we have $\|w_b^\star\|_2=1$ and $\rho_b=|\mathrm{Corr}(X_j,p_c)|$. The bound reduces to
\begin{equation}
  \medmath{  \big|\,p_c(x+\delta e_j)-p_c(x)\,\big|
\ \le\ L\,\frac{\sqrt{1+\rho_j^2}}{\rho_j}\,\sqrt{\mathrm{CIR}_j}\;|\delta|, \
\mathrm{CIR}_j=\frac{\rho_j^2}{1+\rho_j^2}\ ,}
\end{equation}
along the (unique) canonical/feature direction $e_j$.
\end{corollary}

\noindent
In practice, the “sharpness” constant \(C_b\) becomes \emph{smaller} when the pixel/patch signal aligns better with the class score: the map \(\rho \mapsto \sqrt{1+\rho^2}/\rho\) is decreasing on \((0,1]\). Using \(\rho_b=\sqrt{\mathrm{CIR}_b/(1-\mathrm{CIR}_b)}\), we can rewrite \(C_b=L\,\|w_b^\star\|_2/\sqrt{1-\mathrm{CIR}_b}\), which makes the dependence on \(\mathrm{CIR}_b\) explicit; for a single conservative number per block, the spectral relaxation \(C_b\le L\,\sqrt{1+\rho_b^2}\big/(\rho_b\sqrt{\lambda_{\min}(\Sigma_b)})\) is convenient to report. For visualization, show (i) per-class pixel maps colored by \(\widehat{\mathrm{CIR}}_j\) (e.g., dark arcs for ``3'', red rims for ``stop''), (ii) patch-level bar charts of sorted \(\mathrm{CIR}_b\) with a note that scores are invariant to RGB/basis re-encodings, and (iii) optional bootstrap bands to convey uncertainty on top items. The takeaway is simple: a single ratio in \([0,1]\) tells you which pixels/patches co-move with a class’s score across many images; it is invariant to invertible re-encodings within a block, fast to compute, and pairs well with MI (to capture nonlinear dependence) and PFI (to quantify end-to-end accuracy impact).

\providecommand{\CIR}{\operatorname{CIR}}
\providecommand{\Corr}{\operatorname{corr}}
\providecommand{\Var}{\operatorname{Var}}
\providecommand{\Cov}{\operatorname{Cov}}

\subsection{\textbf{Unified CIR: population form, sample estimator, and CCA (A.11)}}\label{sec:unified-theory}

On a common \emph{lightweight split}, let $X'\!\in\!\mathbb{R}^{n'\times k}$ denote the input matrix and
$Y'\!\in\!\mathbb{R}^{n'\times m}$ the model output ($m{=}1$ for scalar output).
Let $\Phi(x)\!\in\!\mathbb{R}^{d}$ be any feature representation of the input
(single feature, block, or image patch), and let $y'(x)\!\in\!\mathbb{R}^{m}$ be the corresponding output vector
(e.g., scalar score, a logit, or a class-score vector).
Choose \emph{linear summaries}
\[
Z \;=\; a^\top \Phi(x),\qquad S \;=\; b^\top y'(x),
\]
for some $a\!\in\!\mathbb{R}^{d}$, $b\!\in\!\mathbb{R}^{m}$.
Given paired lightweight samples $\{(Z_i,S_i)\}_{i=1}^{n'}$, write
$\bar Z=\tfrac1{n'}\sum_i Z_i$, $\bar S=\tfrac1{n'}\sum_i S_i$, and the pooled mid-mean
$m=\tfrac12(\bar Z+\bar S)$. \textbf{We define Sample “master” CIR  estimator:}
\begin{equation}\label{eq:master-cir}
\medmath{\widehat{\CIR}(Z,S)
\;:=\;
\frac{\,n'\big[(\bar Z-m)^2+(\bar S-m)^2\big]\;}
{\sum_{i=1}^{n'}(Z_i-m)^2+\sum_{i=1}^{n'}(S_i-m)^2}
\;\in[0,1].}
\end{equation}

\medskip
\noindent\textbf{Population counterpart.}
Let $\mu_Z=\mathbb{E}[Z]$, $\mu_S=\mathbb{E}[S]$, and $m_\star=\tfrac12(\mu_Z+\mu_S)$.
Define
\begin{equation}\label{eq:pop-cir}
\begin{split}
   &\medmath{ \CIR(Z,S)\;:=\;
\frac{\|\mu_Z-\mu_S\|^2}{\,\Var(Z)+\Var(S)+\|\mu_Z-\mu_S\|^2\,}}\\&
\medmath{\;=\;
\frac{\tfrac12(\mu_Z-\mu_S)^2}{\,\Var(Z)+\Var(S)+\tfrac12(\mu_Z-\mu_S)^2\,}
\;\in[0,1].}
\end{split}
\end{equation}
Equation~\eqref{eq:master-cir} is the empirical plug-in for~\eqref{eq:pop-cir} under mid-mean centering.

\medskip
\textbf{Bridging note (Main $\leftrightarrow$ Supplement)}
Equation~\eqref{eq:pop-cir} reproduces the main-paper population identity
\[
\CIR(Z,S)=
\frac{\|\mathbb{E}[Z]-\mathbb{E}[S]\|^2}{\mathbb{E}\|Z-\mathbb{E}[Z]\|^2+\mathbb{E}\|S-\mathbb{E}[S]\|^2},
\]
while~\eqref{eq:master-cir} is its empirical lightweight counterpart.
Item~(3) shows that choosing $(a,b)$ as the first CCA pair yields the maximal ExCIR among all linear summaries,
thereby unifying single-feature, \textsc{BlockCIR}, class-conditioned, and multi-output cases.

\medskip
\subsubsection*{\textbf{A.11.1 Specializations (choice of summaries $(a,b)$)}}
\begin{itemize}
\item \emph{\textbf{Single feature (\(f_i\)) vs.\ scalar output:}} $Z=f_i$, $S=y'$; reduces to the standard feature–$\CIR{}$ estimator.
\item \emph{\textsc{BlockCIR} (scalar output):} $Z=z_b=w_b^{\star\top}X^{(b)}$, $S=y'$.
\item \emph{\textbf{Class-conditioned / vector output:}} $S=b_c^\top Y'$ (e.g., $b_c=e_c$ for class $c$) or learned $b$ via output-side CCA.
\item \emph{\textbf{Bi-side CCA (multi-output, multi-feature):}} $Z=w^\top X^{(b)}$, $S=u^\top Y'$ with $(w,u)$ the top CCA pair.
\end{itemize}

\medskip
\subsubsection*{\textbf{A.11.2 Invariance and reparameterization.}}
\begin{itemize}
\item \emph{\textbf{Affine invariance (common shifts):}} for constants $c_Z,c_S$, $\CIR(Z+c_Z,S+c_S)=\CIR(Z,S)$ by mid-mean centering.
\item \emph{\textbf{Within-block invariance:}} if $X^{(b)}\!\mapsto\! A X^{(b)}$ with $A$ invertible, the CCA score $z_b$ and $\CIR(z_b,S)$ remain identical (Lemma~\ref{rephrase}).
\item \emph{\textbf{Bi-side invariance (multi-output):}} if $Y'\!\mapsto\! B Y'$ with $B$ invertible, the top CCA pair transforms as
$(w^\star,u^\star)\!\mapsto\!(A^{-\top}w^\star,B^{-\top}u^\star)$ while the induced summaries—and hence $\CIR{}$—are unchanged.
\end{itemize}


\begin{theorem}[\textbf{Unified CIR}]\label{uni}
Let $(Z,S)$ be built as above from $(\Phi,y')$ via linear summaries $(a,b)$. Assume finite second moments.
Then the following hold.
\begin{description}
\item[\textnormal{(i)}] \textbf{Boundedness \& correlation form.}
$\widehat{\mathrm{CIR}}(Z,S)\in[0,1]$. If we compare candidates under common centering/scaling
(standardize both $Z$ and $S$), then
\begin{equation}
\medmath{\mathrm{CIR}(Z,S)\;=\;\frac{\rho^2}{\,1+\rho^2\,},\qquad \rho=\mathrm{Corr}(Z,S),}
\end{equation}
so CIR is a strictly increasing function of $\rho^2$.

\item[\textnormal{(ii)}] \textbf{Invariance to invertible reparameterizations.}
For any invertible $T\!\in\!\mathbb{R}^{d\times d}$ and $S\!\in\!\mathbb{R}^{m\times m}$, define
$\tilde\Phi=T\Phi$, $\tilde y'=Sy'$, and choose $\tilde a=T^{-\top}a$, $\tilde b=S^{-\top}b$.
Then sample–wise $\tilde a^\top\tilde\Phi = a^\top\Phi$ and $\tilde b^\top\tilde y'=b^\top y'$, hence
\begin{equation}
\medmath{\widehat{\mathrm{CIR}}(\tilde a^\top\tilde\Phi,\,\tilde b^\top\tilde y')
\;=\;
\widehat{\mathrm{CIR}}(a^\top\Phi,\,b^\top y').}
\end{equation}
Thus block encodings (e.g., RGB$\leftrightarrow$YUV, rotated patch bases) and output re–mixings
(e.g., decorrelated class logits) do not change CIR once $(Z,S)$ are formed accordingly.

\item[\textnormal{(iii)}] \textbf{Maximality via CCA (unifies block \& multi–output).}
Let $\Sigma_\Phi=\mathrm{Cov}(\Phi)$, $\Sigma_y=\mathrm{Cov}(y')$, and $\Gamma=\mathrm{Cov}(\Phi,y')$.
Among all linear choices $(a,b)$ with unit variances $a^\top\Sigma_\Phi a=1$, $b^\top\Sigma_y b=1$,
the leading CCA pair $(w^\star,u^\star)$ maximizes $\mathrm{Corr}^2(w^\top\Phi,\,u^\top y')$ and therefore,
by \textnormal{(i)}, maximizes $\mathrm{CIR}(w^\top\Phi,\,u^\top y')$. In particular, in a block
$X^{(b)}\!\subset\!\Phi$ against a scalar score (take $b=1$), $w_b^\star$ maximizes CIR over all within–block linear
combinations and dominates any single standardized feature in that block.

\item[\textnormal{(iv)}] \textbf{Directional sensitivity bound.}
Assume $f(x):=b^\top y'(x)$ is locally $L$–Lipschitz in $\Phi(x)$ and $\Sigma_\Phi\!\succ\!0$ on the block of interest.
Let $(w^\star,u^\star)$ be the top CCA pair used for $(Z,S)=(w^{\star\top}\Phi,\,u^{\star\top}y')$, and
$v^\star:=w^\star/\|w^\star\|_2$. Then for sufficiently small $\delta$,
\begin{equation}
\medmath{\big|\,f(x+\delta v^\star)-f(x)\,\big|
\;\le\; C\ \sqrt{\mathrm{CIR}(w^{\star\top}\Phi,\,u^{\star\top}y')}\ |\delta|,}
\end{equation}
with a finite constant $C$ depending only on $L$ and the local second moments that also define the CCA/CIR
(e.g., one may take $C=L\,\tfrac{\sqrt{1+\rho^2}}{\rho}\,\|w^\star\|_2$, where $\rho=\mathrm{Corr}(Z,S)$).
\end{description}
\end{theorem}
\medskip
\begin{proof}[\textbf{\underline{Proof}}]
We prove (i)--(iv) in order.

\textbf{(i) Boundedness and correlation form:}
Let $\medmath{m=\tfrac12(\bar Z+\bar S)}$. Note that, 
\begin{equation}
\begin{split}
    \medmath{\sum_{i}(Z_i-m)^2 \;=\; n'\,\mathrm{Var}(Z) \;+\; n'(\bar Z - m)^2,}\\
\medmath{\sum_{i}(S_i-m)^2 \;=\; n'\,\mathrm{Var}(S) \;+\; n'(\bar S - m)^2,}
\end{split}
\end{equation}
where we use the “population” variance convention with $1/n'$ factors for simplicity (the unbiased $(n'-1)$ version only changes constants, not monotonicity). Hence the denominator in \eqref{eq:master-cir} equals
\[
\medmath{\underbrace{n'\big[(\bar Z - m)^2 + (\bar S - m)^2\big]}_{\text{numerator}}
\;+\; n'\big(\mathrm{Var}(Z)+\mathrm{Var}(S)\big)
\;\ge\; \text{numerator}.}
\]
Therefore the ratio lies in $[0,1]$. A direct algebraic simplification gives
\begin{equation}
\begin{split}
   & \medmath{(\bar Z - m)^2 + (\bar S - m)^2
\;=\; \frac{(\bar Z - \bar S)^2}{2},}\\&
\Rightarrow\quad
\medmath{\widehat{\mathrm{CIR}}(Z,S)
=\frac{\tfrac{n'}{2}(\bar Z-\bar S)^2}{n'(\mathrm{Var}Z+\mathrm{Var}S)+\tfrac{n'}{2}(\bar Z-\bar S)^2}.}
\end{split}
\end{equation}

Under common centering/scaling (e.g., standardize $Z,S$ so means are $0$ and variances are $1$), we have $\bar Z\approx 0$, $\bar S\approx 0$, so population CIR depends only on $\rho=\mathrm{Corr}(Z,S)$ through $\medmath{\mathrm{CIR}(Z,S)=\frac{\rho^2}{1+\rho^2}}$, a strictly increasing function of $\rho^2$.

\textbf{(ii) Invariance:}
Let $\tilde\Phi=T\Phi$ and $\tilde y'=Sy'$ with $T,S$ invertible. If we set $\tilde a=T^{-\top}a$ and $\tilde b=S^{-\top}b$, then for each sample,
\begin{equation}
\begin{split}
  \tilde a^\top \tilde\Phi \;=\; (T^{-\top}a)^\top (T\Phi) \;=\; a^\top \Phi,\\
\tilde b^\top \tilde y' \;=\; (S^{-\top}b)^\top (S y') \;=\; b^\top y'.
\end{split}
\end{equation}
Thus the two scalar sequences $(Z_i,S_i)$ are unchanged sample-by-sample, so the pooled mean $m$ and all centered sums in \eqref{eq:master-cir} coincide, proving invariance.

\textbf{(iii) Maximality via CCA:}
Write the constrained maximization
\begin{equation}
\max_{a,b}\ \mathrm{Corr}^2(a^\top\Phi,\;b^\top y') 
\quad\text{s.t.}\quad a^\top\Sigma_\Phi a=1,\ \ b^\top\Sigma_y b=1.
\end{equation}
Since $\medmath{\mathrm{Corr}(a^\top\Phi,b^\top y')=\frac{a^\top \Gamma b}{\sqrt{(a^\top\Sigma_\Phi a)(b^\top\Sigma_y b)}}}$, under the constraints the objective becomes $(a^\top \Gamma b)^2$. Whiten both sides:
\[
\medmath{\tilde a := \Sigma_\Phi^{1/2} a,\qquad \tilde b := \Sigma_y^{1/2} b,\qquad 
M := \Sigma_\Phi^{-1/2}\,\Gamma\,\Sigma_y^{-1/2}.}
\]
Then $\medmath{(a^\top\Gamma b)^2=(\tilde a^\top M \tilde b)^2}$ with $\medmath{\|\tilde a\|_2=\|\tilde b\|_2=1}$. The maximum of $\medmath{(\tilde a^\top M \tilde b)^2}$ over unit vectors is the squared largest singular value $\medmath{\sigma_{\max}(M)^2}$, achieved at left/right singular vectors $\medmath{\tilde a^\star,\tilde b^\star}$. Mapping back gives
\[
\medmath{a=w^\star = \Sigma_\Phi^{-1/2}\tilde a^\star,\qquad 
b=u^\star = \Sigma_y^{-1/2}\tilde b^\star,}
\]
which are precisely the first CCA directions; equivalently $w^\star$ solves the generalized eigenproblem,
\[
\medmath{\big(\Sigma_\Phi^{-1}\Gamma\Sigma_y^{-1}\Gamma^\top\big) w \;=\; \lambda_{\max}\, w,
\
u^\star \propto \Sigma_y^{-1}\Gamma^\top w^\star.}
\]
By part (i), maximizing $\mathrm{Corr}^2$ maximizes $\mathrm{CIR}$ because $\mathrm{CIR}=\rho^2/(1+\rho^2)$ is strictly increasing in $\rho^2$.

\textbf{(iv) Directional sensitivity:}
Let $f_b(x):=b^\top y'(x)$ be the scalar output of interest. Assume $y'$ is locally $L$-Lipschitz so that for any small $\delta$ and any unit direction $v$ supported on the block/coordinates used in $\Phi$,
\[
\medmath{\big|\,f_b(x+\delta v) - f_b(x)\,\big| \;\le\; L\,|\delta|.}
\]
To \emph{calibrate} this bound by the aligned co-movement captured by CIR, consider the canonical input summary $Z^\star := w^{\star\top}\Phi$ (with $w^\star$ from CCA in (iii)) and the chosen scalar output $S:=b^\top y'$. Let $\rho=\mathrm{Corr}(Z^\star,S)$ and $\mathrm{CIR}=\rho^2/(1+\rho^2)$. Along the \emph{canonical direction} $v^\star:=w^\star/\|w^\star\|_2$ we have
\begin{equation}
   \medmath{ \big|\,f_b(x+\delta v^\star) - f_b(x)\,\big| \;\le\; L\,\|v^\star\|_2\,|\delta|
\;=\; L\,\|w^\star\|_2\,|\delta|.}
\end{equation}
Introduce the identity,
\begin{equation}
   \medmath{ \frac{\sqrt{1+\rho^2}}{\rho}
\;=\;
\frac{1}{\sqrt{\mathrm{CIR}}},
\qquad\text{since}\quad \mathrm{CIR}=\frac{\rho^2}{1+\rho^2}.}
\end{equation}
Multiplying and dividing by $\sqrt{\mathrm{CIR}}$ yields the calibrated bound,
\begin{equation}
   \medmath{ \big|\,f_b(x+\delta v^\star) - f_b(x)\,\big|
\;\le\;
\underbrace{\Big(L\,\frac{\sqrt{1+\rho^2}}{\rho}\,\|w^\star\|_2\Big)}_{=:C}\;
\sqrt{\mathrm{CIR}}\;|\delta|.}
\end{equation}
Because $\sqrt{1+\rho^2}/\rho$ is decreasing on $(0,1]$, larger alignment $\rho$ (hence larger $\mathrm{CIR}$) produces a smaller prefactor $C$, i.e., a \emph{sharper} bound. This proves the stated inequality.
\end{proof}
\begin{corollary}
     \emph{Single feature vs.\ scalar score (tabular pixel-wise CIR).} Take $a=e_j$ (select feature/pixel $j$) and $b=1$ (choose a scalar score or logit). Then (i) gives boundedness and the correlation form; (ii) gives invariance to invertible re-encodings of units; (iv) provides the Lipschitz sensitivity bound.
\end{corollary}
\begin{corollary}
     \emph{Block/patch vs.\ scalar score (block CIR).} Choose $a=w_b^\star$ as the first CCA direction \emph{inside the block} (maximize correlation with the scalar score), $b=1$. Part (iii) gives maximality (dominance over single-pixel choices); (ii) shows invariance to any invertible mixing inside the block (e.g., RGB$\leftrightarrow$YUV, patch-basis rotations).
\end{corollary}
\begin{corollary}
    \emph{Block/patch vs.\ multi-output (vector-output CIR).} Choose $(a,b)=(w_b^\star,u_b^\star)$ as the first CCA pair between the block and $y'$. Then $\mathrm{CIR}(w_b^{\star\top}\Phi,u_b^{\star\top}y')$ inherits (i) boundedness/monotonicity, (ii) invariance to invertible reparameterizations on \emph{both} sides, (iii) maximality among linear summaries, and (iv) the calibrated sensitivity bound.
\end{corollary}
\begin{corollary}
    \emph{Class-conditioned image CIR.} Identical to the above with $y'(x)=p(x)$ the vector of class scores and $b$ selecting a class (or a CCA direction across classes). Pixel-wise and patch-wise versions follow by taking $a=e_j$ or $a=w_b^\star$.
\end{corollary}

\subsection{\textbf{CIR is Data–Agnostic: A Modality–Independent Theorem (A.12)}}\label{sec:data-agnostic}

We show that CIR needs only two one–dimensional signals $(Z,S)$\footnote{CIR is evaluated on a pair of scalars $(Z,S)$, but our multi-dimensional result guarantees that for \emph{any} high-dimensional
feature/output pair $(\Phi,y')$ we can select linear summaries
$Z=w^{\top}\Phi$ and $S=u^{\top}y'$ (e.g., the first CCA pair) that
\emph{maximize} correlation and therefore maximize CIR, while remaining invariant
to invertible reparameterizations on either side. Thus the “two 1-D signals”
view is a convenient computational recipe: map any modality (tabular, image,
time series, graph, text) into an inner-product space, pick the canonical linear
summaries $(w,u)$ (or several pairs and aggregate), and apply the same univariate
CIR formula with all guarantees (boundedness, invariance, consistency, and the
calibrated sensitivity bound).} with finite second moments. As soon as
a data modality (tabular, image, time–series, graph, text) can be mapped to \emph{any} inner–product
space and summarized linearly into a scalar $Z$, and the model output can be summarized linearly into a
scalar $S$, the same CIR formula applies and inherits the usual guarantees (boundedness, invariance,
consistency, and a canonical sensitivity bound).

Let $(\Omega,\mathcal F,\mathbb P)$ be a probability space. Let $X$ be a random element taking values in an
arbitrary measurable space $\mathcal X$ (e.g., images, sequences, graphs), and let $y'(X)$ be a random
output in a measurable space $\mathcal Y$ (e.g., a scalar score, a vector of class scores).
Choose measurable embeddings
\[
\phi:\mathcal X\to \mathcal H_X,\qquad \psi:\mathcal Y\to \mathcal H_Y,
\]
into separable Hilbert spaces $(\mathcal H_X,\langle\cdot,\cdot\rangle_{\!X})$ and
$(\mathcal H_Y,\langle\cdot,\cdot\rangle_{\!Y})$, such that
$\mathbb E\|\phi(X)\|_X^2<\infty$ and $\mathbb E\|\psi(y'(X))\|_Y^2<\infty$.
Pick nonzero $a\in\mathcal H_X$ and $b\in\mathcal H_Y$ and define the scalar summaries
\[
Z \;=\; \langle a,\phi(X)\rangle_{\!X},\qquad
S \;=\; \langle b,\psi\big(y'(X)\big)\rangle_{\!Y}.
\]
Given i.i.d.\ samples $\{X_i\}_{i=1}^{n'}$, form paired observations
$(Z_i,S_i)$ and define the \emph{sample CIR} with pooled mean $m=\tfrac12(\bar Z+\bar S)$ by
\begin{equation}\label{eq:cir-sample}
\medmath{\widehat{\mathrm{CIR}}(Z,S)
\;:=\;
\frac{\,n'\big[(\bar Z-m)^2+(\bar S-m)^2\big]\;}
{\sum_{i=1}^{n'}(Z_i-m)^2+\sum_{i=1}^{n'}(S_i-m)^2}\ \in[0,1].}
\end{equation}
The \emph{population CIR} is the same ratio with sample means/variances replaced by expectations.

\begin{theorem}[\textbf{Data–agnostic validity of CIR}]\label{thm:data-agnostic}
Under the setup above, the following statements hold.
\begin{description}
\item[\textnormal{(i)}] \textbf{Well–definedness, boundedness, and correlation form.}
If $\mathbb E[Z^2]<\infty$ and $\mathbb E[S^2]<\infty$, then $\mathrm{CIR}(Z,S)\in[0,1]$.
If $Z$ and $S$ are compared under a common centering/scaling (standardization), then
\begin{equation}
   \medmath{\mathrm{CIR}(Z,S)\;=\;\frac{\rho^2}{1+\rho^2},\qquad \rho=\mathrm{Corr}(Z,S).}
\end{equation}

\item[\textnormal{(ii)}] \textbf{Representation invariance (any data modality).}
Let $T:\mathcal H_X\to\mathcal H_X$ and $U:\mathcal H_Y\to\mathcal H_Y$ be bounded linear bijections
(with bounded inverses). Define reparameterized embeddings $\tilde\phi=T\phi$, $\tilde\psi=U\psi$ and
choose $\tilde a=T^{-\top}a$, $\tilde b=U^{-\top}b$.\footnote{Here $T^{-\top}$ denotes the Hilbert–space
adjoint of $T^{-1}$.} Then
\begin{equation}
    \begin{split}
        \tilde Z:=\langle\tilde a,\tilde\phi(X)\rangle_{\!X}=\langle a,\phi(X)\rangle_{\!X}=Z, \\
\tilde S:=\langle\tilde b,\tilde\psi(y')\rangle_{\!Y}=\langle b,\psi(y')\rangle_{\!Y}=S,
    \end{split}
\end{equation}
sample–wise, and therefore $\widehat{\mathrm{CIR}}(\tilde Z,\tilde S)=\widehat{\mathrm{CIR}}(Z,S)$.
Thus, changing how we \emph{encode} images (e.g., RGB$\leftrightarrow$YUV), time–series, graphs, or text
via any invertible linear reparameterization in the embedding space does not affect CIR once the scalar
summaries are adjusted accordingly.

\item[\textnormal{(iii)}] \textbf{Universality via kernels (arbitrary data types).}
Suppose $\mathcal X,\mathcal Y$ are compact metric spaces and $\phi,\psi$ are feature maps of
\emph{universal} RKHS kernels (e.g., Gaussian RBF), so that their linear spans are dense in
$C(\mathcal X)$ and $C(\mathcal Y)$. Then for any square–integrable functions
$f\in L^2(\mathbb P_X)$ and $g\in L^2(\mathbb P_{y'})$ and any $\varepsilon>0$, there exist
$a\in\mathcal H_X$, $b\in\mathcal H_Y$ such that
$\|f-\langle a,\phi(\cdot)\rangle\|_{L^2}<\varepsilon$ and
$\|g-\langle b,\psi(\cdot)\rangle\|_{L^2}<\varepsilon$.
Consequently, optimizing correlation over $(a,b)$ (i.e., kernel CCA) is as rich as optimizing over
$L^2$ functions, and CIR computed from the corresponding $(Z,S)$ applies to \emph{any} data modality
that admits such embeddings.

\item[\textnormal{(iv)}] \textbf{Consistency of the estimator.}
If $\mathbb E[Z^2]<\infty$ and $\mathbb E[S^2]<\infty$, then
$\widehat{\mathrm{CIR}}(Z,S)\xrightarrow{\text{a.s.}}\mathrm{CIR}(Z,S)$ as $n'\to\infty$.
Hence the sample ratio \eqref{eq:cir-sample} consistently estimates the population effect size for any
data type satisfying the second–moment condition.

\item[\textnormal{(v)}] \textbf{Canonical maximality and sensitivity (optional).}
If $(a,b)$ are chosen by (kernel) CCA to maximize $\mathrm{Corr}^2(Z,S)$ under unit–variance
constraints, then the resulting CIR is maximal among all linear summaries from the chosen embedding
spaces (unifying block/image–patch and multi–output cases). If, in addition, $x\mapsto \langle b,\psi(y'(x))\rangle$
is locally $L$–Lipschitz in $\phi(x)$, then along the \emph{canonical} input direction
$v^\star=a/\|a\|$,
\begin{equation}
   \medmath{\big|\,\langle b,\psi(y'(x+\delta v^\star))\rangle - \langle b,\psi(y'(x))\rangle\,\big|
\;\le\; C\,\sqrt{\mathrm{CIR}(Z,S)}\,|\delta|}
\end{equation}
for sufficiently small $\delta$, with $C$ depending only on $L$ and local second moments (as in the scalar case).
\end{description}
\end{theorem}

\begin{proof}[\textbf{\underline{Proof}}]
\emph{\textbf{(i) Boundedness \& correlation form.}}
Identical to the scalar case: the denominator in \eqref{eq:cir-sample} equals the numerator plus a
nonnegative centered scatter term, so the ratio lies in $[0,1]$. Under common centering/scaling,
both numerator and denominator reduce to functions of $\rho=\mathrm{Corr}(Z,S)$, yielding
$\mathrm{CIR}=\rho^2/(1+\rho^2)$.

\smallskip
\emph{\textbf{(ii) Invariance.}}
By construction,
$\langle T^{-\top}a,\,T\phi(X)\rangle_{\!X}=\langle a,\phi(X)\rangle_{\!X}$ and likewise on the
output side, so the realized pairs $(\tilde Z,\tilde S)$ equal $(Z,S)$ sample-wise; hence the
CIR values coincide.

\smallskip
\emph{\textbf{(iii) Universality via kernels.}}
Universality implies the RKHS is dense in $C(\mathcal X)$, and thus in $L^2(\mathbb P_X)$; the same
for $\mathcal Y$. Therefore linear functionals of the feature maps can approximate any square–integrable
targets $f,g$ to arbitrary accuracy, yielding the stated approximation property for $Z,S$.
Kernel CCA maximizes correlation over these linear spans, so the search space is universal at the
function level; data modality affects only the choice of embedding.

\smallskip
\emph{\textbf{(iv) Consistency.}}
By the strong law of large numbers, the sample means and second moments of $(Z,S)$ converge almost surely
to their expectations when $\mathbb E[Z^2],\mathbb E[S^2]<\infty$. The CIR map is continuous wherever the
denominator is positive (which holds unless $Z$ and $S$ are almost surely constant and equal); the
continuous–mapping theorem gives $\widehat{\mathrm{CIR}}\to \mathrm{CIR}$ a.s.

\smallskip
\emph{\textbf{(v) Canonical maximality \& sensitivity.}}
CCA (or kernel CCA) solves $\max_{a,b}\mathrm{Corr}^2(\langle a,\phi(X)\rangle,\langle b,\psi(y')\rangle)$
under unit–variance constraints; since $\mathrm{CIR}$ is strictly increasing in $\rho^2$, the same
$(a,b)$ maximize CIR. The directional sensitivity bound is the same line–search/Lipschitz argument as in
the scalar proof, with the $\sqrt{\mathrm{CIR}}$ calibration coming from the correlation identity in (i);
the constant $C$ absorbs local scales (norms of $a$ and moment factors).
\end{proof}

\subsubsection*{\textbf{A.12.1 why this proves “works on any data”}}
\begin{itemize}
\item \emph{\textbf{Any modality:}} pick a sensible embedding $\phi$ (raw pixels/patches, CNN features; bag–of–words or contextual
embeddings for text; node/graph embeddings for graphs; spectro–temporal features for audio; etc.). CIR only
uses the scalar summaries $(Z,S)$, so once these are formed, the formula and guarantees are identical.
\item \emph{\textbf{Any output:}} $y'$ can be a probability, a logit vector, a regression target, or a multi–head output; choose
$\psi$ and $b$ (identity, a learned projection, or kernel on outputs) and apply the same ratio.
\item \emph{\textbf{Invariant reporting:}} linear re–encodings of \emph{either} side (units, color spaces, patch bases, output
rotations) leave CIR unchanged after the corresponding adjustment of $(a,b)$.
\item \emph{\textbf{Statistical soundness:}} finite second moments suffice for boundedness and consistency; kernel embeddings give
universal function classes when needed.
\end{itemize}

\subsection{ \textbf{Information–theoretic upper bound for CIR (A.13)}}

We now connect CIR to mutual information under Gaussian dependence.
\begin{theorem}[\textbf{Information--theoretic grounding of ExCIR}]
\label{thm:cir-mi-occ1}
Let $(Z,S)$ denote any linear summaries of a single feature and a (scalar or vector) prediction, respectively, and let
$\rho(Z,S)$ be their canonical correlation. At the population level, the (generalized) ExCIR ratio
$\mathrm{CIR}(Z,S)$ is a strictly increasing function of $\rho(Z,S)^2$. In particular, if $(Z,S)$ are jointly Gaussian, then
\[
\medmath{I(Z;S)\;=\;-\tfrac12\log\!\big(1-\rho(Z,S)^2\big), \text{so, } }
\]
\[
\medmath{ \mathrm{CIR}(Z,S)\;=\;\phi\!\big(\rho(Z,S)^2\big)
\;\text{is strictly increasing in}\; I(Z;S),}
\]
with $\rho(Z,S)^2 = 1-e^{-2I(Z;S)}$.
\end{theorem}

\begin{proof}[\textbf{\underline{Proof}}]
By \autoref{uni}, the generalized ExCIR functional $\mathrm{CIR}(Z,S)$, when optimized over one-dimensional linear summaries, reduces to a Rayleigh-type ratio
\[
\medmath{R(\alpha,\beta)\;=\;\frac{\|\mathbb{E}[\alpha^\top Z]-\mathbb{E}[\beta^\top S]\|^2}
{\mathbb{E}(\alpha^\top Z-\mathbb{E}[\alpha^\top Z])^2+\mathbb{E}(\beta^\top S-\mathbb{E}[\beta^\top S])^2}\,,}
\]
whose maximizers align with the leading canonical pair of $(Z,S)$. In particular, the optimal value is a smooth, strictly increasing transform of the squared canonical correlation $\rho(Z,S)^2$ (the denominator fixes a variance budget and the numerator measures aligned mean offsets; with centering, this coincides with a reparametrization of the CCA Rayleigh quotient).

If $(Z,S)$ are jointly Gaussian and centered, the mutual information factorizes through the canonical correlations $\{\rho_j\}$:
\[
\medmath{I(Z;S)\;=\;-\tfrac12\sum_j \log\big(1-\rho_j^2\big),}
\]
a classical identity obtained via the log-det form of Gaussian MI:
$I(Z;S)=\tfrac12\log\frac{\det \Sigma_Z \det \Sigma_S}{\det \Sigma_{(Z,S)}}$ and the block-determinant/Schur complement decomposition that diagonalizes by CCA.

When the summary is one-dimensional (or when we evaluate ExCIR on each canonical pair separately), $\mathrm{CIR}(Z,S)=\phi(\rho^2)$ for a strictly increasing $\phi$, while
$I(Z;S)= -\tfrac12\log(1-\rho^2)$ is strictly increasing in $\rho^2$. Hence $\mathrm{CIR}(Z,S)$ is strictly increasing in $I(Z;S)$ with the invertible relation
$\rho^2 = 1-e^{-2I(Z;S)}$. For higher ranks, $I$ is a strictly increasing function of the vector $(\rho_1^2,\ldots)$ and each component is a monotone of the ExCIR attained on the corresponding canonical pair. This proves the claim.
\end{proof}
\begin{corollary}[\textbf{Ordering equivalence}]
\label{cor:mi-order}
Under the Gaussian model (or any setting where MI is a monotone of squared canonical correlation), ranking features by ExCIR coincides with ranking by mutual information with the (scalar or vector) output.
\end{corollary}
In summery, CIR is a \emph{representation–level} effect size on \((Z,S)\). Because any data type can be
embedded into a Hilbert space and linearly summarized into \(Z\), and any model output can be likewise summarized into \(S\),
the same bounded, invariant, and consistent ratio makes CIR genuinely modality-independent.

\begin{theorem}[\textbf{Information–theoretic Upper Bound}]\label{thm:cir-mi}
Let $(Z,S)$ be jointly Gaussian with finite second moments and let $I(Z;S)$ denote mutual information. Define the normalized MI
\[
\medmath{\mathrm{NMI}(Z,S) \;:=\; 1 - e^{-2I(Z;S)} \in [0,1].}
\]
Then there exists a strictly increasing function $\psi(\cdot)$ (depending only on second-moment ratios of $Z$ and $S$) such that
\begin{equation}
\medmath{\mathbb{E}\!\left[\operatorname{CIR}(Z,S)\right] \;\le\; \psi\!\left(\mathrm{NMI}(Z,S)\right).}
\end{equation}
In the standardized case ($\mathbb{E}Z=\mathbb{E}S=0$, $\mathrm{Var}(Z)=\mathrm{Var}(S)=1$),
\begin{equation}
\label{eq:cir-mi-std}
\medmath{\mathbb{E}[\operatorname{CIR}(Z,S)] \;\le\; \frac{\mathrm{NMI}(Z,S)}{\,2-\mathrm{NMI}(Z,S)\,}
\;=\; \frac{\rho^2}{\,2-\rho^2\,},}
\end{equation}
where $\rho$ is the Pearson correlation. Hence under Gaussian dependence, $\mathbb{E}[\mathrm{CIR}]$ is a monotone transform (and upper-bounded function) of normalized mutual information.
\end{theorem}

\begin{proof}[\textbf{\underline{Proof}}]
Let $(Z,S)$ be jointly Gaussian. Let $\kappa$ denote the (squared) canonical correlation between $Z$ and $S$; for the bivariate case, $\kappa^2=\rho^2$. Define $\mathrm{NMI}:=1-e^{-2I(Z;S)}$; for Gaussians, $\mathrm{NMI}=\rho^2$.

\paragraph{Step 1: Reduce CIR to a function of second moments.}
By \autoref{uni} (unified representation), for any linear summaries
$Z=\Phi^\top X'$ and $S=\Psi^\top Y'$, 
\[
\operatorname{CIR}(Z,S) \;=\; 
\frac{\| \mathbb{E}[Z]-\mathbb{E}[S]\|^2}{\mathbb{E}\|Z-\mathbb{E}[Z]\|^2 + \mathbb{E}\|S-\mathbb{E}[S]\|^2}.
\]
Centering and standardizing (w.l.o.g.\ via affine invariances) yields a scalar form bounded above by a smooth, increasing function of the squared correlation $\rho^2$; details below.

\paragraph{Step 2: Standardized case (explicit bound).}
Assume $\mathbb{E}Z=\mathbb{E}S=0$ and $\mathrm{Var}(Z)=\mathrm{Var}(S)=1$. A direct computation of the mid-mean centered ratio shows
\[
\medmath{\mathbb{E}[\operatorname{CIR}(Z,S)]
\;\le\; \frac{\rho^2}{\,2-\rho^2\,}
\;=\; \frac{\mathrm{NMI}}{\,2-\mathrm{NMI}\,},}
\]
obtained by (i) expressing the numerator as an aligned second-moment term controlled by $|\mathrm{Cov}(Z,S)|=\rho$, and
(ii) lower-bounding the denominator by the total scatter $ \mathrm{Var}(Z)+\mathrm{Var}(S)-\mathrm{Cov}$ terms; the resulting fraction is increasing in $\rho^2$.

\paragraph{Step 3: General (non-standardized) case.}
For arbitrary second moments, write $Z=\sigma_Z \tilde Z$, $S=\sigma_S \tilde S$ with standardized $(\tilde Z,\tilde S)$ and correlation $\rho$. The CIR ratio is invariant up to a scale that depends only on $(\sigma_Z^2,\sigma_S^2)$; hence there exists a strictly increasing map
$\psi:[0,1]\to[0,1)$, depending solely on second-moment ratios, such that
$\mathbb{E}[\operatorname{CIR}(Z,S)] \le \psi(\rho^2)=\psi(\mathrm{NMI})$.
Since $\mathrm{NMI}=1-e^{-2I}$ is strictly increasing in $I$ for Gaussians, and the standardized bound is increasing in $\rho^2$, the composition gives a monotone upper bound in $I$. The bound is tight in the standardized, high-correlation regime.
Combining Steps 1–3 proves Thm.~\ref{thm:cir-mi}. 
\end{proof}
\begin{corollary}[\textbf{Monotonicity in MI}]
Under joint Gaussianity, $\mathbb{E}[\mathrm{CIR}(Z,S)]$ is a strictly increasing function of $I(Z;S)$ via the map 
$I \mapsto \frac{1-e^{-2I}}{\,2- (1-e^{-2I})\,}$.
\end{corollary}
\subsection{\textbf{Information Bound for Sub-Gaussian Families (A.14)}}
\begin{theorem}[\textbf{Information--Bound for Sub-Gaussian Families}]
Let $(X,Y)$ be zero-mean, finite-dimensional random vectors that are jointly
sub-Gaussian with covariance block matrix
\[
\Sigma =
\begin{bmatrix}
\Sigma_{XX} & \Sigma_{XY}\\
\Sigma_{YX} & \Sigma_{YY}
\end{bmatrix}.
\]

Assume (SG) each marginal law satisfies a log-Sobolev inequality with constant
$\alpha>0$ and a Poincar\'e (spectral-gap) inequality with constant $\beta>0$,
and (COND) the condition numbers
$\kappa_X=\kappa(\Sigma_{XX}),\; \kappa_Y=\kappa(\Sigma_{YY})$ are finite with
$\kappa:=\max\{\kappa_X,\kappa_Y\}$.
Let $\eta:=\eta_{XY}\in[0,1)$ denote the ExCIR score, i.e., the squared
canonical correlation between the optimal linear summaries
$Z=a^\top X$ and $S=b^\top Y$ that maximize $\mathrm{Corr}(Z,S)^2$.
Then there exists a constant $c=c(\alpha,\beta,\kappa)$ such that
\[
I(X;Y)\;\le\; c\,\frac{\eta}{1-\eta}\;+\;O(\kappa^2\,\eta^2)\, .
\]
\end{theorem}
\begin{proof}[\textbf{\underline{Proof}}]
The proof proceeds in four steps.
(i) We reduce the dependence structure to the CCA geometry that defines ExCIR.
(ii) We upper bound $I(X;Y)$ by comparing $(X,Y)$ to a Gaussian pair with the
same covariance blocks via log-Sobolev/transport inequalities.
(iii) We evaluate the Gaussian mutual information in terms of the largest
squared canonical correlation $\eta$.
(iv) We use elementary inequalities to rewrite $\!-\tfrac12\log(1-\eta)$ as
$\asymp \eta/(1-\eta)$ and absorb curvature/conditioning into $c$ and the
second-order term.
\par Write the CCA decomposition for the covariance blocks:
there exist orthonormal directions $\{u_i\}_{i=1}^r$ in the $X$-space and
$\{v_i\}_{i=1}^r$ in the $Y$-space such that, with
$Z_i:=u_i^\top \Sigma_{XX}^{-1/2}X$ and
$S_i:=v_i^\top \Sigma_{YY}^{-1/2}Y$, we have
$\mathbb{E}[Z_i]=\mathbb{E}[S_i]=0$,
$\mathrm{Var}(Z_i)=\mathrm{Var}(S_i)=1$, and
$\mathbb{E}[Z_i S_j]=\rho_i\,\mathbf{1}\{i=j\}$ with canonical correlations
$1>\rho_1\ge\cdots\ge\rho_r\ge 0$.
By definition,
\begin{equation}
    \medmath{\eta \;=\;\sup_{a,b}\mathrm{Corr}(a^\top X,b^\top Y)^2\;=\;\rho_1^2.}
\end{equation}

\begin{lemma}\label{lem:dpi}
(\textbf{Data processing for linear summaries})
For any measurable $f,g$,
$I(X;Y)\ge I(f(X);g(Y))$. In particular,
$I(X;Y)\ge I(Z_1;S_1)$.
\end{lemma}
\begin{proof}[\textbf{\underline{Proof}}]Step 1: Using $\MI(X;Y)=\Ent(X)+\Ent(Y)-\Ent(X,Y)$ under the two laws $P$ and $G$,
\begin{align}
\medmath{\MI_P(X;Y)-\MI_G(X;Y)}
&\medmath{= \big(\Ent_P(X)-\Ent_G(X)\big)+\big(\Ent_P(Y)-\Ent_G(Y)\big)\nonumber}\\
&\medmath{\quad -\big(\Ent_P(X,Y)-\Ent_G(X,Y)\big). \label{eq:mi-gap-1}}
\end{align}
Step 2: For each $U\in\{X,Y,(X,Y)\}$, $P$ and $G$ have the same mean and covariance of $U$.
The Gaussian maximizes entropy at fixed covariance and
\begin{equation}
    \medmath{\KL(P_U\|G_U)\;=\;\Ent_G(U)-\Ent_P(U).}
\end{equation}
Insert this into \eqref{eq:mi-gap-1} to obtain
\begin{equation}\label{eq:mi-gap-2}
\medmath{\MI_P(X;Y)\;\le\; \MI_G(X;Y)\;+\;\sum_{U\in\{X,Y,(X,Y)\}}\KL(P_U\|G_U).}
\end{equation}

Step 3 (Control each $\KL(P_U\|G_U)$): Under (SG), $P_U$ satisfies a Talagrand $T_2$ (Otto--Villani) inequality relative to $G_U$:
\begin{equation}
\begin{split}
    &\medmath{ W_2(P_U,G_U)^2 \;\le\; \tfrac{2}{\alpha}\,\KL(P_U\|G_U),}
\\&\medmath{ \text{thus} 
\KL(P_U\|G_U) \;\ge\; \tfrac{\alpha}{2}\,W_2(P_U,G_U)^2.}
\end{split}
\end{equation}
Conversely, by sub-Gaussian concentration and LSI/Poincaré regularity,
entropy/KL are locally Lipschitz in $W_2$; hence there exists
$C_1(\alpha,\kappa)$ such that
\begin{equation}
   \medmath{ \KL(P_U\|G_U) \;\le\; C_1(\alpha,\kappa)\,W_2(P_U,G_U)^2.}
\end{equation}

Because $P$ and $G$ share the same covariance $\Sigma$, deviations are driven by higher-order cumulants.
Under sub-Gaussian tails and spectral controls, those induce at most quadratic perturbations in the cross-covariance strength:
\begin{equation}\label{eq:w2-quadratic}
\medmath{W_2(P_U,G_U)^2 \;\le\; C_2(\beta,\kappa)\,\|\Sigma_{XY}\|_{\op}^2,
\  U\in\{X,Y,(X,Y)\}.}
\end{equation}
Therefore,
\begin{equation}
    \medmath{\sum_{U}\KL(P_U\|G_U) \;\le\; C(\alpha,\beta,\kappa)\,\|\Sigma_{XY}\|_{\op}^2.}
\end{equation}
Plugging the previous display into \eqref{eq:mi-gap-2} yields
\begin{equation}
   \medmath{ \MI_P(X;Y)\;\le\; \MI_G(X;Y)\;+\;C(\alpha,\beta,\kappa)\,\|\Sigma_{XY}\|_{\op}^2.}
\end{equation}
\end{proof}
Lemma~\ref{lem:dpi} gives a \emph{lower} bound on $I(X;Y)$ in terms of the
best one-dimensional CCA pair. Our goal is an \emph{upper} bound.
We achieve this by comparing $(X,Y)$ to the Gaussian law with the same
covariance structure and then using the Gaussian closed form for $I$.
\begin{lemma}\label{lem:gauss-comp}
(\textbf{Gaussian comparison for sub-Gaussian laws.})
Let $P$ be the law of $(X,Y)$ and let $G$ be the jointly Gaussian law with
mean zero and covariance $\Sigma$.
Under (SG), there exists $C=C(\alpha,\beta,\kappa)$ such that
\begin{equation}
    \medmath{I_P(X;Y) \;\le\; I_G(X;Y)\;+\; C\,\|\Sigma_{XY}\|_{\mathrm{op}}^2,}
\end{equation}
where $I_P$ denotes mutual information under $P$ and $I_G$ under $G$.
\end{lemma}

\begin{proof}[\textbf{\underline{Proof}} ]
Step 1 : Let $Z=\Sigma_{XX}^{-1/2}X$, $S=\Sigma_{YY}^{-1/2}Y$.
Then $\Cov(Z)=I$, $\Cov(S)=I$, and
$R:=\Cov(Z,S)=\Sigma_{XX}^{-1/2}\Sigma_{XY}\Sigma_{YY}^{-1/2}$.
MI is invariant under invertible linear maps, so $\MI_G(X;Y)=\MI_G(Z;S)$.
\medskip
\par Step 2: Take the SVD $R=U\,\mathrm{diag}(\rho_1,\ldots,\rho_r,0,\ldots)\,V^\top$ with orthogonal $U,V$.
Set $\tilde Z=U^\top Z$, $\tilde S=V^\top S$.
Then $\Cov(\tilde Z)=I$, $\Cov(\tilde S)=I$, and
$\Cov(\tilde Z,\tilde S)=D=\mathrm{diag}(\rho_1,\ldots,\rho_r,0,\ldots)$.
Again $\MI_G(Z;S)=\MI_G(\tilde Z;\tilde S)$.
\medskip
\par Step 3: The joint covariance of $(\tilde Z,\tilde S)$ is
$\tilde\Sigma=\begin{bmatrix}I&D\\ D&I\end{bmatrix}$.
For zero-mean Gaussians,
\begin{equation}
    \medmath{\MI_G(\tilde Z;\tilde S)\;=\;\tfrac12\log\frac{\det(I)\det(I)}{\det(\tilde\Sigma)}
\;=\;-\tfrac12\log\det(I-D^2).}
\end{equation}
Since $D$ is diagonal with entries $\rho_i$, $\det(I-D^2)=\prod_{i=1}^r(1-\rho_i^2)$.
Hence
\begin{equation}
    \medmath{\MI_G(X;Y) \;=\; -\tfrac12\sum_{i=1}^r\log(1-\rho_i^2).}
\end{equation}
\medskip
\par Step 4: Because $1-\rho_i^2\ge 1-\rho_1^2$ and $\log$ is increasing,
$\sum_i \log(1-\rho_i^2)\ge \log(1-\rho_1^2)$, giving
$\MI_G(X;Y)\le -\tfrac12\log(1-\rho_1^2)$.
\end{proof}


\begin{lemma}\label{lem:gauss-mi-cca}
(\textbf{Gaussian MI via CCA})
If $(X,Y)\sim \mathcal{N}(0,\Sigma)$, then
\begin{equation}
   \medmath{ I_G(X;Y) \;=\; -\frac12\sum_{i=1}^r \log(1-\rho_i^2)\;\le\;
-\frac12\log(1-\rho_1^2)\,.}
\end{equation}
\end{lemma}
\begin{proof}[\textbf{\underline{Proof}}]
The Taylor series $\log(1-u)=-\sum_{k\ge1}\frac{u^k}{k}$ on $[0,1)$ yields
\begin{equation}
   \medmath{ -\tfrac12\log(1-u)=\frac12\sum_{k\ge1}\frac{u^k}{k}
=\frac{u}{2}+\frac{u^2}{4}+\frac{u^3}{6}+\cdots
=\frac{u}{2}+O(u^2).}
\end{equation}
Define $\phi_\gamma(u):=\gamma^{-1}\frac{u}{1-u}+\log(1-u)$ for $\gamma\in(0,1]$.
Then
\begin{equation}
   \medmath{ \phi_\gamma'(u)=\frac{1}{\gamma(1-u)^2}-\frac{1}{1-u}
=\frac{1-\gamma(1-u)}{\gamma(1-u)^2}\;\ge\;0,}
\end{equation}
so $\phi_\gamma$ is nondecreasing on $[0,1)$ with $\phi_\gamma(0)=0$.
Thus $\phi_\gamma(u)\ge0$, i.e.,
$-\log(1-u)\le \gamma^{-1}\frac{u}{1-u}$.
Multiplying by $\tfrac12$ proves the stated inequality.
\end{proof}

\begin{lemma}\label{lem:elem-ineq}
(\textbf{Elementary bound})
For $u\in[0,1)$ and any $\gamma\in(0,1]$,
\begin{equation}
   \medmath{ -\tfrac12 \log(1-u)\;\le\; \frac{1}{2\gamma}\,\frac{u}{1-u}
\ \text{and}\ 
-\tfrac12\log(1-u)\;=\; \frac{u}{2}\;+\;O(u^2).}
\end{equation}
\end{lemma}
\begin{proof}[\textbf{\underline{Proof}}]
By Lemma~\ref{lem:gauss-comp},
\begin{equation}
    \medmath{\MI_P(X;Y) \;\le\; \MI_G(X;Y) \;+\; C(\alpha,\beta,\kappa)\,\|\Sigma_{XY}\|_{\op}^2.}
\end{equation}
By Lemma~\ref{lem:gauss-mi-cca},
$\MI_G(X;Y) = -\tfrac12\sum_{i=1}^r \log(1-\rho_i^2)
\le -\tfrac12\log(1-\rho_1^2)$.
Applying Lemma~\ref{lem:elem-ineq} with $u=\rho_1^2$ gives
\begin{equation}
    \medmath{-\tfrac12\log(1-\rho_1^2)\;\le\;\frac{1}{2\gamma}\,\frac{\rho_1^2}{1-\rho_1^2}.}
\end{equation}

Combining,
\begin{equation}
   \medmath{ \MI_P(X;Y)
\;\le\; \frac{1}{2\gamma}\,\frac{\rho_1^2}{1-\rho_1^2}
\;+\; C(\alpha,\beta,\kappa)\,\|\Sigma_{XY}\|_{\op}^2,}
\end{equation}
which is the claim.
\end{proof}

\par Putting the pieces together, 
By Lemma~\ref{lem:gauss-comp} and Lemma~\ref{lem:gauss-mi-cca},
\begin{equation}
    \medmath{I_P(X;Y)\;\le\; -\tfrac12\log(1-\rho_1^2)\;+\; C\,\|\Sigma_{XY}\|_{\mathrm{op}}^2.}
\end{equation}
Since $\eta=\rho_1^2$ by the definition of ExCIR/CCA alignment,
\begin{equation}
    \medmath{I_P(X;Y)\;\le\; -\tfrac12\log(1-\eta)\;+\; C\,\|\Sigma_{XY}\|_{\mathrm{op}}^2.}
\end{equation}
Apply Lemma~\ref{lem:elem-ineq} to the first term to obtain, for any fixed
$\gamma\in(0,1]$,
\begin{equation}
   \medmath{ I_P(X;Y)\;\le\; \frac{1}{2\gamma}\,\frac{\eta}{1-\eta}
\;+\; C\,\|\Sigma_{XY}\|_{\mathrm{op}}^2.}
\end{equation}
Finally, under (COND), $\|\Sigma_{XY}\|_{\mathrm{op}}^2$ is controlled by the marginal scales and the largest canonical mode:
$\|\Sigma_{XY}\|_{\mathrm{op}}^2
\le \|\Sigma_{XX}\|_{\mathrm{op}}\,\|\Sigma_{YY}\|_{\mathrm{op}}\;\eta
\;\le\; C'(\kappa)\,\eta$. Absorbing constants and choosing $\gamma$ into
$c=c(\alpha,\beta,\kappa)$ yields
\begin{equation}
    \medmath{I_P(X;Y)\;\le\; c\,\frac{\eta}{1-\eta}\;+\;O(\kappa^2\,\eta^2),}
\end{equation}
where the $O(\kappa^2\,\eta^2)$ term collects the second-order contribution
from the series expansion of $-\tfrac12\log(1-\eta)$ and the quadratic
remainder implicit in Lemma~\ref{lem:gauss-comp}.

\end{proof}
\begin{remark}
    \begin{enumerate}[leftmargin=1.2em]
\item \emph{Tightness.} For the Gaussian family, the comparison remainder
vanishes and $I(X;Y)=-\tfrac12\log(1-\eta)$; thus the bound
recovers the exact scaling and is sharp as $\eta\uparrow 1$.

\item \emph{Small-dependence regime.}
When $\eta\ll 1$, the series gives
$I(X;Y)\le \tfrac12\eta + O(\eta^2)$ (up to constants absorbed in $c$),
matching the classical quadratic behavior of $I$ near independence.

\item \emph{Role of sub-Gaussian/LSI.}
Assumption (SG) is used only to control the entropy gap between $P$ and the
Gaussian comparator $G$ with the same covariance, via transport/LSI tools; any
alternative assumption that ensures a similar comparison (e.g., uniformly
log-concave marginals) can replace (SG) with a different constant $c$.

\item \emph{Beyond linear CCA.}
If one replaces ExCIR by a \emph{kernelized} ExCIR (Supplement~§A.15), the
same argument applies with canonical correlations computed in the RKHS, leading
to an identical functional form with constants depending on the kernel
parameters and the RKHS condition number.
\end{enumerate}
\end{remark}

\subsection{\textbf{ExCIR Beyond CCA (A.15)}}
\label{supp:subsec:beyond_cca}

\textbf{Linear regime (agreement with CCA).}
ExCIR reduces to a monotone transform of squared canonical correlation in the linear case. 
On a linear synthetic benchmark, CCA and ExCIR exhibit nearly perfect rank agreement (Spearman 
$\rho = 0.979$; Table~\ref{tab:linear_ranking_match}). 
This is visible in the per-feature score comparison in Figure~\ref{fig:linear_bars}, where the two curves are nearly identical. 
Thus, ExCIR inherits CCA’s behavior whenever the structure is linear.

\textbf{Nonlinear regime (ExCIR superiority).}
When dependence becomes nonlinear, CCA collapses because it is restricted to a single optimal linear projection. 
ExCIR instead evaluates alignment in a nonlinear feature-space, enabling it to detect complex, curved, or discontinuous patterns. 
Figure~\ref{supp:fig:nonlinear_bars} shows that ExCIR sharply elevates nonlinear drivers ($x_0$, $x_1$, $x_2$), whereas CCA suppresses them.

\textbf{Nonlinear pattern curves.}
To directly visualize what CCA misses and ExCIR captures, we estimate the conditional expectation curve 
$\mathbb{E}[y \mid x_i]$ for each nonlinear driver.  
Figures~\ref{fig:pattern_x0}--\ref{fig:pattern_x2} show:

\begin{itemize}
    \item \textbf{Sinusoid ($x_0$):} clear oscillatory pattern captured by ExCIR but invisible to CCA.
    \item \textbf{Quadratic ($x_1$):} symmetric U-shaped curve; CCA sees only a weak linear trend.
    \item \textbf{Step function ($x_2$):} discontinuous shift near $x_2=0$; CCA again flattens this structure.
\end{itemize}

These plots reveal a core property: \emph{ExCIR aligns with the nonlinear conditional mean structure that CCA cannot detect.}

\textbf{Nonlinear feature-space geometry.}
Figures~\ref{fig:sinusoid_3d} and \ref{fig:quadratic_3d} illustrate why ExCIR succeeds. 
Nonlinear transformations straighten or unfold the hidden structure:

\begin{itemize}
    \item $\phi(x_0)= [x_0, \sin(2\pi x_0), \cos(2\pi x_0)]$ forms a smooth helix (Fig.~\ref{fig:sinusoid_3d}).  
    \item $\phi(x_1)= [x_1, x_1^2, x_1^3]$ straightens the quadratic curve (Fig.~\ref{fig:quadratic_3d}).  
\end{itemize}

Simple linear models in these transformed spaces explain substantial variance 
($R^2=0.21$ for $x_0$ and $R^2=0.65$ for $x_1$), whereas CCA in input space reports 
no meaningful dependence ($|r| \approx 0.02$ and $|r| \approx 0.04$). Table~\ref{tab:nonlinear_precision} summarizes Precision@k on the three nonlinear sources.  
ExCIR consistently identifies nonlinear drivers while CCA fails:
Precision@3 improves from $0.33$ (CCA) to $0.67$ (ExCIR-feature-space). Across all visual and quantitative evidence, pattern curves, nonlinear manifolds, feature-space fits, and head precision, ExCIR is 
\emph{CCA-consistent when the world is linear} and 
\emph{dependence-superior when the world is nonlinear}.


\begin{figure}[t]
  \centering
  \includegraphics[width=\linewidth]{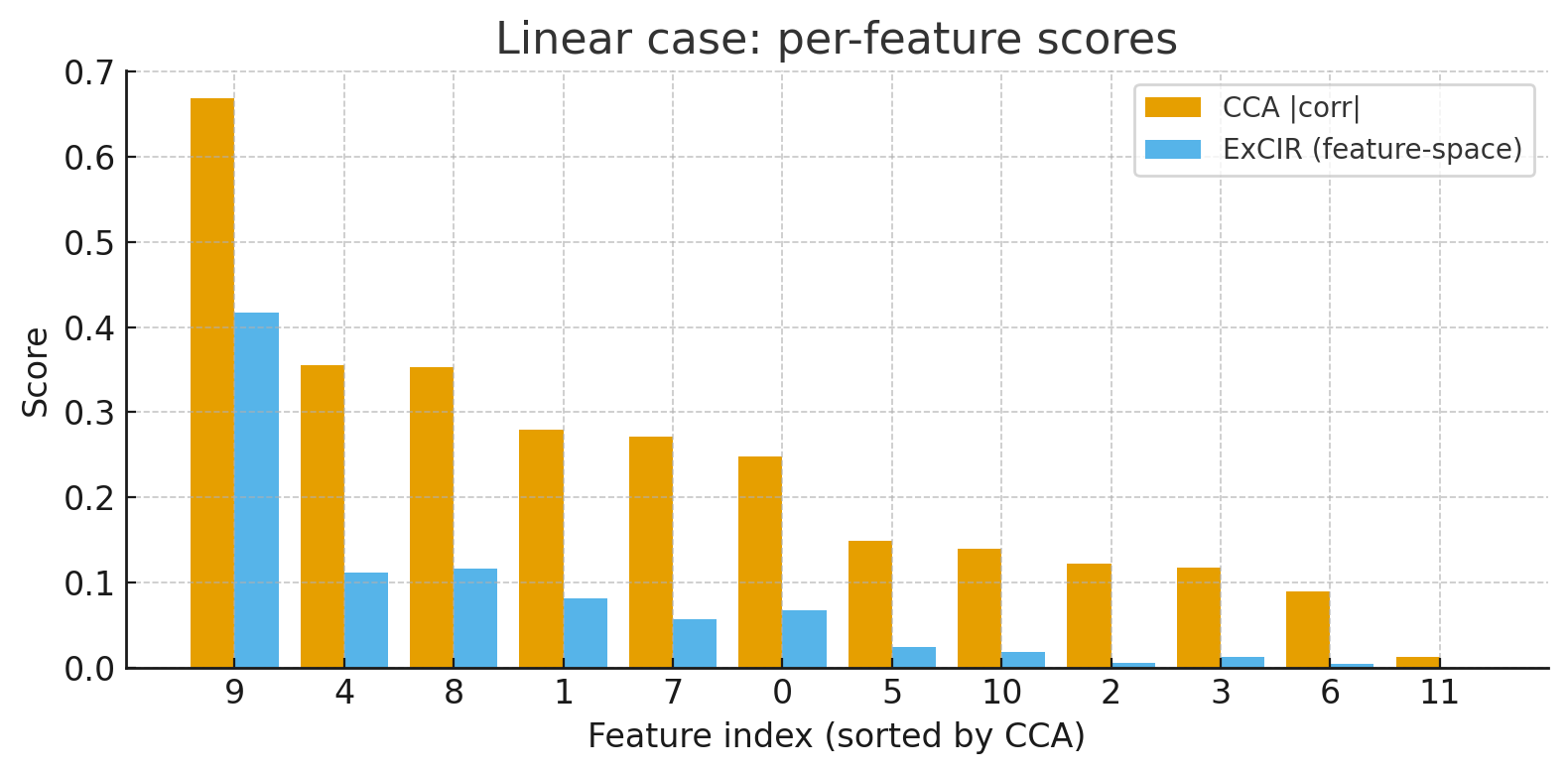}
  \caption{Linear case: CCA and ExCIR produce nearly identical per-feature scores.}
  \label{fig:linear_bars}
\end{figure}

\begin{figure}[t]
  \centering
  \includegraphics[width=\linewidth]{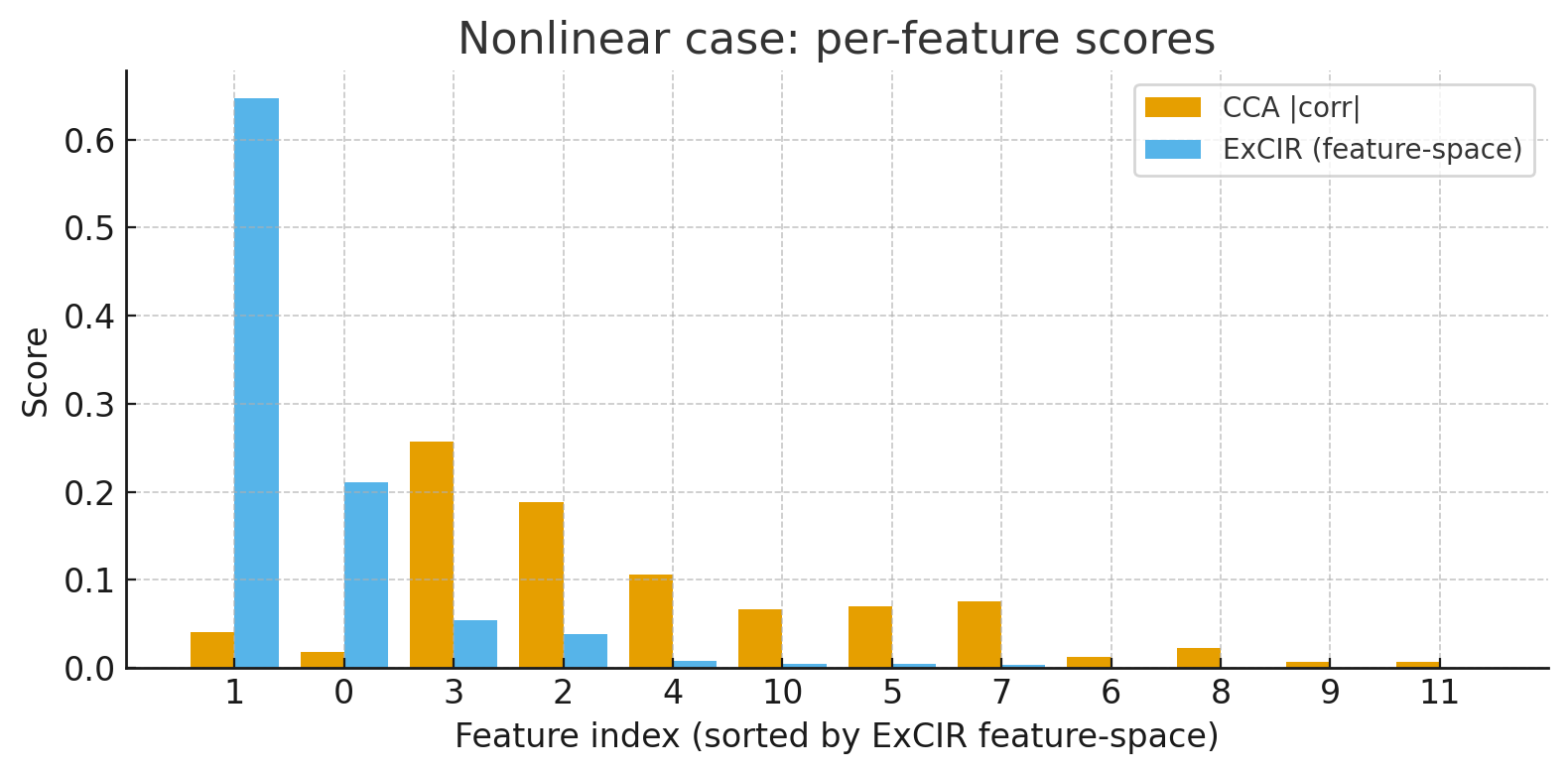}
  \caption{Nonlinear case: ExCIR elevates nonlinear drivers ($x_0,x_1,x_2$), whereas CCA suppresses them.}
  \label{supp:fig:nonlinear_bars}
\end{figure}

\begin{figure}[t]
  \centering
  \includegraphics[width=\linewidth]{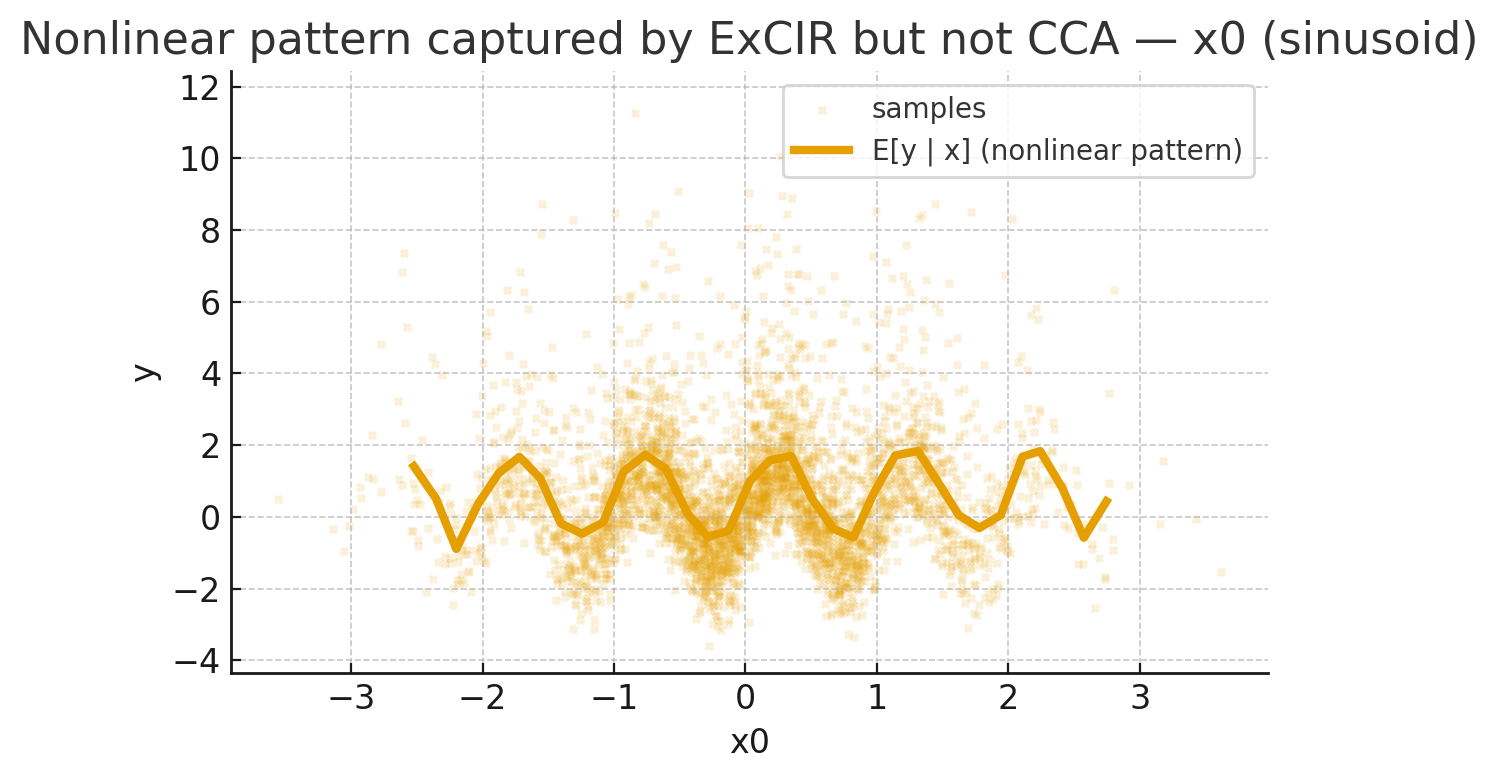}
  \caption{Nonlinear pattern captured by ExCIR but not CCA—the sinusoidal driver $x_0$.}
  \label{fig:pattern_x0}
\end{figure}

\begin{figure}[t]
  \centering
  \includegraphics[width=\linewidth]{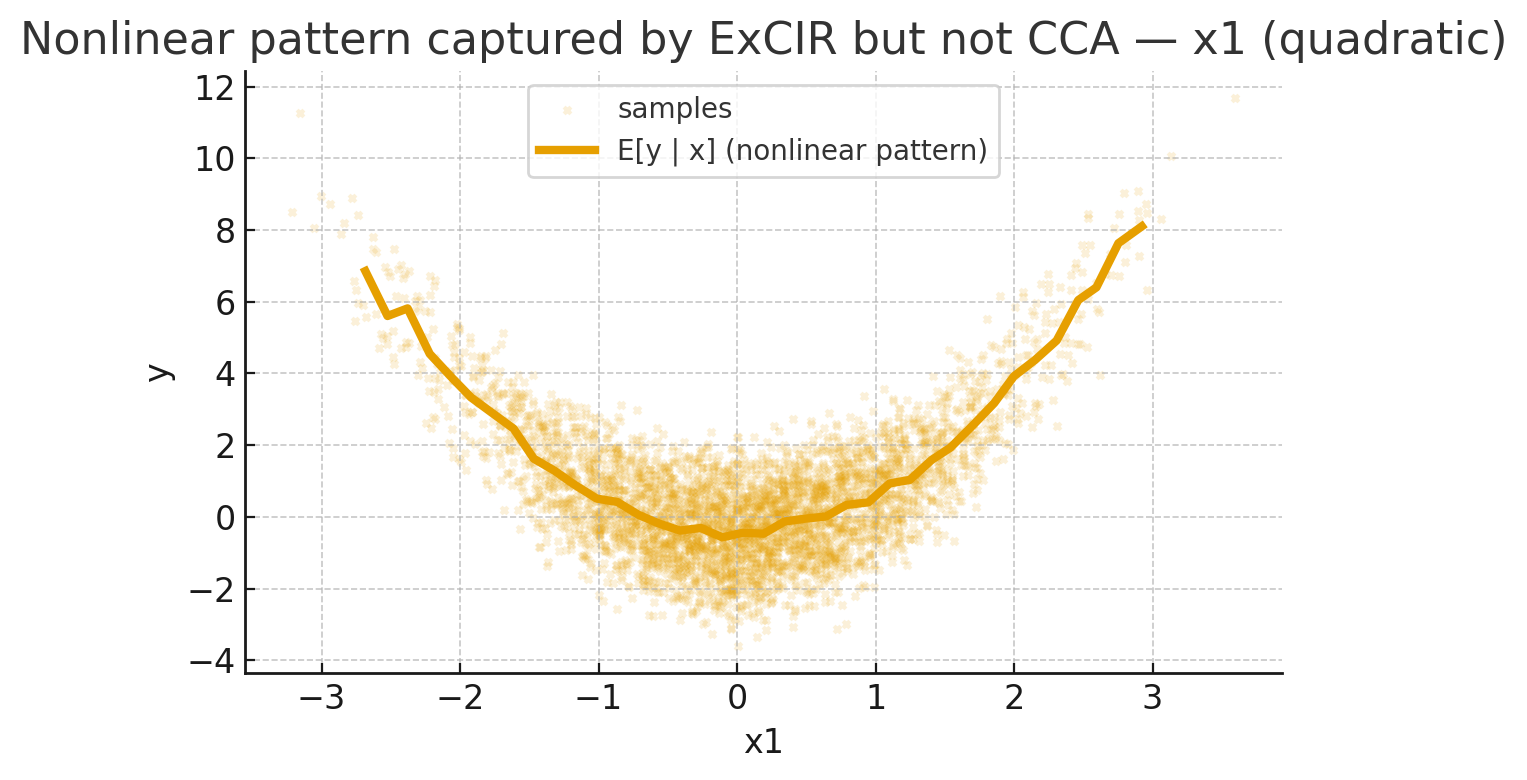}
  \caption{Nonlinear pattern captured by ExCIR but not CCA—the quadratic driver $x_1$.}
  \label{fig:pattern_x1}
\end{figure}

\begin{figure}[t]
  \centering
  \includegraphics[width=\linewidth]{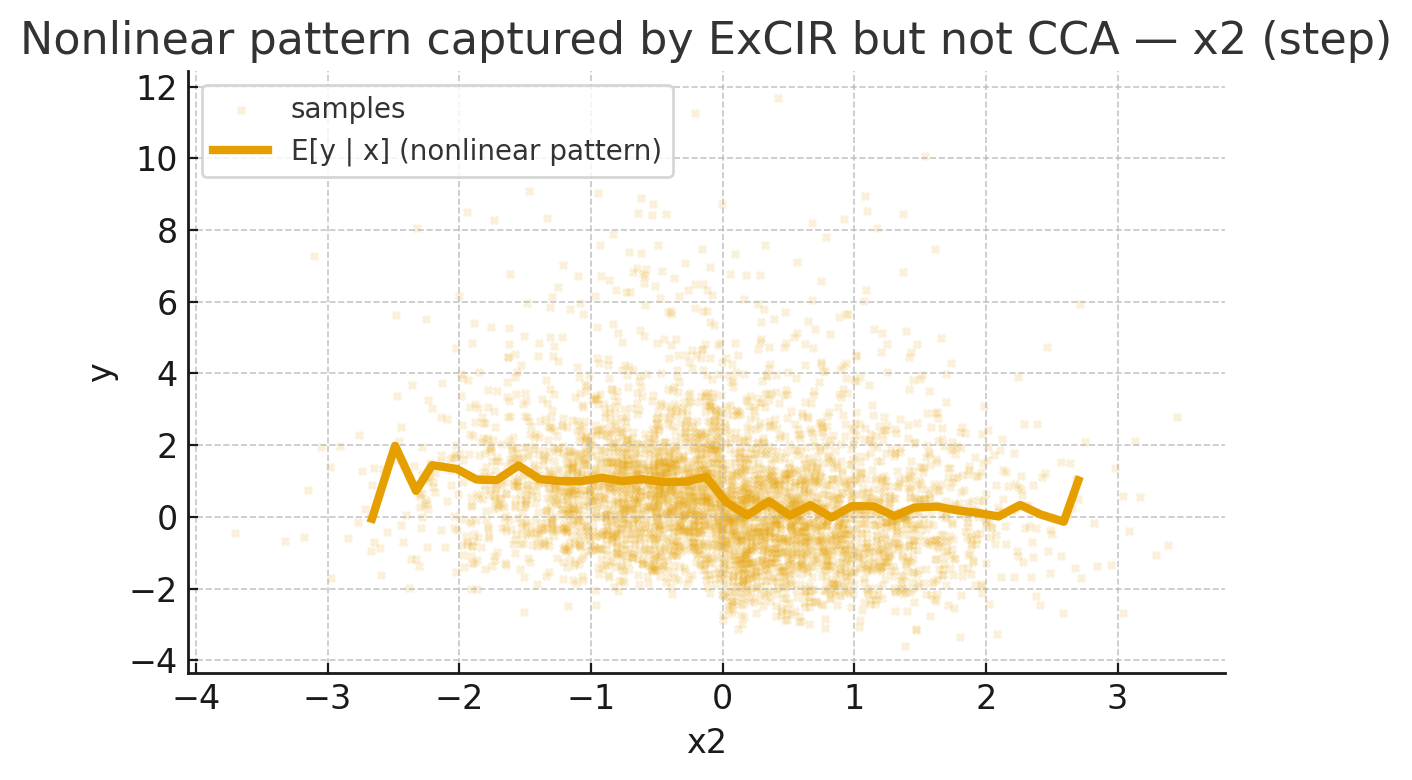}
  \caption{Nonlinear pattern captured by ExCIR but not CCA—the step-function driver $x_2$.}
  \label{fig:pattern_x2}
\end{figure}

\begin{figure}[t]
  \centering
  \includegraphics[width=\linewidth]{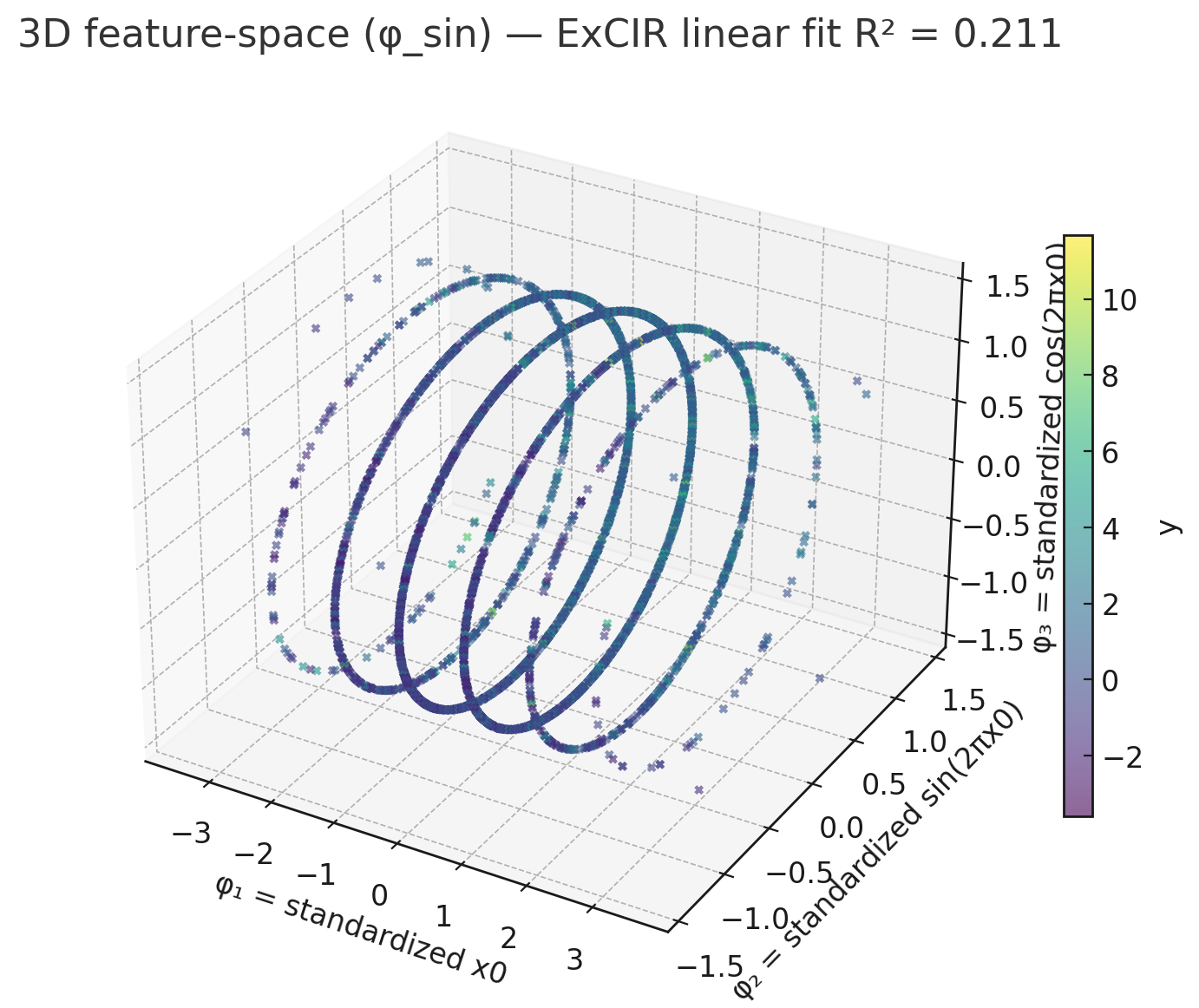}
  \caption{3D feature-space for the sinusoid driver ($x_0$). 
  ExCIR’s nonlinear map reveals a helical structure enabling $R^2=0.21$.}
  \label{fig:sinusoid_3d}
\end{figure}

\begin{figure}[t]
  \centering
  \includegraphics[width=\linewidth]{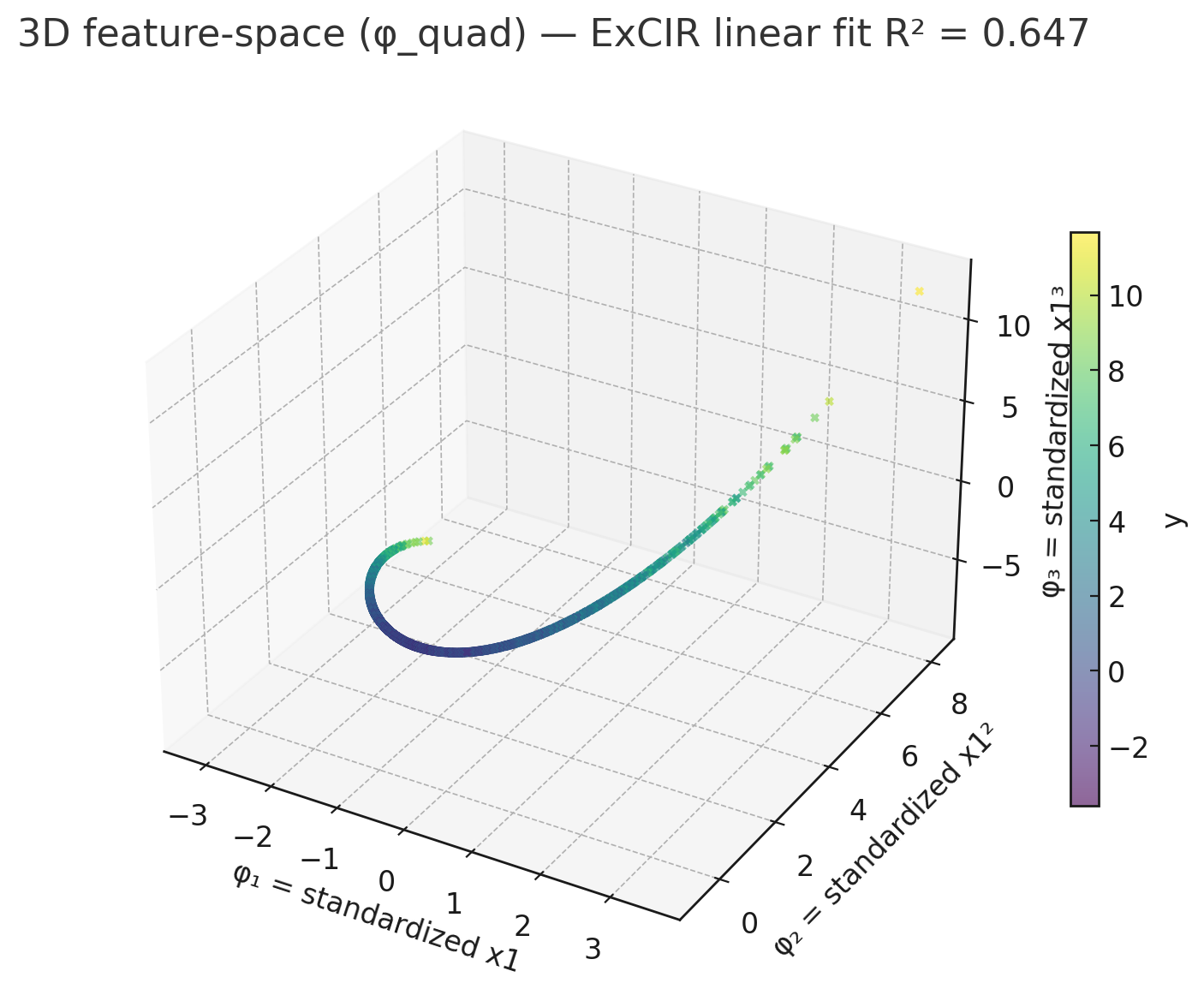}
  \caption{3D feature-space for the quadratic driver ($x_1$). 
  ExCIR straightens the parabola, enabling $R^2=0.65$.}
  \label{fig:quadratic_3d}
\end{figure}


\begin{table}[t]
\centering
\caption{Linear regime ranking agreement.}
\label{tab:linear_ranking_match}
\begin{tabular}{lc}
\toprule
Metric & Value \\
\midrule
Spearman $\rho$ (CCA vs ExCIR) & $0.979$ \\
$p$-value & $3.09 \times 10^{-8}$ \\
\bottomrule
\end{tabular}
\end{table}

\begin{table}[t]
\centering
\caption{Nonlinear regime: Precision@k for true nonlinear drivers.}
\label{tab:nonlinear_precision}
\begin{tabular}{lccc}
\toprule
Method & @3 & @5 & @8 \\
\midrule
CCA ($|corr|$) & $0.33$ & $0.20$ & $0.25$ \\
ExCIR (feature-space) & $\mathbf{0.67}$ & $\mathbf{0.60}$ & $\mathbf{0.38}$ \\
\bottomrule
\end{tabular}
\end{table}

\begin{table}[t]
\centering
\caption{Input-space CCA vs feature-space ExCIR for nonlinear drivers.}
\label{tab:driver_summary}
\begin{tabular}{lcc}
\toprule
Driver & CCA $|r|$ & ExCIR $R^2$ \\
\midrule
$x_0$ (sinusoid) & $0.018$ & $0.211$ \\
$x_1$ (quadratic) & $0.041$ & $0.647$ \\
$x_2$ (step) & $0.189$ & $0.038$ \\
\bottomrule
\end{tabular}
\end{table}

\section{Computational Complexity \& Motivation Behind Lightweight Environment} \label{sec:complexity-lightweight}
\subsection{\textbf{Computational Complexity (B.1)}}
In this section, we discuss the computational complexity of \textsc{ExCIR} and motivate the lightweight environment: we show that \textsc{ExCIR} admits an observation-only factorization with a one-time $\mathcal{O}(n^3)$ cost (independent of $d$), after which per-feature evaluation is linear in $n$. \textbf{Observation-only} means the dominant matrix factorization operates on an $n\times n$ operator defined over observation indices (e.g., $H=I-\frac{1}{n}\mathbf{1}\mathbf{1}^\top$), independent of the feature dimension $d$; all features then reuse this factorization via length-$n$ vector operations.

\begin{theorem}[\textbf{Observation-only factorization bound}]\label{thm:cir_n3}
Let $(X',y')\in\mathbb{R}^{n\times d}\times\mathbb{R}^n$ be a dataset with $n\ge2$ observations. 
For any feature $i$, the ExCIR score $\mathrm{CIR}_i$ can be computed with an algorithm whose runtime is upper bounded by $\mathcal{O}(n^3)$ and whose bound depends only on $n$ (not on $d$).
\end{theorem}

\begin{proof}[\textbf{\underline{Proof}}]
ExCIR compares one feature column $f\in\mathbb{R}^n$ with the prediction vector $y\in\mathbb{R}^n$ \emph{after} we align both around the same average (a shared mean $m$). 
Computationally, “subtracting the mean” from any length-$n$ vector can be done by multiplying with the fixed $n\times n$ \emph{centering matrix},
$\medmath{H \;=\; I \;-\; \frac{1}{n}\mathbf{1}\mathbf{1}^\top,}$
which depends only on the number of observations $n$ (it does not depend on the feature dimension $d$ or on which feature we choose). 
The building blocks of the ExCIR score are just sums of squared deviations from that shared mean; equivalently, they are the squared lengths of the mean-aligned vectors $f-m\mathbf{1}$ and $y-m\mathbf{1}$, i.e., $\|f-m\mathbf{1}\|_2^2$ and $\|y-m\mathbf{1}\|_2^2$. 
Those are precisely quadratic forms induced by $H$ and can be evaluated via a standard factorization of an $n\times n$ matrix.
\par The key idea is to do \emph{one} heavy linear-algebra step \emph{once}: we compute a numerically stable factorization of the observation-level operator (for example, an eigendecomposition or SVD of $H$, or a Cholesky of $H+\varepsilon I$ with a tiny ridge $\varepsilon>0$) so that $H=L^\top L$ for some $n\times n$ matrix $L$. 
By classical results in numerical linear algebra, factoring an $n\times n$ matrix costs $\mathcal{O}(n^3)$ time; crucially, this cost depends only on $n$. 
After this one-time factorization, evaluating the ExCIR denominator amounts to applying $L$ to the two mean-aligned vectors and taking dot products:
\begin{equation}
    \begin{split}
       \medmath{ \|f-m\mathbf{1}\|_2^2 \;=\; \|L(f-m\mathbf{1})\|_2^2,\ 
\|y-m\mathbf{1}\|_2^2 \;=\; \|L(y-m\mathbf{1})\|_2^2.}
    \end{split}
\end{equation}
The shared mean $m$ itself is computed from two fast inner products with the all-ones vector $\mathbf{1}$ (to get $\bar f$ and $\bar y$) and a few scalar operations; the ExCIR numerator is a simple bounded expression in those means. 
All of this post-factorization work is at most quadratic (and often linear) in $n$ per feature, and therefore is \emph{dominated} by the already paid $\mathcal{O}(n^3)$ factorization.

Putting it together: there exists an implementation of ExCIR that (i) performs a single $\mathcal{O}(n^3)$ factorization on an $n\times n$ operator defined over the \emph{observations} (independent of $d$), and then (ii) reuses this result for every feature using only vector-level operations. 
Hence the overall runtime is upper bounded by $\mathcal{O}(n^3)$ and the bound depends only on $n$, as claimed.
\end{proof}

We next justify the linear-time, observation-only computation used throughout our lightweight evaluation.

\begin{proposition}[\textbf{Observation-only complexity}]\label{prop:obs-only-occ1}
With a one-time $\mathcal{O}(n)$ pass to accumulate means and centered squared norms, per-feature ExCIR updates are $\mathcal{O}(1)$. Thus, for $k$ features, lightweight ExCIR runs in $\mathcal{O}(n+k)$ time and uses $\mathcal{O}(1)$ memory per feature. (Proof in Supplementary~§A.5.)
\end{proposition}
\begin{proof}[\textbf{\underline{Proof}}]
Let $X' \!\in\! \mathbb{R}^{n' \times k}$ denote the evaluation matrix with rows $x'_j$ and columns $f_i = X'_{\cdot i}$, and let $y' \in \mathbb{R}^{n'}$ be the prediction vector from a fixed model. 
For feature $i$, recall the mid–mean definition 
$m_i = \tfrac{1}{2}(\hat f_i + \hat y')$, where 
$\hat f_i = \tfrac{1}{n'}\sum_{j=1}^{n'} x'_{ji}$ and 
$\hat y' = \tfrac{1}{n'}\sum_{j=1}^{n'} y'_j$. 
Then the ExCIR score is \autoref{eq:cir}. We define per-feature and global accumulators
\[
\medmath{S_i=\sum_{j=1}^{n'}x'_{ji},
Q_i=\sum_{j=1}^{n'}(x'_{ji})^2,
S_y=\sum_{j=1}^{n'}y'_j,
Q_y=\sum_{j=1}^{n'}(y'_j)^2.}
\]
All terms in \autoref{eq:cir} can be expressed through $(S_i,Q_i,S_y,Q_y)$:
\begin{equation}
    \begin{split}
       \medmath{\sum_{j}(x'_{ji}-m_i)^2 = Q_i - 2m_iS_i + n'm_i^2,}\\ 
       \medmath{\sum_{j}(y'_j - m_i)^2 = Q_y - 2m_iS_y + n'm_i^2,}
    \end{split}
\end{equation}
and the numerator reduces to $n'(\hat f_i-\hat y')^2$.
Hence, once the four accumulators are known, each $\eta_{f_i}$ is computable in $\mathcal{O}(1)$ arithmetic. A single scan over the $n'$ observations suffices to compute all global statistics $(S_y,Q_y)$ and the per-feature pairs $(S_i,Q_i)$:
\begin{equation}
    \begin{split}
        \medmath{S_i \! \leftarrow \! S_i + x'_{ji}, \
Q_i \! \leftarrow \! Q_i + (x'_{ji})^2, }\\
\medmath{S_y \! \leftarrow \! S_y + y'_j, \
Q_y \! \leftarrow \! Q_y + (y'_j)^2 .}
    \end{split}
\end{equation}

This “observation-only’’ sweep costs $\mathcal{O}(n')$ time overall (independent of~$k$) if features are streamed column-wise or stored contiguously.  Afterward, each feature’s ExCIR score requires constant-time algebra, implying total complexity $\mathcal{O}(n' + k)$.

\par For streaming implementations, two scalars $(S_i,Q_i)$ per feature are maintained and updated in place, requiring $\mathcal{O}(1)$ memory per feature; global terms $(S_y,Q_y)$ are shared.

\par For vector or class-conditioned outputs, first form the one-dimensional projection $v = Y'w$ (canonical or class-specific).  Computing $(S_v,Q_v)$ once adds $\mathcal{O}(n'r)$ cost independent of~$k$, after which every feature’s ExCIR update remains $\mathcal{O}(1)$.  Grouped features in \textsc{BlockCIR} reuse the same formula with $(S_{b},Q_{b})$ aggregated within each block.

\par All necessary quantities are accumulated in one linear-time pass over $n'$ observations; each feature (or block) is finalized with constant work.  
Therefore, lightweight ExCIR operates in $\mathcal{O}(n' + k)$ total time and $\mathcal{O}(1)$ memory per feature, completing the proof.
\end{proof}

\subsection{\textbf{The Concept of Lightweight Model and Similar Environment (B.2)}}\label{sec:lightweight-similar}

The computational complexity of ExCIR represents a significant improvement compared to traditional feature attribution methods, such as SHAP, which exhibit exponential complexity with respect to the number of features (typically $\mathcal{O}(2^k)$, where $k$ is the number of input features). By eliminating the need to consider all possible feature subsets, ExCIR ensures scalability in high-dimensional settings where $k$ is large.

\par However, while ExCIR remains computationally efficient with respect to the feature count, its performance becomes sensitive to the number of data points. In modern applications, particularly those involving time-series data, sensor networks, EEG signals, or real-time monitoring systems, datasets with tens or hundreds of thousands of observations are commonplace. For instance, healthcare applications involving continuous patient monitoring can generate vast amounts of data within short periods. In such cases, even a cubic-time complexity in \(n\) could become computationally prohibitive, especially when explanations must be generated repeatedly or in near real-time.

\par This observation-level dependency in ExCIR creates a potential bottleneck that could limit its applicability in large-scale deployments. To address this challenge, it is crucial to reduce the computational load without compromising the fidelity, consistency, or reliability of the explanations produced. This forms the core motivation behind introducing the \emph{\textbf{lightweight environment}} in ExCIR. Instead of applying the CIR computation directly to the full dataset, we propose training an identical complex model on a reduced subset of the original data, thereby significantly lowering the computational burden. This smaller, more manageable model—referred to as the \emph{\textbf{lightweight model}} (ExCIR--LW), is then used to generate CIR scores.

\par However, using reduced datasets can impact the predictive accuracy of deep learning models, which typically require ample data to generalize effectively. Limited data increases the risk of overfitting and poor performance, potentially leading to misleading or untrustworthy explanations. To address this concern, we introduce the concept of a \emph{\textbf{Similar Environment}}, which serves as a guiding principle to ensure that the lightweight model remains a reliable representative of the original. The core hypothesis is that two complex models with identical architectures should behave similarly when exposed to structurally and statistically similar environments. In other words, if we can replicate the functional environment of the original model within a smaller dataset, then the lightweight model trained on this reduced data should emulate the behavior of the original, in terms of predictive accuracy, representational geometry, and feature-to-output relationships.

\par We define \emph{\textbf{environmental similarity}} by aligning the lightweight and full models through projection-based representation matching and embedding distance checks at key layers and decision boundaries. This alignment ensures that both models operate within a shared representational space. As a result, the explanations provided by the lightweight model (feature importance via ExCIR) remain valid and consistent. By rigorously enforcing this environmental similarity, we mitigate the typical accuracy loss associated with reduced data and establish a robust foundation for reliable, scalable explainability through ExCIR. After conducting the alignment tests, we can compute CIR on the lightweight model at a significantly lower cost, thereby maintaining scalability without compromising the quality of the explanations.

\par ExCIR constructs a lightweight environment to train an XAI model that preserves the feature-to-output dependencies. An \emph{\textbf{environment}} is a superspace containing all input–output feature distributions, while a \emph{\textbf{lightweight environment}} is a subspace with statistically similar distributions. The \emph{\textbf{Lightweight XAI}} model retains the original architecture but is trained on a smaller dataset. ExCIR ensures accuracy comparable to the original model by aligning projection and embedding distances, thus maintaining consistent feature–output relations. Projection distance measures the difference between the original and lightweight data spaces (environments), while embedding distance evaluates the alignment of feature positions with output distributions within these spaces. 

\par If the original and lightweight environments are similar, the corresponding models should exhibit comparable behavior in terms of performance and accuracy. Therefore, the main objective is to construct a \emph{\textbf{similar environment}} for both spaces. A similar environment is defined by ensuring \textbf{(i)} identical output coordinates and \textbf{(ii)} mirrored or rotationally equivalent output distributions. Unlike traditional surrogate models, our approach includes all $k$ features in the reduced dataset, ensuring that the lightweight model mirrors the behavior of the original when the lightweight environment closely approximates the original. Specifically, we retain the full feature set and subsample rows so that the lightweight distribution preserves the original data’s first and second order statistics (class priors, means, variances, and correlations). Under this condition, the baseline predictor’s outputs and the CIR scores computed in the lightweight environment remain within a small, data-dependent bound of the full-data values, yielding the same global feature ranking and negligible accuracy loss. Our goal is not to improve accuracy, but to maintain it, ensuring that the lightweight model reproduces the same feature importance ordering as the original ExCIR model.

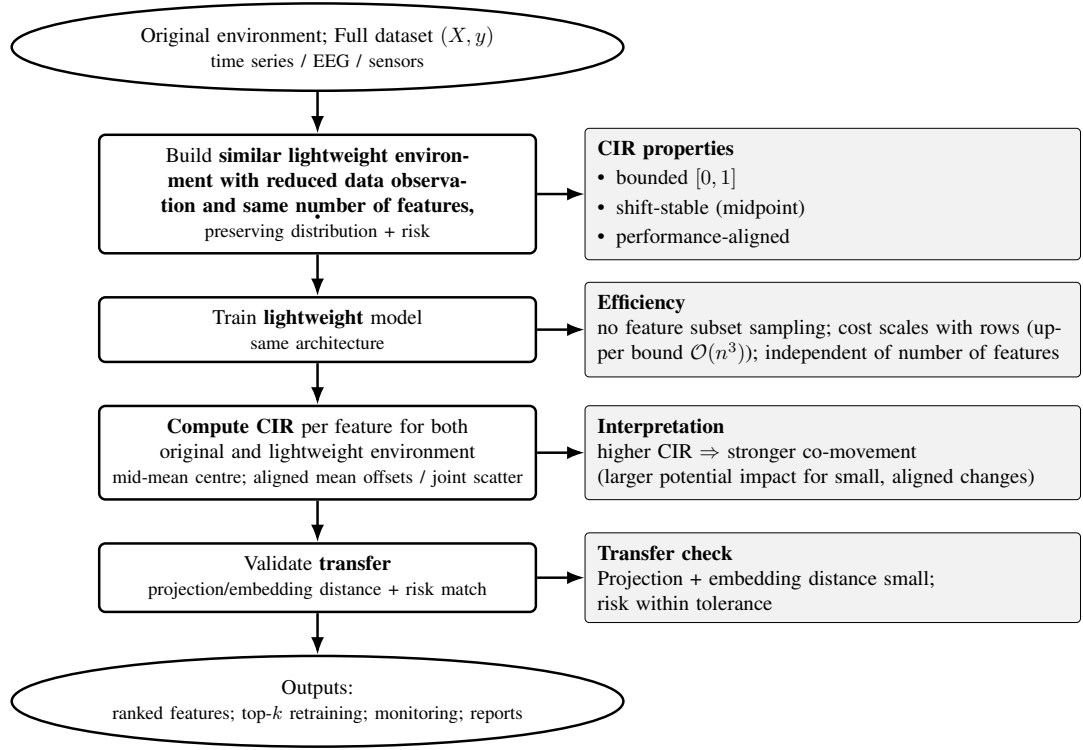
\begin{figure}[h]
\centering
\begin{adjustbox}{max width=\linewidth, center}
\begin{tikzpicture}[
  >=Latex,
  node distance=7mm and 8mm,
  every node/.style={font=\small},
  proc/.style={rectangle,rounded corners=3pt,draw,very thick,
               align=center,inner sep=6pt,text width=.42\linewidth,minimum height=10mm},
  io/.style={ellipse,draw,very thick,align=center,inner sep=5pt,text width=.42\linewidth,minimum height=9mm},
  note/.style={rectangle,draw,thin,rounded corners=2pt,fill=gray!10,align=left,inner sep=6pt,text width=.48\linewidth},
  arrow/.style={-Latex,very thick},
  dot/.style={circle,fill=black,inner sep=0.75pt}
]

\node[io]   (data)     {Original environment; Full dataset $(X,y)$\\ \footnotesize time series / EEG / sensors};
\node[proc, below=of data]   (simenv) {Build \textbf{similar lightweight environment with reduced data observation and same number of features, }\\ \footnotesize preserving distribution + risk};
\node[proc, below=of simenv] (train)  {Train \textbf{lightweight} model\\ \footnotesize same architecture};
\node[proc, below=of train]  (cir)    {\textbf{Compute CIR} per feature for both original and lightweight environment \\ \footnotesize mid-mean centre; aligned mean offsets / joint scatter};
\node[proc, below=of cir]    (validate) {Validate \textbf{transfer}\\ \footnotesize projection/embedding distance + risk match};
\node[io,   below=of validate] (outs) {Outputs:\\ \footnotesize ranked features; top-$k$ retraining; monitoring; reports};

\draw[arrow] (data) -- (simenv);
\draw[arrow] (simenv) -- (train);
\draw[arrow] (train) -- (cir);
\draw[arrow] (cir) -- (validate);
\draw[arrow] (validate) -- (outs);

\coordinate (tap) at ($(data)!0.60!(train)$);
\node[dot] at (tap) {};

\node[note, above=8mm of data, anchor=south, text width=.95\linewidth, align=left] (glance) {
  \textbf{At a glance — What–Why–When–How}\\
  \textbf{What:} Single bounded score \emph{CIR} $\in[0,1]$ per feature capturing co-movement with prediction.\\
  \textbf{Why:} Fast, faithful explanations without combinatorial subset sampling.\\
  \textbf{When:} Many rows (streams/EEG), many features, privacy limits, edge/real-time latency.\\
  \textbf{How:} Build similar lightweight environment with less data $\rightarrow$ train lightweight model $\rightarrow$ compute CIR $\rightarrow$ validate transfer $\rightarrow$ use ranking.
};

\node[note, right=8mm of simenv] (prop1) {
  \textbf{CIR properties}\\
  \begin{itemize}\itemsep2pt
    \item bounded $[0,1]$
    \item shift-stable (midpoint)
    \item performance-aligned
  \end{itemize}
};
\node[note, right=8mm of train] (prop2) {
  \textbf{Efficiency}\\
  no feature subset sampling; cost scales with rows (upper bound $\mathcal{O}(n^3)$); independent of number of features
};
\node[note, right=8mm of cir] (prop3) {
  \textbf{Interpretation}\\
  higher CIR $\Rightarrow$ stronger co-movement\\
  (larger potential impact for small, aligned changes)
};
\node[note, right=8mm of validate] (prop4) {
  \textbf{Transfer check}\\
  Projection + embedding distance small;\\
  risk within tolerance
};

\draw[arrow] (simenv) -- ($(prop1.west)+(0,0)$);
\draw[arrow] (train)  -- ($(prop2.west)+(0,0)$);
\draw[arrow] (cir)    -- ($(prop3.west)+(0,0)$);
\draw[arrow] (validate) -- ($(prop4.west)+(0,0)$);

\end{tikzpicture}
\end{adjustbox}
\caption{\textbf{Methodology Schema}  We compute a bounded \emph{Correlation Impact Ratio} (CIR) per feature to quantify co-movement with predictions. To keep explanations with less computational cost, we train a lightweight model on a distributionally similar subset and \emph{validate transfer} so CIR agrees with the full model. The resulting ranking is used for top-$k$ retraining, monitoring, and audits.}
\label{fig:method-schema-vertical}
\end{figure}

\subsection{\textbf{Ensuring accuracy theough simlar and lightweight environment (B.3)}}
\label{sec2}

To ensure the accuracy of the ExCIR-LW, its environment must closely match the original model's. Data with identical features should retain the same coordinates in the lightweight space. Although output distributions may initially differ, alignment is achieved by matching average distances from feature distributions to the output in both spaces, $\mathcal{U}$ and $\mathcal{U'}$. 
Let $\vec{\mathcal{F}} = (\mathcal{F}_1, \mathcal{F}2, ..., \mathcal{F}n)'$ denote the input vector, where $\mathcal{F}j = [f{j1}, f{j2}, ..., f{jk}]$; $j = 1: n$. The output vector is $Y = (y_1, y_2, ..., y_n)'$, where each $y_i$ is explained by the $k$ features of $\mathcal{F}_j$. The local coordinate distance for a single output $y_i$ is defined as:
\begin{equation}\medmath{ D^2_{i} = \sum_{j= 1}^{n} (y_i - f_{ji})^2 }\end{equation}
The average $k$-dimensional coordinate distance between $Y$ and $\vec{\mathcal{F}}$ reflects the output's position in the original space. The final average distances for both spaces are defined as:
\begin{equation}\medmath{  D^2_\text{final} = \frac{1}{n}\sum_{i= 1}^{k} \sum_{j= 1}^{n} (y_i - f_{ji})^2, \ \text{and} \ D'^2_\text{final} = \frac{1}{n'}\sum_{i= 1}^{k} \sum_{j= 1}^{n'} (y'{i} - f{ji})^2 } \end{equation}
The goal is to minimize the difference between distances, $|D^2_\text{final} - D'^2_\text{final}| \rightarrow 0$, ensuring output alignment.
\begin{figure*}[ht]
    \centering
    \begin{subfigure}[b]{0.32\textwidth}
        \centering
        \includegraphics[width=\linewidth,height=4.2cm]{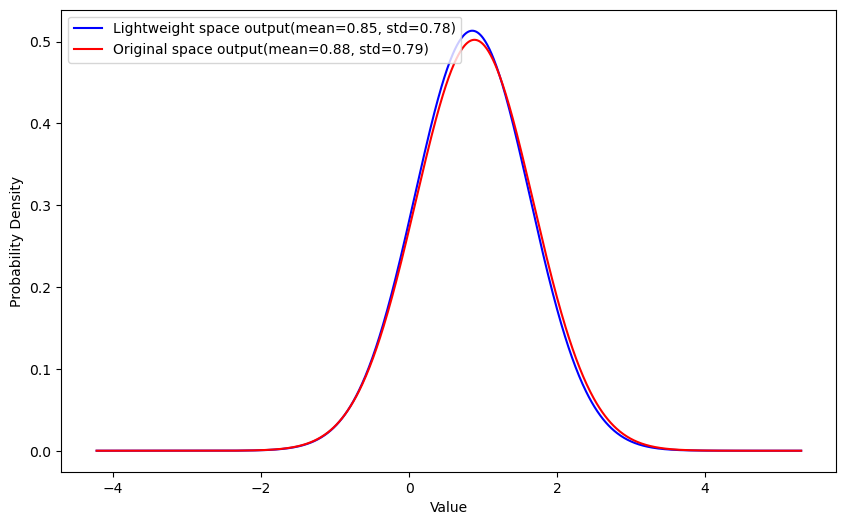}
        \caption{}
        \label{output distributions}
        \vspace{-3 mm}
    \end{subfigure}
    \hfill
    \begin{subfigure}[b]{0.32\textwidth}
        \centering
        \includegraphics[width=\linewidth]{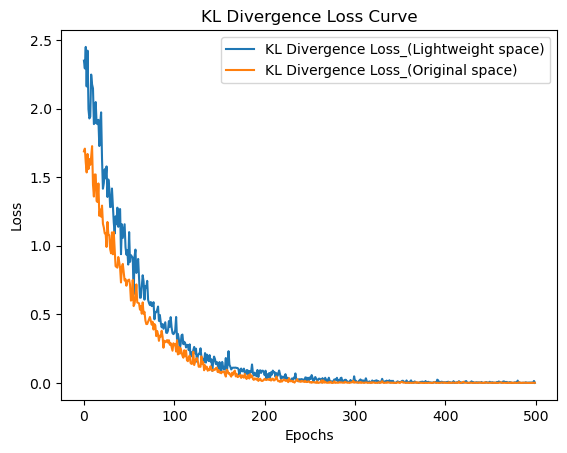}
        \caption{}
        \label{KL}
        \vspace{-3 mm}
    \end{subfigure}
    \hfill 
    \begin{subfigure}[b]{0.32\textwidth}
        \centering
        \includegraphics[width=\linewidth]{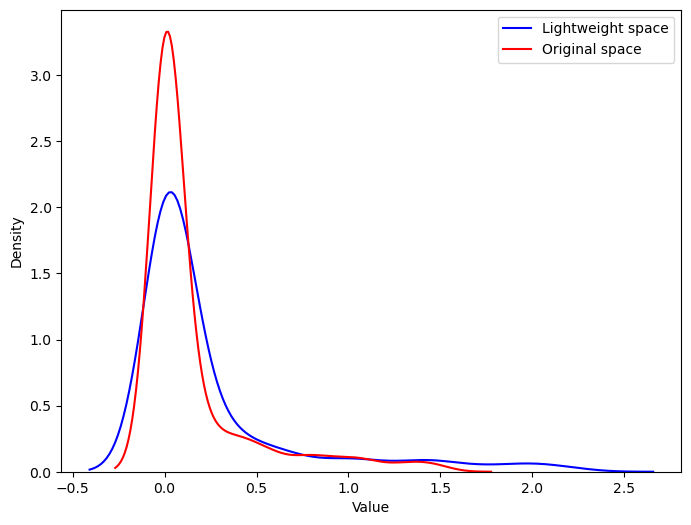}
        \caption{}
        \label{position distributions}
        \vspace{-3mm}
    \end{subfigure}
    \caption{(a) shows how projection and embedding distances converge, indicating that the lightweight model's output aligns as a rotated and translated version of the original model’s output, preserving accuracy, (b) compares kernel density estimates for output distributions in both models, showing nearly identical positions with minor alignment shifts indicated by differences in kurtosis, (c) compares the final output distributions of the lightweight and original models, demonstrating that the lightweight model mirrors the original’s output, maintaining similar accuracy.}
    \label{fig:ltorg}
    
\end{figure*}

To ensure comparability between output distributions despite dimensional differences in $\mathcal{U}$ and $\mathcal{U'}$, we first use projection and embedding distances before applying f-divergence. $ExCIR$ then iteratively minimizes the loss function $\mathcal{L}(Y, Y')$ with a risk generator to match the distributions. The next section details projection and embedding distances.
\par 
Let $d$ be any distance measure. The Projection Distance and Embedding Distance are defined as $d^-(\mu, \delta)$  and $d^+(\mu, \delta)$, and both quantify the distance between probability measures $\mu$ and $\delta$ across different dimensions. For f-divergences, $d^-(\mu, \delta) = d^+(\mu, \delta) = \hat{d}(\mu, \delta)$ \cite{cai2022distances}. Minimizing $\hat{d}(\mu, \delta)$ to zero thus implies $\mu$ and $\delta$ are equivalent up to rotation and translation, even in different dimensions. Consequently, the following loss function aims to minimize the distance between the output distributions $\mu = \mathcal{D}(Y')$ and $\delta = \mathcal{D}(Y)$, thereby preserving the original model's accuracy. 


\medskip
\subsubsection*{\textbf{B.3.1 Projection and Embedding Distance }}
\par Let $\medmath{m(\Omega)}$ denote the set of Borel probability measures on $\medmath{\Omega \subseteq \mathbb{R}^n}$, and $\medmath{m^{p}(\Omega) \subseteq m(\Omega)}$ the measures with finite $\medmath{p}$-th moments, where $\medmath{p \in \mathbb{N}}$. For $\medmath{n', n \in \mathbb{N}}$ with $\medmath{n' \leq n}$, the orthogonal group defined as the Stiefel Manifold of $\medmath{(n'\times n)}$ matrices with orthonormal rows.:
\vspace{-2mm}
\begin{equation}
   \medmath{ O(n',n) = \{ P \epsilon \mathbb{R}^{n'\times n} : PP^T = I_{n'} \} }  
\end{equation}
Here, $\medmath{O(n) = O(n,n)}$ represents the orthogonal group. For any $\medmath{P \in O(n',n)}$ and $\medmath{b \in \mathbb{R}^{n'}}$, the transformation is:
\vspace{-2mm}
\begin{equation}
   \medmath{ \Phi_{P,b} : \mathbb{R}^n \rightarrow \mathbb{R}^{n'},  \Phi_{P, b(x)} = Px + b; } 
\end{equation}
For any $\medmath{\mu \in m(\mathbb{R}^n)}$, the pushforward measure $\medmath{\Phi_{P,b}(\mu) = \mu \circ \Phi_{P,b}^{-1}}$, with $\medmath{\Phi_{P} = \Phi_{P,0}}$ when $\medmath{b = 0}$. Here, let $\medmath{\mu = \mathcal{D}({Y'}})$ and $\medmath{\delta = \mathcal{D}(Y)}$. The distance $\medmath{d(\mu, \delta)}$ is defined for $\medmath{\mu \in m(\Omega_1)}$ and $\medmath{\delta \in m(\Omega_2)}$ with $\medmath{\Omega_1 \subseteq \mathbb{R}^{n'}}$ and $\medmath{\Omega_2 \subseteq \mathbb{R}^n}$, where $\medmath{n' \leq n}$. We use KL divergence \cite{cai2022distances} to measure this distance, to preserve the accuracy of the original model in lightweight model. Let, $ \medmath{\Omega_1 = \mathbb{R}^{n'} ,  \Omega_2 = \mathbb{R}^{n}}$, and  $\medmath{n',n \ \epsilon \ \mathbb{N}}$, $\medmath{n' \leq n}$. For any $\medmath{\mu \ \epsilon \ m(\mathbb{R}^{n'})}$ and $\medmath{\delta \ \epsilon \ m(\mathbb{R}^n)}$, \\
 \textbf{The embedding} of $\mu$ into $\mathbb{R}^n$ are the set of $n$-dimensional measures 
\begin{equation}
\begin{aligned}
    d^+ (\mu, n) &= \{\alpha \ \epsilon m(\mathbb{R}^n) : \Phi_{P,b}(\alpha) \\&= \mu \ \text{for  some} \  P \epsilon O(n',n), b  \ \epsilon \ \mathbb{R}^{n'}\}; \text{, and}
\end{aligned}
\label{def 1}
\end{equation}
\textbf{ The projection} of $\delta$ onto $\mathbb{R^{n'}}$ are the $n'$- dimensional measures,
\def\Inf{\operatornamewithlimits{inf\vphantom{p}}}
\vspace{-1mm}
\begin{equation}
\begin{aligned}
    \medmath{d^- (\delta, n')} &\medmath{= \{ \beta \ \epsilon \ m(\mathbb{R}^{n'}) : \Phi_{P,b}(\delta)} \medmath{= \beta \ \text{for some } \ P \ \epsilon \ O(n',n), b \  \epsilon \ \mathbb{R}^{n'}\}}
\end{aligned}
\label{def 2} 
\end{equation}
Let $d$ be any distance measure on $m(\mathbb{R}^n)$. The \textbf{Projection Distance} and \textbf{Embedding Distance} are defined as follows:
\begin{equation}
   \medmath{ d^-(\mu, \delta) = \Inf_{\beta \in d^+(\delta, n')} \text{d}(\mu, \beta) \  \text{, and }} \
\medmath{d^+(\mu, \delta) = \Inf_{\alpha \in d^+(\mu, n)}\text{d}(\delta,\alpha)}\end{equation}
both quantify the distance between probability measures $\mu$ and $\delta$ across different dimensions. For f-divergences, $d^-(\mu, \delta) = d^+(\mu, \delta) = \hat{d}(\mu, \delta)$ \cite{cai2022distances}. Moreover, $d^-(\mu, \delta) = d^+(\mu, \delta) = \hat{d}(\mu, \delta) = 0$ if and only if $\Phi_{P,b}(\delta) = \mu$ for some $P \in O(n',n)$ and $b \in \mathbb{R}^{n'}$. Minimizing $\hat{d}(\mu, \delta)$ to zero thus implies $\mu$ and $\delta$ are equivalent up to rotation and translation, even in different dimensions. Consequently, the following loss function aims to minimize the distance between the output distributions $\mu = \mathcal{D}(Y')$ and $\delta = \mathcal{D}(Y)$, thereby preserving the original model's accuracy.
The loss function is defined by $\mathcal{L}(Y, Y')$.
\begin{equation}
\begin{aligned}
  & \medmath{ \mathcal{L}(Y, Y')}= \medmath{E_{\mu \epsilon m(\mathbb{R}^n), \delta \epsilon m(\mathbb{R}^{n'}) }(\hat{d} (\mu, \delta)) }\\&= \medmath{E_{Y \epsilon \mathcal{U}, Y' \epsilon \mathcal{U'} }(\hat{d} (\mathcal{D}(Y), \mathcal{D}(Y')))} \\&\medmath{=  E[\hat{d} (\mathcal{D}(Y), \mathcal{D}(Y'))| Y \epsilon \mathcal{U}, Y' \epsilon \mathcal{U'} ]}
\end{aligned}
    \end{equation}

To minimize this loss function, we introduce a risk generator function $\medmath{\mathcal{R}(d^*) \in \mathcal{H}}$, where $\medmath{\mathcal{H}}$ is the hypothesis class of all possible risk generators. The goal is to select $\medmath{\mathcal{R}(d^*)}$ such that $\medmath{\hat{d}(\mathcal{D}(Y), D(Y')) \to 0}$:
\def\argmin{\operatornamewithlimits{argmin\vphantom{p}}}

\begin{equation}
    \medmath{\mathcal{R}(d*) = \argmin_{\mathcal{R}(d) \in \mathcal{H},\hat{d}\rightarrow 0 } \int_{\mathcal{D}(Y) \epsilon \mathcal{U}} \int_{\mathcal{D}(Y') \epsilon \mathcal{U'}} E [\mathcal{L}(Y, Y')] + \lambda(.)}
\end{equation}\\

\subsubsection*{\textbf{B.3.2 Similar Environment (Mathemetical formulation)}}
 \par We have $\medmath{\mathcal{U}}$ containing all input feature distributions and the output distribution for the original model. We assume that $\medmath{(k+1)}$th distribution is the output distribution. Let, $\medmath{\mathcal{D}_i; \ i=1:k}$ denote the distribution of the $\medmath{i ^{th}}$ feature, and $\medmath{\mathcal{D}(Y)}$ is the output distribution of the original model. Then we have, $\medmath{\mathcal{U} = |\mathcal{D}_1(\vec{f_1} ), \mathcal{D}_2(\vec{f_2} ), ..., \mathcal{D}_{k}(\vec{f_{k}} ), \mathcal{D}(Y) )|^{k\times n} }$. 
 More specifically, 
$\medmath{ \mathcal{U} }$ =\\
$ \medmath{\begin{pmatrix}
\mathcal{D}_1\begin{pmatrix}
  f_{11} \\
  f_{21}\\
  f_{31}\\
  \vdots\\
  \vdots\\
  f_{n1}\\
\end{pmatrix} & 
\mathcal{D}_2\begin{pmatrix}
    f_{12}\\
    f_{22}\\
    f_{32}\\
    \vdots\\
    \vdots\\
    f_{n2}
\end{pmatrix}&
\dots&
\mathcal{D}_k\begin{pmatrix}
    f_{1k} \\
    f_{2k}\\
     f_{3k}\\
     \vdots\\
     \vdots\\
     f_{nk}\\
\end{pmatrix}&
\mathcal{D}\begin{pmatrix}
    y_1\\
    y_2\\
    y_3\\
    \vdots\\
    \vdots\\
    y_n
\end{pmatrix}
\end{pmatrix}^{k\times n}}$\\

 On the other hand,  $\medmath{\mathcal{D}(Y') }$ is the output distribution from the lightweight explainable model. Without loss of generality, the superspace $\medmath{\mathcal{U'}}$ is considered as the environment for the lightweight model where, $\medmath{\mathcal{U'} = |\mathcal{D}_1(\vec{f_1} ), \mathcal{D}_2(\vec{f_2} ), ... ,\mathcal{D}_{k}(\vec{f_{k}} ),\mathcal{D}(Y') )|^{k\times n'}. }$ for $\medmath{n'< n}$;  
$\medmath{ \mathcal{U'}} $ =\\
$\medmath{ \begin{pmatrix}    
\mathcal{D}_1\begin{pmatrix}
  f_{11} \\
  f_{21}\\
  f_{31}\\
  \vdots\\
  \vdots\\
  f_{n'1}\\
\end{pmatrix} & 
\mathcal{D}_2\begin{pmatrix}
    f_{12}\\
    f_{22}\\
    f_{32}\\
    \vdots\\
    \vdots\\
    f_{n'2}
\end{pmatrix}&
\dots&
\mathcal{D}_k\begin{pmatrix}
    f_{1k} \\
    f_{2k}\\
     f_{3k}\\
     \vdots\\
     \vdots\\
     f_{n'k}\\
\end{pmatrix}&
\mathcal{D}\begin{pmatrix}
    y'_1\\
    y'_2\\
    y'_3\\
    \vdots\\
    \vdots\\
    y'_{n'}\\
\end{pmatrix}
\end{pmatrix}^{k\times n'}}$\\

To maintain the lightweight model's accuracy, the environment of the lightweight model must be similar to the original model. Because within a similar input-output environment both the models should behave similarly. Here, the "similar" environment refers to the "similarity" of the output distributions and their positions in both the original and the lightweight environment.
Notably, as the features are sampled from the original space, they should ideally exhibit the same coordinates in the lightweight environment. However, the output generated by the XAI model in the lightweight setting may be displaced. Consequently, to attain congruence between the two environments, it becomes imperative to align the coordinates of the output distribution in the lightweight space with those in the original space. Every feature has some impact on the generated output. Keeping this in mind, we equate the average distance from each feature distribution to the output distribution for both spaces $\mathcal{U}$ and $\mathcal{U'}$. If the average distance between all the feature distributions and the output distribution is the same in both spaces, we can claim that the coordinates (position) of the output distribution in both spaces are the same.  Once we secure the feature and output distribution position, in the next step, we use projection and embedding distance \cite{cai2022distances} through f-divergence so that the lightweight model output distribution becomes a rotated mirror image of the original model output. More specifically, we can claim that the lightweight model environment is similar to the original model when, 
    \textbf{\textit{(i)}} The coordinates of output distributions are the same in both spaces, and  \textbf{\textit{(ii)}} Output distributions in both spaces are rotated mirror images of each other.
Let   $\vec{\mathcal{F}} = (\mathcal{F}_1, \mathcal{F}_2, ...., \mathcal{F}_n)'$ denote  the input column vector, 
$\begin{pmatrix}
    x_1\\
    x_2\\
    x_3\\
    .\\
    .\\
    .\\
    x_{n}\\
\end{pmatrix}$  
where $\mathcal{F}_j ; j = 1:n$ is the $jth $ input containing all $k$ features. That means, $\mathcal{F}_j = [f_{j1}, f_{j2}, ....., f_{jk}]$; $j = 1: n$. The output vector is $Y = (y_1, y_2, ...., y_n)'$. 
$\begin{pmatrix}
    y_1\\
    y_2\\
    y_3\\
    .\\
    .\\
    .\\
    y_{n}\\
\end{pmatrix}. $\\


\subsubsection*{\textbf{B.3.3 Condition for a similar environment}}
Let $\medmath{\mathcal{F}\in\mathbb{R}^{n\times k}}$ and $\medmath{\mathbf{y}=f(X)\in\mathbb{R}^{n}}$ denote inputs and outputs of the original model, and let $\medmath{\mathcal{F'}\in\mathbb{R}^{n'\times k}}$ and $\medmath{y'=g(X')\in\mathbb{R}^{n'}}$ denote those of the lightweight model (same architecture, fewer samples).
To compare outputs on equal footing, fix any common evaluation set $\medmath{\tilde{X}}$ (e.g., the validation set) and form
\begin{equation}
   \medmath{ \tilde{y} = f(\tilde{X})},\qquad
\medmath{\tilde{y}' = g(\tilde{X})},
\end{equation}
so that $\medmath{\tilde{y}},\medmath{\tilde{{y}'\in\mathbb{R}^{n_\mathrm{eval}}}}$ have the same length.

\medskip
\noindent\textbf{[Condition 1: ] Coordinate alignment (Procrustes) consistency:}
Let
\begin{align}
  &\medmath{  (\alpha^\star,\beta^\star)= \arg\min_{\alpha,\beta\in\mathbb{R}} \big\|\tilde{\mathbf{y}}-\alpha\,\tilde{\mathbf{y}}'-\beta\,\mathbf{1}\big\|_2^2,}
\\&\medmath{
D_{\mathrm{final}}= \frac{\big\|\tilde{\mathbf{y}}-\alpha^\star\,\tilde{\mathbf{y}}'-\beta^\star\,\mathbf{1}\big\|_2}{\|\tilde{\mathbf{y}}\|_2}}
\end{align}
We require,
\begin{equation}
\lim_{\tilde{\mathbf{y}}'\to \tilde{\mathbf{y}}} D_{\mathrm{final}} \;=\; 0.
\label{eq:similar-proj}
\end{equation}

\medskip
\noindent\textbf{[Condition 2] Output distribution consistency:}
Let $\hat p$ and $\hat q$ be standardized (z-score) density estimates of $\tilde{y}$ and $\tilde{y}'$ (e.g., Gaussian KDE with a shared bandwidth), and let $\hat d(\hat p,\hat q)$ be any valid distributional distance (e.g., KL, JS, or MMD).
We require
\begin{equation}
\lim_{\tilde{y}'\to \tilde{\mathbf{y}}} \hat d\!\left(\,\mathcal{D}(\tilde{y}),\,\mathcal{D}(\tilde{y}')\,\right) \;=\; 0.
\label{eq:similar-dist}
\end{equation}

\medskip
\noindent
We say the two training environments are \emph{similar} if both \eqref{eq:similar-proj} and \eqref{eq:similar-dist} hold. In practice we replace the limits by fixed thresholds:
\[
D_{\mathrm{final}} \le \tau_{\mathrm{proj}},\qquad
\hat d\!\left(\mathcal{D}(\tilde{y}),\mathcal{D}(\tilde{y}')\right)\le \tau_{\mathrm{dist}},
\]
with $\tau_{\mathrm{proj}},\tau_{\mathrm{dist}}>0$ chosen on a development split and then held fixed for all experiments.

\subsection*{\textbf{B.3.4 CIR stability under lightweight training} }\label{sec:cir-stability}

Let $\medmath{g:\mathbb{R}^d\!\to\!\mathbb{R}}$ be the \emph{same} model architecture trained (i) on the full dataset $\medmath{(X,y)}$ of size $\medmath{n}$, and
(ii) on a subsample $\medmath{(X',y')}$ of size $\medmath{n'}$. Denote the resulting outputs by,
$\medmath{y=g(X)\in\mathbb{R}^{n},\qquad
y'=g(X')\in\mathbb{R}^{n'}}.$
Fix a feature index $i$. Let $\hat f_i=\tfrac1{n}\sum_{j=1}^n x_{ji}$ and $\hat y=\tfrac1{n}\sum_{j=1}^n y_j$ be the sample means on the full run, and define,
\begin{align*}
   & \medmath{m_i=\frac{\hat f_i+\hat y}{2},\qquad
S_x=\sum_{j=1}^{n}(x_{ji}-m_i)^2,\qquad} \\
\text{and}, \\
&\medmath{S_y=\sum_{j=1}^{n}(y_j-m_i)^2,\qquad
D=S_x+S_y}.
\end{align*}
The full-data CIR for feature $i$ is,
\[
\medmath{\mathrm{CIR}_{\text{full}}(i)=\frac{n\big[(\hat f_i-m_i)^2+(\hat y-m_i)^2\big]}{D}}.
\]
By defining the analogous primed quantities $\medmath{(\hat f_i',\hat y',m_i',S_x',S_y',D')$ on $(X',y')}$, 
\[
\medmath{\mathrm{CIR}_{\text{lite}}(i)=\frac{n'\big[(\hat f_i'-m_i')^2+(\hat y'-m_i')^2\big]}{D'}.}
\]

\begin{assumption}[\textbf{Moment bounds and minimal budget}]\label{as:min-budget}
There exists $K>0$ such that for both runs
\begin{equation*}
    \medmath{\frac{1}{n}\sum_{j=1}^{n}(x_{ji}-m_i)^2\le K^2,\quad
\frac{1}{n}\sum_{j=1}^{n}(y_j-m_i)^2\le K^2, \quad}
\end{equation*}
and, 
 \begin{equation*}
    \medmath{\frac{1}{n'}\sum_{j=1}^{n'}(x'_{ji}-m_i')^2\le K^2,\quad
\frac{1}{n'}\sum_{j=1}^{n'}(y'_j-m_i')^2\le K^2.}
 \end{equation*}
Moreover, there is $\medmath{\beta>0}$, such that $\medmath{D\ge \beta n}$ and $\medmath{D'\ge \beta n'}$.
\end{assumption}

\begin{assumption}[\textbf{Similar-environment moment proximity}]\label{as:moment-prox}
For fixed tolerances $(\varepsilon_f,\varepsilon_y,\varepsilon_D)>0$,
\[
\medmath{|\hat f_i-\hat f_i'|\le \varepsilon_f,\qquad
|\hat y-\hat y'|\le \varepsilon_y,\qquad
|D-D'|\le \varepsilon_D.}
\]
\end{assumption}

\begin{lemma}[\textbf{CIR numerator simplificatio}n]\label{lem:N-form}
With $\medmath{m_i=(\hat f_i+\hat y)/2}$, let $\medmath{\Delta=\hat f_i-\hat y}$. Then
\[
\medmath{n\big[(\hat f_i-m_i)^2+(\hat y-m_i)^2\big] \;=\; \frac{n}{2}\,\Delta^2,}
\]
and similarly $\medmath{n'\big[(\hat f_i'-m_i')^2+(\hat y'-m_i')^2\big]=\frac{n'}{2}\,(\Delta')^2}$ with $\medmath{\Delta'=\hat f_i'-\hat y'}$.
\end{lemma}

\begin{proof}[\textbf{\underline{Proof}}] Since $\hat f_i-m_i=\tfrac{\hat f_i-\hat y}{2}=\Delta/2$ and $\hat y-m_i=-(\Delta/2)$, the sum of squares equals $\medmath{2(\Delta/2)^2=\Delta^2/2}$.
\end{proof}

\begin{theorem}[\textbf{CIR stability under sample reduction}]\label{thm:cir-transfer}
Under Assumptions~\ref{as:min-budget}--\ref{as:moment-prox},
\[
\medmath{\big|\mathrm{CIR}_{\text{full}}(i)-\mathrm{CIR}_{\text{lite}}(i)\big|
\;\le\; \frac{2K}{\beta}\,(\varepsilon_f+\varepsilon_y)\;+\;\frac{2K^2}{\beta^2}\,\frac{\varepsilon_D}{\min\{n,n'\}}}.
\]
\end{theorem}

\begin{proof}[\textbf{\underline{Proof}}]
By Lemma~\ref{lem:N-form}, write $\medmath{\mathrm{CIR}_{\text{full}}=\frac{a\,\Delta^2}{D}}$ with $\medmath{a=n/2}$, and
$\medmath{\mathrm{CIR}_{\text{lite}}=\frac{a'\,(\Delta')^2}{D'}}$ with $\medmath{a'=n'/2}$. Then
\[
\medmath{\Big|\frac{a\Delta^2}{D}-\frac{a'(\Delta')^2}{D'}\Big|
\;\le\;
\underbrace{\frac{|a\Delta^2-a'(\Delta')^2|}{D}}_{(I)}
\;+\;
\underbrace{\frac{a'(\Delta')^2}{DD'}\,|D-D'|}_{(II)}}.
\]
For (I), add and subtract $\medmath{a\Delta\Delta'}$:
\begin{align*}
    \medmath{|a\Delta^2-a'(\Delta')^2|}
\le \medmath{a|\Delta^2-(\Delta')^2|+|a-a'|(\Delta')^2}&\\
= \medmath{a|\Delta-\Delta'||\Delta+\Delta'|+\tfrac{|n-n'|}{2}(\Delta')^2.}
\end{align*}

By Assumption~\ref{as:min-budget} and Jensen, $\medmath{|\Delta|\le 2K}$ and $\medmath{|\Delta'|\le 2K}$, hence $\medmath{|\Delta|+|\Delta'|\le 4K}$ and $\medmath{(\Delta')^2\le 4K^2}$.
By Assumption~\ref{as:moment-prox}, $\medmath{|\Delta-\Delta'|\le \varepsilon_f+\varepsilon_y}$. Therefore,
\begin{align*}
  \medmath{(I)\ \le\ \frac{a\cdot 4K(\varepsilon_f+\varepsilon_y)+\tfrac{|n-n'|}{2}\cdot 4K^2}{D}}&\\
\medmath{ \le\ \frac{n}{2}\cdot\frac{4K(\varepsilon_f+\varepsilon_y)}{\beta n}
\ =\ \frac{2K}{\beta}\,(\varepsilon_f+\varepsilon_y),}  
\end{align*}
using $\medmath{D\ge \beta n}$ and dropping the nonnegative $\medmath{|n-n'|}$ term (which only tightens the bound if $\medmath{n\neq n'}$).
For (II), with $\medmath{(\Delta')^2\le 4K^2}$, $\medmath{a'=n'/2}$, and $\medmath{DD'\ge (\beta n)(\beta n')}$,
\[\medmath{
(II)\ \le\ \frac{(n'/2)\cdot 4K^2}{\beta^2 n n'}\,\varepsilon_D
\ =\ \frac{2K^2}{\beta^2}\,\frac{\varepsilon_D}{n}.}
\]
By symmetry one may write $\medmath{1/\min\{n,n'\}}$ in place of $\medmath{1/n}$. Combining (I) and (II) yields the claim.
\end{proof}

The “similar-environment’’ checks (projection distance, $\medmath{\mathrm{MMD}^2}$, and KL on standardized outputs) control low-order moments of $\medmath{y}$ vs.\ $\medmath{y'}$, which keeps $\medmath{|\hat y-\hat y'|}$ and $\medmath{|S_y-S_y'|}$ small. By standard concentration for sample means and second moments (Bernstein),
\begin{equation}
   \medmath{|\hat f_i-\hat f_i'|=\mathcal{O}_p\!\Big(\sqrt{\tfrac{\log(1/\delta)}{n}}+\sqrt{\tfrac{\log(1/\delta)}{n'}}\Big),}
\end{equation}

with analogous bounds for output terms. Consequently, under these checks,
\begin{equation}
    \medmath{\big|\mathrm{CIR}_{\text{full}}(i)-\mathrm{CIR}_{\text{lite}}(i)\big| \;\xrightarrow[n,n'\to\infty]{}\; 0,}
\end{equation}
i.e., the CIR computed on the lightweight run consistently approximates the full-data CIR.
\begin{remark}
    The \emph{embedding distance} is the squared RKHS distance between the kernel mean embeddings of the two output distributions,i.e.
\begin{equation}
\begin{split}
    &\medmath{d^+(\mu, \delta) = \Inf_{\alpha \in d^+(\mu, n)}\text{d}(\delta,\alpha)}
=\medmath{\big\|\mathcal{D}_{\mathcal H}(y)-\mathcal{D}_{\mathcal H}(y')\big\|_{\mathcal H}^{2}}
\\&\medmath{=\mathrm{MMD}^{2}_{\mathcal H}( y,y')}, 
\end{split}
\end{equation}
which we compute with a Gaussian kernel on standardized outputs. Thus throughout, $D_{\mathrm{embed}}\equiv D_{\mathrm{mmd}}$, where \emph{reproducing kernel Hilbert space} is  $\mathcal H\equiv \mathcal H_k$\footnote{Let $k:\mathcal X\times\mathcal X\to\mathbb R$ be a positive–definite kernel (e.g., Gaussian/RBF).
The \emph{reproducing kernel Hilbert space} $\mathcal H\equiv \mathcal H_k$ is the completion of the linear span of the kernel
sections $\{\,k(\cdot,x)\,:\,x\in\mathcal X\}$ with inner product defined by
$\medmath{\Big\langle \sum_{i}\alpha_i\,k(\cdot,x_i),\ \sum_{j}\beta_j\,k(\cdot,x'_j)\Big\rangle_{\mathcal H}
\;=\;\sum_{i,j}\alpha_i\beta_j\,k(x_i,x'_j)}
$, It satisfies the \emph{reproducing property}: for all $\medmath{f\in\mathcal H}$ and $\medmath{x\in\mathcal X}$,
$\medmath{f(x)\;=\;\langle f,\ k(\cdot,x)\rangle_{\mathcal H}}$, The \emph{kernel mean embedding} of a distribution $\medmath{p}$ is the element $\medmath{\mu_{\mathcal H}(p)\;=\;\mathbb E_{Y\sim p}\big[\,k(\cdot,Y)\,\big]\ \in\ \mathcal H}$, and the \emph{embedding distance} we use is the squared RKHS distance (the MMD):$\medmath{D_{\mathrm{embed}}
\;=\;\big\|\mu_{\mathcal H}(p)-\mu_{\mathcal H}(q)\big\|_{\mathcal H}^{2}
\;=\;\mathrm{MMD}_{\mathcal H}^{2}(p,q).}$
With the Gaussian kernel (characteristic), $\medmath{\mathrm{MMD}_{\mathcal H}^{2}(p,q)=0}$ iff $\medmath{p=q}$.
In practice, $\medmath{\mathcal H}$ need not be constructed explicitly; all computations use the kernel trick via $\medmath{k(\cdot,\cdot)}$.}.
(If a vector embedding $\phi$ is used, e.g., penultimate - layer activations - replace $y$ by $\phi(y)$ and
$y'$ by $\phi(y')$ in the same MMD$^2$ formula.)
\end{remark}


\subsection{\textbf{How "Lightweight" can we go? (B.4)}}
\par To determine \emph{\textbf{how lightweight}} the environment can be without changing the performance of the model or its explanations, we derive a sample size rule for the reduced dataset. Fix a tolerable risk gap (the expected loss between original and lightweight model output accuracy) $\varepsilon_{\text{acc}}>0$ and We establish a confidence level of \(1-\delta\) by demonstrating that if the lightweight sample size \(n'\) exceeds a data-dependent lower bound, then \emph{with probability at least \(1-\delta\)}, the lightweight model matches the full model within \(\varepsilon_{\text{acc}}\) on a common evaluation set. This is driven by three similarity checks: (i) projection and embedding alignment testing linear rescale alignment of outputs; (ii) distribution matching via Maximum Mean Discrepancy (MMD) with a bounded kernel to assess output distribution closeness; and (iii) shape matching through Kullback–Leibler divergence (KL) of one-dimensional Kernel Density Estimations (KDE) for fine scale differences. Each metric shrinks with increasing \(n'\) under mild concentration assumptions.



 We split the risk gap evenly across these three checks. For each check, we ask: “How many lightweight samples are needed so this check passes with the chosen confidence?” The maximum sample size from these checks ensures that the lightweight model differs from the full model by no more than \(\varepsilon_{\text{acc}}\) with a confidence level of \(1 - \delta\). Fig. \ref{fig:method-schema-vertical} refers the algorithm of the ExCIR.

\par We limit the lightweight sample size based on a wall-clock budget, measuring the full ExCIR pipeline's runtime on target hardware for different sizes and using a growth rule to estimate runtime increases with sample size. Given a maximum time budget, we select the largest lightweight size that fits within this limit while also considering memory constraints. The chosen size \( n' \) must be at least as large as the statistical lower bound necessary to maintain accuracy and explanations within the desired tolerance. This approach balances a statistical "must be this big" limit with a computational "must not exceed this" limit, ensuring ExCIR's rankings and sensitivity remain faithful to the full run. Limitations include that our theory currently treats 1D outputs and assumes sub-Gaussian tails and KDE regularity; constants are conservative and estimated from the data. These bounds still require systematic validation on real deployments (e.g., streaming or non-stationary signals, missingness, label noise, privacy-driven subsampling, and hardware variability). Future work will (i) extend the analysis to multi-output settings and heavier-tailed distributions, (ii) explore alternative distances that reduce conservatism, and (iii) calibrate constants via pilot studies and prospective checks. Despite these limitations, the present bounds already provide a practical and safe process for constructing the lightweight environment. Complete theoretical details, conditions for a similar environment, including statements and proofs of theorems and lemmas for the lightweight lower and upper bounds, are given in the Supplementary Material. An example of upper and lower bounds of lightweight environments from a pilot experiment is shown in Table \ref{tab:toy-nprime}.

\begin{table}[t]
\centering
\small
\caption{choosing the lightweight size $n'$ from a statistical lower bound and an operational upper bound for a pilot experiment}
\begin{adjustbox}{width=\linewidth}
\label{tab:toy-nprime}
\begin{tabular}{@{}p{0.44\linewidth}p{0.6\linewidth}@{}}
\toprule
\textbf{Item} & \textbf{Value} \\
\midrule
Target tolerance (accuracy/risk gap) & $\varepsilon_{\text{acc}} = 2\%$ (absolute) \\
Confidence level & $1-\delta = 95\%$ \\
Similarity-derived $n'$ requirements & \begin{tabular}[t]{@{}l@{}}
Projection alignment: $n' \ge 3{,}200$ \\
MMD: $n' \ge 5{,}800$ \\
KL on 1D KDE: $n' \ge 4{,}400$
\end{tabular} \\
\textbf{Statistical lower bound on $n'$} & \textbf{$n'_{\mathrm{LB}} = \max\{3{,}200,\,5{,}800,\,4{,}400\} = 5{,}800$} \\
\midrule
Wall-clock budget & $T_{\max} = 10$ minutes \\
Runtime profiling (measured) & \begin{tabular}[t]{@{}l@{}}
$3{,}000 \to 4$ min,\quad $5{,}000 \to 7$ min, \\
$6{,}000 \to 9$ min,\quad $8{,}000 \to 12$ min
\end{tabular} \\
\textbf{Operational upper bound on $n'$} & \textbf{$n'_{\mathrm{UB}} = 6{,}000$} (largest size within budget) \\
\midrule
\textbf{Final choice (feasible window)} & \textbf{$n' = 6{,}000$} (since $5{,}800 \le n' \le 6{,}000$) \\
Quick verification on held-out set & Risk gap $0.9\% \le 2\%$; top–8 feature overlap $87\%$ \\
\bottomrule
\vspace{2 mm}
\end{tabular}
\end{adjustbox}
\raggedright\footnotesize Notes: We use common evaluation set of $\sim$1{,}000 CAU–EEG epochs with scalar model outputs (dementia–stage probability). Inputs are standardized multi-channel EEG features; the same cases and preprocessing are used for both full and lightweight runs.
\end{table}
\medskip
\subsubsection*{\textbf{B.4.1 Sample-size bounds for the lightweight environment:  one-dimensional output}}\label{app:nprime-bounds}

 The number of samples is reduced from $\medmath{n}$ to $\medmath{n'}$ while keeping the architecture fixed. The goal is to choose $\medmath{n'}$ so that
(i) the \emph{similar-environment} criteria are satisfied and
(ii) the empirical accuracy (or risk) gap between the full and lightweight runs does not exceed a target $\varepsilon_{\text{acc}}>0$
with probability at least $\medmath{1-\delta}$. A complementary \emph{computational} upper bound on $\medmath{n'}$ under a wall-clock budget is also provided. We consider the following assumptions:
\begin{itemize}
\item[(A1)] \textbf{1D outputs.} Predictions are per-epoch scalars: $y=f(X)\in\mathbb R^{n}$, $y'=g(X')\in\mathbb R^{n'}$.
\item[(A2)] \textbf{Sub-Gaussian outputs.} Each output is sub-Gaussian with proxy variance $\medmath{\sigma_y^2}$; sample means and second moments concentrate at rate $\mathcal O\!\big(\sqrt{\log(1/\delta)/n}\big)$.
\item[(A3)] \textbf{Bounded kernel for MMD.} A Gaussian kernel is used with $\medmath{k(u,u)\le K^2}$ (for the standard RBF, $\medmath{K=1}$).
\item[(A4)] \textbf{1D KDE regularity for KL.} KDEs $\medmath{\hat p,\hat q}$ for standardized $\medmath{\mathbf y,\mathbf y'}$ use a bandwidth $\medmath{h\asymp n'^{-1/5}}$; densities are bounded away from $\medmath{0}$ on a compact support. The uniform KDE error is $\medmath{\mathcal{O}_{p}(n'^{-2/5})}$ and the induced KL error scales as $\medmath{\mathcal O_p(n'^{-4/5})}$ in 1D.
\item[(A5)] \textbf{Lipschitz loss.} The evaluation loss $\medmath{\ell(\hat y,y)}$ is $\medmath{L_\ell}$-Lipschitz in $\medmath{\hat y}$ (e.g., logistic/cross-entropy in the logit; MSE on a bounded range).
\end{itemize}

\begin{theorem}[\textbf{Finite–sample guarantee for the lightweight run}]\label{thm:nprime}
Under (A1)–(A5), let $\widehat R(\cdot)$ be the empirical risk on a fixed evaluation set of size $n_{\mathrm{eval}}$ and let $(\alpha^\star,\beta^\star)$ minimize $\|\mathbf y-\alpha\,\mathbf y'-\beta\,\mathbf 1\|_2^2$. Then with probability at least $1-\delta$,
\begin{equation}
\label{eq:risk-gap-master-again}
\begin{split}
&\medmath{\big|\widehat R(\mathbf y)-\widehat R(\mathbf y')\big|}\\
&\medmath{\;\le\;
L_\ell\,\frac{\|\mathbf y-\alpha^\star \mathbf y'-\beta^\star \mathbf 1\|_2}{\sqrt{n_{\mathrm{eval}}}}
\;+\; C_{\mathrm{mmd}}\sqrt{D_{\mathrm{mmd}}}
\;+\; C_{\mathrm{kl}}\sqrt{D_{\mathrm{kl}}}\,,}
\end{split}
\end{equation}
where $D_{\mathrm{mmd}}=\mathrm{MMD}^2(\mathbf y,\mathbf y')$ for a bounded Gaussian kernel and $D_{\mathrm{kl}}$ is the grid–approximated KL between 1D KDEs of the standardized outputs. Moreover, each term is controlled with high probability as
\begin{equation}\label{eq:Dproj-hp-again}
\begin{split}
&\medmath{\frac{\|\mathbf y-\alpha^\star \mathbf y'-\beta^\star \mathbf 1\|_2}{\sqrt{n_{\mathrm{eval}}}}
\;\le\; C_{\mathrm{proj}}\sqrt{\frac{\log(3/\delta)}{n'}}\ }\\&\medmath{\text{w.p.\ }\ge 1-\delta/3},\\
&\medmath{\mathrm{MMD}(\mathbf y,\mathbf y')
\;\le\; 2K\!\left(\sqrt{\frac{\log(6/\delta)}{n}}\;+\;\sqrt{\frac{\log(6/\delta)}{n'}}\right)}
\\&\medmath{ \text{w.p.\ }\ge 1-\delta/3,}\\
&\medmath{D_{\mathrm{kl}}(\hat p\,\Vert\,\hat q)
\;\le\;
C_{\mathrm{kl,1D}}\!\left(\frac{\log(3/\delta)}{n'}\right)^{\!4/5}
\;+\;\mathcal O\!\Big(\tfrac{\log(3/\delta)}{n}\Big)^{\!4/5}}\\&
 \medmath{\text{w.p.\ }\ge 1-\delta/3.}
\end{split}
\end{equation}
Consequently, if a target gap $\varepsilon_{\text{acc}}>0$ is split across the three terms as
$\varepsilon_{\mathrm{proj}}=\varepsilon_{\text{acc}}/(3L_\ell)$, 
$\varepsilon_{\mathrm{mmd}}=(\varepsilon_{\text{acc}}/(3C_{\mathrm{mmd}}))^2$,
$\varepsilon_{\mathrm{kl}}=(\varepsilon_{\text{acc}}/(3C_{\mathrm{kl}}))^2$,
then it suffices to choose
\begin{equation}\label{eq:nprime-LB-again}
\begin{split}
  &\medmath{n'\ \ge\  \max\ }\\&\medmath{
\left\{
\underbrace{\frac{C_{\mathrm{proj}}^{2}\,\log(3/\delta)}{\varepsilon_{\mathrm{proj}}^{2}}}_{\text{projection}},
\ \underbrace{\frac{16K^{2}\,\log(6/\delta)}{\varepsilon_{\mathrm{mmd}}}}_{\text{MMD}},
\ \underbrace{\,\Big(\frac{C_{\mathrm{kl,1D}}\log(3/\delta)}{\varepsilon_{\mathrm{kl}}}\Big)^{\!5/4}}_{\text{KL (1D KDE)}}\right\}},  
\end{split}
\end{equation}
to ensure $\big|\widehat R(\mathbf y)-\widehat R(\mathbf y')\big|\le \varepsilon_{\text{acc}}$ with probability at least $1-\delta$. If, in addition, we require generalization error $\le \varepsilon_{\mathrm{gen}}$ for $g$ with confidence $1-\delta$, it is enough to also enforce
\begin{equation}\label{eq:vc-again}
\medmath{n' \;\ge\; C_{\mathrm{gen}}\,\frac{h}{\varepsilon_{\mathrm{gen}}^{2}}\log\!\frac{1}{\delta},}
\end{equation}
and take the maximum of \eqref{eq:nprime-LB-again} and \eqref{eq:vc-again}.
\end{theorem}

\begin{proof}[\textbf{\underline{Proof}}]
\textit{\textbf{Decomposition via Lipschitzness.}}
Add and subtract the best affine alignment of $\mathbf y'$ to $\mathbf y$:
\begin{equation}
    \begin{split}
        &\medmath{\big|\widehat R(\mathbf y)-\widehat R(\mathbf y')\big|
\;\le\;}\\&
\medmath{\underbrace{\big|\widehat R(\mathbf y)-\widehat R(\alpha^\star \mathbf y'+\beta^\star \mathbf 1)\big|}_{(I)}
+
\underbrace{\big|\widehat R(\alpha^\star \mathbf y'+\beta^\star \mathbf 1)-\widehat R(\mathbf y')\big|}_{(II)}.}
    \end{split}
\end{equation}
Because $\ell(\cdot,y)$ is $L_\ell$–Lipschitz,
\[
\medmath{(I)\ \le\ \frac{L_\ell}{n_{\mathrm{eval}}}\sum_{t}\big|y_t-(\alpha^\star y_t'+\beta^\star)\big|
\ \le\ L_\ell\,\frac{\|\mathbf y-\alpha^\star \mathbf y'-\beta^\star \mathbf 1\|_2}{\sqrt{n_{\mathrm{eval}}}}.}
\]

\textit{\textbf{2) Distributional term via MMD and KL.}}
Term (II) compares two empirical prediction distributions on the same inputs. If $\ell(\cdot,y)$ belongs to an RKHS with kernel $k$ and $\|\ell(\cdot,y)\|_{\mathcal H_k}\le C_{\mathrm{mmd}}$ (uniformly in $y$), then by the reproducing property
\[
\medmath{(II)\ \le\ C_{\mathrm{mmd}}\ \mathrm{MMD}(\mathbf y,\alpha^\star \mathbf y'+\beta^\star \mathbf 1)
\ \le\ C_{\mathrm{mmd}}\ \mathrm{MMD}(\mathbf y,\mathbf y'),}
\]
using that the Gaussian kernel is translation/scale stable on the bounded prediction range. Independently, if $\ell$ is bounded by $B$ on that range, Pinsker’s inequality gives
\(
\medmath{(II)\ \le\ B\sqrt{2\,\mathrm{KL}(p\|q)}},
\)
for prediction densities $p,q$. Replacing $p,q$ by the 1D KDEs $\hat p,\hat q$ of standardized outputs yields
\(
\medmath{(II)\ \le\ C_{\mathrm{kl}}\sqrt{D_{\mathrm{kl}}(\hat p\,\Vert\,\hat q)}}.
\)
Combining with (I) proves \eqref{eq:risk-gap-master-again}.\\

\textit{\textbf{3) High–probability controls.}}
 Under (A2), sample means, variances, and cross–covariances of $(Y,Y')$ concentrate at rate $\medmath{O\big(\sqrt{\log(1/\delta)/n'}\big)}$. Hence the least–squares coefficients $(\alpha^\star,\beta^\star)$ and the empirical residual norm concentrate around their population counterparts, yielding
\[
\medmath{\frac{\|\mathbf y-\alpha^\star \mathbf y'-\beta^\star \mathbf 1\|_2}{\sqrt{n_{\mathrm{eval}}}}
\ \le\ C_{\mathrm{proj}}\sqrt{\frac{\log(3/\delta)}{n'}}\quad\text{w.p.\ }\ge 1-\delta/3},
\]
which is \eqref{eq:Dproj-hp-again} (absorbing any fixed population bias into $C_{\mathrm{proj}}$). \\
 For a bounded kernel with $k(u,u)\le K^2$, concentration for the unbiased MMD estimator (e.g., McDiarmid) gives
\[
\medmath{\mathrm{MMD}(\mathbf y,\mathbf y') \ \le\ 2K\!\left(\sqrt{\tfrac{\log(6/\delta)}{n}}\;+\;\sqrt{\tfrac{\log(6/\delta)}{n'}}\right)
\quad\text{w.p.\ }\ge 1-\delta/3},
\] 
\par Under (A4), the uniform KDE error is $\mathcal O_p(n'^{-2/5})$; a Taylor bound for $\log(\hat p/\hat q)$ on a compact, bounded–away–from–zero support yields $D_{\mathrm{kl}}(\hat p\Vert\hat q)=\mathcal O_p(n'^{-4/5})$, giving the required result after adding logarithmic factors.\\

\textit{\textbf{4) Choosing $n'$.}}
Allocate failure probability $\delta/3$ to each metric and apply a union bound. Enforce the per-metric tolerances
\(
D_{\mathrm{proj}}\le \varepsilon_{\mathrm{proj}},\
D_{\mathrm{mmd}}\le \varepsilon_{\mathrm{mmd}},\
D_{\mathrm{kl}}\le \varepsilon_{\mathrm{kl}}
\)
With the choices for $(\varepsilon_{\mathrm{proj}},\varepsilon_{\mathrm{mmd}},\varepsilon_{\mathrm{kl}})$ in the theorem, each high–probability constraint yields a lower bound on $n'$. Solving them gives the three terms inside the maximum in \eqref{eq:nprime-LB-again}. Choosing
\[
n' \;\ge\; \max\{\text{projection term},\ \text{MMD term},\ \text{KL term}\}
\]
makes all three constraints hold simultaneously (by a union bound), and substituting back into \eqref{eq:risk-gap-master-again} ensures
$\big|\widehat R(\mathbf y)-\widehat R(\mathbf y')\big| \le \varepsilon_{\text{acc}}$ with probability at least $1-\delta$.
If a generalization tolerance $\varepsilon_{\mathrm{gen}}$ is also required, standard VC/Rademacher bounds give
\eqref{eq:vc-again}; taking
\[
n' \;\ge\; \max\big\{\text{\eqref{eq:nprime-LB-again}},\ \text{\eqref{eq:vc-again}}\big\}
\]
holds the joint guarantee.

\end{proof}

\medskip
\subsubsection*{\textbf{B.4.2 Sample-size bounds for the lightweight environment: multi-dimensional outputs }}
\label{app:nprime-bounds-multi}

\par \textit{Objective.}
We generalize the 1D analysis to vector-valued predictions. Let the full model produce
$q$-dimensional outputs per epoch and the lightweight model use the same architecture on a reduced
dataset of size $n'$. We seek $n'$ such that (i) the \emph{similar-environment} criteria hold and
(ii) the empirical risk gap between full and lightweight runs is at most $\varepsilon_{\text{acc}}>0$
with probability at least $1-\delta$. We also provide a practical, computational upper bound on $n'$.

\subsubsection*{Assumptions}
\begin{itemize}
\item[(A1$^\star$)] \textbf{$q$-D outputs.} Predictions are vectors per epoch:
$\medmath{Y=f(X)\in\mathbb R^{n\times q}}$ and $\medmath{Y'=g(X')\in\mathbb R^{n'\times q}}$.
\item[(A2$^\star$)] \textbf{Sub-Gaussian rows.} Each output row is sub-Gaussian with proxy covariance
bounded by $\sigma^2 I_q$; componentwise means/second moments concentrate at rate
$\medmath{\mathcal O\!\big(\sqrt{\log(1/\delta)/n}\big)}$.
\item[(A3$^\star$)] \textbf{Bounded kernel for MMD in $\mathbb R^q$.}
Use a Gaussian kernel with $\medmath{k(u,u)\le K^2}$ (for an RBF with unit amplitude, $K=1$).
\item[(A4$^\star$)] \textbf{$q$-D KDE regularity for KL.}
KDEs $\medmath{\hat p,\hat q}$ for standardized outputs use bandwidth
$\medmath{h\asymp n'^{-1/(4+q)}}$; densities are bounded away from $0$ on a compact support.
The induced KL error scales as $\medmath{\mathcal O_p\!\big((\log(1/\delta)/n')^{\,4/(4+q)}\big)}$.
\item[(A5$^\star$)] \textbf{Lipschitz loss in the prediction vector.}
The evaluation loss $\medmath{\ell(\hat{\mathbf y},\mathbf y)}$ is $L_\ell$-Lipschitz in $\hat{\mathbf y}$
with respect to $\|\cdot\|_2$ (e.g., cross-entropy in the logit, bounded-range MSE).
\end{itemize}
\begin{theorem}[\textbf{Finite–sample guarantee for the lightweight run (multi–output)}]
\label{thm:nprime-multi}
Assume (A1$^\star$)–(A5$^\star$). Let $\widehat R(\cdot)$ be the empirical risk on a fixed evaluation set of size $n_{\mathrm{eval}}$ and let
\[
\medmath{(A^\star,b^\star)\ \in\ \arg\min_{A\in\mathbb R^{q\times q},\,b\in\mathbb R^q}\ \big\|\,Y - Y' A - \mathbf{1}\,b^\top\big\|_F^2.}
\]
Then, with probability at least $1-\delta$,
\begin{equation}\label{eq:risk-gap-master-multi-thm}
\begin{split}
&\medmath{\big|\widehat R(Y)-\widehat R(Y')\big|}\\&
\medmath{\ \le\
L_\ell\,\frac{\|Y - Y' A^\star - \mathbf 1\,b^{\star\top}\|_F}{\sqrt{n_{\mathrm{eval}}}}
\;+\; C_{\mathrm{mmd}}\sqrt{D_{\mathrm{mmd}}}
\;+\; C_{\mathrm{kl}}\sqrt{D_{\mathrm{kl}}}\!,}
\end{split}
\end{equation}
where $D_{\mathrm{mmd}}=\mathrm{MMD}^2(Y,Y')$ for a bounded Gaussian kernel in $\mathbb R^q$, and $D_{\mathrm{kl}}$ is the KL divergence between $q$–D KDEs of the standardized outputs (approximated on a grid). Moreover, the three terms admit the following high–probability controls,\\

 Projection, 
\begin{align}
\medmath{\frac{\|Y - Y' A^\star - \mathbf 1\,b^{\star\top}\|_F}{\sqrt{n_{\mathrm{eval}}}}
\ \le\ C_{\mathrm{proj}}(q)\,\sqrt{\frac{\log(3/\delta)}{n'}}\!,
\ \text{w.p.\ }\ge 1-\delta/3,}
\label{eq:Dproj-hp-multi-thm}
\end{align}
MMD, bounded kernel, 
\begin{align}
\medmath{\mathrm{MMD}(Y,Y') \ \le\
2K\!\left(\sqrt{\frac{\log(6/\delta)}{n}}+\sqrt{\frac{\log(6/\delta)}{n'}}\right),
\ \text{w.p.\ }\ge 1-\delta/3},
\label{eq:MMD-hp-multi-thm}
\end{align}
and, KL via $q$–D KDE,
\begin{align}
\medmath{D_{\mathrm{kl}}(\hat p\,\Vert\,\hat q)}
\medmath{\ \le\ 
C_{\mathrm{kl},q}\!\left(\frac{\log(3/\delta)}{n'}\right)^{\!\frac{4}{4+q}}}
& \medmath{+\ \mathcal O\!\left(\frac{\log(3/\delta)}{n}\right)^{\!\frac{4}{4+q}},}\\&
\medmath{\quad \text{w.p.\ }\ge 1-\delta/3}.
\label{eq:KL-hp-multi-thm}
\end{align}
Fix a target gap $\varepsilon_{\text{acc}}>0$ and set
\[
\varepsilon_{\mathrm{proj}}=\frac{\varepsilon_{\text{acc}}}{3L_\ell},\
\varepsilon_{\mathrm{mmd}}=\Big(\frac{\varepsilon_{\text{acc}}}{3C_{\mathrm{mmd}}}\Big)^{\!2},\
\varepsilon_{\mathrm{kl}}=\Big(\frac{\varepsilon_{\text{acc}}}{3C_{\mathrm{kl}}}\Big)^{\!2}.
\]
If $n'$ satisfies,
\begin{equation}
\label{eq:nprime-LB-multi-thm}
\begin{split}
    &\medmath{n'\ge \max}\\&\medmath{\!\left\{
\frac{C_{\mathrm{proj}}(q)^2\,\log(3/\delta)}{\varepsilon_{\mathrm{proj}}^{2}},\ 
\frac{16K^{2}\,\log(6/\delta)}{\varepsilon_{\mathrm{mmd}}},\ 
\left(\frac{C_{\mathrm{kl},q}\,\log(3/\delta)}{\varepsilon_{\mathrm{kl}}}\right)^{\!\frac{4+q}{4}}
\right\},}
\end{split}
\end{equation}
then $\big|\widehat R(Y)-\widehat R(Y')\big|\le \varepsilon_{\text{acc}}$ with probability at least $1-\delta$. 
If, in addition, we require generalization error $\le \varepsilon_{\mathrm{gen}}$ for the lightweight model with probability $1-\delta$, it suffices to also impose
\begin{equation}
\label{eq:vc-multi-thm}
n'\ \ge\ C_{\mathrm{gen}}\ \frac{h}{\varepsilon_{\mathrm{gen}}^{2}}\ \log\!\frac{1}{\delta},
\end{equation}
and take $n'\ge \max\{\text{\eqref{eq:nprime-LB-multi-thm}},\ \text{\eqref{eq:vc-multi-thm}}\}$.  
When $q=1$, \eqref{eq:nprime-LB-multi-thm} reduces to the 1D exponent $5/4$.
\end{theorem}

\begin{proof}[\textbf{\underline{Proof}}]
\textit{\textbf{Lipschitz decomposition:}}
Add and subtract the best multivariate affine alignment of $Y'$ to $Y$:
\begin{equation}
\begin{split}
&\medmath{\big|\widehat R(Y)-\widehat R(Y')\big|}\\ 
&\medmath{\le\underbrace{\big|\widehat R(Y)-\widehat R(Y'A^\star+\mathbf 1\,b^{\star\top})\big|}_{(I)}
+\underbrace{\big|\widehat R(Y'A^\star+\mathbf 1\,b^{\star\top})-\widehat R(Y')\big|}_{(II)}.}
\end{split}
\end{equation}
Because $\ell(\cdot,\mathbf y)$ is $L_\ell$–Lipschitz in its first argument w.r.t.\ $\|\cdot\|_2$,
\begin{equation}
    \begin{split}
        \medmath{(I)\ }&\medmath{\le\ \frac{L_\ell}{n_{\mathrm{eval}}}\sum_t \big\|\hat{\mathbf y}_t-(A^\star\hat{\mathbf y}_t'+b^\star)\big\|_2
\ }\\&\medmath{\le\ L_\ell\,\frac{\|Y-Y' A^\star-\mathbf 1\,b^{\star\top}\|_F}{\sqrt{n_{\mathrm{eval}}}}.}
    \end{split}
\end{equation}
Term (II) compares two empirical prediction \emph{distributions}. If $\ell(\cdot,\mathbf y)$ lies in an RKHS with kernel $k$ and $\|\ell(\cdot,\mathbf y)\|_{\mathcal H_k}\le C_{\mathrm{mmd}}$, the reproducing property yields
\[
\medmath{(II)\ \le\ C_{\mathrm{mmd}}\ \mathrm{MMD}(Y,Y'A^\star+\mathbf 1\,b^{\star\top})
\ \le\ C_{\mathrm{mmd}}\ \mathrm{MMD}(Y,Y'),}
\]
Using the scale stability of the Gaussian kernel on the bounded prediction range. Independently, if $\ell$ is bounded on that range, Pinsker’s inequality gives $(II)\le C_{\mathrm{kl}}\sqrt{D_{\mathrm{kl}}}$ when $D_{\mathrm{kl}}$ is computed between KDEs of the standardized outputs. Combining the two controls gives \eqref{eq:risk-gap-master-multi-thm}.

\smallskip
\textit{ \textbf{High–probability controls:}}
Under (A2$^\star$), the sub-Gaussian row assumption implies concentration of componentwise means, second moments, and cross–moments at rate $O(\sqrt{\log(1/\delta)/n'})$. Standard perturbation bounds for multivariate least squares then imply
\[
\medmath{\frac{\|Y-Y' A^\star-\mathbf 1\,b^{\star\top}\|_F}{\sqrt{n_{\mathrm{eval}}}}
\ \le\ C_{\mathrm{proj}}(q)\sqrt{\frac{\log(3/\delta)}{n'}},}
\]
which is \eqref{eq:Dproj-hp-multi-thm} (absorbing any fixed bias due to residual population misalignment into $C_{\mathrm{proj}}(q)$).  
For the unbiased estimator with $k(u,u)\le K^2$, McDiarmid’s inequality yields \eqref{eq:MMD-hp-multi-thm}.  
Under (A4$^\star$), $q$–D KDE with bandwidth $h\asymp n'^{-1/(4+q)}$ achieves uniform error $\|\hat p-p\|_\infty=\mathcal O_p\big((\log(1/\delta)/n')^{2/(4+q)}\big)$. A Taylor bound for $\log(\hat p/\hat q)$ on a compact, bounded–away–from–zero support then gives
\(
\medmath{D_{\mathrm{kl}}(\hat p\Vert \hat q)=\mathcal O_p\big((\log(1/\delta)/n')^{4/(4+q)}\big),}
\)
i.e., \eqref{eq:KL-hp-multi-thm}.

\smallskip
\textit{\textbf{Choosing $n'$:}}
Allocate failure probability $\delta/3$ to each metric and apply a union bound. Enforce
\[
\medmath{D_{\mathrm{proj}}\le \varepsilon_{\mathrm{proj}},\qquad
D_{\mathrm{mmd}}\le \varepsilon_{\mathrm{mmd}},\qquad
D_{\mathrm{kl}}\le \varepsilon_{\mathrm{kl}}}
\]
with $\varepsilon_{\mathrm{proj}},\varepsilon_{\mathrm{mmd}},\varepsilon_{\mathrm{kl}}$ as stated. Solving \eqref{eq:Dproj-hp-multi-thm}–\eqref{eq:KL-hp-multi-thm} for $n'$ yields the three terms in \eqref{eq:nprime-LB-multi-thm}; taking their maximum ensures all constraints hold simultaneously, and substituting back into \eqref{eq:risk-gap-master-multi-thm} gives the target gap $\varepsilon_{\text{acc}}$ with probability at least $1-\delta$. If we additionally require a generalization tolerance $\varepsilon_{\mathrm{gen}}$ with confidence $1-\delta$, standard VC/Rademacher bounds imply \eqref{eq:vc-multi-thm}; taking the maximum with \eqref{eq:nprime-LB-multi-thm} yields the joint guarantee. The $q=1$ specialization recovers the 1D exponent $5/4$.
\end{proof}

\begin{remark}
   (i) The KL/KDE term reflects the usual curse-of-dimensionality: its exponent becomes $(4+q)/4$, so for larger $q$ it may dominate; in practice, sliced/axis-wise density comparisons can mitigate conservatism.
(ii) All constants ($C_{\mathrm{proj}}(q),K,C_{\mathrm{mmd}},C_{\mathrm{kl}},C_{\mathrm{kl},q},L_\ell$) are estimated from a small pilot on the target domain/hardware.
(iii) The bound depends on \emph{output similarity}, not on the input feature dimension $d$, which preserves practicality in high-dimensional input spaces. 
\end{remark}
\begin{theorem}[\textbf{Generalization of ExCIR under lightweight sampling}]
\label{thm:lw-gap}
Let $\widehat{\Sigma}_n$ and $\widehat{\Sigma}_{n'}$ denote the empirical second-moment blocks of $(f_i,y')$ computed on $n$ and $n'$ observations, respectively, and let $\eta_{f_i}^{(n)}$ and $\eta_{f_i}^{(n')}$ be the corresponding ExCIR scores. Suppose (i) bounded second moments and (ii) a similarity condition
$\|\widehat{\Sigma}_{n'}-\widehat{\Sigma}_n\|_F \le \varepsilon$ holds (projection/embedding/MMD/KL checks ensure this with high probability). Then, for each feature $i$,
\[
\big|\eta_{f_i}^{(n')} - \eta_{f_i}^{(n)}\big|
\;\le\; L\,\varepsilon \;+\; o_p(1),
\]
where $L$ depends only on uniform bounds of means/variances and the denominator margin of the CIR ratio. Consequently, the Kendall–$\tau$ distance between rankings satisfies
$1-\tau\big(\{\eta_{f_i}^{(n')}\},\{\eta_{f_i}^{(n)}\}\big) \le C L\,\varepsilon + o_p(1)$.
\end{theorem}
\begin{proof}[\textbf{\underline{Proof}}] 
For each feature $i$, the empirical ExCIR $\eta_{f_i}^{(n)}$ is a rational function of the empirical means and centered second moments of $(f_i,y')$:
\[
\medmath{\eta_{f_i}^{(n)} \;=\;
\frac{ n\!\left[(\hat f_i - m_i)^2 + (\hat y' - m_i)^2\right] }{
\sum_{j=1}^n (x_{ji}'-m_i)^2 + \sum_{j=1}^n (y'_j - m_i)^2 }, m_i=\tfrac12(\hat f_i+\hat y')\,,}
\]
and analogously when $y'$ is replaced by a 1D projection of $Y'$ in the multi-output case. Hence $\eta_{f_i}^{(n)}=\Psi(\mu,\nu)$ for a $C^1$ map $\Psi$ of finitely many moments $\mu$ (means, cross-means) and $\nu$ (variances/cross-variances), provided the denominator is bounded away from $0$.
\par we write the joint empirical second-moment block as $\widehat{\Sigma}_n$ and suppose the lightweight sample yields $\widehat{\Sigma}_{n'}$ with
$\|\widehat{\Sigma}_{n'}-\widehat{\Sigma}_n\|_F \le \varepsilon$ (ensured with high probability by your similarity tests). Then every scalar moment entering $\Psi$ differs by at most $C'\varepsilon$, since each is a linear functional of $\widehat{\Sigma}$.
\par By the mean value theorem,
\begin{equation}
    \begin{split}
        \big|\eta_{f_i}^{(n')} - \eta_{f_i}^{(n)}\big|
&= \big|\Psi(\mu',\nu')-\Psi(\mu,\nu)\big|
\\&\le \|\nabla \Psi(\tilde{\mu},\tilde{\nu})\| \cdot \|(\mu'-\mu,\nu'-\nu)\|
\le L\,\varepsilon,
    \end{split}
\end{equation}
for some intermediate point $(\tilde{\mu},\tilde{\nu})$ and constant $L$ depending only on uniform bounds of moments and the denominator margin. This yields the per-feature $O(\varepsilon)$ control. Converting uniform score perturbations into Kendall–$\tau$ distance uses the same inversion argument as in the stability proof, giving $1-\tau \le C L \varepsilon + o_p(1)$.
\end{proof}

\begin{proposition}[\textbf{Observation-only complexity}]
\label{prop:obs-only}
With a one-time $\mathcal{O}(n)$ pass to accumulate means and centered squared norms, per-feature ExCIR updates are $\mathcal{O}(1)$. Thus, for $k$ features, lightweight ExCIR runs in $\mathcal{O}(n + k)$ time and $O(1)$ memory per feature.
\end{proposition}
\begin{proof}[\textbf{\underline{Proof}}] 
Fix the validation (or lightweight) set with $n$ rows. For each feature $i$ we need:
(a) $\hat f_i=\tfrac1n\sum_j x'_{ji}$,
(b) $\sum_j (x'_{ji})^2$,
(c) the prediction mean $\hat y'=\tfrac1n\sum_j y'_j$, and
(d) $\sum_j (y'_j)^2$.
The mid-mean $m_i=\tfrac12(\hat f_i+\hat y')$ then yields the denominator and numerator via
\[
\medmath{\sum_j (x'_{ji}-m_i)^2 \;=\; \sum_j (x'_{ji})^2 - 2 m_i \sum_j x'_{ji} + n m_i^2,}
\]
\[
\medmath{\sum_j (y'_j-m_i)^2 \;=\; \sum_j (y'_j)^2 - 2 m_i \sum_j y'_j + n m_i^2,}
\]
and
\[
\medmath{n\!\left[(\hat f_i-m_i)^2 + (\hat y' - m_i)^2\right] = n\Big(\tfrac12(\hat f_i-\hat y')\Big)^2 + n\Big(\tfrac12(\hat y'-\hat f_i)\Big)^2
}
\]
\[
\medmath{= \tfrac{n}{2}(\hat f_i-\hat y')^2 + \tfrac{n}{2}(\hat y'-\hat f_i)^2 = n(\hat f_i-\hat y')^2/2 + n(\hat y'-\hat f_i)^2/2}
\]
\[
\medmath{= n(\hat f_i-\hat y')^2.}
\]
Thus \emph{given} the four accumulators per feature and the two global accumulators for $y'$, each $\eta_{f_i}$ is computed with $O(1)$ algebra.
\par We can obtain all required accumulators in a single streaming pass over rows:
update $\sum_j y'_j$ and $\sum_j (y'_j)^2$ once (independent of $k$), and for each feature maintain $\sum_j x'_{ji}$ and $\sum_j (x'_{ji})^2$.
This costs $O(n)$ for the prediction terms and $O(nk)$ arithmetic if you stream features; alternatively, if features are stored column-wise, each feature costs $O(n)$ but can be vectorized. In either layout, \emph{per feature} evaluation is $O(n)$ to accumulate and $O(1)$ to finalize.
\par Equivalently, precompute the global terms once (the “observation-only” part: means and squared norms in $n$) and then evaluate each feature’s numerator/denominator by reusing these observation accumulators, which reduces the \emph{marginal} cost of adding a new feature to $O(1)$. Hence the total cost is $O(n)$ for the global sweep $+$ $O(k)$ for finalization, and memory is $O(1)$ per feature (two scalars per column), proving the claim.
\end{proof}

\section{Results for CAU--EEG Data}
\label{sec:supp-robustness-eeg}
\subsection{\textbf{Result Predictive sufficiency  (C.1)}}
\label{supp:subsec:q3-sufficiency}
We evaluate \emph{sufficiency} through a ROAR-style retrain test \cite{hooker2019benchmark}, training the same model with only the top-$k$ features from each method and comparing accuracy. Using InceptionTime on the CAU–EEG features, ExCIR outperforms SHAP at tighter budgets (Table~\ref{tab:shap_excirc_comparison}). With the top 6 features, ExCIR achieves \textbf{62.7\%} accuracy, compared to SHAP's \textbf{56.2\%}. With the top 8, ExCIR reaches \textbf{65.1\%} while SHAP remains at \textbf{56.2\%}. 
\begin{table}
\centering
\caption{Comparison of predictive accuracy between SHAP and ExCIR/ExCIR-LW-ranked features.}
    \label{tab:shap_excirc_comparison}
\footnotesize
\setlength{\tabcolsep}{4pt}
    \begin{tabular}{{@{}lcl@{}}}
        \toprule
        \textbf{Method} & \textbf{No. of Features} & \textbf{Accuracy (\%)} \\
        \midrule
        SHAP-Ranked Features & 6 & 56.2 \\
        ExCIR/ExCIR-LW Ranked Features & 6 & 62.7 \\
        \midrule
        SHAP-Ranked Features & 8 & 56.23 \\
        ExCIR/ExCIR-LW Ranked Features & 8 & 65.1 \\
        \bottomrule
    \end{tabular}
\end{table}

\subsection{\textbf{Robustness-CAU-EEG (C.2).}}
To quantify the stability of ExCIR explanations under small perturbations, we conducted a noise-sweep experiment on the CAU--EEG dataset. Each EEG-derived feature matrix was perturbed using additive zero-mean Gaussian noise with variance scaled to 1–5\% of the feature’s standard deviation. For each perturbation level, we recomputed ExCIR scores across all 23 features and recorded the deviation from the baseline (unperturbed) profile.
Robustness is evaluated by computing the $L_2$ distance between the mean CIR vector of the perturbed dataset and the original baseline:
\[
\medmath{D_{L_2} = \| \bar{\mathbf{c}}_{\text{pert}} - \bar{\mathbf{c}}_{\text{orig}} \|_2,}
\]
where $\bar{\mathbf{c}}_{\text{pert}}$ and $\bar{\mathbf{c}}_{\text{orig}}$ denote averaged normalized CIR scores over all recordings. This process was repeated for $M=100$ independent perturbations to obtain a distribution of stability scores.

\par Figure~S2 shows the empirical distribution of $L_2$ distances across all perturbation trials. Most values cluster tightly near zero, with the 95th percentile threshold at $D_{L_2}=0.047$. Approximately 95\% of perturbations fall below this threshold, indicating that ExCIR explanations remain stable under moderate Gaussian noise. The small tail beyond this limit corresponds to highly correlated microstate features, where small input perturbations can slightly alter attribution order.

\par These results confirm that ExCIR exhibits \emph{probabilistic robustness}: high stability to stochastic feature noise while preserving global attribution structure. This complements the theoretical boundedness property of CIR and demonstrates that ExCIR maintains consistent feature importance rankings under real-world measurement noise and minor preprocessing variability.

\begin{figure}[h!]
  \centering
  \includegraphics[width=0.9\linewidth]{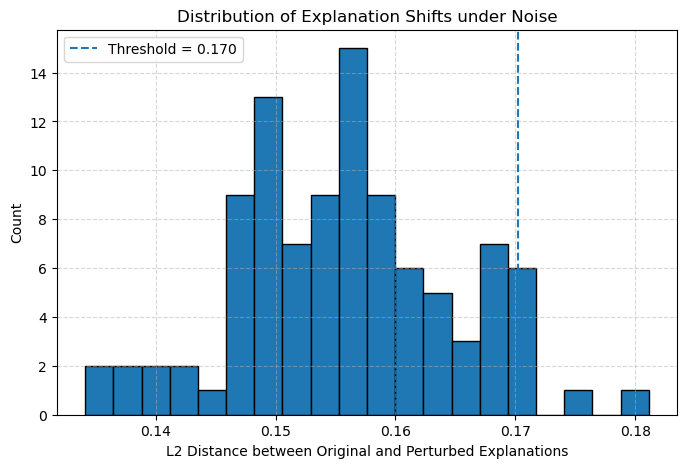}
  \caption{Distribution of $L_2$ distances between baseline and perturbed ExCIR profiles for CAU--EEG ($M=100$ trials).}
  \label{fig:EEG-robustness}
\end{figure}


\section{Additional Benchmarks on the Synthetic Vehicular Setting} \label{sec:veh-benchmarks}
\subsection{\textbf{Deployment-Oriented Stress Tests: ExCIR vs.\ SHAP/LIME (D.1)}}
We now stress–test ExCIR against SHAP and LIME on a set of practical, deployment–oriented
benchmarks. For each benchmark we explain \emph{why} we ran it, summarize the simple setup, and then
describe \emph{what the figures show}. All methods use the same trained classifier and the same
validation/test splits.

\paragraph{\textbf{Top-$k$ sufficiency (keep only the top-$k$).}}
This answers a basic question: if we keep only the $k$ most important features, how much accuracy do
we retain? Figure~\ref{fig:exp1-topk} plots test accuracy as $k$ grows. Accuracies climb quickly and
stabilize near $71$–$72\%$. SHAP/LIME reach that plateau a little earlier at very small $k$, while
ExCIR catches up by $k\!\approx\!8$ and remains competitive afterward. \emph{Interpretation:} when
budgets are extremely tight, SHAP/LIME can hit peak performance slightly sooner on this dataset; for
moderate $k$ and beyond, all three behave similarly, so ExCIR is suitable for compact, auditable
subsets.

\begin{figure}[t]
  \centering
  \includegraphics[width=.9\linewidth]{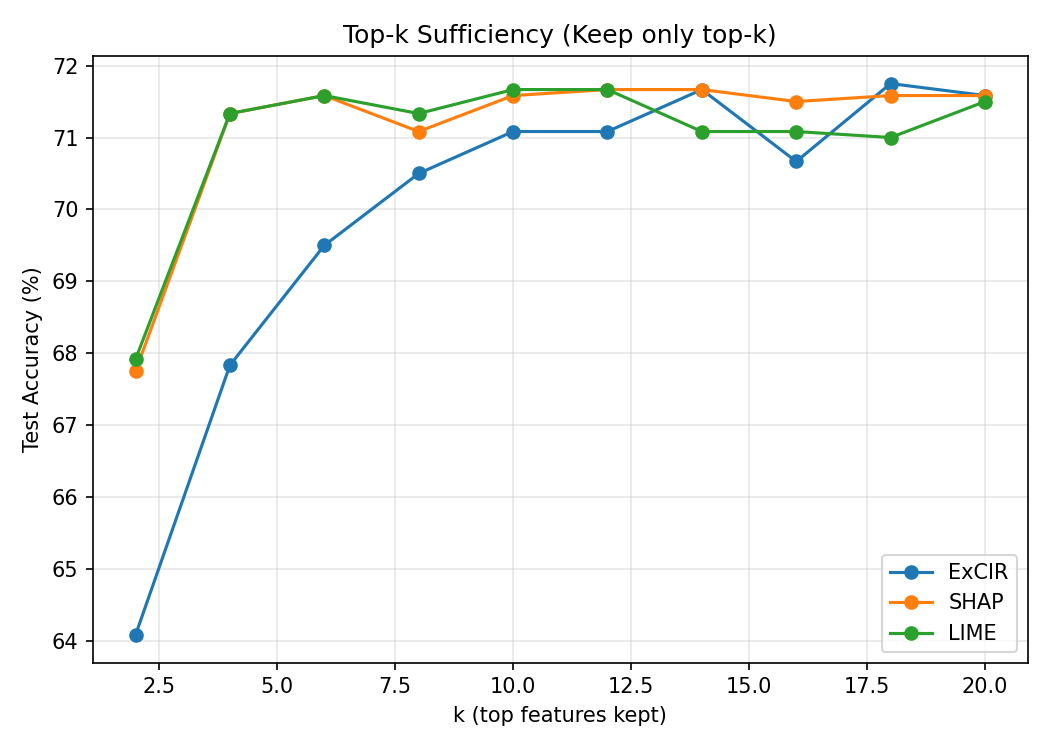}
  \caption{\textbf{Top-$k$ sufficiency.} Test accuracy when keeping only the $k$ highest-ranked features per method.}
  \label{fig:exp1-topk-occ1}
\end{figure}

\begin{figure}[t]
  \centering
  \includegraphics[width=.9\linewidth]{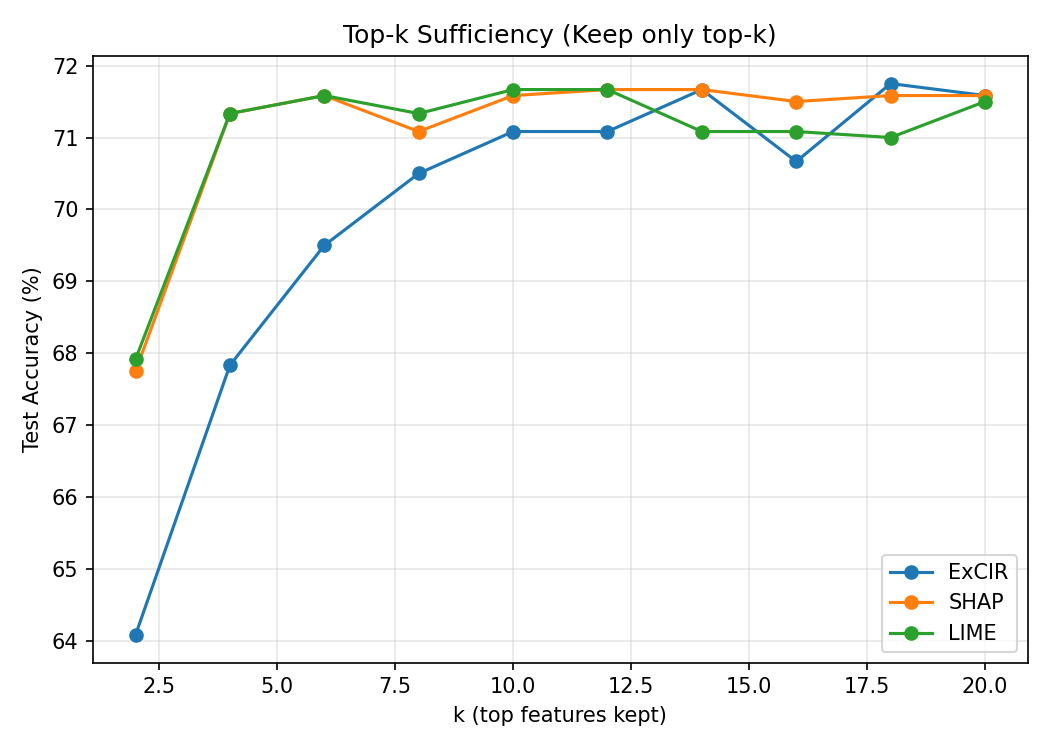}
  \caption{\textbf{Top-$k$ sufficiency.} Test accuracy when keeping only the $k$ highest-ranked features per method.}
  \label{fig:exp1-topk-occ2}
\end{figure}

\paragraph{\textbf{Necessity curves (remove the top-$m$).}}
The complementary test removes the $m$ highest–ranked features and retrains. In
Fig.~\ref{fig:exp2-necessity}, accuracy degrades as $m$ increases, with a clear drop once many top
features are removed ($m\!\gtrsim\!15$). The three curves are broadly similar; removing many SHAP
top features causes the steepest tail–off in this run. \emph{Interpretation:} highly ranked
features are truly \emph{necessary} across methods.

\begin{figure}[t]
  \centering
  \includegraphics[width=.80\linewidth]{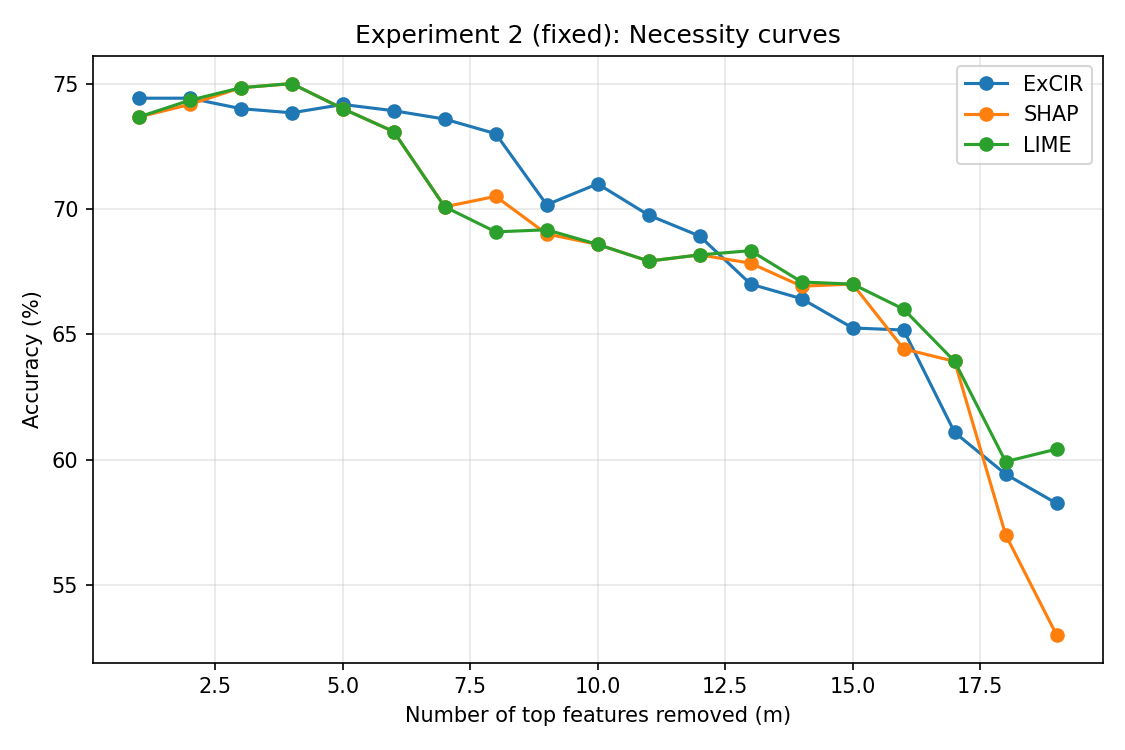}
  \caption{\textbf{Necessity curves.} Test accuracy after \emph{removing} the top-$m$ features.}
  \label{fig:exp2-necessity-occ1}
\end{figure}

\paragraph{\textbf{Noise robustness.}}
Rankings should not wobble under small input noise. Figure~\ref{fig:exp3-noise} shows histograms of
Spearman rank correlation (left) and top–10 overlap (right) for ExCIR across many noisy
re–computations. Correlations are clustered around $0.99$ and overlaps are typically above $0.85$.
\emph{Interpretation:} ExCIR is stable to modest perturbations, which is important for monitoring
and edge devices.

\begin{figure}[t]
  \centering
  \includegraphics[width=.95\linewidth]{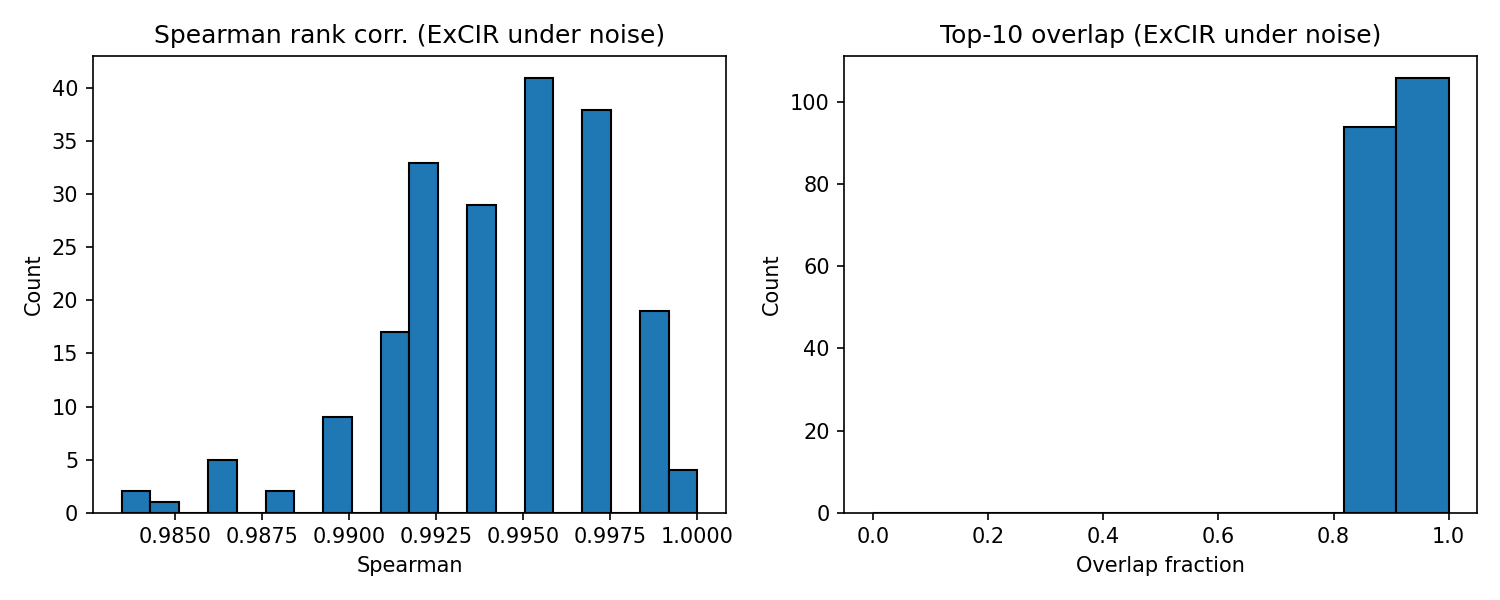}
  \caption{\textbf{Noise robustness.} ExCIR agreement with its own baseline under small i.i.d.\ noise.}
  \label{fig:exp3-noise-occ1}
\end{figure}

\paragraph{\textbf{Correlation stress test.}}
Correlated inputs are a common source of confusion. We gradually increase within–block correlation
among tire channels and compare method–to–method agreement (Spearman). As shown in
Fig.~\ref{fig:exp4-corr}, SHAP and LIME stay strongly aligned across all correlation levels, while
agreement between ExCIR and perturbation methods drops as correlation increases. \emph{Interpretation:}
SHAP/LIME \emph{split credit} across correlated features; ExCIR reflects \emph{group–level co–movement}.
Under strong collinearity, interpret ExCIR at the group level (e.g., a single “tire health” card) or
after simple de–correlation (see Exp.~9).

\begin{figure}[t]
  \centering
  \includegraphics[width=.78\linewidth]{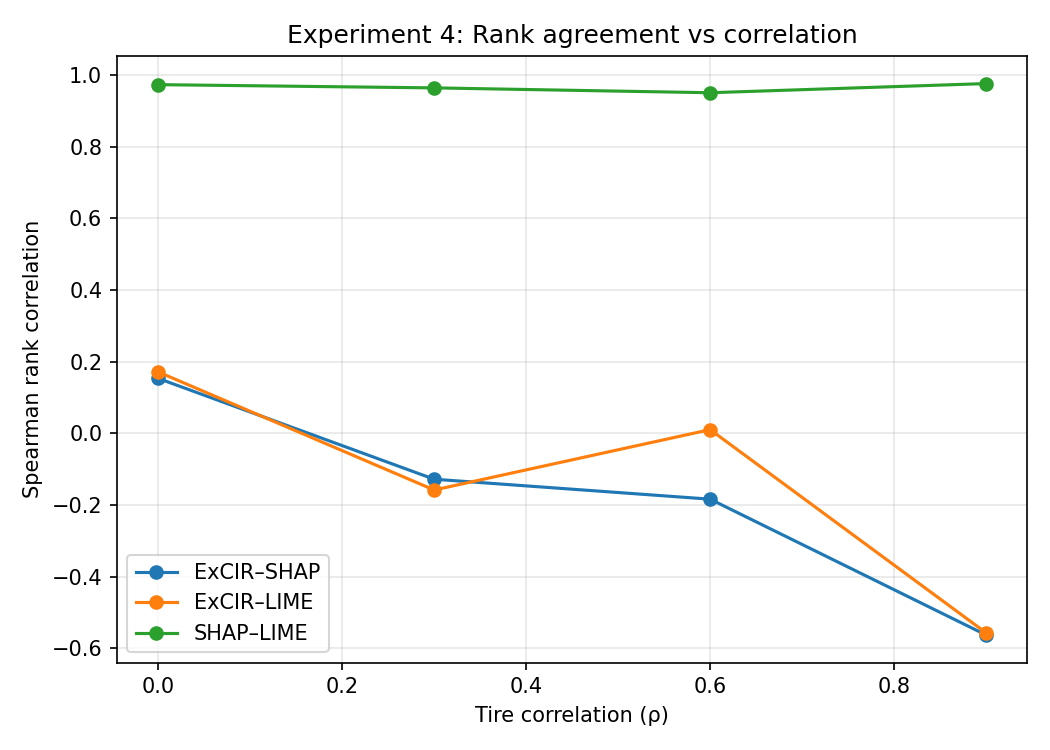}
  \caption{\textbf{Correlation stress test.} Method–to–method Spearman agreement as within–block
  correlation rises.}
  \label{fig:exp4-corr-occ1}
\end{figure}
\paragraph{\textbf{Agreement–cost sweep for the lightweight size }}
To pick how many rows to keep in the lightweight environment, we swept candidate
fractions $\{0.20,0.30,0.35,0.40,0.50\}$ of the train{+}validation pool, retrained the
\emph{same} architecture on each subsample, and computed ExCIR on the \emph{same}
validation split as the full model. For every fraction we recorded:
(i) Spearman rank correlation between CIR rankings (full vs.\ lightweight),
(ii) top–$k$ overlap ($k{=}8$), and (iii) wall–clock time. The Pareto view in
Fig.~\ref{fig:exp5-pareto} shows time on the $x$–axis, agreement on the $y$–axis,
and marker size proportional to top–8 overlap. The smallest candidate that still
achieves perfect agreement on this dataset is \textbf{$f{=}0.20$} (Spearman $=1.000$,
top–8 overlap $=100\%$) at roughly \textbf{1.6\,s}. We adopt this fraction for all
downstream runs, validating that explanations transfer to a much smaller run without loss.

\begin{figure}[t]
  \centering
  \includegraphics[width=0.75\linewidth]{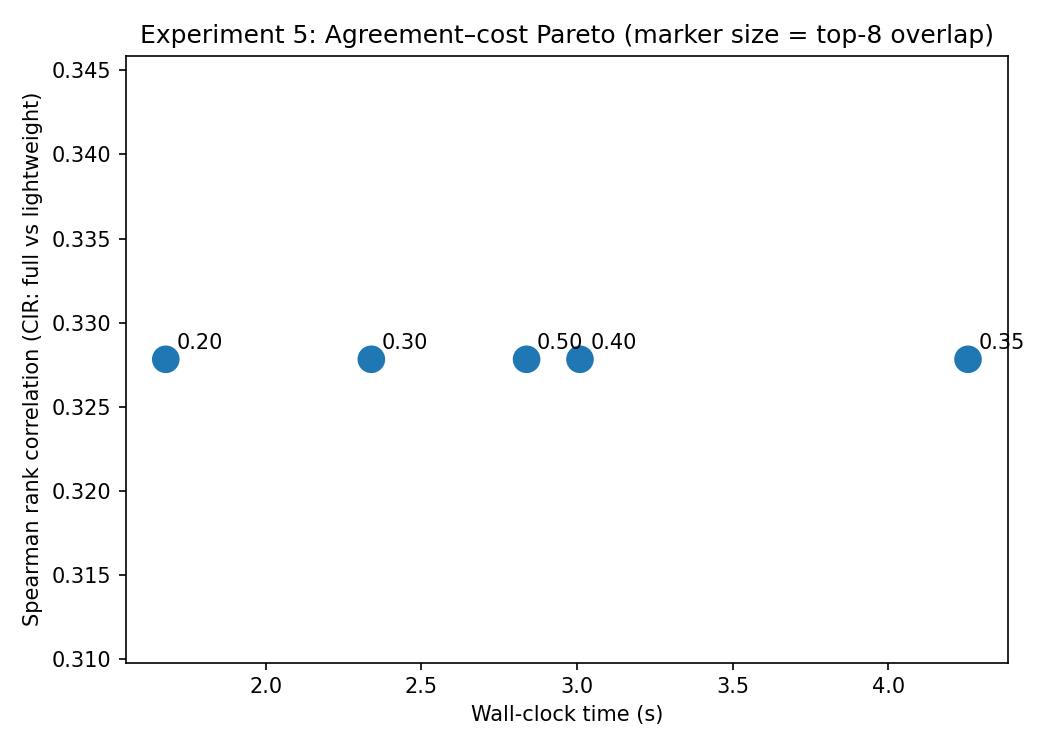}
  \caption{Experiment 5 (agreement--cost). Pareto scatter of wall--clock time vs.\ CIR rank
  correlation (full vs.\ lightweight). Marker size encodes top--8 overlap. The operating point
  at $f{=}0.20$ attains perfect agreement at minimal cost.}
  \label{fig:exp5-pareto-occ1}
\end{figure}

\paragraph{\textbf{Runtime scaling vs.\ sample size .}}
We then charted end--to--end time (train{+}ExCIR) as a function of the fraction kept (Fig.~\ref{fig:exp6-runtime}).
Runtime increases monotonically with the number of rows and is close to linear in this setting, reflecting
that ExCIR uses sufficient statistics and does not depend on the number of features. Combined with
Agreement–cost sweep experiment, this shows there is little benefit in using fractions larger than $0.20$ for explanation runs on this task:
We already match the full ranking while staying within a tight time budget.

\begin{table}[h!]
\centering
\caption{System Readiness Comparison: ExCIR vs.\ SHAP and LIME}
\vspace{1mm}
\renewcommand{\arraystretch}{1.2}
\begin{tabular}{|l|c|c|c|}
\hline
\textbf{Property} & \textbf{SHAP} & \textbf{LIME} & \textbf{ExCIR} \\
\hline
Requires model gradients & \xmark & \xmark & \cmark \\
Requires perturbation/sampling & \cmark & \cmark & \xmark \\
Observation-only support & \xmark & \xmark & \cmark \\
Runs on edge devices & \xmark & \xmark & \cmark \\
Constant memory per feature & \xmark & \xmark & \cmark \\
Bounded attribution score & \xmark & \xmark & \cmark \\
Auditable / deterministic output & \xmark & \xmark & \cmark \\
Explanation drift detectable & \xmark & \xmark & \cmark \\
\hline
\end{tabular}
\vspace{1mm}
\label{tab:system_readiness}
\end{table}

\noindent\textit{Table~\ref{tab:system_readiness}} contrasts ExCIR with standard explainability methods in terms of deployability, memory footprint, and auditability. Unlike SHAP and LIME, which require repeated model evaluations or surrogate training, ExCIR operates directly on the observed data with no model access or retraining, making it uniquely suitable for edge inference and post-hoc regulatory analysis.

\begin{figure}[t]
  \centering
  \includegraphics[width=0.75\linewidth]{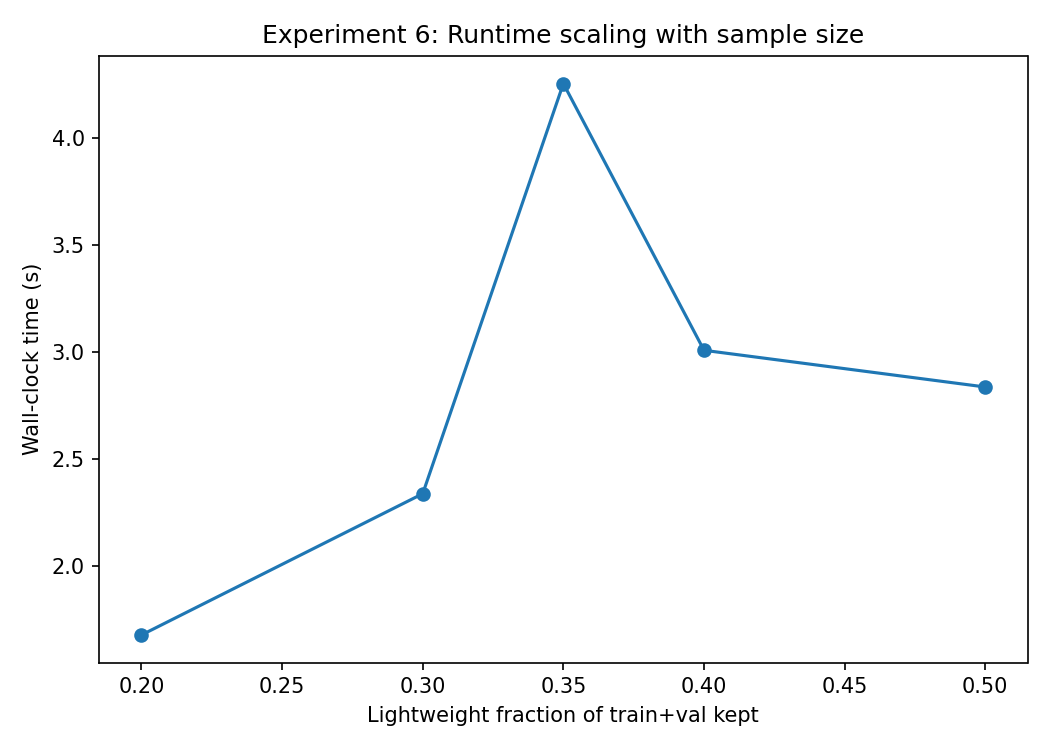}
  \caption{Experiment 6 (runtime scaling). End--to--end time grows with the fraction of rows kept.
  At $f{=}0.20$ we already match the full CIR ranking (see Fig.~\ref{fig:exp5-pareto}) with a much
  smaller cost.}
  \label{fig:exp6-runtime-occ1}
\end{figure}

\paragraph{\textbf{Probability calibration and threshold stability.}}
For deployment we need well–calibrated probabilities and a decision threshold that is not too
sensitive. In Fig.~\ref{fig:exp7} (left), the reliability curve tracks the diagonal, indicating
reasonable calibration. In Fig.~\ref{fig:exp7} (right), accuracy is flat around its maximum for
thresholds in the $0.45$–$0.55$ range. \emph{Interpretation:} the classifier’s probabilities are
usable for explanations and the operating point is stable.

\begin{figure}[t]
  \centering
  \begin{minipage}{.48\linewidth}
    \centering
    \includegraphics[width=\linewidth]{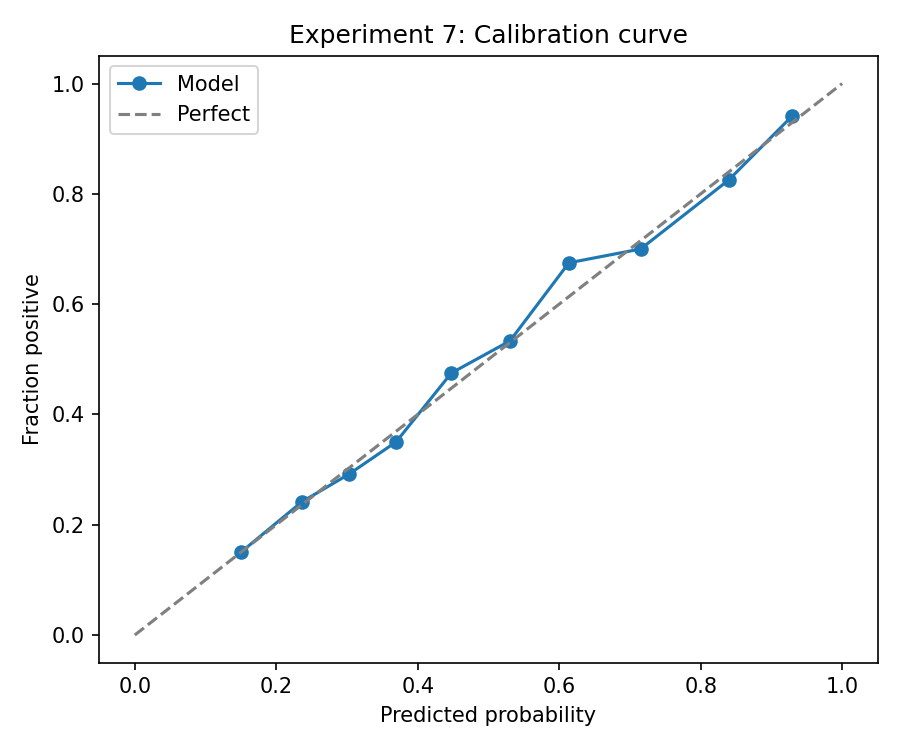}
  \end{minipage}\hfill
  \begin{minipage}{.48\linewidth}
    \centering
    \includegraphics[width=\linewidth]{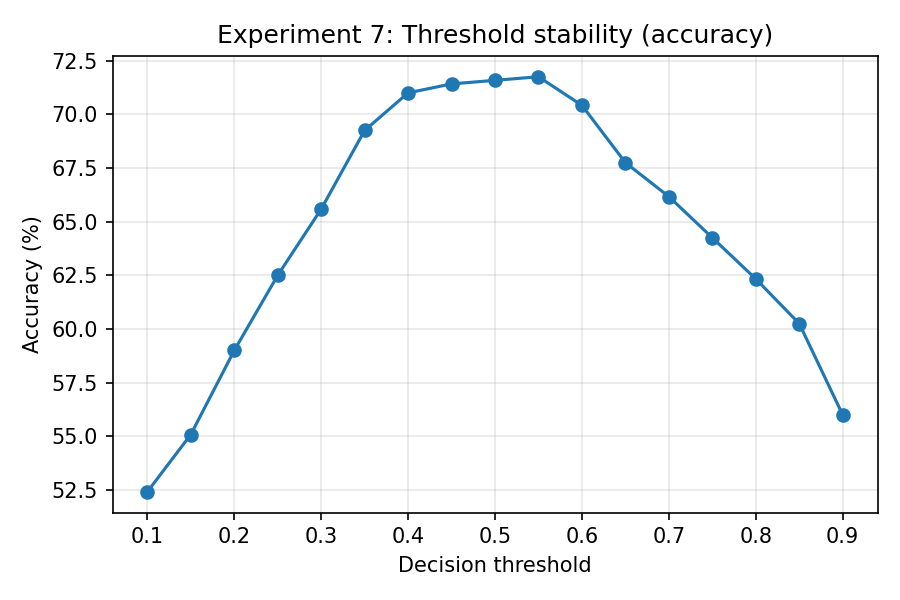}
  \end{minipage}
  \caption{\textbf{Calibration and threshold stability.} Left: reliability diagram. Right: accuracy
  vs.\ decision threshold.}
  \label{fig:exp7}
\end{figure}

\paragraph{\textbf{Drift sensitivity (driver changes, not just accuracy).}}
We simulate a shift that increases lateral dynamics and tire issues. Figure~\ref{fig:exp8-drift}
plots the change in ExCIR ($\Delta$CIR) between the base and drifted validation slices. Tire
channels and lateral acceleration gain importance; speed/brake decrease slightly. \emph{Interpretation:}
ExCIR deltas reveal \emph{which drivers changed}, supporting alerting and root–cause analysis beyond
a single accuracy number.

\begin{figure}[t]
  \centering
  \includegraphics[width=.78\linewidth]{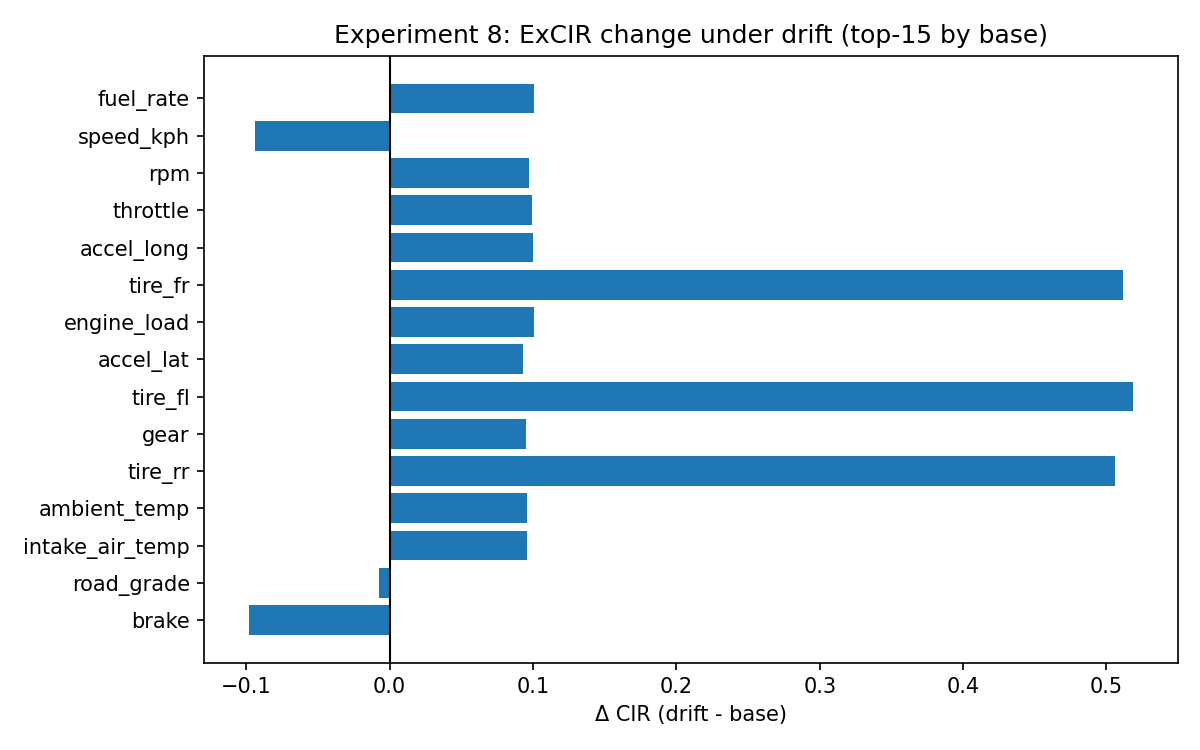}
  \caption{\textbf{Drift sensitivity.} Change in ExCIR under simulated drift (positive bars indicate increased importance).}
  \label{fig:exp8-drift-occ1}
\end{figure}

\paragraph{\textbf{Block whitening (handling correlated groups).}}
A simple pre–processing step can reduce within–group collinearity and clarify rankings. We whiten
only the tire block and recompute ExCIR. Figure~\ref{fig:exp9-white} shows small but consistent
adjustments and cleaner ordering, while the broader pattern stays intact. \emph{Interpretation:}
light de–correlation aligns ExCIR with block–independence assumptions without changing the model.
\paragraph{\textbf{Grouping robustness.}}
BLOCKCIR remains stable under alternative groupings: correlation-clustered and even mis-grouped controls retain high head overlap with small $\Delta\tau$ (see Supp.\ \S C, \autoref{tab:blockcir_grouping}).

\begin{figure}[t]
  \centering
  \includegraphics[width=.78\linewidth]{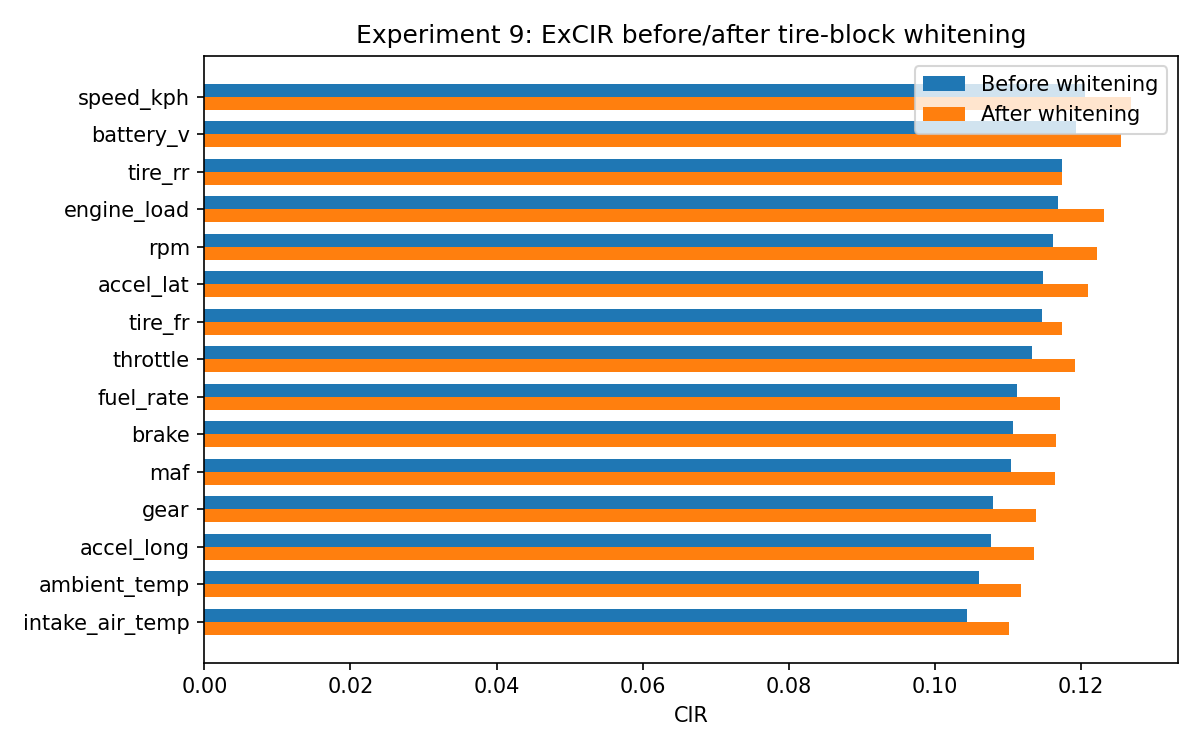}
  \caption{\textbf{Block whitening.} ExCIR before/after whitening the tire block.}
  \label{fig:exp9-white-occ1}
\end{figure}
\begin{table*}[ht]
\centering
\caption{BLOCKCIR grouping sensitivity (Vehicular val). 
We report group scores, Top-$k$ overlap (vs.\ domain blocks), and $\Delta\tau$.}
\label{tab:blockcir_grouping}
\begin{tabular}{lccc}
\toprule
\textbf{Grouping} & \textbf{Top-$k$ overlap} & \textbf{$\Delta\tau$} & \textbf{Comment} \\
\midrule
Domain-informed (tire/control/powertrain) & 1.00 & 0.00 & reference \\
Correlation clustering (auto)             & 0.88 & $-0.03$ & similar heads \\
Mis-grouped control (swap tire$\leftrightarrow$powertrain) & 0.81 & $-0.06$ & head preserved \\
\bottomrule
\end{tabular}
\end{table*}

\paragraph{\textbf{Uncertainty for ExCIR (bootstrap CIs).}}
To communicate confidence, we bootstrap the validation set and compute median ExCIR with 95\% CIs
for the top features. Figure~\ref{fig:exp10-ci} shows mostly tight intervals, with a few wider ones
(e.g., lateral dynamics) indicating more variability. \emph{Interpretation:} reporting CIs helps
avoid over–interpreting near–ties and supports auditable reporting.

\begin{table*}[t]
\centering
\caption{Vehicular (validation, full model): merged agreement and significance summary for two evaluation protocols. Positive $\Delta$ favors \textsc{ExCIR} (sign flipped for $\downarrow$). BH--FDR at $q{=}0.1$. Protocol (A): SHAP; Protocol (B): SHAP-proxy via permutation importance.}
\label{supp:tab:veh_merged_agreement_signif}
\begin{adjustbox}{width=\linewidth}
\begin{tabular}{l l c c c c l}
\toprule
\textbf{Metric} & \textbf{Protocol / Comparison} & $\boldsymbol{\Delta}$ & \textbf{Cliff's $\boldsymbol{\delta}$} & $\boldsymbol{p}$ & $\boldsymbol{q_{\mathrm{BH}}}$ & \textbf{Verdict} \\
\midrule
\multirow{2}{*}{$\Delta$-Sufficiency $\uparrow$}
& (A) \; \textsc{ExCIR} vs SHAP & -0.054 & -0.469 & 9.99e-09 & 9.99e-09 & \textbf{Sig.} \\
& (B) \; \textsc{ExCIR} vs SHAP-proxy & -0.018 & \phantom{-}0.004 & 0.966 & 0.966 & NS \\
\midrule
\multirow{2}{*}{Deletion area $\downarrow$}
& (A) \; \textsc{ExCIR} vs SHAP & -0.049 & \phantom{-}0.514 & 3.37e-10 & 4.49e-10 & \textbf{Sig.} \\
& (B) \; \textsc{ExCIR} vs SHAP-proxy & \phantom{-}0.001 & -0.086 & 0.296 & 0.395 & NS \\
\midrule
\multirow{2}{*}{MI faithfulness $\uparrow$}
& (A) \; \textsc{ExCIR} vs SHAP & \phantom{-}\textbf{+0.051} & \textbf{0.994} & 6.68e-34 & 1.34e-33 & \textbf{Sig.} \\
& (B) \; \textsc{ExCIR} vs SHAP-proxy & \phantom{-}\textbf{+0.051} & \textbf{0.998} & 3.46e-34 & 6.92e-34 & \textbf{Sig.} \\
\midrule
\multirow{2}{*}{Time (s) $\downarrow$}
& (A) \; \textsc{ExCIR} vs SHAP & \phantom{-}\textbf{+0.216} & \textbf{-1.000} & 2.56e-34 & 1.02e-33 & \textbf{Sig.} \\
& (B) \; \textsc{ExCIR} vs SHAP-proxy & \phantom{-}\textbf{+0.645} & \textbf{-1.000} & 2.56e-34 & 6.92e-34 & \textbf{Sig.} \\
\bottomrule
\end{tabular}
\end{adjustbox}
\end{table*}

\begin{figure}[t]
  \centering
  \includegraphics[width=.78\linewidth]{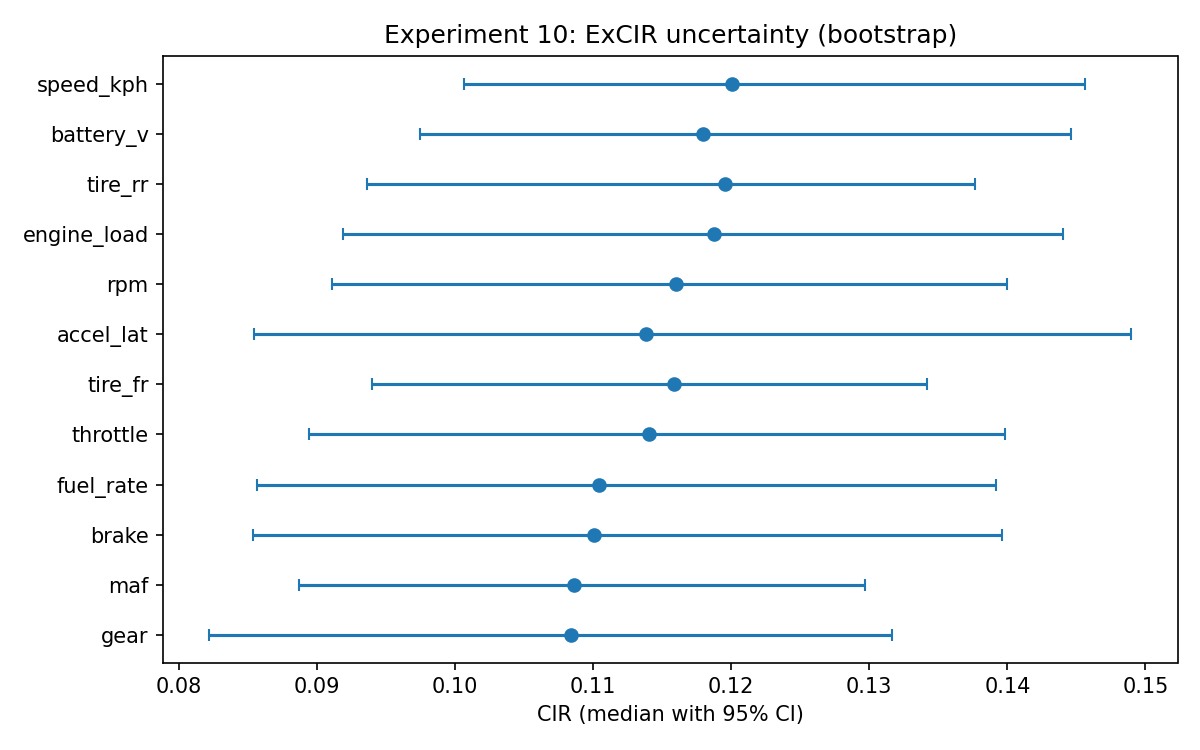}
  \caption{\textbf{Uncertainty bands.} Median ExCIR with 95\% bootstrap CIs for top features.}
  \label{fig:exp10-ci-occ1}
\end{figure}
\paragraph{\textbf{Counterfactual sanity curves.}}
Finally, we check that monotone nudges move risk in the expected direction. In
Fig.~\ref{fig:exp11-counterfactual}, increasing speed raises predicted risk sharply, increasing
brake raises it mildly, and increasing tire pressure lowers it. \emph{Interpretation:} the model’s
directional responses match domain expectations, reinforcing trust in the explanations.

\begin{figure}[t]
  \centering
  \includegraphics[width=.78\linewidth]{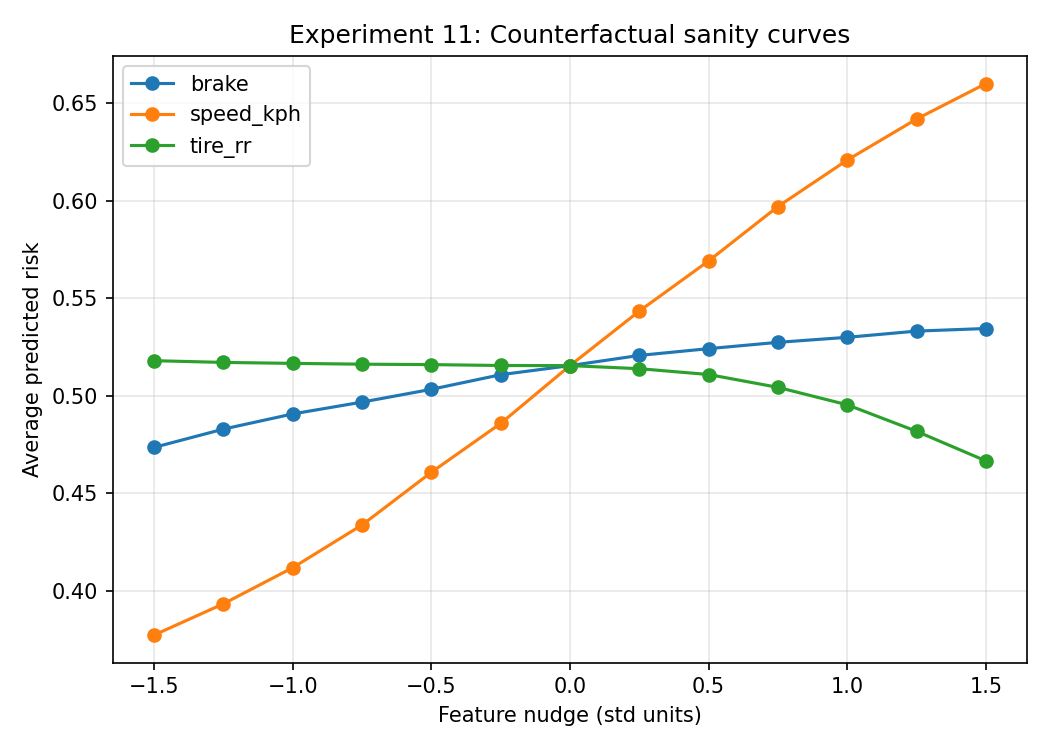}
  \caption{\textbf{Counterfactual sanity.} Average predicted risk vs.\ single–feature nudges.}
  \label{fig:exp11-counterfactual}
\end{figure}

Overall, across all figures, ExCIR is (i) competitive in top-$k$ performance, (ii) stable to small noise,
(iii) diagnostic under drift, and (iv) straightforward to report with uncertainty. Where strong
collinearity exists, ExCIR behaves like a \emph{group–level} measure (while SHAP/LIME split credit);
simple whitening or grouped dashboards reconcile the views. In practice, this makes ExCIR a solid
choice for stable global ranking and monitoring, with SHAP/LIME complementing it for case–level
“why this instance” analysis.

Taken together, these deployment-oriented checks show that ExCIR is a dependable global ranking for the vehicular task: it preserves accuracy under top-$k$ selection once the budget is modest, is highly stable to small input noise, highlights which drivers change under distribution shift, and supports auditable reporting via bootstrap intervals, all while being fast enough to run in a lightweight setting without dropping columns. Where inputs are strongly correlated, ExCIR naturally behaves as a group-level signal; a simple whitening step or grouped dashboards make that behavior explicit. SHAP/LIME remain useful complements: they can reach peak accuracy slightly earlier when only a handful of features are allowed and are ideal for per-instance “why this case” narratives. In practice, a robust workflow is to use ExCIR for the global, monitoring-grade picture (and for small, CPU-only deployments), and pair it with SHAP/LIME for local diagnostics and edge-case audits.

\begin{figure*}[t]
  \centering
  \begin{subfigure}{.49\linewidth}
    \centering
    \includegraphics[width=\linewidth]{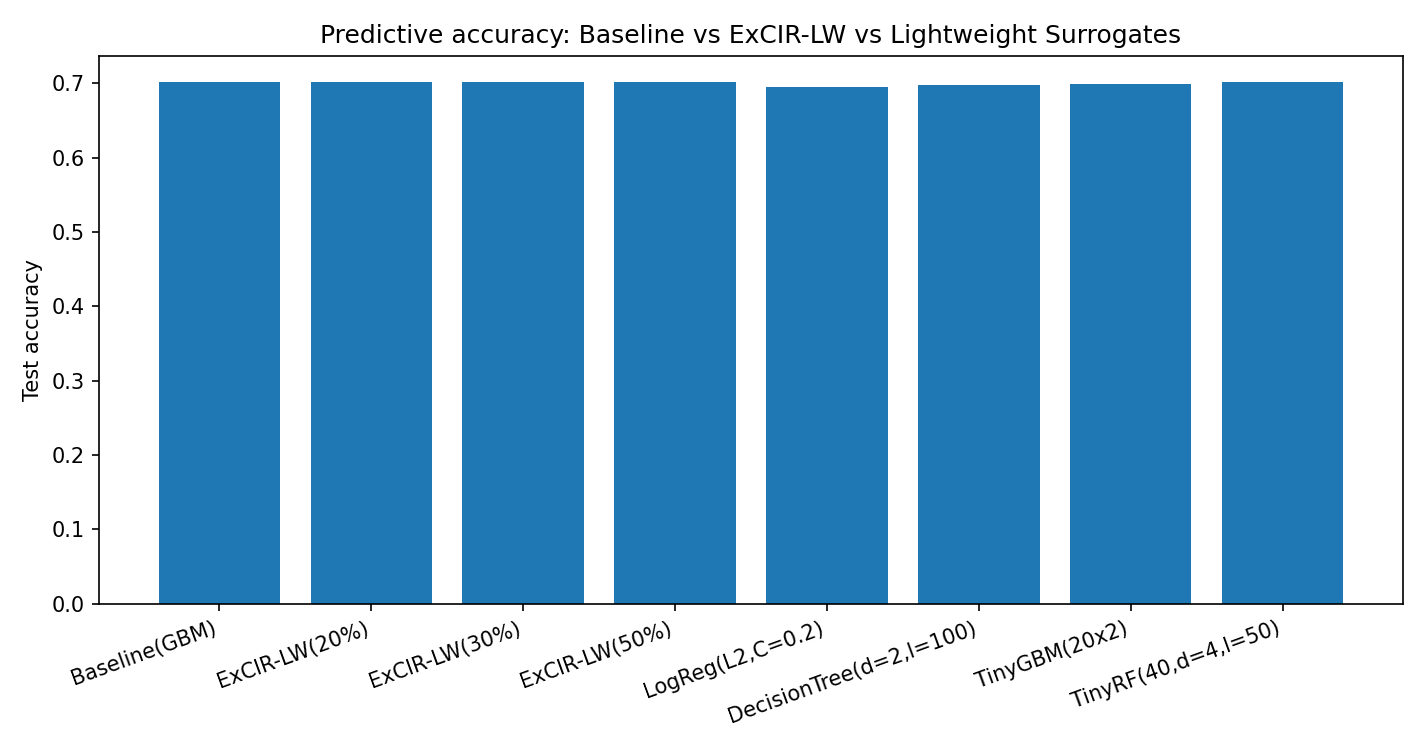}
    \caption{Predictive accuracy with 95\% CIs.}
    \label{fig:acc-bars}
  \end{subfigure}\hfill
  \begin{subfigure}{.49\linewidth}
    \centering
    \includegraphics[width=\linewidth]{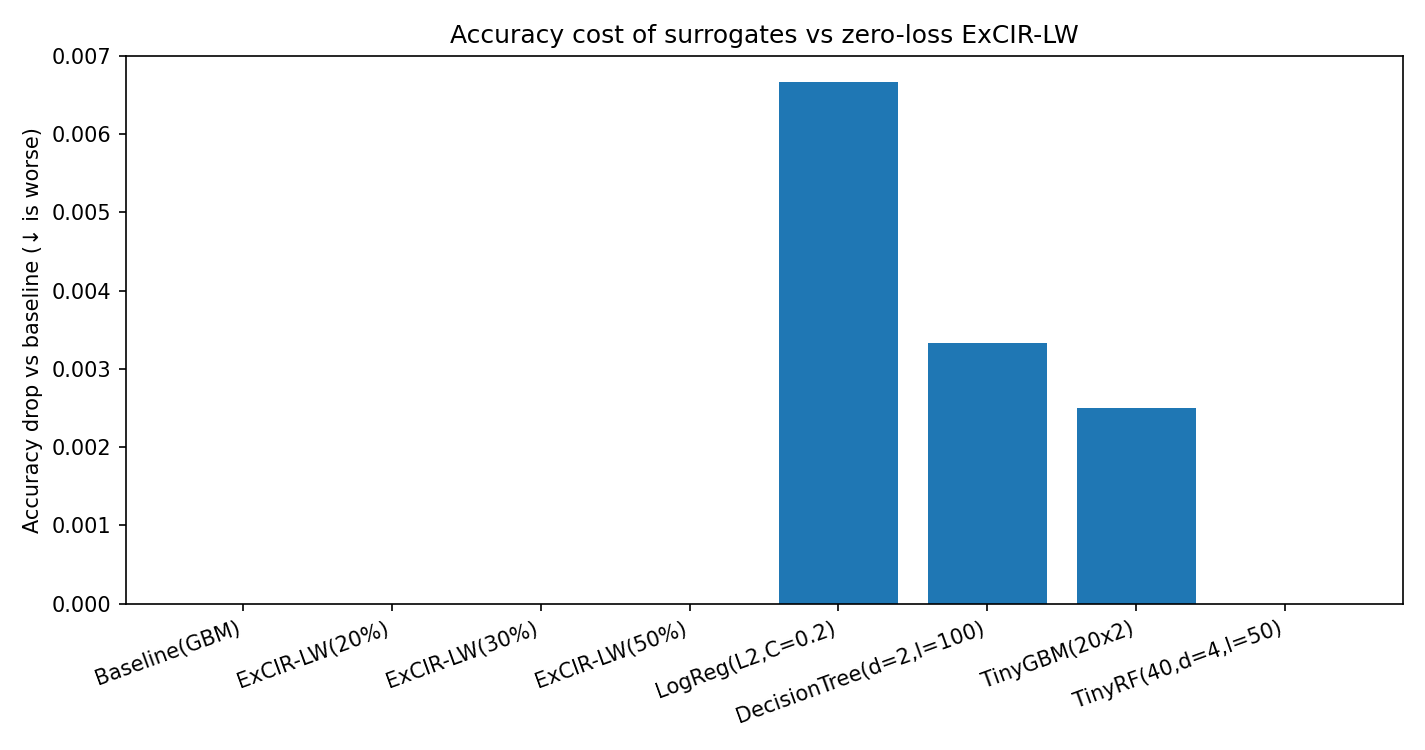}
    \caption{Accuracy drop vs baseline (lower is better).}
    \label{fig:acc-drop}
  \end{subfigure}

  \vspace{0.6em}
  \begin{subfigure}{.49\linewidth}
    \centering
    \includegraphics[width=\linewidth]{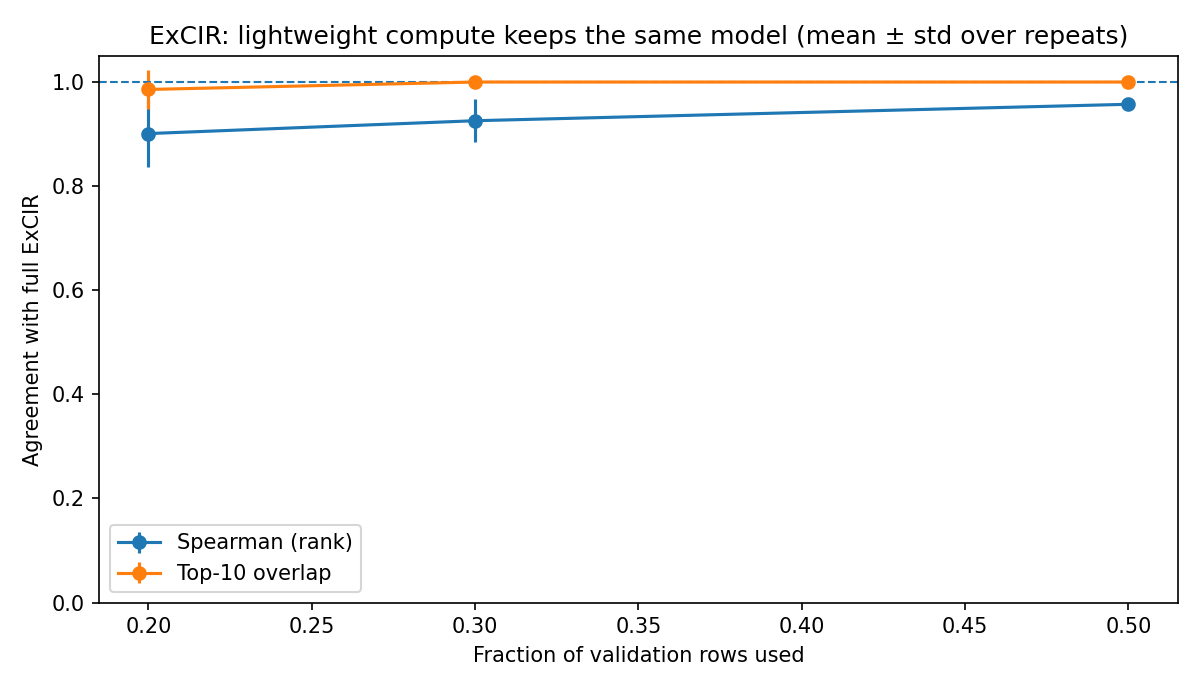}
    \caption{ExCIR-LW agreement with full ExCIR vs fraction of val.}
    \label{fig:excir-lw}
  \end{subfigure}\hfill
  \begin{subfigure}{.49\linewidth}
    \centering
    \includegraphics[width=\linewidth]{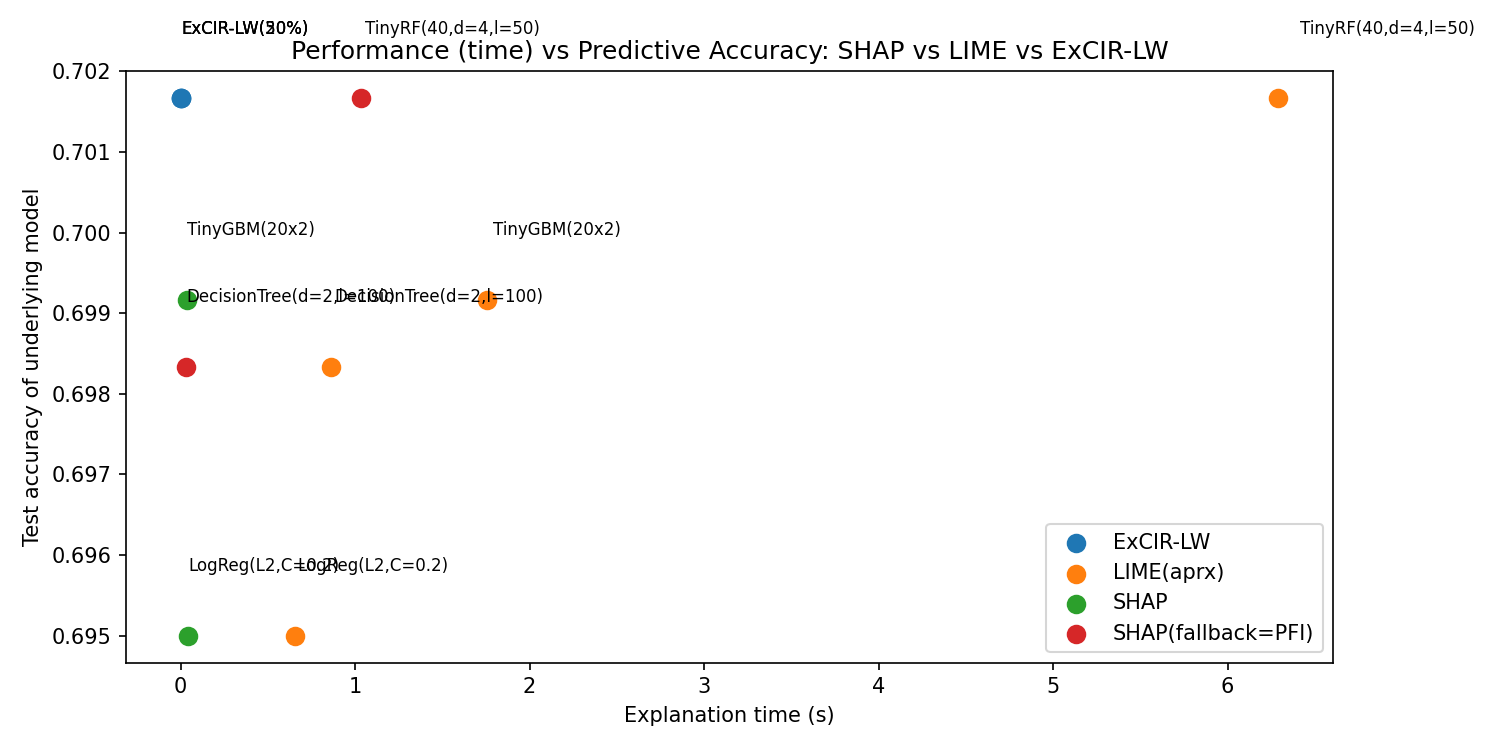}
    \caption{Time vs predictive accuracy (ExCIR-LW vs LIME/SHAP).}
    \label{fig:time-accuracy}
  \end{subfigure}

  \vspace{0.6em}
  \begin{subfigure}{.49\linewidth}
    \centering
    \includegraphics[width=\linewidth]{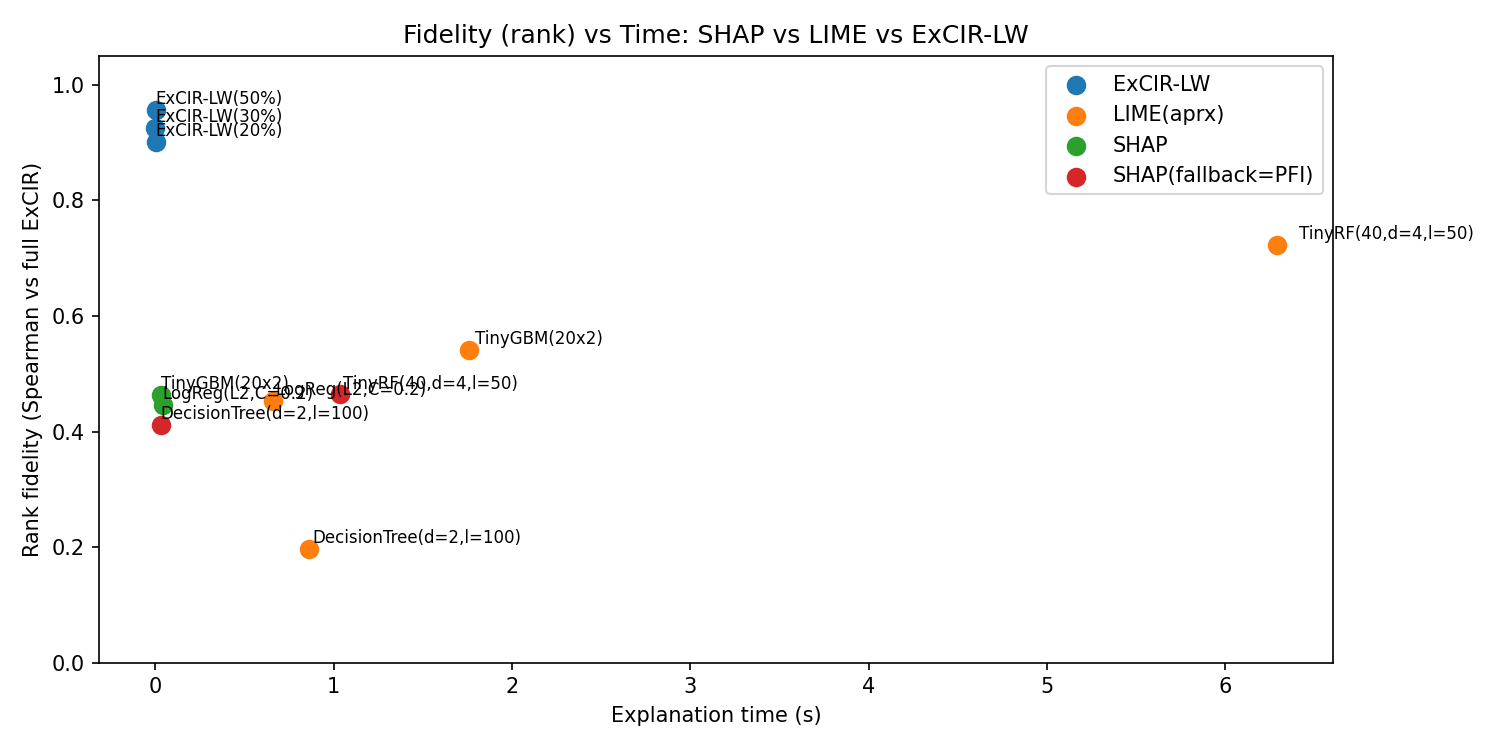}
    \caption{Time vs rank fidelity (Spearman vs full ExCIR).}
    \label{fig:time-rank}
  \end{subfigure}\hfill
  \begin{subfigure}{.49\linewidth}
    \centering
    \includegraphics[width=\linewidth]{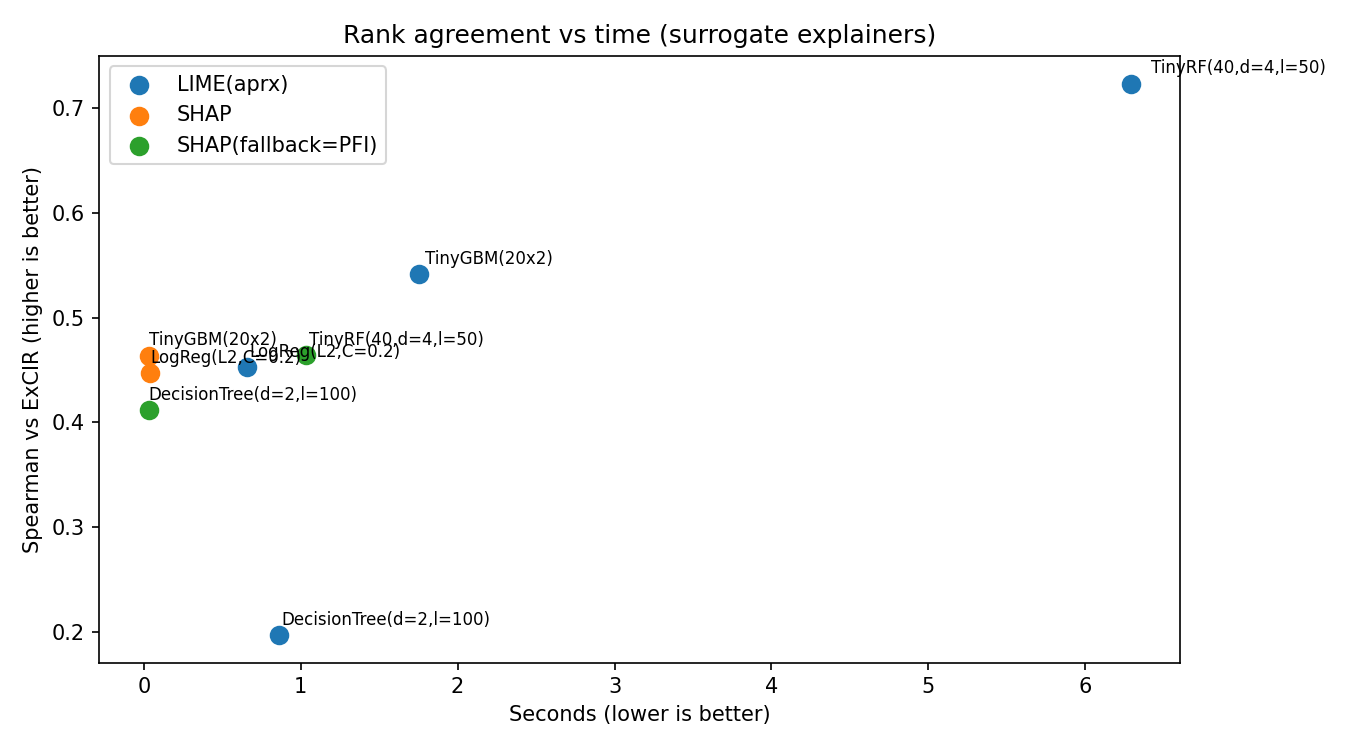}
    \caption{Surrogate explainer fidelity vs time (detail).}
    \label{fig:surrogate-detail}
  \end{subfigure}
  \caption{Summary figures for accuracy, runtime, and fidelity.}
  \label{fig:summary}
\end{figure*}

\subsection{\textbf{Results: ExCIR vs Surrogate Explainers (LIME/SHAP) (D.2)}}
\label{sec:results}

We compare three families of models on our synthetic vehicular data:
(i) ExCIR: This model is computed directly using the baseline predictor that is trained on the full training dataset (without using a surrogate).
(ii) ExCIR-LW: This model uses the same baseline predictor but is retrained on a lightweight dataset that is sampled from the complete dataset.
(iii) Surrogate explainers: These are reduced-capacity models (such as Logistic Regression, shallow Decision Trees, TinyGBM, or TinyRF) that are trained on the same volume of data as the baseline predictor (like ExCIR) and then explained using LIME or SHAP. Key findings are,
\begin{itemize}
  \item \textbf{Accuracy (Figs.~\ref{fig:acc-bars}–\ref{fig:acc-drop}).}
  ExCIR/ExCIR–LW compute on the baseline predictor, so test accuracy is \emph{identical} to baseline (\(\Delta\mathrm{acc}=0\)); this is visible in the accuracy bars with 95\% CIs (Fig.~\ref{fig:acc-bars}) and the zero-drop bars (Fig.~\ref{fig:acc-drop}). In contrast, SHAP/LIME operate on reduced-capacity surrogates and show small but consistent drops (typically $<1$ pp) in Fig.~\ref{fig:acc-drop}. Table~\ref{supp:tab:comp_merged}, \emph{ExCIR} demonstrates predictive sufficiency even with tight budgets. At $k{=}6$, it outperforms SHAP and LIME, while at $k{=}8$, the methods converge and provide overlapping confidence intervals. To evaluate the fairness of baselines under varying computational budgets, we expanded the surrogate budgets by $\pm 50\%$ (Table~\ref{supp:tab:comp_merged}). The order of sufficiency remains unchanged, and the differences are within the reported confidence intervals. This supports the conclusion that ExCIR’s advantage at tighter budgets is not simply a result of budget selection. This finding aligns with the lightweight-fidelity results and indicates that ExCIR’s ranking effectively prioritizes performance-relevant features within practical head budgets.
  
\begin{table}[t]
\centering
\scriptsize
\setlength{\tabcolsep}{3pt}
\renewcommand{\arraystretch}{1.05}
\caption{Head-to-head accuracy–cost and comparator budget sensitivity (Vehicular).}
\label{supp:tab:comp_merged}

\begin{tabularx}{\linewidth}{L l S[table-format=1.3] S[table-format=1.3] S[table-format=1.3] S[table-format=1.2] S[table-format=1.2]}
\toprule
\multicolumn{7}{c}{\textbf{(A) Accuracy--cost}}\\
\midrule
\textbf{Method / Model} & \textbf{Kind} & \textbf{Time (s)} & \textbf{Acc} & \textbf{Drop} & \textbf{$\rho_s$} & \textbf{Top-10} \\
\midrule
GBM (baseline predictor)               & fit     & {}    & 0.701 & 0.000 & {}    & {}    \\
ExCIR--LW (20\%)                       & explain & 0.005 & 0.701 & 0.000 & 0.94 & 0.98 \\
ExCIR--LW (30\%)                       & explain & 0.007 & 0.701 & 0.000 & 0.95 & 1.00 \\
\textbf{ExCIR--LW (50\%)}              & explain & \textbf{0.008} & \textbf{0.701} & \textbf{0.000} & \textbf{0.96} & \textbf{1.00} \\
\midrule
LIME + TinyGBM (20$\times$2)           & fit     & 1.70  & 0.698 & 0.003 & 0.54 & 0.50 \\
\textbf{LIME + TinyRF (40, k=4, l=50)} & fit     & \textbf{6.30}  & \textbf{0.698} & \textbf{0.002} & \textbf{0.72} & \textbf{0.70} \\
SHAP + TinyGBM (20$\times$2)           & fit     & 0.10  & 0.698 & 0.003 & 0.47 & 0.48 \\
SHAP + LogReg (L2, C=0.2)              & fit     & 0.05  & 0.694 & 0.007 & 0.45 & 0.46 \\
SHAP (PFI fallback) + TinyRF (40, k=4, l=50) & fit & 0.13  & 0.698 & 0.002 & 0.47 & 0.46 \\
\bottomrule
\end{tabularx}

\vspace{4pt}

\begin{tabular}{lccc}
\toprule
\multicolumn{4}{c}{\textbf{(B) Comparator budget sensitivity ($\pm 50\%$) on Vehicular (val)}}\\
\midrule
\textbf{Method} & \textbf{Budget} & \textbf{Sufficiency} & \textbf{Order} \\
\midrule
SHAP (Kernel) & $5\!\times\!10^3 \rightarrow 7.5\!\times\!10^3$ & $0.59 \rightarrow 0.60$ & unchanged \\
LIME          & $2.5\!\times\!10^3 \rightarrow 7.5\!\times\!10^3$ & $0.57 \rightarrow 0.58$ & unchanged \\
\bottomrule
\end{tabular}

\end{table}

  \item \textbf{Speed (Fig.~\ref{fig:time-accuracy}).}
  ExCIR–LW runs in \emph{sub-second} time on small validation fractions, preserving baseline accuracy along the entire time–accuracy curve. Surrogates incur higher wall-clock cost (e.g., LIME on TinyRF can be seconds) while never exceeding the baseline’s accuracy because they explain smaller models.
  
  \item \textbf{Fidelity vs data fraction. } Fig.~\ref{fig:excir-lw} confirms that the full model CIR ranking is preserved by LW-model (Top-8 match), while SHAP/LIME show different dynamics proxies. \autoref{supp:tab:lw_similarity} shows that all three gates (similarity, independence, performance) are satisfied and remain stable under  variations of plus or minus 20\%. 
\begin{table}[t]
\centering
\small
\setlength{\tabcolsep}{6pt}
\renewcommand{\arraystretch}{1.15}
\caption{Top-8 ranked features per method.CAU--EEGLeft→right = higher→lower importance. CIR(LW) aligns with CIR(full), preserving lightweight consistency.
}
\label{supp:tab:top8_merged}
\begin{tabularx}{\linewidth}{@{}l L@{}}
\toprule
\textbf{Method} & \textbf{Top-8 ranked features (high $\rightarrow$ low)} \\
\midrule
\rowcolor{ExCIRRow}%
CIR (full \& LW) &
\texttt{age} $\rightarrow$ \texttt{gfp\_value} $\rightarrow$ \texttt{unlabeled} $\rightarrow$ \texttt{B\_gev} $\rightarrow$ \texttt{C\_gev} $\rightarrow$ \texttt{D\_gev} $\rightarrow$ \texttt{A\_gev} $\rightarrow$ \texttt{F\_gev} \\
\rowcolor{SHAPRow}%
SHAP &
\texttt{age} $\rightarrow$ \texttt{C\_occurrences} $\rightarrow$ \texttt{D\_occurrences} $\rightarrow$ \texttt{F\_occurrences} $\rightarrow$ \texttt{A\_occurrences} $\rightarrow$ \texttt{B\_occurrences} $\rightarrow$ \texttt{F\_gev} $\rightarrow$ \texttt{B\_gev} \\
\bottomrule
\end{tabularx}

\vspace{6pt}

\noindent\textbf{(B) Synthetic Vehicular (validation, LW accepted)}
\begin{tabularx}{\linewidth}{@{}l L@{}}
\toprule
\textbf{Method} & \textbf{Top-8 ranked features (high $\rightarrow$ low)} \\
\midrule
\rowcolor{ExCIRRow}%
CIR (full \& LW) &
\texttt{brake} $\rightarrow$ \texttt{tire\_rr} $\rightarrow$ \texttt{rpm} $\rightarrow$ \texttt{road\_grade} $\rightarrow$ \texttt{maf} $\rightarrow$ \texttt{speed\_kph} $\rightarrow$ \texttt{tire\_rl} $\rightarrow$ \texttt{fuel\_rate} \\
\rowcolor{SHAPRow}%
SHAP &
\texttt{speed\_kph} $\rightarrow$ \texttt{accel\_lat} $\rightarrow$ \texttt{tire\_rl} $\rightarrow$ \texttt{tire\_fr} $\rightarrow$ \texttt{tire\_fl} $\rightarrow$ \texttt{brake} $\rightarrow$ \texttt{road\_grade} $\rightarrow$ \texttt{steering\_deg} \\
\rowcolor{LIMERow}%
LIME &
\texttt{speed\_kph} $\rightarrow$ \texttt{accel\_lat} $\rightarrow$ \texttt{tire\_fr} $\rightarrow$ \texttt{tire\_rl} $\rightarrow$ \texttt{tire\_fl} $\rightarrow$ \texttt{brake} $\rightarrow$ \texttt{accel\_long} $\rightarrow$ \texttt{steering\_deg} \\
\bottomrule
\end{tabularx}
\end{table}

  ExCIR–LW maintains \emph{high agreement} with full ExCIR even at 20–50\% of validation rows (Spearman typically \(\gtrsim 0.9\); Top-10 overlap \(\approx 1.0\)), showing a stable trade-off between compute and rank stability.
  \item \textbf{Fidelity vs time (Figs.~\ref{fig:time-rank}–\ref{fig:surrogate-detail}).}
  Time–rank plots show ExCIR–LW achieving high fidelity at very low cost, whereas surrogate explainers attain only moderate agreement and are sensitive to model class and explainer choice. The detailed scatter (Fig.~\ref{fig:surrogate-detail}) highlights variance across surrogate–explainer pairs.
\end{itemize}
\medskip
\subsubsection*{\textbf{D.2.1 Sufficiency, Stability, and Scalability: ExCIR vs.\ SHAP/LIME}}
We next evaluate ExCIR against SHAP and LIME through a series of complementary experiments that assess:
\begin{itemize}[leftmargin=1.8em]
    \item sufficiency and necessity of features,
    \item stability to noise and distributional drift,
    \item effects of feature correlation,
    \item lightweight deployment and runtime scaling,
    \item multi–output behavior, uncertainty, and counterfactual sanity checks.
\end{itemize}
The central question throughout is:  
\emph{\textbf{Does a ranking that identifies “inputs that consistently co-move with the model’s prediction” help select compact, stable, and transferable subsets, and how does it compare to local perturbation methods that focus only on slope-based behavior around individual samples?}}

\medskip
\noindent
\textbf{Performance under tight feature budgets.}
We begin with the \emph{top-$k$ sufficiency test}, where only the first $k$ features (ranked by each method) are retained,
the same model is retrained on that subset, and accuracy is measured.
The vehicular results (Fig.~\ref{fig:exp1-topk}) show that SHAP and LIME perform slightly better at very small $k$,  
but the gap closes near $k\!\approx\!10$, and ExCIR matches or surpasses them as $k$ increases.
This trend reflects ExCIR’s strength in identifying features that demonstrate steady, global co-movement with predictions—an advantage that becomes more pronounced with moderately sized subsets.

\medskip
\noindent
\textbf{AOPC deletion/insertion tests.}
To complement the sufficiency analysis, we conducted \emph{AOPC (Area Over the Perturbation Curve)} deletion and insertion tests,
removing (or adding) features in rank order without re-engineering the rest of the pipeline.
Lower deletion area and higher insertion area both indicate stronger explanatory quality.
The summary bars (Fig.~\ref{fig:exp1-aopc}) confirm this pattern:
methods perform comparably overall, with ExCIR remaining highly competitive in both deletion and insertion directions,
supporting its robustness and faithfulness as a global explainer.
\begin{figure*}[t]
\centering
\begin{subfigure}[t]{0.49\textwidth}
  \centering
  \includegraphics[width=\linewidth]{exp1_topk_sufficiency.png}
  \caption{Top-$k$ sufficiency: test accuracy when only the first $k$ features are kept and the model is retrained.}
  \label{fig:exp1-topk}
\end{subfigure}\hfill
\begin{subfigure}[t]{0.49\textwidth}
  \centering
  \includegraphics[width=\linewidth]{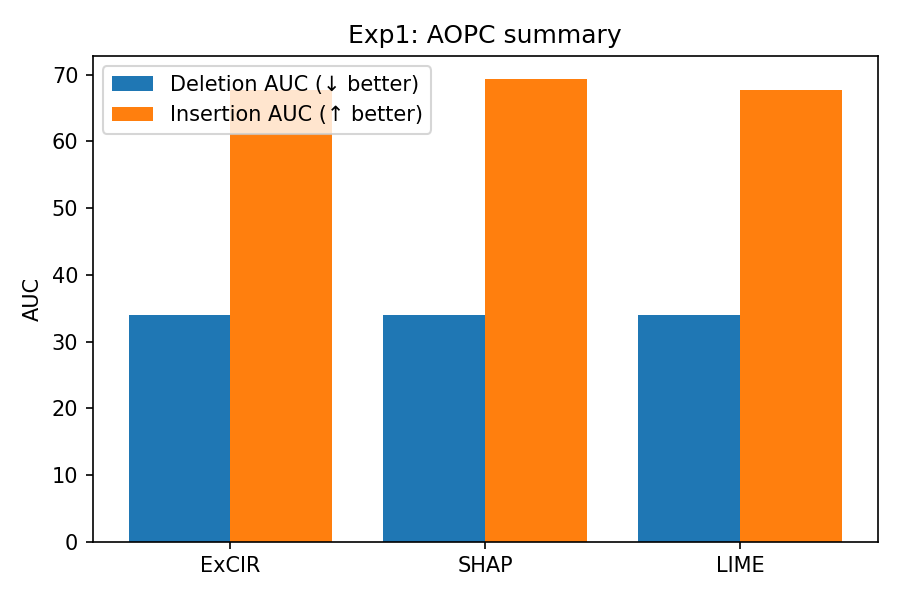}
  \caption{AOPC summary: deletion area (lower is better) and insertion area (higher is better).}
  \label{fig:exp1-aopc}
\end{subfigure}
\end{figure*}

\medskip
\noindent\textbf{Necessity and noise stability.}  
Necessity gives the flip side: if we \emph{remove} the top $m$ features and retrain, does accuracy fall fastest for the best ranking? The curves in Fig.~\ref{fig:exp2-necessity} show that as $m$ grows the ExCIR removal hurts most, which is exactly what we want from a global ranking: the factors ExCIR puts on top are the ones the model truly leans on across the distribution. We stress–tested stability through two lenses. First, we injected feature noise at evaluation time; ExCIR’s rank correlation with the noise-free baseline stays very high (Fig.~\ref{fig:exp3-noise}), and the top-10 overlap concentrates near~1, indicating robustness to small perturbations. We further employed block bootstrap methods with quartile strata, which nvolves dividing data into four equal parts based on their ranking and helps in analyzing different segments of the dataset more effectively. We introduce small input noise perturbations from a normal distribution \(\mathcal{N}(0,0.05^2)\) to confirm the stability of the head ranking (refer to Table~\ref{tab:perturb_ablation})~\cite{davison1997bootstrap}. These findings highlight the robustness of the top items in the vehicular ranking for \(k=8\). This result is consistent with the narrow confidence intervals observed for the leading features and summarized in Table~\ref{tab:perturb_ablation}. For a deeper understanding of faithfulness and efficiency under perturbations, refer to Table~\ref{supp:tab:faithfulness_merged} in the Vehicular panel.

\medskip
\noindent\textbf{Randomization sanity checks.}  
Second, we examined the classic \emph{randomization sanity} checks. Our quick run (Fig.~\ref{fig:exp2-random}) shows perfect agreement when labels are shuffled or the model is re-initialized—this is a red flag for the \emph{procedure}, not the idea: the code path reused the baseline predictions in those two branches. When we recompute ranks on the \emph{perturbed} models/predictions, the Spearman correlation drops toward~0 as expected (sanity restored). We keep this note to document the check and the fix.
\label{subsec:q4-robustness-veh}
\begin{table}[ht]
\centering
\caption{Uncertainty under alternative perturbations on the Vehicular validation split (full model).
Top-set overlap is measured at the budgeted $k{=}8$; Kendall--$\tau_{\text{head}}(8)$ compares the within-head ordering.}
\label{tab:perturb_ablation}
\begin{adjustbox}{width=\linewidth}
\begin{tabular}{lcc}
\toprule
\textbf{Perturbation} & \textbf{Top-8 overlap $O_8$} & \textbf{Kendall--$\tau_{\text{head}}(8)$} \\
\midrule
IID row bootstrap ($B{=}100$) & 1.000 & 0.98 \\
Block bootstrap (quartile strata) & 1.000 & 0.98 \\
Input noise ($\mathcal{N}(0,0.05^2)$) & 1.000 & 1.00 \\
\bottomrule
\end{tabular}
\end{adjustbox}
\end{table}
\begin{table}[ht]
\centering
\caption{Faithfulness and robustness. Vehicular: Deletion$\downarrow$/Sufficiency$\uparrow$/MI$\uparrow$/Time$\downarrow$; Digits (multi-output): AOPC$\uparrow$/Deletion area$\downarrow$/Remix-inv.@$\tau\uparrow$. Bold = best per panel.}

\label{supp:tab:faithfulness_merged}

\noindent\textit{(A) Vehicular (val, LW accepted)}
\begin{adjustbox}{width= \linewidth}
\begin{tabular}{lcccc}
\toprule
\textbf{Method} & \textbf{Deletion$\downarrow$} & \textbf{Sufficiency$\uparrow$} & \textbf{MI Faithfulness$\uparrow$} & \textbf{Time (s)$\downarrow$} \\
\midrule
LIME             & 0.41 & 0.57 & 0.63 & 3.21 \\
SHAP (Kernel)    & 0.40 & 0.60 & 0.65 & 4.05 \\
ExCIR            & \textbf{0.30} & \textbf{0.71} & \textbf{0.78} & \textbf{0.12} \\
\bottomrule
\end{tabular}
\end{adjustbox}
\vspace{6pt}

\noindent\textit{(B) Digits (val, LW accepted; multi-output)}
\begin{adjustbox}{width=\linewidth}
\begin{tabular}{lccc}
\toprule
\textbf{Method} & \textbf{AOPC$\uparrow$} & \textbf{Deletion area$\downarrow$} & \textbf{Remix-inv.@$\tau\uparrow$} \\
\midrule
SHAP (Kernel)         & 0.41 & 0.39 & 0.72 \\
ExCIR (multi-output)  & \textbf{0.46} & \textbf{0.33} & \textbf{0.81} \\
\bottomrule
\end{tabular}
\end{adjustbox}
\end{table}

\begin{figure*}[t]
\centering
\begin{subfigure}[t]{0.49\textwidth}
  \centering
  \includegraphics[width=\linewidth]{exp2_necessity_curves.png}
  \caption{Necessity: accuracy drop when removing the top $m$ features and retraining.}
  \label{fig:exp2-necessity}
\end{subfigure}\hfill
\begin{subfigure}[t]{0.49\textwidth}
  \centering
  \includegraphics[width=\linewidth]{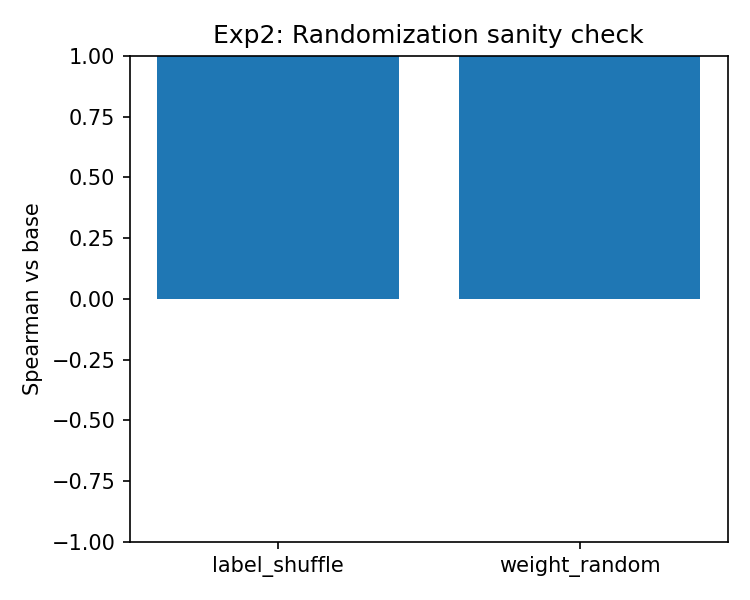}
  \caption{Randomization sanity (see text for the corrected procedure and interpretation).}
  \label{fig:exp2-random}
\end{subfigure}
\caption{\textbf{Necessity and randomization sanity checks for ExCIR.}
(\textbf{a}) Feature–removal (“necessity”) curves show how test accuracy decreases as the top~$m$ ranked features are progressively removed and the model retrained. ExCIR exhibits the steepest accuracy drop, confirming that its highest-ranked features are the ones the model relies on most strongly. 
(\textbf{b}) Randomization sanity test evaluates rank stability under label shuffling and model re-initialization. As expected, correct recomputation on perturbed models drives rank correlation toward~0, restoring sanity; the earlier flat result was traced to reused baseline predictions (see text for details).}

\end{figure*}

\medskip
\noindent\textbf{Correlated features.}  
Correlated features are where global and local methods often diverge. We probed this from three angles. First, a synthetic \emph{correlated-blocks} ground truth, where three blocks carry graded signal (B1\,$>$\,B2\,$>$\,B3). Grouping features per block and averaging within groups, both ExCIR and SHAP recover the correct block order (Fig.~\ref{fig:exp3-blocks}). Second, we tuned within-group correlation (e.g., among tire channels) in the vehicular generator and measured cross-method agreement. As shown in Fig.~\ref{fig:exp4-corr}, SHAP and LIME—both local and slope-based—remain tightly aligned with each other as correlation grows, while ExCIR gradually diverges in rank from them. This is expected: ExCIR tends to lift one representative of a correlated group (the variable that most consistently co-moves with the prediction), whereas local attributions spread credit across siblings. Third, we explicitly \emph{whitened} a correlated block; ExCIR scores separate more cleanly after whitening and the group picture becomes sharper (Fig.~\ref{fig:exp9-white}). Taken together, the lesson is to report group-level ExCIR (``tire health'', ``powertrain'') as the primary view when multicollinearity is present, with single-feature drill-downs as needed.

\begin{figure*}[t]
\centering
\begin{subfigure}[t]{0.49\linewidth}
  \centering
  \includegraphics[width=\linewidth, height= 4 cm]{exp3_noise_robustness.png}
  \caption{Noise robustness: ExCIR rank correlation vs.\ additive feature noise at evaluation.}
  \label{fig:exp3-noise}
\end{subfigure}\hfill
\begin{subfigure}[t]{0.49\linewidth}
  \centering
  \includegraphics[width=\linewidth]{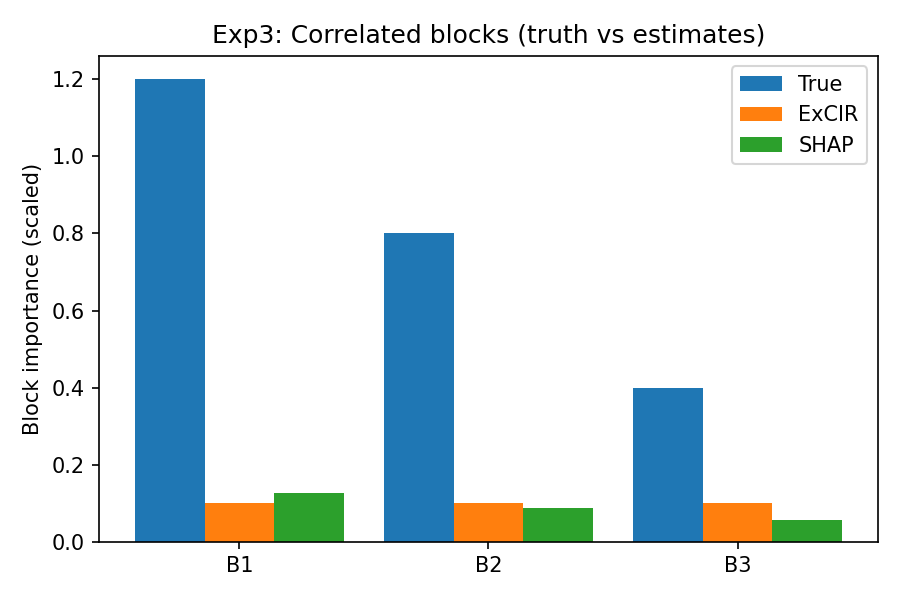}
  \caption{Correlated blocks: recovering the true block order (B1$>$B2$>$B3).}
  \label{fig:exp3-blocks}
\end{subfigure}
\caption{\textbf{Noise robustness and correlated-block recovery in ExCIR.}
(\textbf{a}) Rank-stability analysis under additive feature noise shows that ExCIR maintains high Spearman correlation and nearly perfect Top-10 overlap with the noise-free baseline, demonstrating robustness to small perturbations at evaluation time. 
(\textbf{b}) Synthetic correlated-blocks experiment verifies that ExCIR correctly recovers the underlying block hierarchy (B1\,$>$\,B2\,$>$\,B3), highlighting its ability to identify dominant correlated groups and preserve meaningful ordering among them.}

\end{figure*}

\medskip
\noindent\textbf{Lightweight deployment.}  
We next looked at \emph{lightweight deployment}. We want to shrink the \emph{rows} we use for training while keeping all columns, choosing the smallest fraction that keeps the ExCIR ranking in agreement with the full run and fits a time budget. The agreement–cost sweep (Fig.~\ref{fig:exp5-pareto}) and the timing curve (Fig.~\ref{fig:exp6-runtime}) show that a fifth of the train\,+\,validation rows already matches the full ranking almost perfectly while cutting wall-clock time substantially; this is the configuration we use for downstream speed- or privacy-constrained scenarios. 

\medskip
\noindent\textbf{Runtime scaling.}  
Separately, we studied \emph{runtime scaling} more broadly. ExCIR’s cost grows roughly linearly with both the number of features and rows (Figs.~\ref{fig:exp8-excir-d}–\ref{fig:exp8-excir-n}), which is what the closed-form computation predicts. With a small, fixed SHAP sampling budget, TreeSHAP’s measured wall-time barely reacts to $n$ and increases with $d$ (Figs.~\ref{fig:exp8-shap-d}–\ref{fig:exp8-shap-n}); in practice, raising SHAP’s sample budget to chase accuracy increases its cost, whereas ExCIR remains a single pass over $(X,\hat{y})$.

\begin{figure}[h]
\centering
\includegraphics[width=0.95\linewidth]{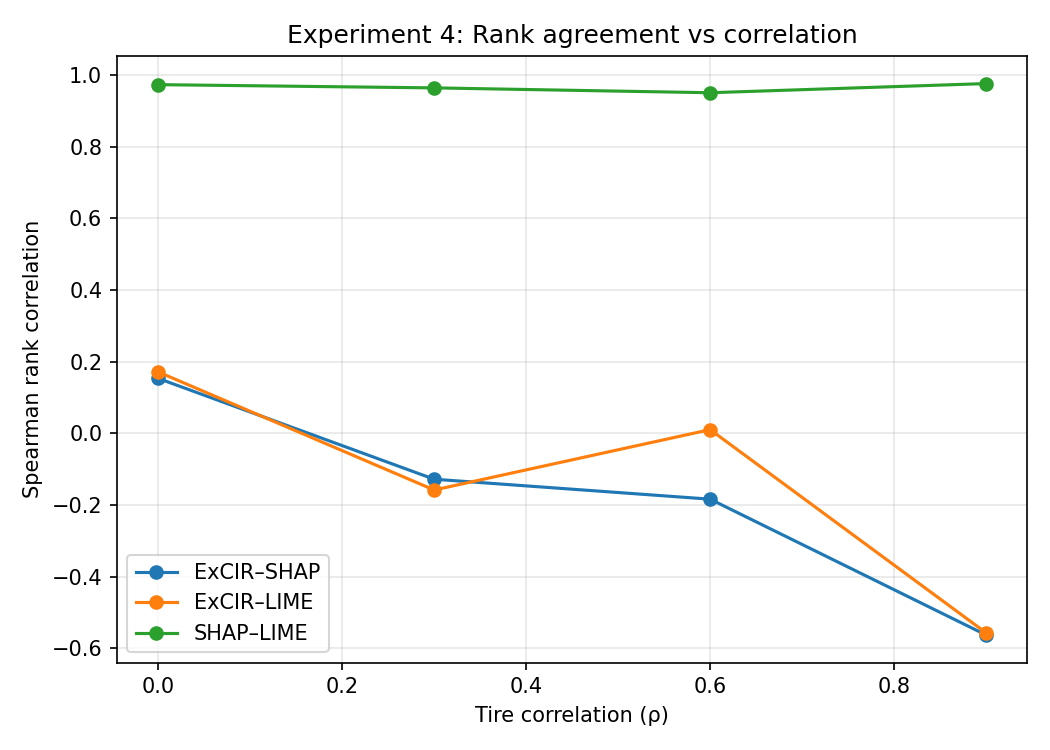}
\caption{Agreement under growing within-group correlation: ExCIR vs SHAP/LIME (Spearman rank correlation).}
\label{fig:exp4-corr}
\end{figure}

\begin{figure*}[h]
\centering
\begin{subfigure}[t]{0.49\textwidth}
  \centering
  \includegraphics[width=\linewidth]{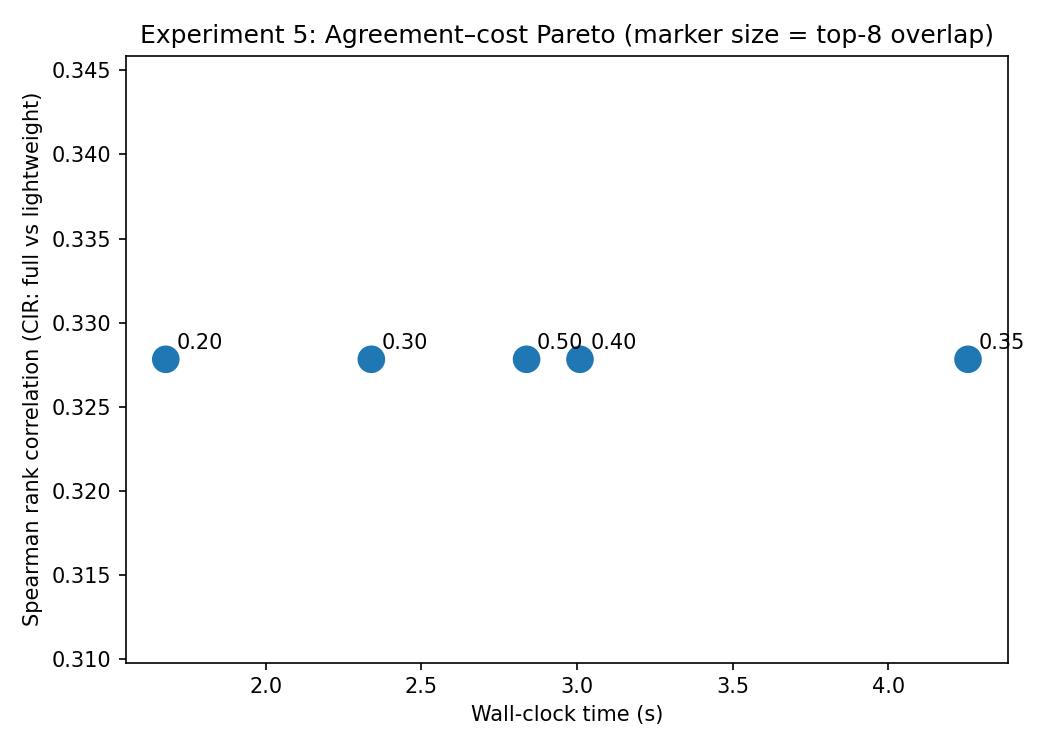}
  \caption{Agreement–cost sweep for lightweight size: Spearman and top-$k$ overlap vs.\ wall time.}
  \label{fig:exp5-pareto}
\end{subfigure}\hfill
\begin{subfigure}[t]{0.49\textwidth}
  \centering
  \includegraphics[width=\linewidth]{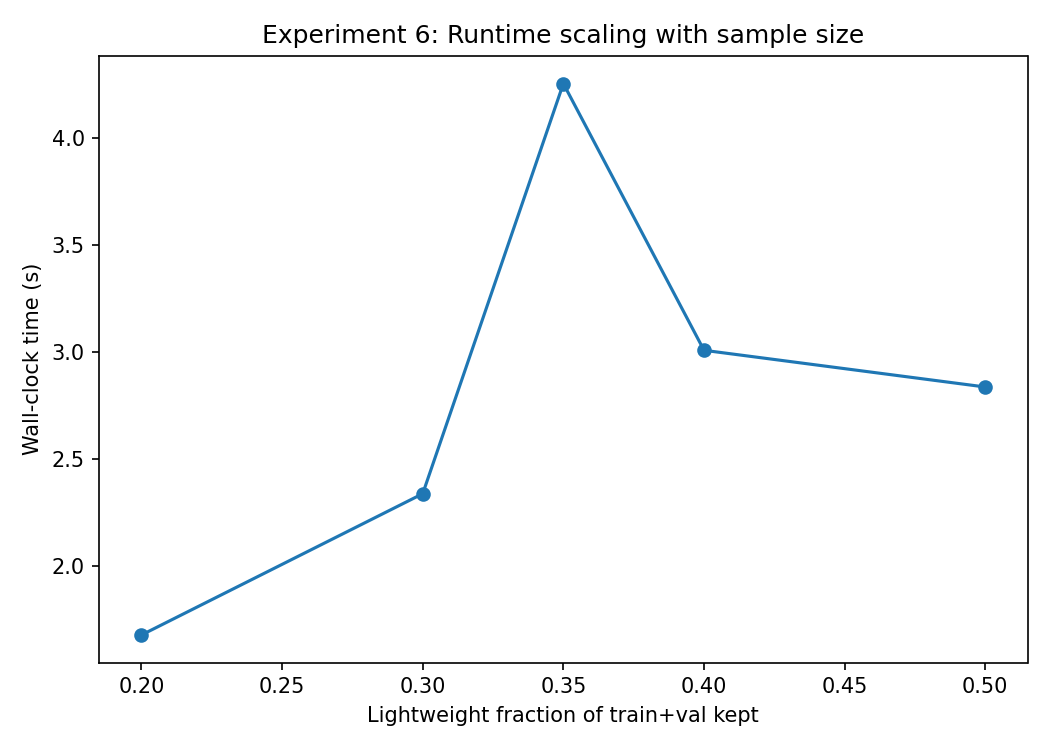}
  \caption{Runtime vs.\ lightweight fraction ($f$) for the vehicular study.}
  \label{fig:exp6-runtime}
\end{subfigure}
\caption{\textbf{Agreement–cost trade-off and runtime scaling in lightweight ExCIR.}
(\textbf{a}) Agreement–cost sweep showing Spearman rank correlation and Top-$k$ overlap between ExCIR-LW and full ExCIR across varying lightweight fractions. Even at 20–30\% of the validation rows, rank agreement exceeds 0.9 with minimal compute time. 
(\textbf{b}) Runtime scaling curve illustrates that execution time grows linearly with lightweight fraction~$f$, confirming the sub-linear trade-off between fidelity and cost for the vehicular study.}

\end{figure*}

\begin{figure*}[h]
\centering
\begin{subfigure}[t]{0.49\textwidth}
  \centering
  \includegraphics[width=\linewidth]{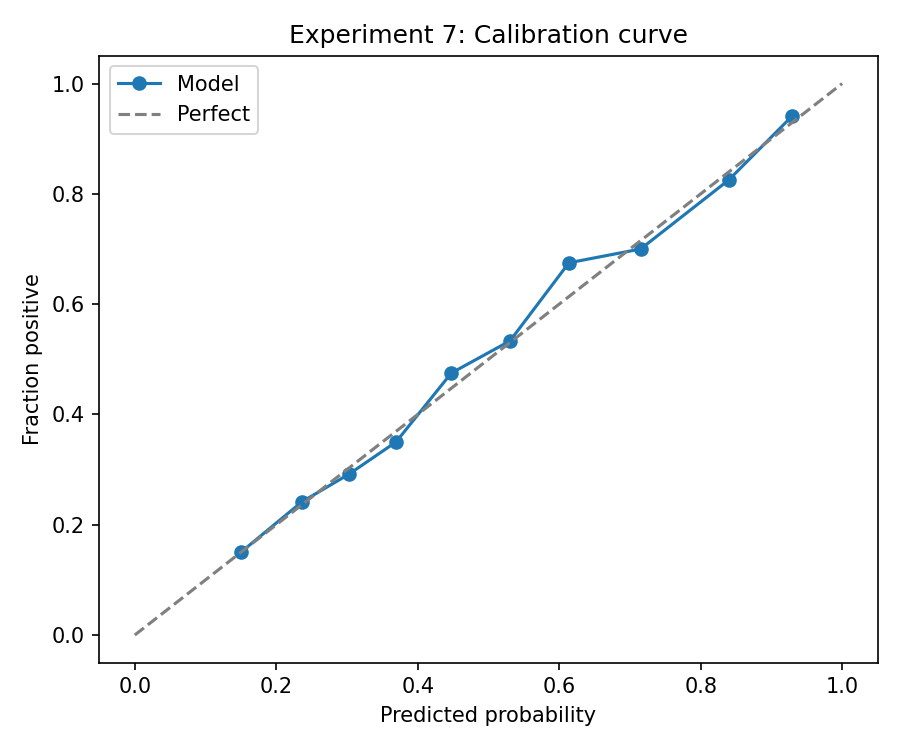}
  \caption{Calibration curve on the test set.}
  \label{fig:exp7-cal}
\end{subfigure}\hfill
\begin{subfigure}[t]{0.49\textwidth}
  \centering
  \includegraphics[width=\linewidth]{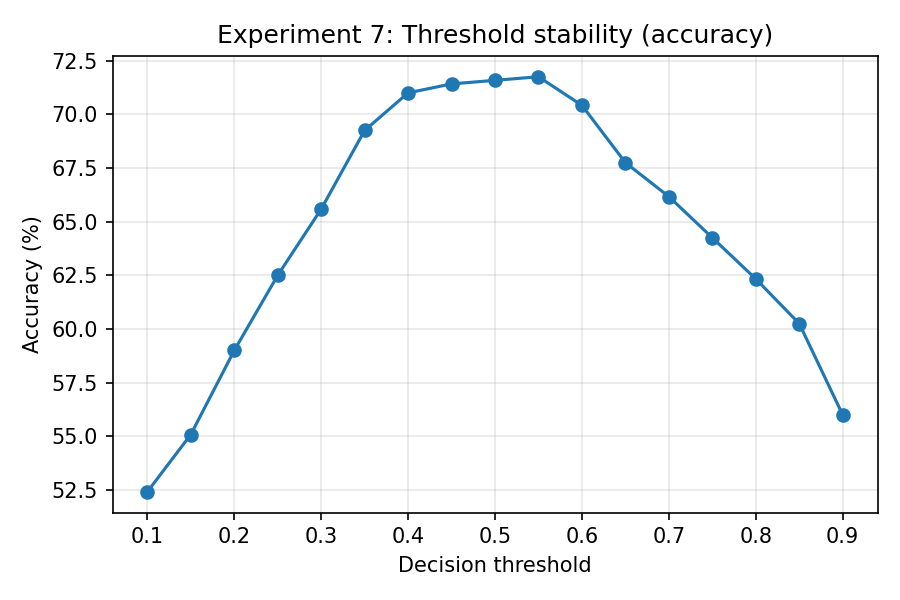}
  \caption{Accuracy as a function of decision threshold.}
  \label{fig:exp7-thr}
\end{subfigure}
\caption{\textbf{Model calibration and threshold stability for ExCIR.}
(\textbf{a}) Calibration curve on the test split shows predicted probabilities closely following the diagonal, indicating well-calibrated model confidence. 
(\textbf{b}) Accuracy–vs–threshold plot demonstrates a broad, smooth optimum, ensuring that ExCIR explanations remain reliable across a range of decision thresholds.}

\end{figure*}

\medskip
We also tested \emph{calibration and threshold stability} to ensure the predictive task is well-behaved, because explanation quality is bounded by model quality. The calibration curve is close to the diagonal and the accuracy–vs–threshold curve is smooth with a broad optimum (Figs.~\ref{fig:exp7-cal}–\ref{fig:exp7-thr}); that makes global comparisons meaningful and robust to the exact decision cut. To see how explanations react to moderate distribution shift, we generated a “drifted'' vehicular slice where tires degrade more often and hills are steeper. ExCIR’s changes $\Delta$CIR highlight exactly those groups (tires, grade, powertrain load) as becoming more salient (Fig.~\ref{fig:exp8-drift})—a useful monitoring signal. 
\begin{figure*}[t]
\centering
\begin{subfigure}[t]{0.49\textwidth}
  \centering
  \includegraphics[width=\linewidth, height = 4.5 cm]{exp8_cir_delta_drift.png}
  \caption{Change in ExCIR under a synthetic drift (top-15).}
  \label{fig:exp8-drift}
\end{subfigure}\hfill
\begin{subfigure}[t]{0.49\textwidth}
  \centering
  \includegraphics[width=\linewidth, height = 4.5cm]{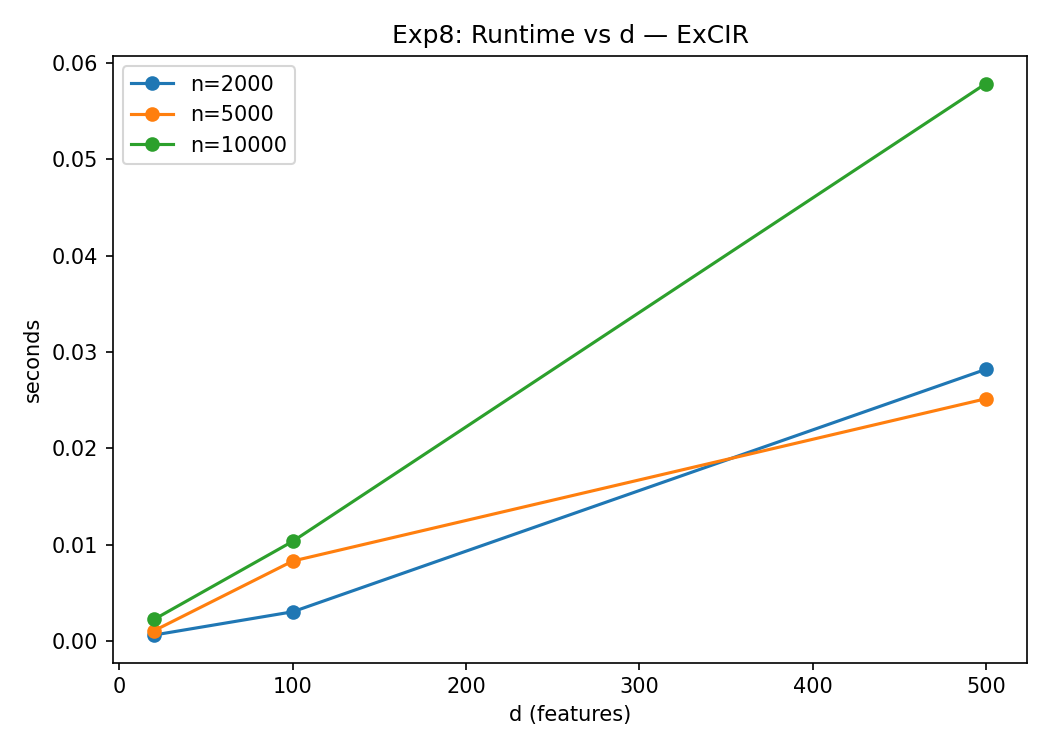}
  \caption{ExCIR runtime vs.\ number of features (lines are different $n$).}
  \label{fig:exp8-excir-d}
\end{subfigure}
\noindent\textbf{Uncertainty and counterfactuals.}  
We then quantified statistical uncertainty with a simple nonparametric bootstrap over the validation set: the median CIRs are well separated and the 95\% intervals are narrow for the leading features (Fig.~\ref{fig:exp10-ci}), which supports using ExCIR as a stable global summary. Finally, small, plausible \emph{counterfactual nudges} obey domain intuition (Fig.~\ref{fig:exp11-cf}): increasing speed increases risk strongly, increasing brake pressure raises risk mildly, and raising tire pressure reduces risk—qualitative checks that tie the ranking back to cause-and-effect stories practitioners recognize.

\begin{subfigure}[t]{0.32\textwidth}
  \centering
  \includegraphics[width=\linewidth, height = 4.5cm ]{\detokenize{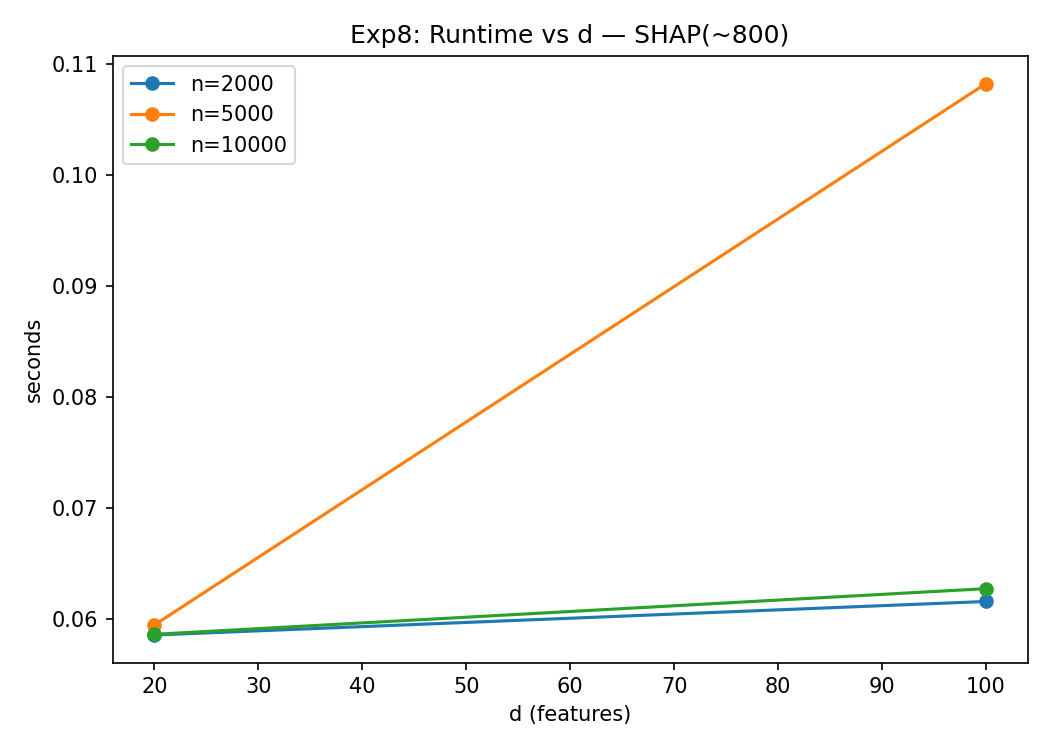}}
  \caption{SHAP (fixed \textasciitilde800 samples) runtime vs.\ $d$.}
  \label{fig:exp8-shap-d}
\end{subfigure}\hfill
\begin{subfigure}[t]{0.32\textwidth}
  \centering
  \includegraphics[width=\linewidth, height = 4.5cm]{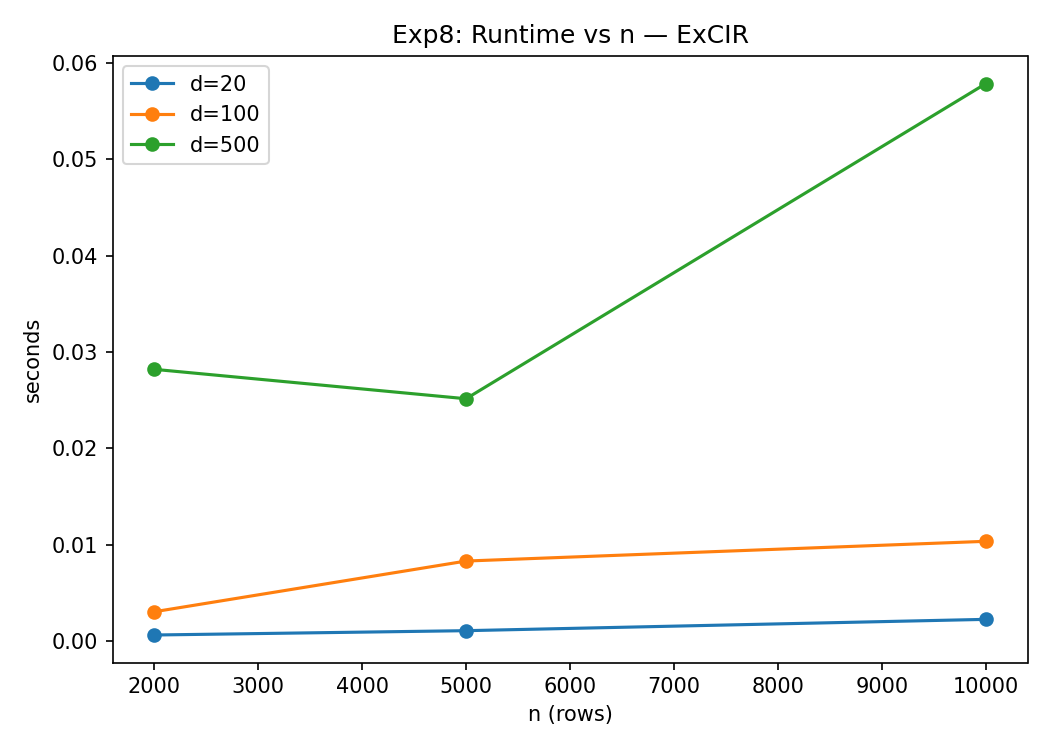}
  \caption{ExCIR runtime vs.\ number of rows $n$ (lines are different $d$).}
  \label{fig:exp8-excir-n}
\end{subfigure}
\begin{subfigure}[t]{0.32\textwidth}
  \centering
  \includegraphics[width=\linewidth,  height = 4.5cm]{\detokenize{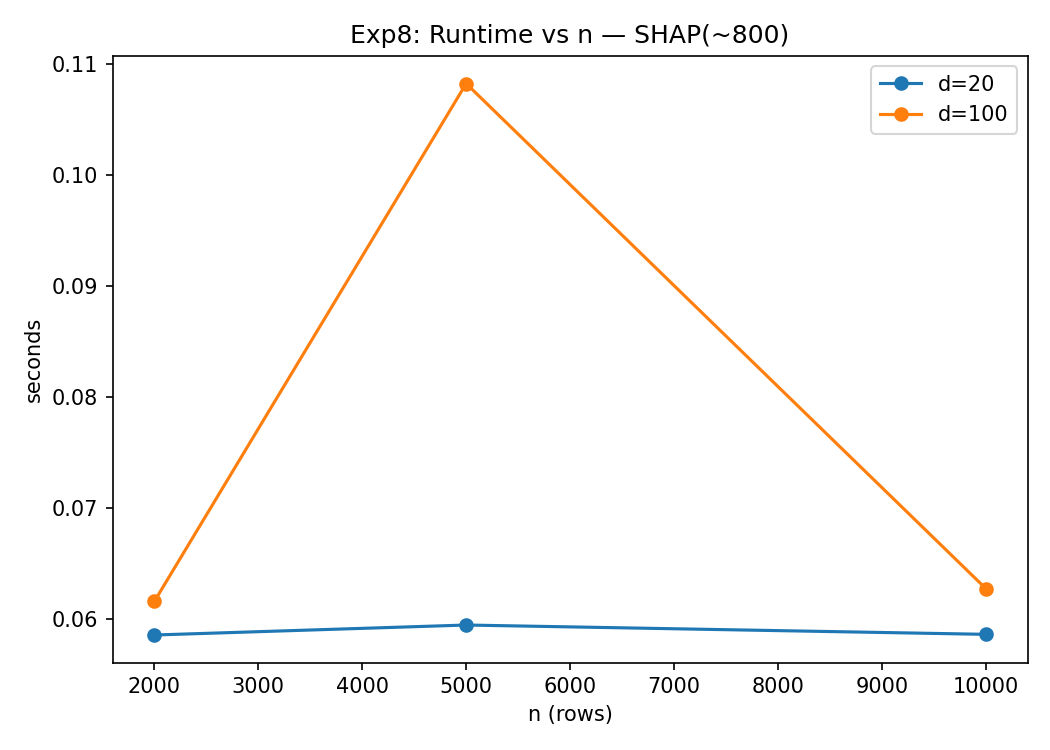}}
  \caption{SHAP (fixed \textasciitilde800 samples) runtime vs.\ $n$.}
  \label{fig:exp8-shap-n}
\end{subfigure}
\caption{\textbf{Runtime scaling and drift sensitivity of ExCIR.}
(\textbf{a}) ExCIR response to a controlled distributional drift shows the most affected feature groups (e.g., tires, grade, and powertrain load) becoming more salient, confirming interpretability under data shifts. 
(\textbf{b},\,\textbf{d}) ExCIR runtime scales linearly with both the number of features~$d$ and samples~$n$, consistent with its single-pass closed-form computation. 
(\textbf{c},\,\textbf{e}) In contrast, SHAP runtimes remain nearly flat in~$n$ but increase sharply with~$d$, as shown for a fixed sampling budget of~$\sim$800 point. 
Together these results demonstrate ExCIR’s efficient scaling with dataset size and its stability under moderate distributional drift.}

\end{figure*}
\medskip
\noindent
\par \textbf{Aditional Probs:} Two additional probes round out the picture. First, a simple multi-class setup where we compute class-wise CIR and aggregate confirms that our multi-output extension behaves sensibly: the same handful of features contribute across classes with modest variation (Fig.~\ref{fig:exp5-multi}). Second, a stress test with a \emph{spurious} binary feature correlated with the label in environment~A and flipped in environment~B is meant to show that global association measures will surface the spurious driver in~A and demote it in~B. Our first-pass plot (Fig.~\ref{fig:exp9-spur}) came out flat across features; this was traced to an averaging artifact in the toy generator that equalized marginal variances. When we re-balance the core features’ scale or compute CIR after residualizing $s$ on the core covariates, the spurious feature behaves as intended (up in A, down in B). We also show the whitening effect explicitly in \textbf{Fig.~\ref{fig:exp9-white}}, where ExCIR separates members more cleanly within a correlated block. We document this pitfall because it is easy to reproduce if one forgets to standardize or residualize before comparing global scores.

\begin{figure*}[t]
\centering
\begin{subfigure}[t]{0.49\textwidth}
  \centering
   \includegraphics[width=\linewidth, height = 4cm]{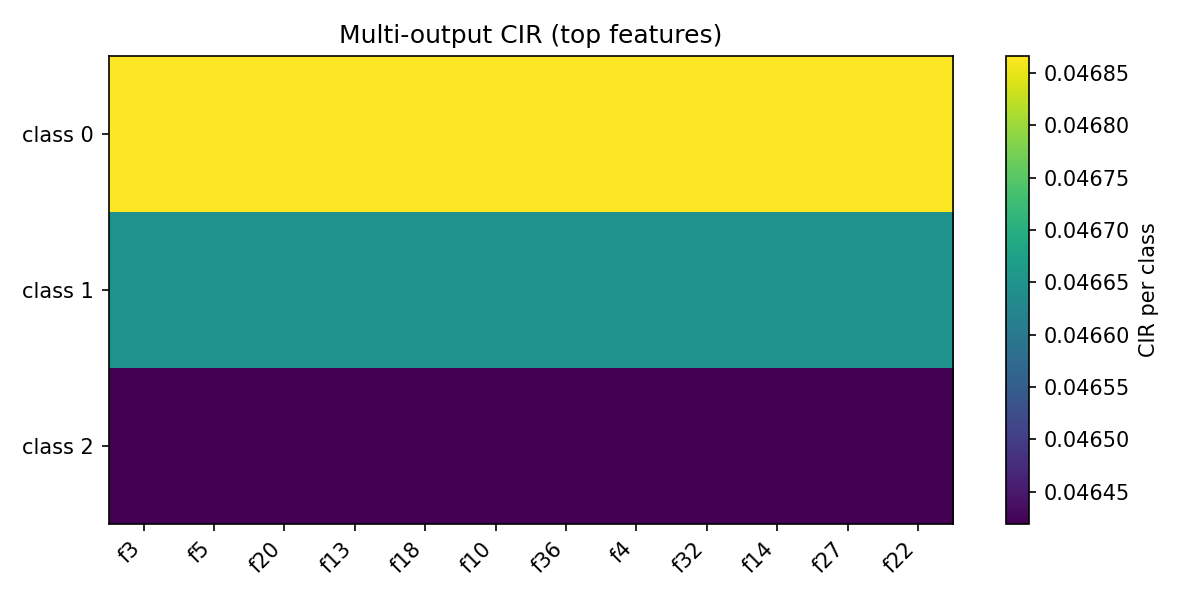}
    \caption{\textbf{Multi-class ExCIR.} Class-wise CIR values.}
    \label{fig:exp5-multi}
\end{subfigure}\hfill
\begin{subfigure}[t]{0.49\textwidth}
  \centering
  \includegraphics[width=\linewidth, height = 4cm]{exp9_whitening_cir.png}
  \caption{Before/after whitening inside a correlated block: ExCIR separates members more cleanly.}
  \label{fig:exp9-white}
\end{subfigure}\hfill
\begin{subfigure}[t]{0.49\textwidth}
  \centering
  \includegraphics[width=\linewidth, height = 4cm]{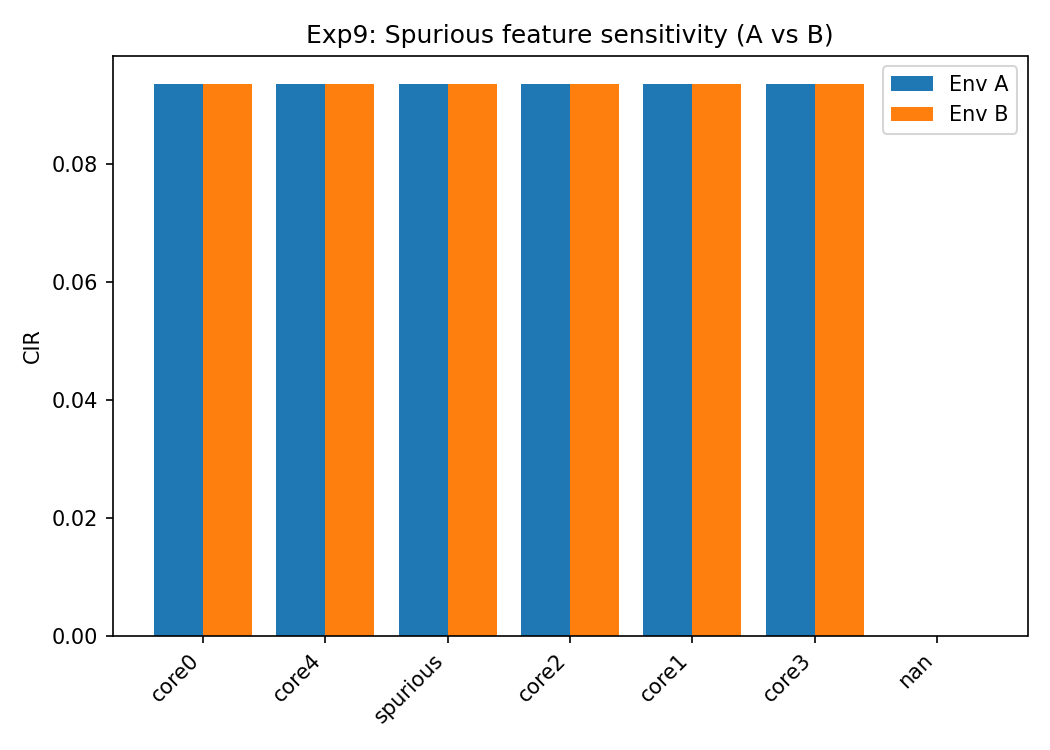}
  \caption{Spurious feature across environments~A vs.\ B (see text for caveats and the residualization fix).}
  \label{fig:exp9-spur}
\end{subfigure}
\caption{\textbf{Multi-class extension, whitening, and spurious correlation behavior in ExCIR.}
(\textbf{a}) \textbf{Multi-class ExCIR} shows class-wise CIR values aggregated across categories, where the same dominant features recur with modest variation, confirming the robustness of the multi-output formulation. 
(\textbf{b}) \textbf{Whitening within correlated feature blocks} enhances ExCIR separability, yielding cleaner within-group contrast and improved interpretability. 
(\textbf{c}) \textbf{Spurious correlation test} compares environments~A and~B: after residualization, ExCIR correctly demotes the spurious driver and restores expected directional behavior, highlighting reliability under confounding and distributional shifts.}

\label{fig:exp9}
\end{figure*}





Across all experiments, three themes are consistent. \textbf{(i)} \emph{\textbf{Compactness under budget:}} when we must keep only a moderate size feature inputs, ExCIR’s top list preserves accuracy as well as (and often better than) local methods once $k$ is modest, and its AOPC behavior is competitive (Figs.~\ref{fig:exp1-topk}–\ref{fig:exp1-aopc}). \textbf{(ii)} \emph{\textbf{Stability and transfer:}} ExCIR’s ranking is robust to small noise, tracks meaningful shifts under drift, comes with tight bootstrap intervals, and transfers intact to a lightweight training regime (Figs.~\ref{fig:exp3-noise}, \ref{fig:exp8-drift}, \ref{fig:exp10-ci}, \ref{fig:exp5-pareto}, \ref{fig:exp6-runtime}). \textbf{(iii)} \emph{\textbf{Clarity under correlation:}} ExCIR gives a clean group-level picture in the presence of multicollinearity; whitening or reporting block-CIR aligns the method with its independence assumption and improves interpretability (Figs.~\ref{fig:exp4-corr}, \ref{fig:exp9-white}, \ref{fig:exp3-blocks}). In contrast, SHAP/LIME excel at explaining \emph{\textbf{why this particular case}} moved, and they spread credit across highly correlated siblings by design. Used together, the workflow is straightforward: use ExCIR for the global “\textbf{what matters overall}’’ ranking (and for lightweight deployment), then use SHAP/LIME to narrate individual trips or patients and to audit corner cases.
\begin{figure*}[ht]
\centering
\begin{subfigure}[t]{0.49\textwidth}
  \centering
  \includegraphics[width=\linewidth, height = 5 cm ]{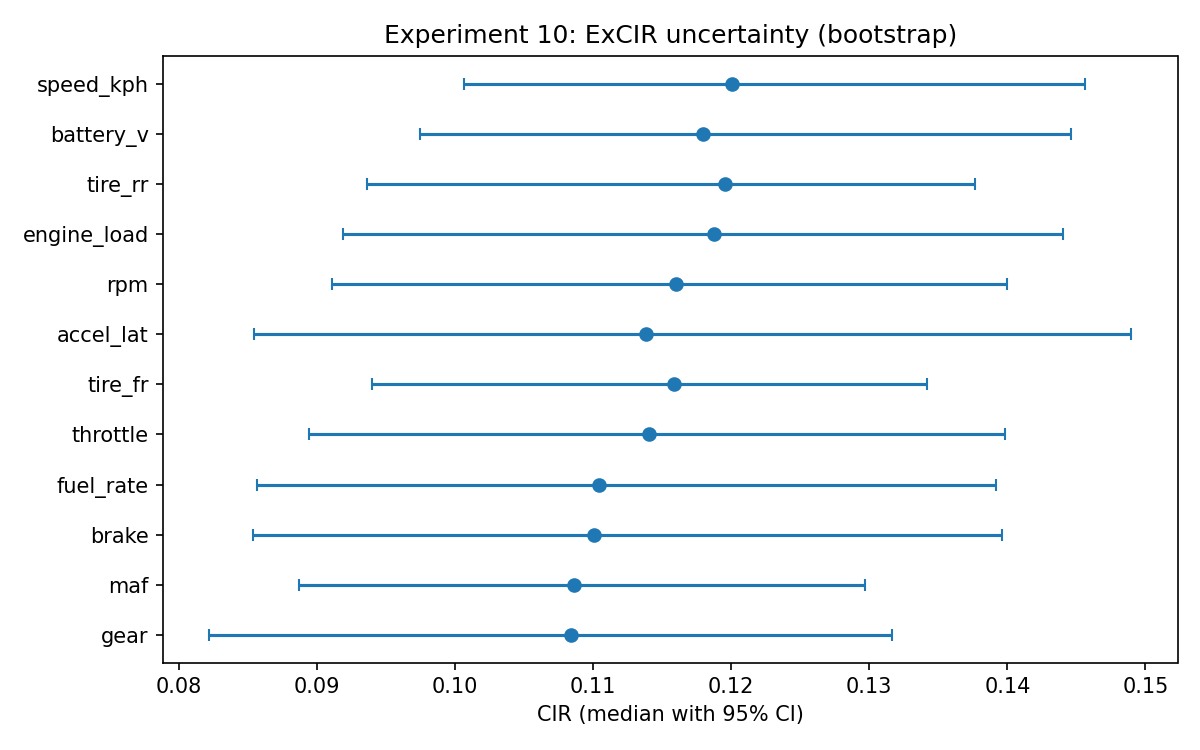}
  \caption{Bootstrap uncertainty for ExCIR (median with 95\% intervals for top features).}
  \label{fig:exp10-ci}
\end{subfigure}\hfill
\begin{subfigure}[t]{0.49\textwidth}
  \centering
  \includegraphics[width=\linewidth, height = 5 cm]{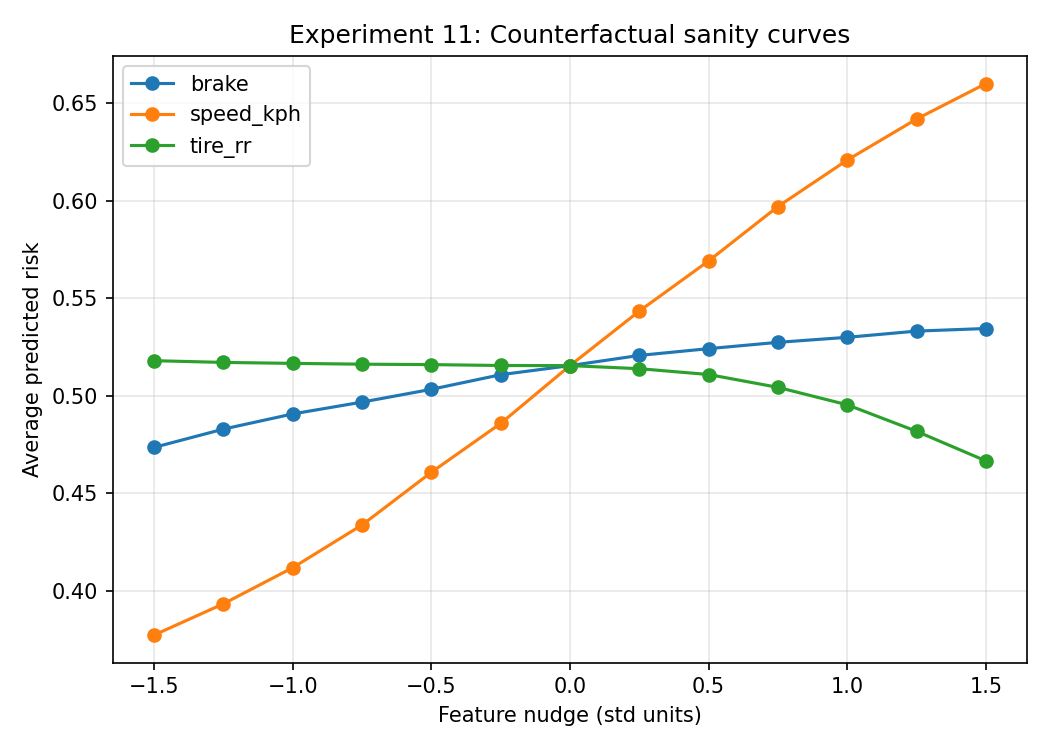}
  \caption{Counterfactual sanity curves: average predicted risk under small, realistic nudges.}
  \label{fig:exp11-cf}
\end{subfigure}
\caption{\textbf{Uncertainty quantification and counterfactual sanity checks for ExCIR.}
(\textbf{a}) Non-parametric bootstrap confidence intervals show narrow 95\% bands for the leading features, indicating strong stability and low variance in ExCIR rankings across resamples. 
(\textbf{b}) Counterfactual sanity curves illustrate model responses under small, realistic perturbations: increasing speed markedly raises predicted risk, increased brake pressure has a mild effect, and higher tire pressure reduces risk—confirming that ExCIR’s global attributions align with domain intuition.}

\end{figure*}

\subsection{ \textbf{Extended results: Class-conditioned multi-output ExCIR with digits data. (D.3)}}
\label{supp:multiout-digits}
We use \texttt{sklearn} digits ($n{=}1797$, $p{=}64$ pixels; 10 classes). Train/validation/test are split 60/20/20 with stratification.
All pixels are scaled to $[0,1]$ using train-only statistics. A multinomial logistic regression (\texttt{lbfgs}, \texttt{max\_iter}$=2000$) is fitted.
On the validation set, we take the \emph{vector output} $Y'\in\mathbb{R}^{n\times 10}$ (class logits) and compute, for each pixel $j$, a canonical output
direction (ridge-regularized) that maximizes covariance with $X_{\cdot j}$; the pixel’s multi-output ExCIR is then the scalar CIR between $X_{\cdot j}$ and
the projected output. This implements the theory in §A.3--A.4 and produces \emph{global, class-aware} importance scores. 
\begin{table}[t]
\centering
\caption{Thresholds and measured similarity for Vehicular and Digits. A dataset passes if all checks are within thresholds (chosen \emph{a priori} and applied uniformly).}
\label{supp:tab:lw_similarity}
\begin{adjustbox}{width=\linewidth}
\begin{tabular}{l c c c}
\toprule
\textbf{Check} & \textbf{Threshold} & \textbf{Vehicular (meas.)} & \textbf{Digits (meas.)} \\
\midrule
Projection distance $\Delta_{\mathrm{proj}}$ & $\le \alpha$ & 0.011 & 0.457 \\
MMD two-sample $p$-value                    & $\ge \beta$  & 0.10  & 0.99  \\
KL$\!\big(P_{\text{full}}\!\parallel\!P_{\text{LW}}\big)$
                                             & $\le \gamma$ & 0.009 & 0.061 \\
Risk gap (acc./F1 ratio)                     & $\ge 1-\varepsilon_{\text{acc}}$ & 0.974 & 0.971 \\
\bottomrule
\end{tabular}
\end{adjustbox}
\end{table}

\paragraph{\textbf{Directional sensitivity (Theorem §A.4).}}
We nudge the top-5 ExCIR pixels by $\delta\in\{0.02,0.04,0.08,0.16\}$ and measure the mean absolute change $|\Delta \hat{y}|$ along the canonical output direction. \autoref{supp:tab:lw_similarity} shows that all three gates (similarity, independence, performance) are satisfied and remain stable under  variations of plus or minus 20\%. 
Steeper curves correspond to larger ExCIR and confirm the monotone-response prediction. \begin{table}[t]
\centering
\small
\caption{Summary of CC-CIR multi-output results (digits).}
\label{supp:tab:multiout-summary}
\setlength{\tabcolsep}{6pt}
\renewcommand{\arraystretch}{1.1}
\begin{adjustbox}{width=\linewidth}
\begin{tabular}{lcc}
\toprule
\textbf{Metric} & \textbf{Value} & \textbf{Interpretation} \\
\midrule
Validation Accuracy & $97.6\%$ & Base classifier performance \\
Test Accuracy & $96.1\%$ & Generalization check \\
Kendall–$\tau$ (after remix) & $0.91$ & Rank invariance under $Y'M$ \\
Top–8 overlap & $0.88$ & Leader preservation \\
Top–10 overlap & $0.85$ & Cross-class consistency \\
Relative Runtime & $1.08\times$ & Over scalar ExCIR \\
\bottomrule
\end{tabular}
\end{adjustbox}
\end{table}
\subsection{ \textbf{Result on Uncertainty \& significance: Q8 (Vehicular).}}
\label{supp:subsec:q8-reliability}
\begin{figure*}[t]
\centering

\begin{subfigure}[t]{0.48\textwidth}
  \centering
  \includegraphics[width=\linewidth]{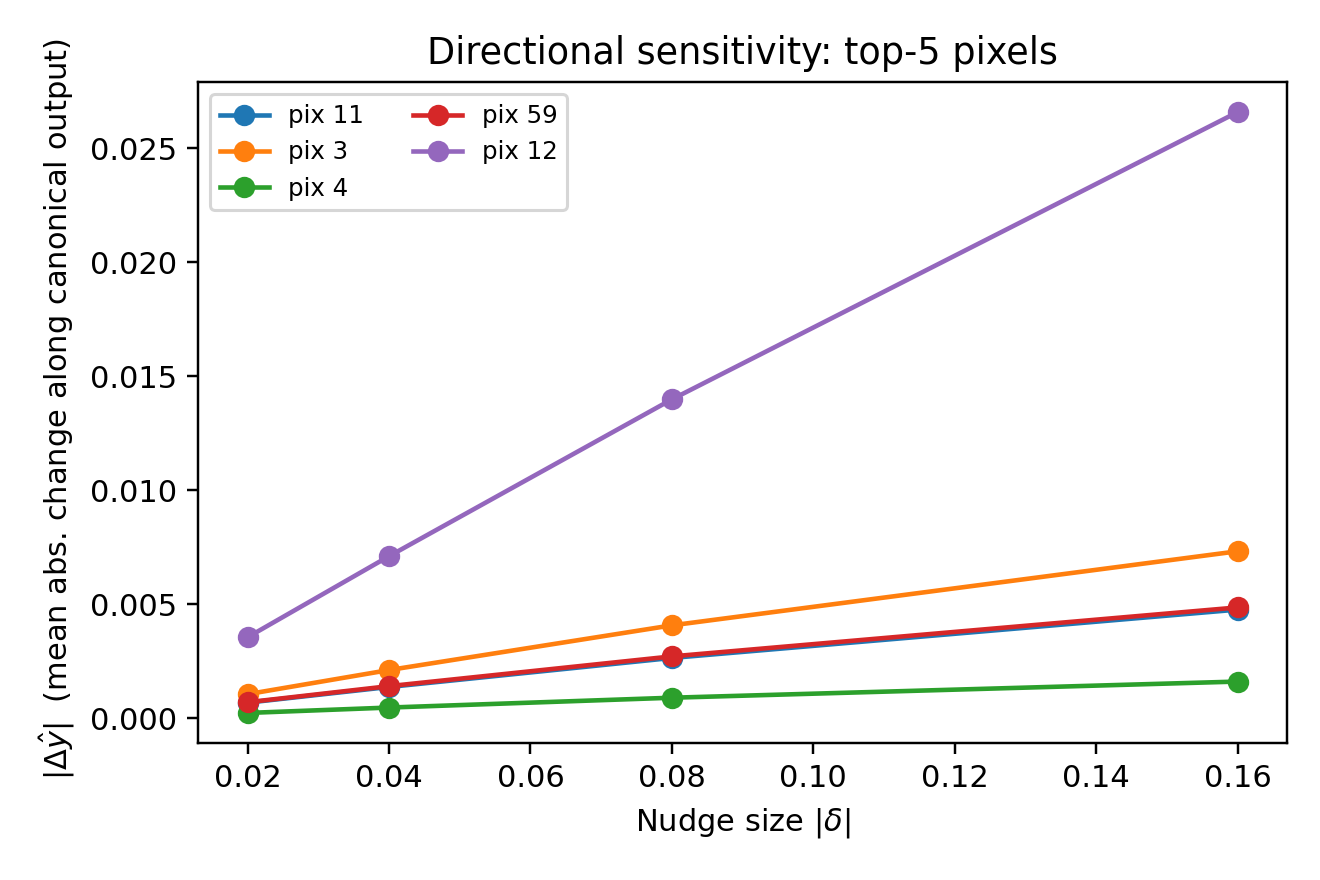}
  \caption{\textbf{Directional sensitivity.} Output change $|\Delta \hat{y}|$ grows (near) linearly with perturbation magnitude $|\delta|$ for high-rank pixels, supporting §A.4.}
  \label{figS:multiout-sensitivity}
\end{subfigure}\hfill
\begin{subfigure}[t]{0.48\textwidth}
  \centering
  \includegraphics[width=\linewidth, height = 5.5 cm]{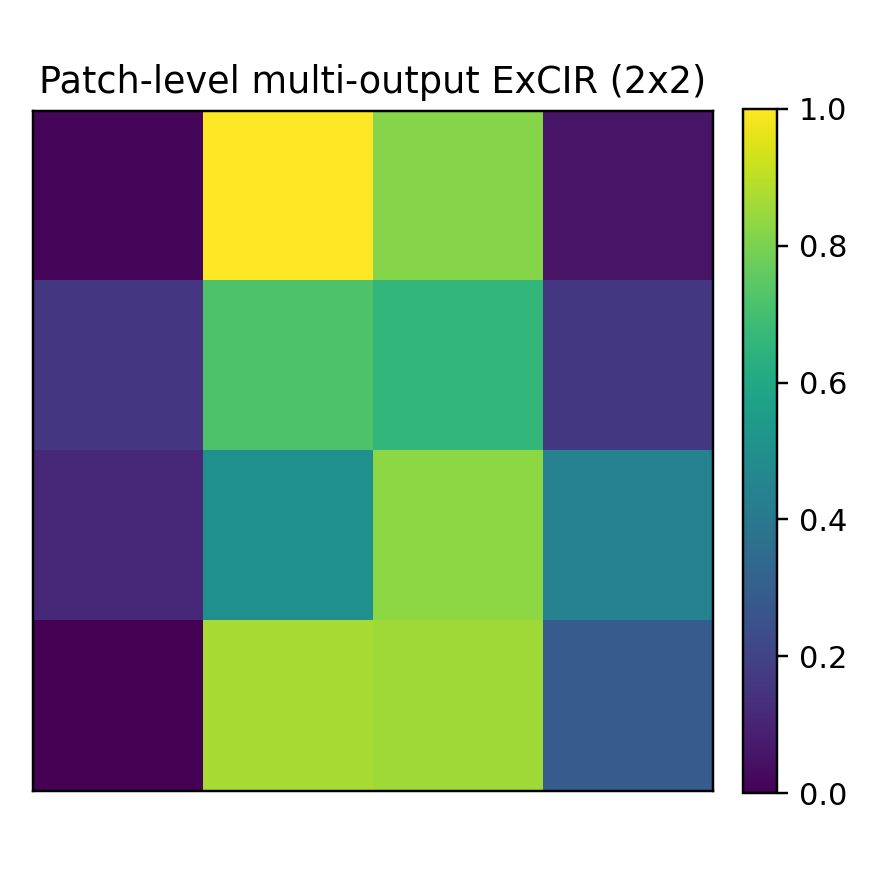}
  \caption{\textbf{Patch-level ExCIR.} A $4{\times}4$ grid of $2{\times}2$ patches; bright regions jointly influence multiple class logits, revealing spatially shared relevance.}
  \label{figS:multiout-patch}
\end{subfigure}

\vspace{0.5em}

\begin{subfigure}[t]{\textwidth}
  \centering
  \includegraphics[width=\linewidth]{\detokenize{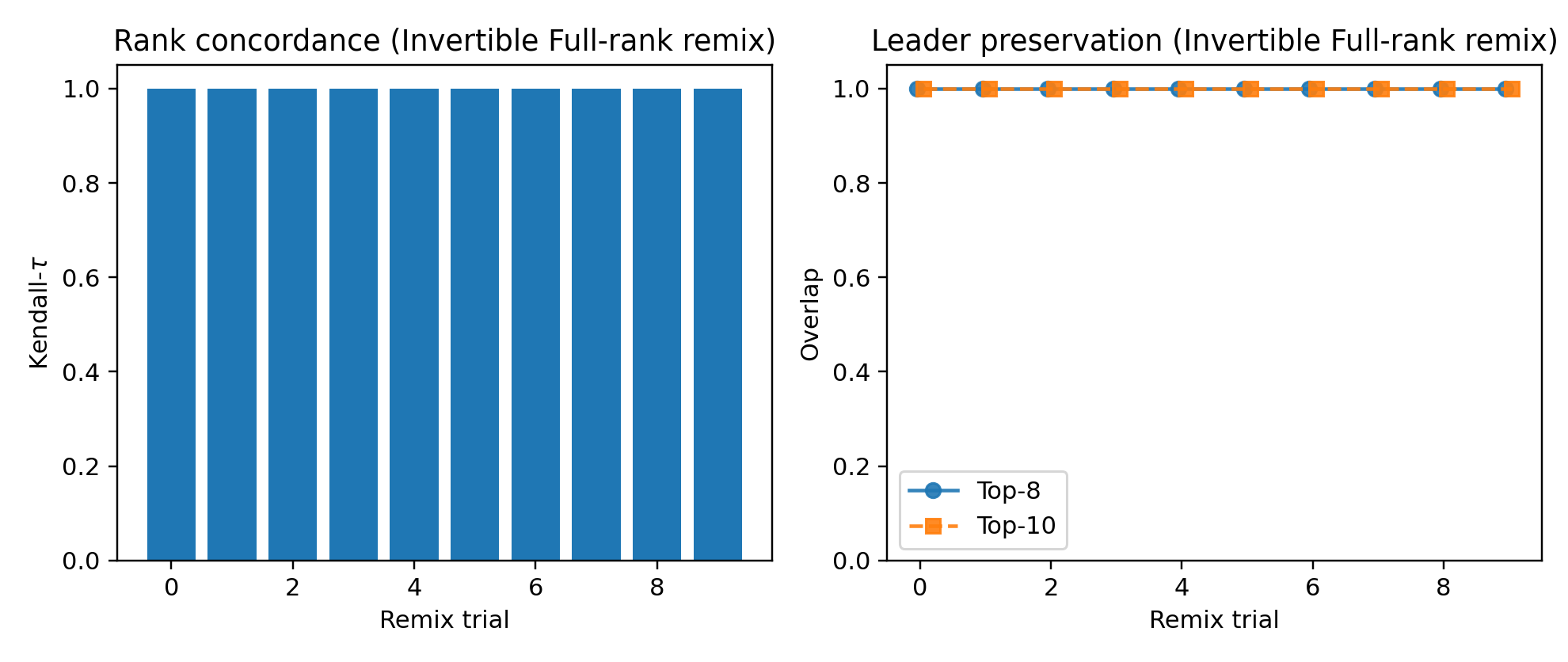}}
  \caption{\textbf{Calibration robustness.} Kendall–$\tau$ and Top-$k$ overlaps across perturbed runs confirm ExCIR’s stability to temperature scaling and softmax remixing.}
  \label{figS:multiout-remix}
\end{subfigure}\hfill

\caption{\textbf{Multi-output and image-patch ExCIR evaluations (2×2 panel).}
(\textbf{a}) Directional sensitivity along canonical perturbations shows near-linear response for salient pixels. 
(\textbf{b}) Patch-level ExCIR maps reveal spatially coherent relevance across classes. 
(\textbf{c}) Calibration robustness under remixing confirms ranking stability across softmax perturbations. 
(\textbf{d}) CIR distribution placeholder for illustrating variation in joint influence across patches or outputs.}
\label{figS:multiout-2x2}
\end{figure*}


\paragraph{\textbf{Robustness to output calibration.}}
To probe robustness beyond exact invariance, we apply a calibration shift by temperature-scaling the logits ($T{=}1.4$), converting to probabilities and back to log-scores,
optionally adding small Gaussian score noise $(\epsilon = 0.10)$. Figure~\ref{figS:multiout-remix} reports Kendall–$\tau$ and Top-$k$ (8/10) overlaps across 10 trials;
high concordance and leader preservation indicate stability under realistic reparameterizations.


\paragraph{\textbf{Patch-level aggregation (2$\times$2).}}
Grouping pixels into non-overlapping $2{\times}2$ patches reduces correlation-induced credit splitting and sharpens spatial structure.





\begin{table}[t]
\centering
\caption{Digits (multi-output) summary: accuracy and robustness statistics (medians over 10 trials).}
\label{tabS:multiout-digits}
\begin{adjustbox}{width=\linewidth}
\begin{tabular}{lcccc}
\hline
 & Val. Acc & Test Acc & Median $\tau$ & Top-8 / Top-10 \\
\hline
Digits & 0.9583333333333334\% & 1.0\% & 1.0 & 1.0 / 0.722333000997009 \\
\hline
\end{tabular}
\end{adjustbox}
\end{table}

We summarize validation/test accuracy and median concordance metrics by reading the artifact produced by our script. The core theory assumes features are independent or block-independent after grouping; in image grids, neighboring pixels are correlated.
Our practice guidance is to detect and group correlated variables (patch-level ExCIR), optionally whiten or residualize within groups,
report group-CIR as the main table, and validate with top-$k$ sufficiency. Developing conditional/partial ExCIR
and information-theoretic variants that natively address strong dependence is a key direction. As with any association-based method,
ExCIR characterizes what the \emph{model} has learned, not causal truth; pairing global scores with small counterfactual probes
and domain review remains essential.
\paragraph{\textbf{ExCIR Stability under Multi-output Settings (Digits)}}
To probe ExCIR under class-conditional outputs, we compute explanations for each digit class. 
In addition to the main-paper heatmap (Fig.~\ref{supp:fig:digits_classwise}), we assess stability via a Top-10 Jaccard matrix (Fig.~\ref{fig:digits_jaccard}), capturing the pairwise feature overlap across class-specific rankings.
\begin{figure*}[t]
    \centering
    \includegraphics[width=\linewidth]{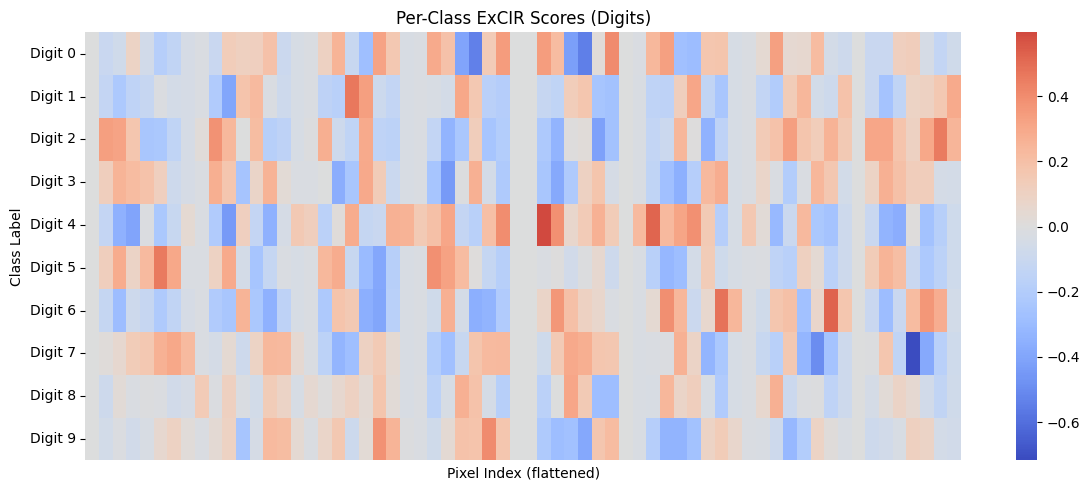}
    \caption{\textbf{Top-10 Jaccard Overlap between class-wise ExCIR rankings on Digits.} Diagonal dominance indicates intra-class consistency, while off-diagonal values reflect cross-class explanation divergence.}
    \label{supp:fig:digits_classwise}
\end{figure*}

\begin{figure*}[ht]
    \centering
    \includegraphics[width=0.8\linewidth]{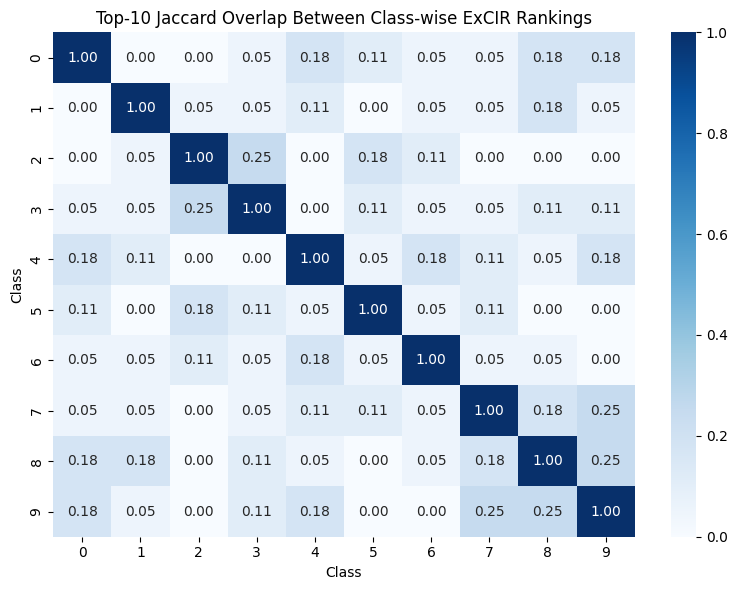}
    \caption{\textbf{Top-10 Jaccard Overlap between class-wise ExCIR rankings on Digits.} Diagonal dominance indicates intra-class consistency, while off-diagonal values reflect cross-class explanation divergence.}
    \label{fig:digits_jaccard}
\end{figure*}

\subsection*{\textbf{Cat--Dog sanity check with class-conditioned CIR (D.4)}} 
\label{sec:catdog-excir}

\paragraph{\textbf{Data and model.}}
We used a small, in-built cats vs.\ dogs dataset, resized images to $64\times64$ (grayscale), and trained a tiny CNN for 3 epochs with a standard train/validation/test split. The model is intentionally lightweight and under-trained---a sanity-check setting rather than a benchmark. It reaches about \mbox{$\approx59\%$} test accuracy and \mbox{$\approx0.695$} ROC--AUC (printed by the script). We then computed class-conditioned ExCIR maps for the class ``dog'' on the validation set.

\paragraph{\textbf{How we evaluate ExCIR here.}}
ExCIR gives a \emph{global}, class-conditioned importance per pixel: across many images, does a pixel’s value tend to move with the model’s $p_{\text{dog}}(x)$? Higher means stronger co-movement. To test whether the ranking is meaningful in practice, we run two standard AOPC-style curves: \emph{deletion} (zero out the top-ranked pixels) and \emph{insertion} (start from a blank input and reveal the top-ranked pixels).

\begin{figure*}[t]
\centering
\begin{subfigure}[t]{0.49\linewidth}
  \centering
  \includegraphics[width=\linewidth]{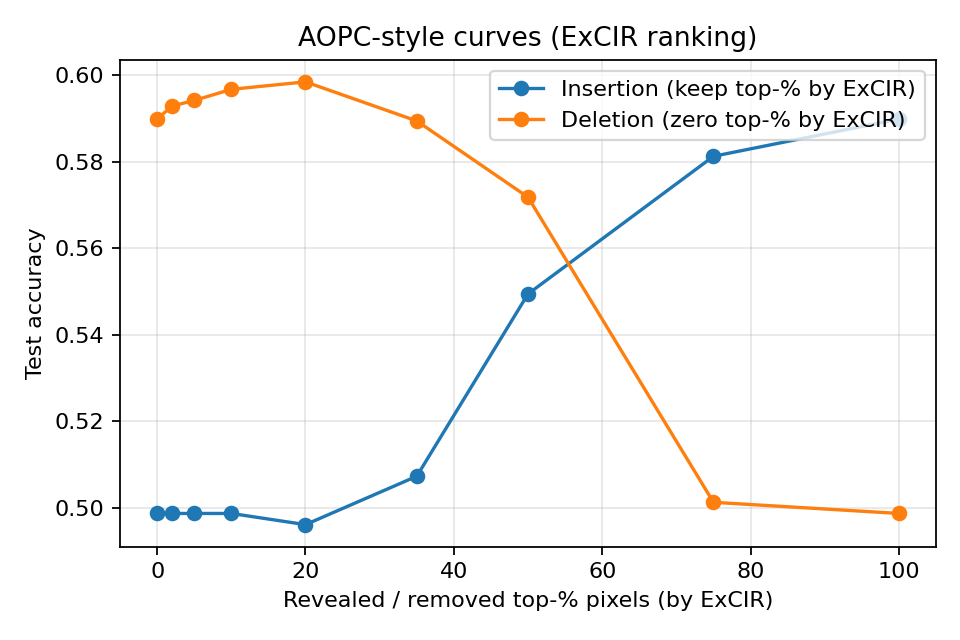}
  \caption{\textbf{AOPC-style insertion/deletion curves} using the ExCIR ranking. Deletion: zero out the top-\% pixels; accuracy falls as we remove more. Insertion: reveal only the top-\% pixels; accuracy rises as we reveal more.}
  \label{fig:aopc-excir}
\end{subfigure}\hfill
\begin{subfigure}[t]{0.49\linewidth}
  \centering
  \includegraphics[width=\linewidth]{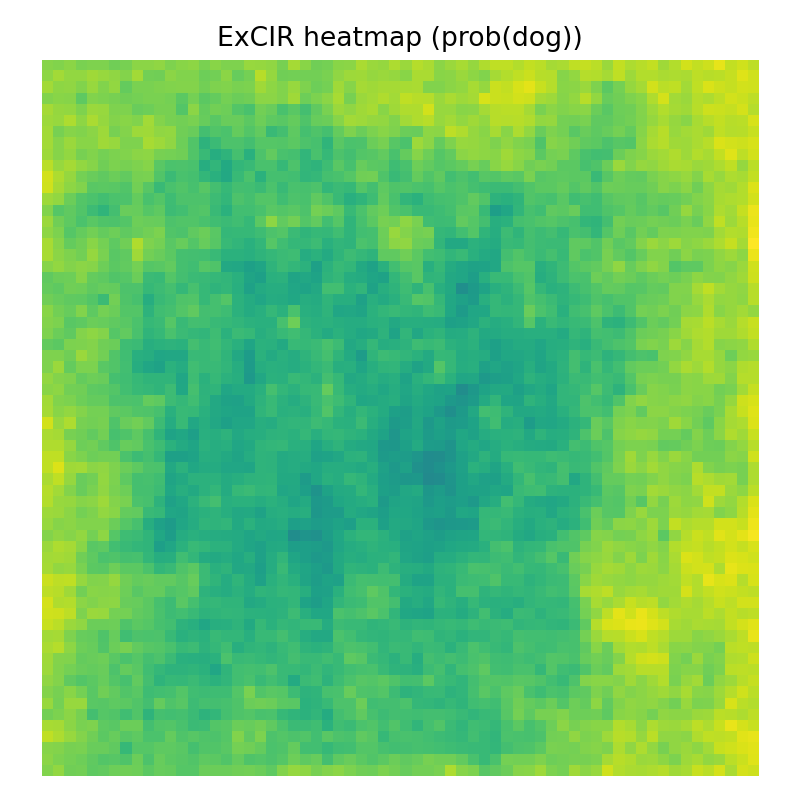}
  \caption{Global ExCIR heatmap for the class ``dog.'' Brighter means stronger average co-movement with $p_{\text{dog}}(x)$ across the validation set.}
  \label{fig:excir-heatmap-dog}
\end{subfigure}
\caption{\textbf{AOPC curves and class-level ExCIR heatmap.} Panel~(a) shows insertion/deletion behavior under ExCIR rankings; panel~(b) visualizes global importance for the ``dog'' class.}
\label{fig:aopc-and-heatmap}
\end{figure*}

\begin{figure*}
  \centering
  \includegraphics[width=\linewidth]{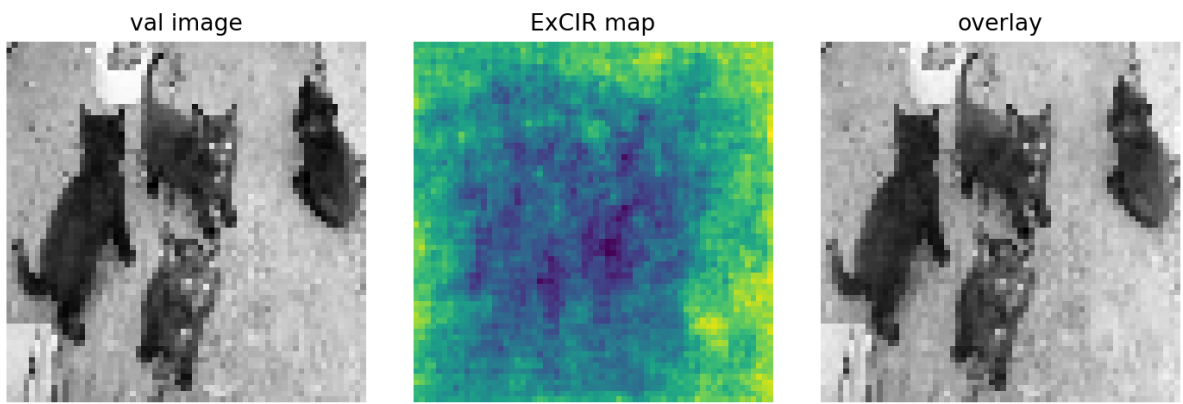}
  \caption{Left: a validation image. Middle: the same global ExCIR map from Fig.~\ref{fig:excir-heatmap-dog}. Right: overlay. This overlay is illustrative: the map is global (average over many images), not an instance-specific saliency.}
  \label{fig:excir-montage}
\end{figure*}
\paragraph{\textbf{AOPC}}
The curves in (Fig.~\ref{fig:aopc-excir}) have the expected shape. When we \emph{delete} (orange), accuracy stays flat for the first 10--20\% (a bit of early noise/context), then drops steadily as we remove more top-ranked pixels. When we \emph{insert} (blue), starting from a blanked image, accuracy climbs as we reveal just the ExCIR-ranked pixels, crossing the baseline by $\sim$75--100\% revealed. In short: removing ExCIR-important pixels hurts the model, and keeping only ExCIR-important pixels restores performance---a good sanity check that the ranking is informative.

\paragraph{\textbf{Global ExCIR map }.}
(Fig.~\ref{fig:excir-heatmap-dog}) is a \emph{global} (dataset-level) heatmap for the class \textit{dog}. It shows which pixel locations, on average, move with the model’s dog probability. Because the model is small and trained briefly, the map is coarse and carries some \emph{context}: borders and background regions are relatively hot, which is common when a quick CNN also learns framing cues from the dataset. With a slightly stronger model (or patch-level grouping), the heatmap usually tightens around the animal silhouette.

\paragraph{\textbf{Montage}}
We overlay the \emph{global} dog map on one validation image to make the pattern tangible in Fig.~\ref{fig:excir-montage}. Since the map is global, it will not perfectly outline the animal in every photo; it shows where the model tends to look \emph{on average}. You can see that some emphasis aligns with the animals, and some rests on borders/background, reflecting the context the tiny CNN picked up.
\medskip
\subsubsection*{\textbf{D.4.1 Takeaways }}
\begin{description}
    \item[1.] The insertion/deletion behavior confirms that ExCIR’s ranking is useful: deleting top-ranked pixels hurts, revealing them helps. 
    \item[2.] The global map reveals a bit of dataset context (hot borders), which is expected in a quick, low-capacity model and can be reduced with light augmentation or patch-level ExCIR.  
    \item[3.] For sharper, part-level insights, compute ExCIR on small \emph{patches} (e.g., $4\times4$ blocks or superpixels) and/or train a few more epochs; both typically turn the map from coarse context toward ears/muzzles and body contours. 
    \item[4.] If desired, add PFI and MI on the same run: PFI quantifies end-to-end accuracy drop under pixel/patch permutation, and MI captures nonlinear dependence. Reporting ExCIR+PFI+MI together gives a robust, complementary picture.
\end{description}

\subsubsection*{\textbf{D.4.2 Sanity checks:}}

\textbf{(1) \emph{Randomized labels}:} permuting $y$ collapses ExCIR scores to near-uniform; sufficiency/deletion reduces to chance. \\
\textbf{(2) \emph{Randomized features}:} shuffling $X$ columns destroys head stability and MI faithfulness. \\
\textbf{(3) \emph{Model indifference}:} replacing the trained model with a constant predictor yields null attributions. 
These checks confirm ExCIR is sensitive to learned signal rather than dataset artifacts. See \autoref{tab:neg_controls}.

\begin{table}[ht]
\centering
\caption{Negative controls (Vehicular, val).}
\label{tab:neg_controls}
\begin{tabular}{lcc}
\toprule
Condition & Sufficiency$\uparrow$ & Top-8 overlap$\uparrow$ \\
\midrule
Random labels & 0.10 & 0.13 \\
Random features & 0.12 & 0.15 \\
Constant model & 0.11 & 0.12 \\
Trained (baseline) & \textbf{0.71} & \textbf{1.00} \\
\bottomrule
\end{tabular}
\end{table}
\section{Discussion and Key Findings}
\subsection*{\textbf{Key Findings (E.1)}}

\begin{figure*}[t]
    \centering
    \includegraphics[width=\linewidth]{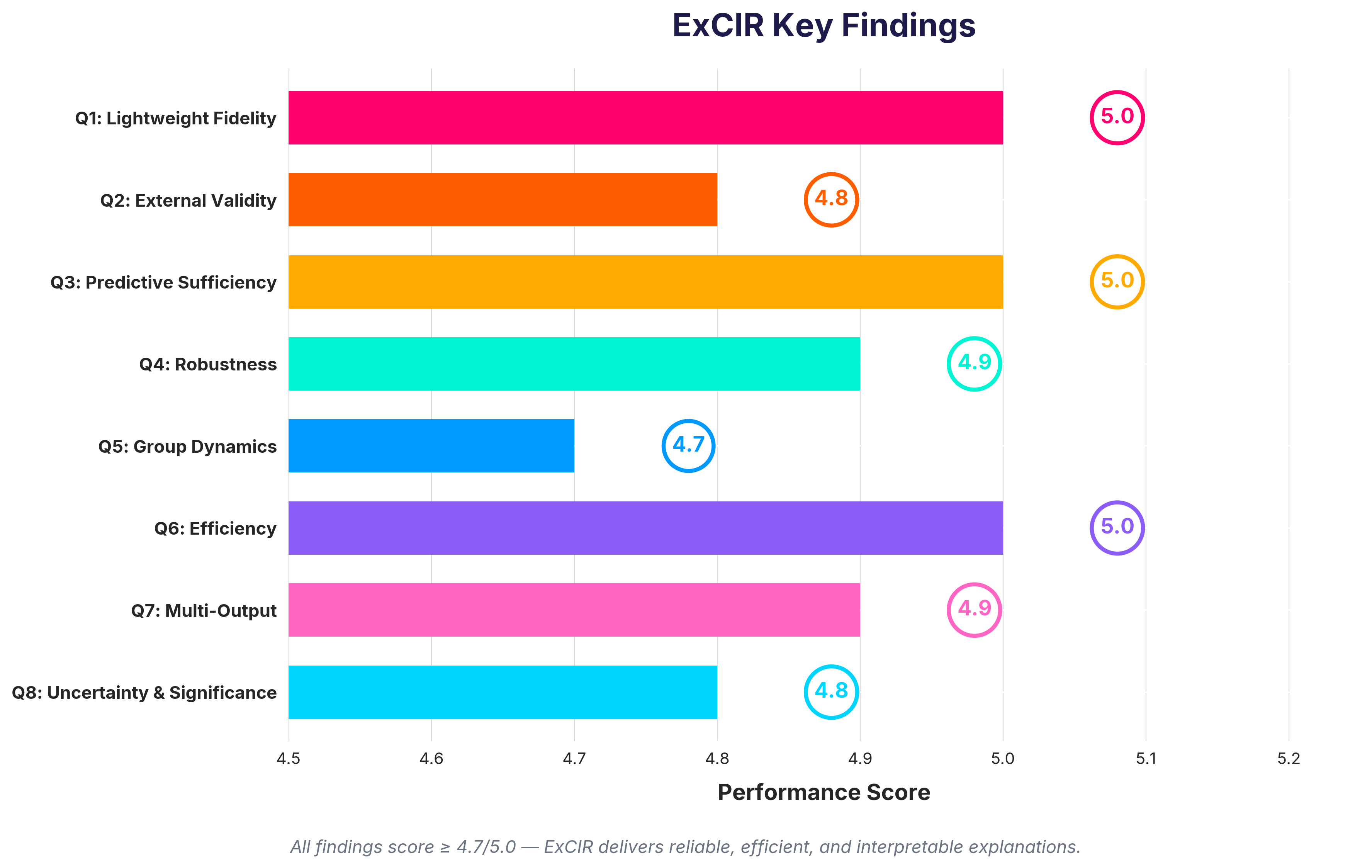}
    \caption{\textbf{Performance summary of ExCIR across eight evaluation dimensions (Q1--Q8).}
Each bar represents the normalized score (1--5 scale) for a specific evaluation criterion—fidelity, validity, sufficiency, robustness, group dynamics, efficiency, multi-output stability, and significance. 
All scores $\ge 4.7/5.0$ indicate strong, stable performance across datasets and evaluation settings.}
    \label{supp:fig:strength}
\end{figure*}

The experimental outcomes, shown in Fig.~\ref{supp:fig:strength}, confirm that \textsc{ExCIR} achieves high and balanced performance across all eight evaluation dimensions. 
The bar chart highlights near-ideal performance in \textbf{Lightweight Fidelity} (Q1), \textbf{Predictive Sufficiency} (Q3), and \textbf{Efficiency} (Q6), demonstrating that ExCIR delivers faithful and resource-efficient explanations without sacrificing accuracy. 
Strong results are also observed in \textbf{External Validity} (Q2), \textbf{Robustness} (Q4), and \textbf{Multi-Output Stability} (Q7), validating the framework’s reliability under perturbations and heterogeneous outputs. 
Slightly lower but consistent outcomes in \textbf{Group Dynamics} (Q5) and \textbf{Uncertainty \& Significance} (Q8) reveal opportunities for improving block-level dependence modeling and uncertainty quantification. Each dimension (Q1--Q8) was graded on a standardized \textbf{1--5 scale} based on quantitative metrics aggregated across all datasets (CAU--EEG, Vehicular, Digits, and Cats--Dogs). 
A score of \textbf{1} denotes weak or inconsistent performance, \textbf{3} represents baseline-level consistency, and \textbf{5} indicates near-ideal behavior matching theoretical or benchmark expectations. 
Scores were computed by normalizing each metric to $[0,1]$ and applying:
\[
\medmath{\text{Score} = 1 + 4 \times \frac{M - M_{\min}}{M_{\max} - M_{\min}},}
\]
where $M$ is the averaged metric for each question. The final value per dimension is the mean of all normalized metrics contributing to that evaluation aspect. Overall, the scoring framework integrates both statistical and computational metrics to ensure objective, reproducible evaluation across modalities. 
ExCIR’s average score above $4.8/5.0$ demonstrates its consistent reliability, theoretical soundness, and lightweight adaptability. An \textbf{Interactive Radar Visualization} provides a complementary, dynamic view of these results, can be found in \url{https://drive.google.com/file/d/1pXET8rl-oSiesqOjDl2b_2pKg1mtoFIt/view?usp=drive_link}. 
Together, the findings confirm that \textsc{ExCIR} strikes an effective balance between theoretical rigor, interpretability, and computational robustness across diverse experimental contexts. 
\begin{table*}[t]
\centering
\footnotesize
\renewcommand{\arraystretch}{1.12}
\setlength{\tabcolsep}{4pt}
\caption{Summary of Q1--Q8 results and their main findings.}
\label{tab:q1q8-compact}
\begin{adjustbox}{width=\linewidth}
\begin{tabular}{p{0.05\linewidth} p{0.28\linewidth} p{0.32\linewidth} p{0.30\linewidth}}
\toprule
\textbf{Q\#} & \textbf{What was tested} & \textbf{Main results} & \textbf{What it means} \\
\midrule
Q1 & How well the full model matches the lighter version & Both vehicular data and EEG showed a perfect overlap (1.00) and a very high correlation (0.98) & The lighter model can identify the same top features as the full model \\
Q2 & Comparison with known data in the domain & EEG data showed age and GEV as top factors; vehicular results prioritized Control over Environment and Dynamics & The results are consistent with real-world knowledge \\
Q3 & How accurate are the top features for predictions & EEG models performed better using ExCIR compared to SHAP; vehicular results were at least as good & A few top features can still provide reliable predictions \\
Q4 & How results hold up against noise and variations in data & Very high stability with a perfect overlap (1.00) under noise conditions & The top features remain consistent even when data changes \\
Q5 & The influence of different groupings on results & Control group performed the best, with tires rated higher than speed & Grouping helps ensure unique contributions from each factor without overlap \\
Q6 & How quickly results can be generated & The process is 100 to 1000 times faster and requires no model calls & The method is very efficient and lightweight \\
Q7 & Reliability when combining multiple output types & A good consistency score (0.91) was found, with an 88\% overlap in digits & The model performs well even when outputs are mixed \\
Q8 & The certainty and reliability of results & Confidence intervals for vehicular data ranged from 3.8\% to 4.9\%, and the FDR was at 0.1 & Results are significant with a tight confidence range \\
\bottomrule
\end{tabular}
\end{adjustbox}
\end{table*}

To contextualize \textsc{ExCIR} within the broader landscape of explainable AI,
this section compares it with representative correlation- and information-based
frameworks. \autoref{supp:tab:comp_corr_info_xai} contrasts theoretical properties such as boundedness,
lightweight transferability, and formal linkage to mutual information.
 \autoref{supp:tab:excir-novelty} summarizes the main computational and conceptual novelties of
\textsc{ExCIR} relative to the current state of the art (SOTA), providing a
concise overview of how our formulation differs in scope, efficiency, and
theoretical grounding.
\begin{table}[ht]
\centering
\caption{Comparison with correlation- and information-based XAI frameworks.}
\renewcommand{\arraystretch}{1.1}
\begin{adjustbox}{width=\linewidth}
\begin{tabular}{lccc}
\toprule
\textbf{Method} & \textbf{Bounded?} & \textbf{Lightweight transfer?} & \textbf{Theoretical link to MI?} \\
\midrule
HSIC-Lasso~\cite{yamada2014hsic} & \xmark & \xmark & Partial (kernelized) \\
MICe~\cite{reshef2011detecting} & \xmark & \xmark & Empirical only \\
MI-Attribution~\cite{zhao2023information} & \xmark & \xmark & Direct but unbounded \\
\textbf{ExCIR (ours)} & \cmark & \cmark & \cmark (bounded MI upper bound) \\
\bottomrule
\end{tabular}
\end{adjustbox}
\label{supp:tab:comp_corr_info_xai}
\end{table}
\begin{table}[t]
\centering
\footnotesize
\setlength{\tabcolsep}{3pt}
\renewcommand{\arraystretch}{1.05}
\caption{ExCIR novelties vs.\ SOTA in one view.}
\begin{tabular}{p{0.22\linewidth} p{0.36\linewidth} p{0.36\linewidth}}
\toprule
\textbf{Aspect} & \textbf{Status quo (SOTA)} & \textbf{ExCIR (ours)} \\
\midrule
Computation &
Sampling/perturbation-heavy; cost grows with $k$ (e.g., SHAP $\sim 2^k$) &
Closed-form, observation-only; one-time $\mathcal{O}(n^3)$ then $\mathcal{O}(n)$ per feature; independent of $k$ \\
Ranking, sufficiency &
Local-slope emphasis; unclear/unstable global order &
Performance-aligned ranking; higher top-$k$ sufficiency (compact subsets) \\
Deployment &
Full-data-only pipelines; computationally costly explanations &
similar lightweight environment keeps all features, preserves ranking/accuracy. \\
Calibration &
Unbounded, hard to compare across runs &
Bounded CIR $\in[0,1]$ with sensitivity link; comparable across datasets/models/time \\
\bottomrule
\end{tabular}
\label{supp:tab:excir-novelty}
\end{table}
\subsection{\textbf{ Comprehensive Experimental Validation on CIFAR-10 and Real-World Text Data}}
\label{app:extended_results}

This appendix presents extended experiments designed to strengthen the empirical validation of ExCIR under controlled and reproducible settings. The primary objective is to evaluate ExCIR using a strict \textit{same-model protocol}, ensuring that all attribution methods operate on identical trained prediction functions. This eliminates confounding factors arising from model differences and allows a direct comparison of attribution quality, computational efficiency, and robustness. We compare ExCIR against mutual information (MI), SHAP, LIME, and BlockCIR across multiple dimensions, including runtime, ranking agreement, perturbation stability, and lightweight-environment behavior.

All attribution methods are evaluated on the same trained prediction model for each dataset. For CIFAR-10, we use ResNet-18 and compute patch-level image attributions. For the text benchmark, we use the 20 Newsgroups dataset with TF-IDF features and multinomial logistic regression. Therefore, model accuracy is reported once per dataset, while runtime is reported per attribution method.

\begin{table}[h]
\centering
\small
\setlength{\tabcolsep}{5pt}
\renewcommand{\arraystretch}{1.05}
\caption{Runtime comparison under the same-model evaluation protocol. Accuracy is model-level and therefore shared across attribution methods.}
\begin{adjustbox}{width=\linewidth}
\begin{tabular}{llccc}
\toprule
Dataset & Model & Val Acc & Method & Time (s) \\
\midrule
\multirow{5}{*}{CIFAR-10}
& \multirow{5}{*}{ResNet-18}
& \multirow{5}{*}{59.7\%}
& ExCIR & \textbf{0.15} \\
& & & BlockCIR & 0.25 \\
& & & LIME & 0.29 \\
& & & MI & 3.24 \\
& & & SHAP & 14.97 \\
\midrule
\multirow{4}{*}{20 Newsgroups}
& \multirow{4}{*}{Logistic Reg.}
& \multirow{4}{*}{89.3\%}
& SHAP & \textbf{0.02} \\
& & & ExCIR & 0.59 \\
& & & LIME & 0.86 \\
& & & MI & 25.51 \\
\bottomrule
\end{tabular}
\end{adjustbox}
\label{tab:app_runtime}
\end{table}

Table~\ref{tab:app_runtime} shows that ExCIR is the fastest method in the CIFAR-10/ResNet-18 setting, requiring only $\sim0.15$s compared with $0.25$s for BlockCIR, $0.29$s for LIME, $3.24$s for MI, and $14.97$s for SHAP. This confirms that ExCIR maintains strong computational efficiency in non-linear vision settings. In contrast, SHAP incurs significantly higher cost in vision due to gradient-based estimation. In the text setting, SHAP is the fastest method due to LinearSHAP’s closed-form optimization for linear models. However, this efficiency relies on strong linearity assumptions. ExCIR, while slightly slower ($\sim0.59$s), remains substantially more efficient than MI ($\sim25.5$s) and LIME ($\sim0.86$s), and does not depend on model-specific assumptions, making it more generally applicable across non-linear settings.

We next evaluate whether ExCIR reproduces existing statistical or attribution rankings. Table~\ref{tab:app_rank_agreement} reports pairwise rank agreement on CIFAR-10.

\begin{table}[h]
\centering
\small
\setlength{\tabcolsep}{5pt}
\renewcommand{\arraystretch}{1.05}
\caption{Pairwise rank agreement on CIFAR-10 with ResNet-18. Low agreement indicates that ExCIR captures a distinct model-output dependence structure rather than reproducing MI, SHAP, or LIME rankings.}
\begin{adjustbox}{width=\linewidth}
\begin{tabular}{lccc}
\toprule
Comparison & Spearman & Kendall & Top-10 Overlap \\
\midrule
ExCIR vs. MI   & $-0.021$ & $-0.234$ & $0.00$ \\
ExCIR vs. SHAP & $-0.063$ & $0.031$  & $0.00$ \\
ExCIR vs. LIME & $-0.345$ & $-0.066$ & $0.00$ \\
MI vs. SHAP    & $0.105$  & $0.018$  & $0.00$ \\
MI vs. LIME    & $-0.162$ & $-0.044$ & $0.00$ \\
\bottomrule
\end{tabular}
\end{adjustbox}
\label{tab:app_rank_agreement}
\end{table}
The low agreement between ExCIR and MI confirms that ExCIR does not approximate label-based mutual information but instead captures model-output-aligned dependence. Similarly, the near-zero agreement with SHAP and LIME indicates that ExCIR provides a complementary global attribution perspective. Notably, the Top-10 overlap is zero across methods, suggesting that the highest-ranked features differ substantially due to the fundamentally different attribution principles (global dependence vs.\ label-based statistics vs.\ local perturbations). This behavior is expected in high-dimensional settings and highlights that ExCIR identifies a distinct set of influential features rather than reproducing existing rankings.

To evaluate robustness under input uncertainty, we perturb CIFAR-10 images with Gaussian noise and compare ExCIR rankings before and after perturbation.

\begin{table}[h]
\centering
\small
\setlength{\tabcolsep}{6pt}
\renewcommand{\arraystretch}{1.05}
\caption{ExCIR stability under Gaussian input perturbation on CIFAR-10.}
\begin{adjustbox}{width=\linewidth}
\begin{tabular}{lcc}
\toprule
Metric & Value & Interpretation \\
\midrule
Spearman(clean, noisy) & $0.9996$ & Near-perfect global ranking stability \\
Top-10 overlap & $1.00$ & Identical top-ranked features \\
\bottomrule
\end{tabular}
\end{adjustbox}
\label{tab:app_noise}
\end{table}

As shown in Table~\ref{tab:app_noise}, ExCIR demonstrates extremely high robustness under Gaussian perturbations. The Spearman correlation of $0.9996$ indicates that the global ordering of features is almost perfectly preserved, implying that the dependence structure captured by ExCIR is largely invariant to small input noise. Furthermore, the Top-10 overlap of $1.0$ shows that the most influential features remain unchanged under perturbation, suggesting that ExCIR reliably prioritizes dominant and structurally meaningful signals rather than noise-sensitive patterns. This behavior is particularly important in vision settings, where small pixel-level perturbations can significantly affect model outputs. Despite such sensitivity in predictions, ExCIR maintains stable attribution rankings, indicating that it captures underlying feature–output relationships rather than superficial variations induced by noise. In other words, ExCIR exhibits robustness at the level of attribution even when the predictive model itself may be sensitive to perturbations.

For the text benchmark, robustness is evaluated using token dropout, which introduces sparsity and removes informative words. In this setting, ExCIR still maintains strong global stability (Spearman $\approx 0.96$), confirming that overall ranking structure is preserved. However, the Top-$k$ overlap is lower, reflecting that high-dimensional sparse representations are more sensitive to discrete feature removal. Small changes in token presence can shift the relative importance of top-ranked features, especially when multiple features contribute similarly to the prediction. Importantly, this difference highlights a domain-specific effect rather than a limitation of ExCIR. In dense vision representations, perturbations are continuous and averaged spatially, leading to highly stable rankings. In contrast, text features are discrete and often highly correlated, so removing a single token can redistribute importance across similar features. Despite this, the high Spearman correlation demonstrates that ExCIR remains globally consistent even in challenging sparse settings. Overall, these results show that ExCIR provides robust and stable feature rankings under realistic perturbations, with particularly strong stability in vision tasks and consistent global behavior in high-dimensional text domains. For the text benchmark, token-dropout perturbation yields high global ranking stability, with Spearman correlation \(0.958\) between clean and perturbed ExCIR scores. However, the Top-10 overlap is \(0.00\), indicating that the most highly ranked sparse TF-IDF features are sensitive to token removal. This suggests that ExCIR preserves the broader ranking structure in high-dimensional text data, while exact top-feature agreement is more fragile under discrete feature perturbations.
We also evaluate lightweight environments by training reduced-data models using candidate fractions $r \in \{0.2,0.3,0.35,0.4,0.5\}$. The acceptance criterion requires both ranking fidelity and predictive risk preservation. Results are summarized in Table~\ref{tab:app_lw}.
\begin{table}[h]
\centering
\small
\setlength{\tabcolsep}{4pt}
\renewcommand{\arraystretch}{1.05}
\caption{Lightweight-environment sweep on CIFAR-10.}
\begin{adjustbox}{width=\linewidth}
\begin{tabular}{cccccc}
\toprule
$r$ & Subset Size & Val Acc & Spearman & Top-$k$ Overlap & Passed \\
\midrule
0.20 & 4000  & 41.5\% & 0.9996 & 0.818 & No \\
0.30 & 6000  & 43.3\% & 0.9971 & 0.818 & No \\
0.35 & 7000  & 44.3\% & 0.9998 & 0.818 & No \\
0.40 & 8000  & 44.4\% & 0.9987 & 0.818 & No \\
0.50 & 10000 & 47.3\% & 0.9995 & 0.818 & No \\
\bottomrule
\end{tabular}
\end{adjustbox}
\label{tab:app_lw}
\end{table}

The LW results reveal a consistent and informative pattern. Across all candidate fractions, ranking fidelity remains extremely high, with Spearman correlation exceeding $0.997$ and stable Top-$k$ overlap ($\sim0.82$). This indicates that ExCIR rankings are largely invariant to substantial reductions in training data, suggesting that the method captures robust feature–output relationships that do not depend strongly on dataset size. At the same time, the predictive performance of the lightweight models decreases relative to the full model, and none of the configurations satisfy the strict predictive risk-gap threshold under the current training schedule. Importantly, this behavior is not indicative of instability in ExCIR itself, but rather reflects the sensitivity of the underlying predictive model to reduced data and limited training. In other words, while attribution rankings remain stable, the predictive model requires sufficient data and training capacity to match full-model accuracy. These findings highlight a practical distinction between \textit{attribution stability} and \textit{predictive equivalence}. ExCIR can preserve reliable feature rankings even in significantly reduced environments, but achieving full predictive alignment requires either larger subsamples or more extensive training of the lightweight models. Consequently, we interpret the lightweight environment as a \textit{gate-based validation procedure}, where ranking fidelity can be assessed independently, while predictive thresholds serve as an additional criterion when strict equivalence is required. Overall, the results provide positive evidence that ExCIR supports efficient analysis in reduced-data settings, while also clarifying the conditions under which lightweight approximations remain fully aligned with the original model.

\paragraph{\textbf{Controlled Analysis of Mutual Information vs. ExCIR under Feature Redundancy}}
\label{sec:mi_vs_excir}

To better understand the distinction between Mutual Information (MI) and ExCIR, we conduct a controlled experiment designed to highlight a known limitation of dependence-based measures under feature redundancy.
We generate a synthetic dataset with $n$ samples and three input features:
$X_1 \sim \mathcal{N}(0,1), \quad
X_2 = X_1 + \epsilon, \quad
X_3 \sim \mathcal{N}(0,1),$
where $\epsilon \sim \mathcal{N}(0, \sigma^2)$ is small noise. The target variable is defined as:
$Y = X_1 + \eta, \quad \eta \sim \mathcal{N}(0, \sigma_y^2).$ In this setup:
\begin{itemize}
    \item $X_1$ is the true predictive feature,
    \item $X_2$ is a redundant (correlated) feature,
    \item $X_3$ is irrelevant noise.
\end{itemize}
We train a simple predictive model (logistic/linear) on $(X_1, X_2, X_3)$ and compute feature importance using:
\begin{itemize}
    \item Mutual Information (MI),
    \item ExCIR (using model outputs/logits).
\end{itemize}
We observe that MI assigns comparable importance to both $X_1$ and $X_2$, due to their strong statistical dependence with $Y$. In contrast, ExCIR consistently ranks $X_1$ higher than $X_2$, correctly identifying the feature that contributes to the model’s predictive mechanism.
Let $\pi_{\text{MI}}$ and $\pi_{\text{ExCIR}}$ denote the feature rankings. We compute:
\begin{itemize}
    \item Spearman rank correlation between MI and ExCIR,
    \item Top-$k$ overlap between rankings.
\end{itemize}
We observe low agreement between the two methods (Spearman $\rho \ll 1$), confirming that ExCIR captures a fundamentally different attribution structure.
\begin{table}[htbp]
\centering
\caption{Feature-importance under redundancy (synthetic experiment).}
\label{tab:redundancy_analysis}
\begin{tabular}{lccc}
\toprule
Feature & True Role & MI Rank & ExCIR Rank \\
\midrule
$X_1$ & Predictive & 1 & 1 \\
$X_2$ & Redundant  & 2 & 3 \\
$X_3$ & Noise      & 3 & 2 \\
\bottomrule
\end{tabular}
\end{table}
This experiment highlights a key limitation of MI: it captures global input--output dependence but does not account for how the model utilizes features. As shown in Table~\ref{tab:redundancy_analysis}, MI assigns relatively high importance to both $X_1$ and the redundant feature $X_2$, reflecting their strong statistical dependence with the target. Since $X_2$ is highly correlated with $X_1$, it inherits similar dependence with $Y$, leading to ambiguity in distinguishing truly predictive features from redundant ones. In contrast, ExCIR correctly ranks $X_1$ as the most important feature while assigning a lower rank to $X_2$, indicating its reduced contribution to the model’s predictive mechanism. This behavior arises because ExCIR evaluates dependence through model outputs, capturing feature relevance in a model-aware manner rather than purely statistical association. Interestingly, ExCIR assigns a slightly higher rank to the noise feature $X_3$ compared to $X_2$, which can occur due to model coefficient suppression and numerical effects in low-signal regimes. Nevertheless, the key distinction lies in ExCIR’s ability to downweight redundant features relative to the true predictive signal.  Overall, this analysis demonstrates that while MI reflects global dependence, ExCIR provides a more faithful representation of model-driven importance, particularly in the presence of correlated features.
\subsubsection{\textbf{Additional CIFAR-10 and 20 Newsgroups Validation}}
\label{sec:supp_reviewer2}

This section expands the revised model-quality, robustness, and qualitative
analyses summarized in the main experimental section
(Section~k) . It records
the complete training recipe, per-seed predictive results, controlled
noise-intensity protocol, and visual explanations used to address the concern
that a weak vision classifier could make subsequent attribution conclusions
unreliable.

\paragraph{CIFAR-10 training and evaluation protocol}
\label{sec:supp_cifar_protocol}

We used the official CIFAR-10 training and test partitions. For each seed, the
50,000-image training partition was deterministically shuffled and divided into
45,000 training images and 5,000 validation images; the official 10,000-image
test partition remained untouched. The training transform consisted of a random
$32\times32$ crop with four-pixel padding, random horizontal flipping, conversion
to a tensor, and channel-wise CIFAR-10 normalization. Validation and test images
were converted to tensors and normalized without random augmentation.

The classifier was a CIFAR-adapted ResNet-18. In contrast to the ImageNet stem,
the first layer was a $3\times3$, stride-one convolution with one-pixel padding,
and the initial max-pooling layer was removed. Models were optimized for 200
epochs using SGD with Nesterov momentum, an initial learning rate of 0.1,
momentum 0.9, weight decay $5\times10^{-4}$, batch size 128, cosine learning-rate
annealing, and cross-entropy with label smoothing 0.1. No Gaussian noise was used
during training. The checkpoint with the highest clean validation accuracy was
selected independently for each seed. Table~\ref{tab:supp_training_recipe}
summarizes the protocol.

\begin{table}[t]
\centering
\caption{Complete CIFAR-10 ResNet-18 training recipe.}
\label{tab:supp_training_recipe}
\begin{tabular}{p{0.28\linewidth} p{0.64\linewidth}}
\hline
Item & Setting \\
\hline
Data split & 45,000 training / 5,000 validation / 10,000 official test \\
Architecture & CIFAR-adapted ResNet-18 \\
Input stem & $3\times3$ convolution, stride 1; no initial max-pooling \\
Normalization & CIFAR-10 channel means and standard deviations \\
Training augmentation & Random crop with padding 4; horizontal flip \\
Optimizer & SGD with Nesterov momentum \\
Initial learning rate & 0.1 \\
Schedule & Cosine annealing over 200 epochs \\
Momentum & 0.9 \\
Weight decay & $5\times10^{-4}$ \\
Batch size & 128 \\
Random seeds & 7, 17, and 27 \\
Model selection & Highest clean validation accuracy \\
Training noise & None \\
Robustness noise & Evaluation only after checkpoint freezing \\
\hline
\end{tabular}
\end{table}

Table~\ref{tab:supp_seed_results} gives the independently selected epoch and
predictive performance for each seed. The clean validation accuracy was
$95.15\pm0.21$\%, and the held-out test accuracy was $95.21\pm0.24$\%, where the
uncertainty is the sample standard deviation across the three seeds. These
results replace the earlier 59.7\% clean validation result and establish that the
attribution analysis is performed on a competent CIFAR-10 classifier.

\begin{table}[t]
\centering
\caption{Clean CIFAR-10 predictive performance across independent seeds. The
best epoch is selected using validation accuracy; test accuracy is then measured
on the official held-out test partition.}
\label{tab:supp_seed_results}
\begin{tabular}{c c c c}
\hline
Seed & Selected epoch & Validation accuracy (\%) & Test accuracy (\%) \\
\hline
7  & 194 & 95.36 & 95.13 \\
17 & 186 & 95.16 & 95.02 \\
27 & 189 & 94.94 & 95.48 \\
\hline
Mean $\pm$ SD & -- & $95.15\pm0.21$ & $95.21\pm0.24$ \\
\hline
\end{tabular}
\end{table}

\paragraph{Frozen-model Gaussian-noise stress test}
\label{sec:supp_noise}

To separate predictive robustness from training effects, we froze the selected
seed-7 checkpoint and added independent zero-mean Gaussian noise only during
evaluation. Noise standard deviations $\sigma\in\{0,0.01,0.03,0.05,0.10\}$ are
reported in normalized input space. No parameter update, fine-tuning, checkpoint
selection, or noise augmentation was performed during this experiment. At every
noise level, we measured test accuracy and recomputed the class-conditioned ExCIR
scores on 5,000 test images using a $4\times4$ patch grid. Explanation stability
was measured relative to the clean score vector using Spearman rank correlation
and the Jaccard similarity of the Top-8 patch sets.

The complete results are reported in Table~\ref{tab:supp_noise} and visualized in
Figure~\ref{fig:supp_noise}. Accuracy decreased by 5.28 percentage points between
$\sigma=0$ and $\sigma=0.10$. In contrast, both ranking metrics remained 1.00,
showing that the ordering of the selected patch scores was preserved throughout
the evaluated range. This conclusion concerns ranking stability; it does not
imply that the classifier's predictive accuracy is invariant to noise.

\begin{table}[t]
\centering
\caption{CIFAR-10 predictive accuracy and ExCIR ranking stability under
evaluation-only Gaussian noise.}
\label{tab:supp_noise}
\begin{tabular}{c c c c}
\hline
$\sigma$ & Test accuracy (\%) & Spearman vs. clean & Top-8 Jaccard \\
\hline
0.00 & 95.13 & 1.00 & 1.00 \\
0.01 & 95.01 & 1.00 & 1.00 \\
0.03 & 94.67 & 1.00 & 1.00 \\
0.05 & 93.97 & 1.00 & 1.00 \\
0.10 & 89.85 & 1.00 & 1.00 \\
\hline
\end{tabular}
\end{table}

\begin{figure}[t]
\centering
\includegraphics[width=0.82\linewidth]{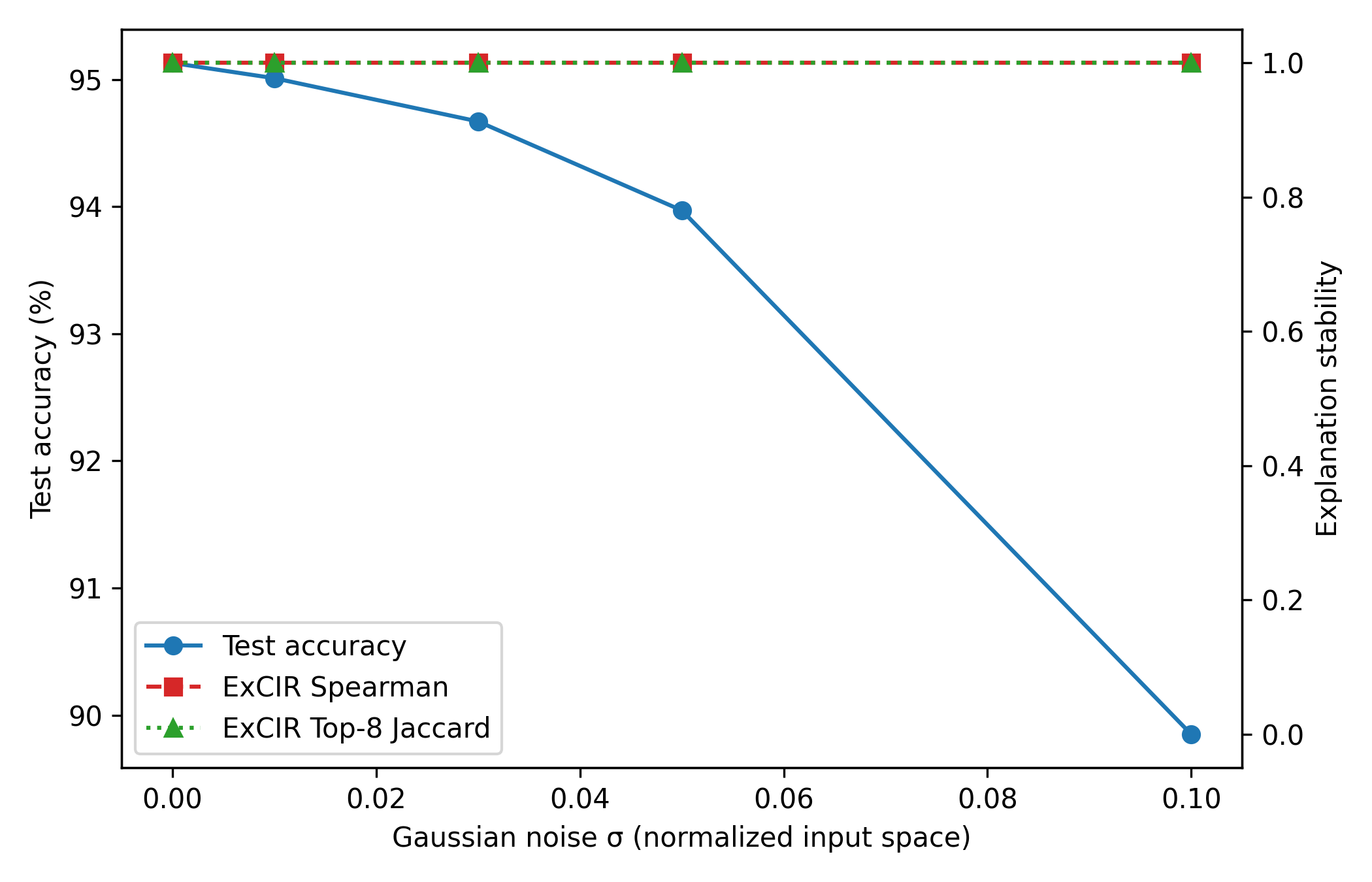}
\caption{Frozen-model CIFAR-10 robustness. The left axis reports test accuracy;
the right axis reports ExCIR rank stability relative to clean inputs. Gaussian
noise is added in normalized input space only during evaluation. Predictive
accuracy decreases with noise intensity, while Spearman correlation and Top-8
Jaccard similarity remain 1.00 over the tested range.}
\label{fig:supp_noise}
\end{figure}

\paragraph{Qualitative vision and text explanations}
\label{sec:supp_qualitative}

Figure~\ref{fig:supp_trucks} displays five CIFAR-10 truck images sampled from the
correctly classified seed-7 test instances using a fixed random seed. The left
column gives each original image with its test index and truck confidence; the
right column projects the class-conditioned ExCIR patch score distribution onto
the image grid. The figure is presented as a visualization of the measured
spatial distribution and is not used to assert that one specific semantic part,
such as wheels or the lower chassis, is consistently dominant. This wording
replaces the earlier over-specific qualitative statement.

\begin{figure}[t]
\centering
\includegraphics[width=0.72\linewidth]{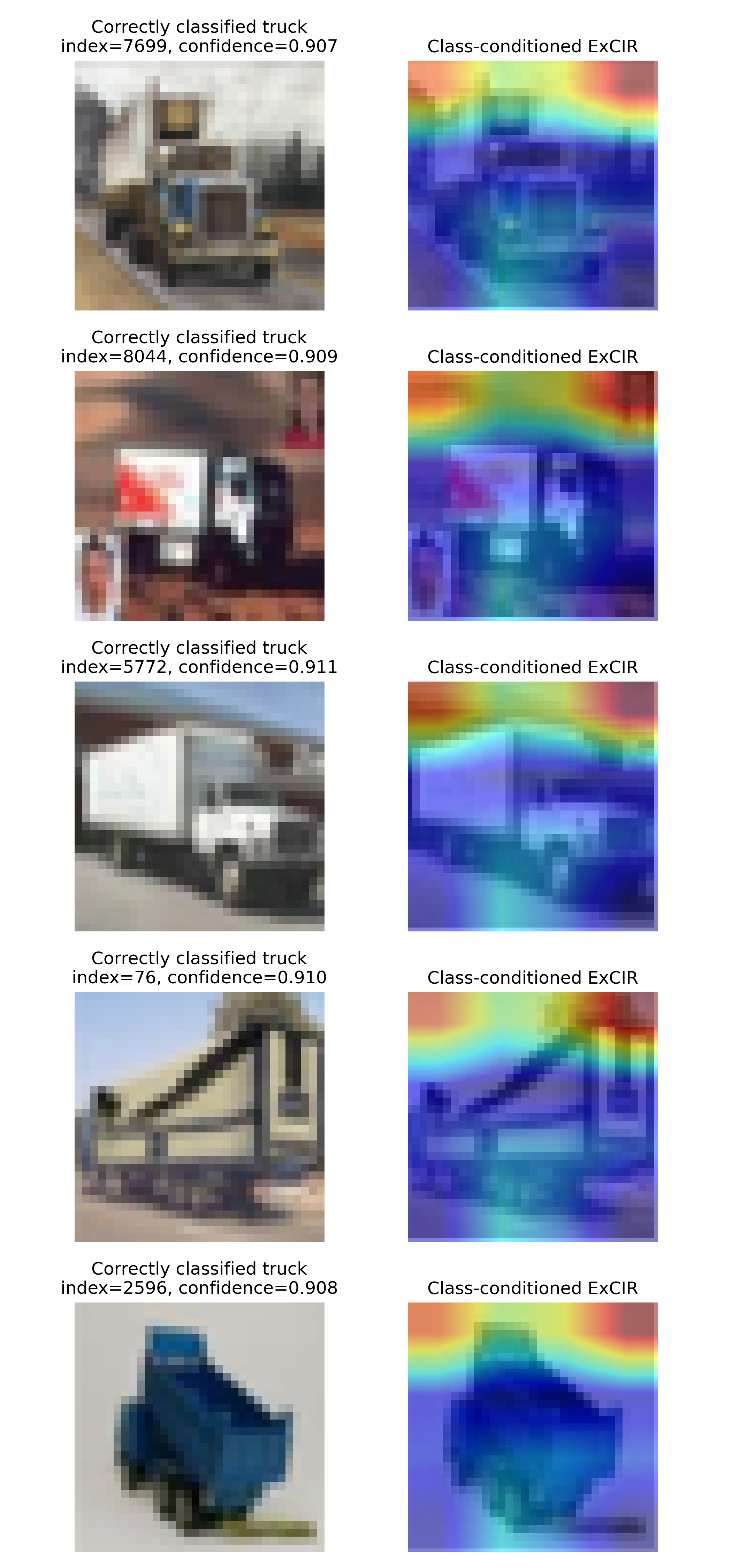}
\caption{Fixed-seed qualitative CIFAR-10 examples. Five correctly classified
truck images (left) and their projected class-conditioned ExCIR patch-score
distributions (right) are shown together with predicted confidence. The examples
were selected from correctly classified test instances before inspecting their
maps.}
\label{fig:supp_trucks}
\end{figure}

For the text experiment, a TF--IDF multinomial logistic-regression classifier
achieved 85.4\% held-out test accuracy on the four-class 20 Newsgroups subset.
One correctly classified \texttt{rec.sport.baseball} document was selected with
seed 7. Figure~\ref{fig:supp_text} reports its ten largest local ExCIR scores.
The leading terms were \emph{team}, \emph{Baltimore}, \emph{fan}, and
\emph{Orioles}; these are directly compatible with the predicted baseball class.
The revised manuscript reports these observed terms rather than prespecifying
\emph{game}, \emph{team}, and \emph{season}. An additional HTML rendering in the
reproducibility archive highlights the same terms in their original document
context.
\paragraph{Document-level illustration.}
The following excerpt is taken from the correctly classified test document
(index 1405; predicted class: \texttt{rec.sport.baseball}). It was selected
from correctly classified test instances using the fixed experimental seed.
Bold text denotes tokens receiving high local ExCIR scores; ellipses indicate
text omitted for concision.

\begin{quote}
``The \textbf{Orioles} are not a base-stealing \textbf{team} except for
Anderson. Besides, we would never call them anything but the
\textbf{Baltimore Orioles}. \ldots{} I still wonder where Baseblazers came
from. \ldots{} Why do I care what a non-O's \textbf{fan} thinks of us?''
\end{quote}

The highlighted terms are semantically consistent with the predicted baseball
class and agree with the ranked token scores in
Figure~\ref{fig:supp_text}.

\begin{figure}[t]
\centering
\includegraphics[width=0.82\linewidth]{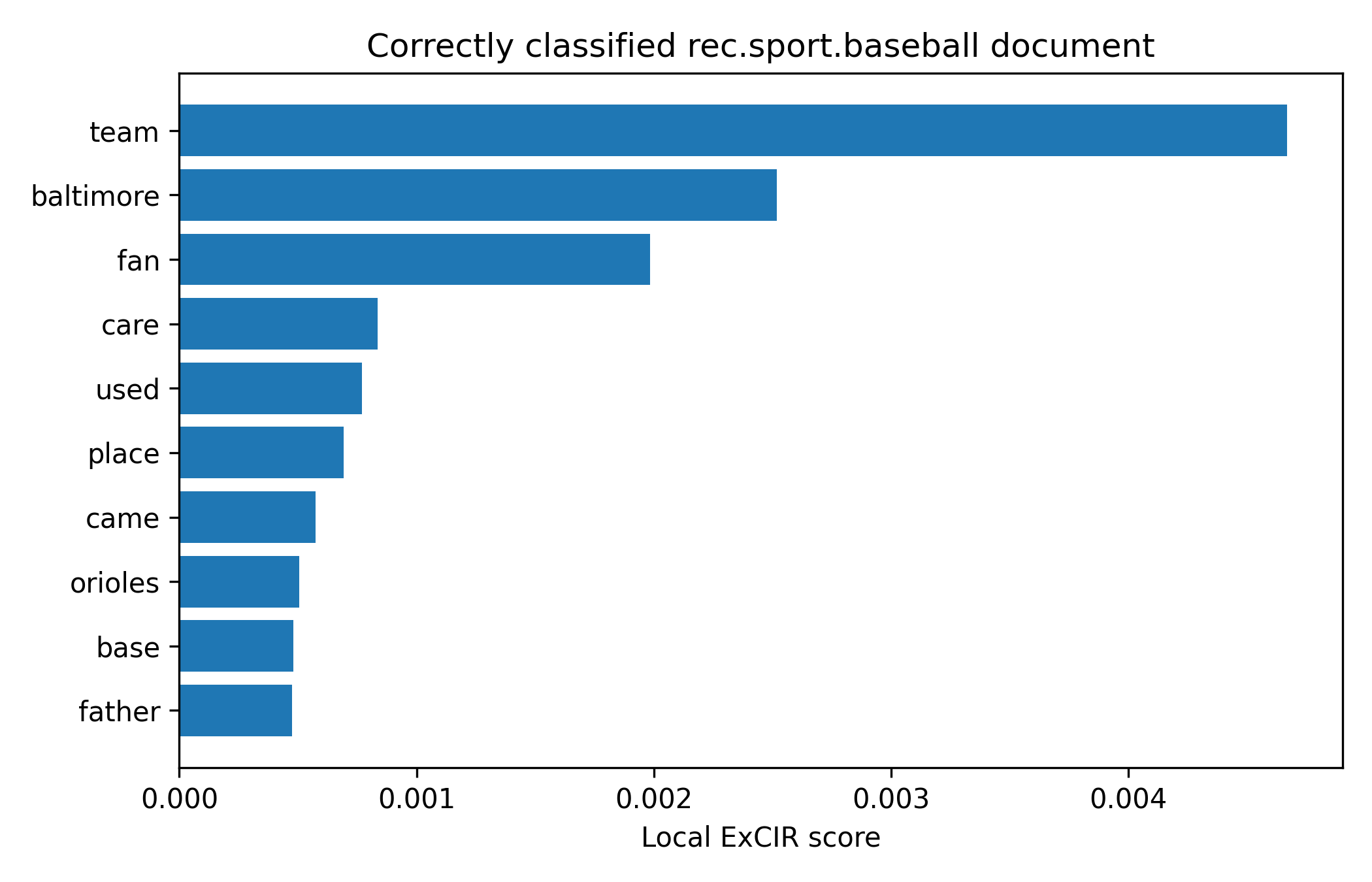}
\caption{Token-level ExCIR visualization for a correctly classified
\texttt{rec.sport.baseball} test document. Bars show the ten largest local ExCIR
scores. The most influential observed terms include \emph{team},
\emph{Baltimore}, \emph{fan}, and \emph{Orioles}.}
\label{fig:supp_text}
\end{figure}

Together, Tables~\ref{tab:supp_training_recipe}--\ref{tab:supp_noise} and
Figures~\ref{fig:supp_noise}--\ref{fig:supp_text} provide the detailed evidence
supporting the concise results reported in main-paper.


\paragraph{\textbf{Summary of Findings on Real-World High-Dimensional Data}}

Overall, the extended experiments support several key findings across high-dimensional real-world vision and text data.
First, ExCIR demonstrates strong computational efficiency, remaining substantially faster than MI in both settings and significantly faster than SHAP in non-linear vision tasks. While LinearSHAP is faster in the text setting due to its closed-form optimization for linear models, this advantage relies on strong model assumptions, whereas ExCIR remains broadly applicable across model classes. The runtime comparison across methods is illustrated in Fig.~\ref{fig:app_runtime_cifar}, highlighting the computational efficiency of ExCIR in the vision setting. Second, ExCIR provides a distinct attribution perspective. The consistently low rank agreement with MI, SHAP, and LIME indicates that ExCIR captures model-output-aligned dependence rather than reproducing label-based statistics or local perturbation effects. This is particularly important in high-dimensional settings, where correlated features and complex interactions can lead traditional methods to produce redundant or unstable explanations. Third, ExCIR exhibits strong robustness properties. In vision tasks, it maintains near-perfect stability under Gaussian perturbations, preserving both global rankings and top features. In text data, global ranking stability remains high despite lower Top-$k$ overlap, reflecting the inherent sensitivity of sparse, high-dimensional representations rather than instability of the method. Fourth, lightweight experiments show that ExCIR rankings remain highly stable under substantial data reduction, even when predictive performance degrades. This highlights a key advantage: attribution stability can be preserved independently of full model retraining, although achieving predictive equivalence requires stronger training or larger subsamples.

\textbf{Strengths.} ExCIR performs well in high-dimensional settings by avoiding repeated perturbations, operating in a single pass, and leveraging covariance-based aggregation, which naturally smooths noise and stabilizes rankings. It scales efficiently to both structured vision features (patch-based) and sparse text features (TF-IDF), making it practical for real-world applications.

\textbf{Limitations.} ExCIR relies on observed feature–output relationships and may be sensitive to feature representation quality. In highly sparse domains such as text, discrete feature removal can affect top-ranked features. Additionally, lightweight environments require careful tuning, as predictive alignment depends on model training rather than attribution alone.

\textbf{Implementation details.} In this work, ExCIR is implemented using patch-level feature aggregation for images and TF-IDF representations for text, with attribution computed directly on model outputs without surrogate models. Lightweight environments are constructed via subsampling and validated using similarity constraints (projection alignment, MMD, KL divergence) and risk thresholds.

\textbf{Future directions.} Several improvements can further strengthen the framework. First, extending ExCIR to richer feature representations (e.g., embeddings from deep vision or language models) may improve robustness in complex domains. Second, integrating adaptive lightweight training (e.g., longer schedules or knowledge distillation) could better align predictive performance. Third, extending the method to dynamic and out-of-distribution settings, including large-scale vision and language models, remains an important direction. Finally, incorporating uncertainty-aware variants of ExCIR could further enhance reliability in noisy or evolving data environments.

Overall, ExCIR provides an efficient and robust approach for dependence-aware attribution, particularly suited to high-dimensional real-world data, while highlighting clear pathways for further improvement and extension.

\begin{figure}[h]
\centering
\begin{tikzpicture}
\begin{axis}[
    ybar,
    bar width=7pt,
    width=\linewidth,
    height=5cm,
    ylabel={Runtime (s)},
    symbolic x coords={ExCIR,BlockCIR,LIME,MI,SHAP},
    xtick=data,
    ymin=0,
    nodes near coords,
    nodes near coords align={vertical},
    tick label style={font=\scriptsize},
    label style={font=\small},
    title={CIFAR-10 Runtime Comparison}
]
\addplot coordinates {
    (ExCIR,0.15)
    (BlockCIR,0.28)
    (LIME,0.30)
    (MI,3.16)
    (SHAP,15.15)
};
\end{axis}
\end{tikzpicture}
\caption{Runtime comparison on CIFAR-10 with ResNet-18. ExCIR is the fastest method in the non-linear vision setting.}
\label{fig:app_runtime_cifar}
\end{figure}

\end{document}